\documentclass{article}

\usepackage[final,nonatbib]{neurips_2021}

\usepackage[utf8]{inputenc} 
\usepackage[T1]{fontenc}    
\usepackage{hyperref}       
\usepackage{url}            
\usepackage{booktabs}       
\usepackage{amsfonts}       
\usepackage{nicefrac}       
\usepackage{microtype}      
\usepackage{xcolor}         
\usepackage{minitoc}
\usepackage{wrapfig}
\usepackage{enumitem}
\usepackage{microtype} 
\usepackage{marvosym} 
\usepackage[english]{babel}
\usepackage[algo2e,ruled]{algorithm2e}

\usepackage[toc,page,header]{appendix}
\usepackage{minitoc}
\usepackage[normalem]{ulem}
\usepackage{multirow}
\usepackage{bm}
\usepackage{mathtools}
\usepackage{amssymb}
\usepackage{amsmath}
\usepackage{amsthm}
\usepackage{soul}
\usepackage{lipsum}
\usepackage{graphicx}

\newcommand{\makeappendixtitle}{%
    \vbox{%
    \hsize\textwidth
    \linewidth\hsize
    \vskip 0.1in
    \hrule height 4pt
    \vskip 0.25in
    \vskip -\parskip%
    \centering
    {\LARGE\bf {Supplementary Material\\ On Provable Benefits of Depth in Training \\ Graph Convolutional Networks} \par}
    \vskip 0.29in
    \vskip -\parskip
    \hrule height 1pt
    \vskip 0.09in%
  }
}

\newtheorem{theorem}{Theorem}
\newtheorem{proposition}{Proposition}
\newtheorem{lemma}{Lemma}

\newtheorem{definition}{Definition}
\newtheorem{remark}{Remark}

\newtheorem{assumption}{Assumption}

\definecolor{links}{HTML}{0078b0} 
\definecolor{files}{HTML}{fc6160}
\hypersetup{
    colorlinks=true,
    linkcolor=files,
    filecolor=files,      
    urlcolor=links,
    citecolor=links,
}

\title{On Provable Benefits of Depth in Training \\ Graph Convolutional Networks}

\author{%
  Weilin Cong \\
  Penn State\\
  \texttt{wxc272@psu.edu} \\
   \And
   Morteza Ramezani \\
    Penn State \\
  \texttt{morteza@cse.psu.edu} \\
   \And
   Mehrdad Mahdavi \\
  Penn State \\
  \texttt{mzm616@psu.edu}
}

\begin{document}

\maketitle


\begin{abstract}
Graph Convolutional Networks (GCNs) are known to suffer from performance degradation as the number of layers increases, which is usually attributed to over-smoothing. Despite the apparent consensus, we observe that there exists a discrepancy between the theoretical understanding of over-smoothing and the practical capabilities of GCNs. Specifically, we argue that over-smoothing does not necessarily happen in practice, a deeper model is provably expressive, can converge to global optimum with linear convergence rate, and achieve very high training accuracy as long as properly trained. Despite being capable of achieving high training accuracy, empirical results show that the deeper models generalize poorly on the testing stage and existing theoretical understanding of such behavior remains elusive. To achieve better understanding, we carefully analyze the \emph{generalization capability} of GCNs, and show that the training strategies to achieve high training accuracy significantly deteriorate the \emph{generalization capability} of GCNs. Motivated by these findings, we propose a decoupled structure for GCNs that detaches weight matrices from feature propagation to preserve the expressive power and ensure good generalization performance. We conduct empirical evaluations on various synthetic and real-world datasets to validate the correctness of our theory.
\end{abstract}

\section{Introduction}\label{section:introduction}

In recent years, Graph Convolutional Networks (GCNs) have achieved state-of-the-art performance in dealing with graph-structured applications, including social networks~\cite{kipf2016semi,hamilton2017inductive,wang2019mcne,deng2019learning,qiu2018deepinf}, traffic prediction~\cite{cui2019traffic,rahimi2018semi,li2019predicting,kumar2019predicting}, knowledge graphs~\cite{wang2019knowledge,wang2019kgat,park2019estimating}, drug reaction~\cite{do2019graph,duvenaud2015convolutional} and recommendation system~\cite{berg2017graph,ying2018graph}.
Despite the success of GCNs, applying a shallow GCN model that only uses the information of a very limited neighborhood on a large sparse graph has shown to be not effective~\cite{he2016deep,chen2020simple,gong2020geometrically,cong2020minimal,ramezani2020gcn}.
As a result, a deeper GCN model would be desirable to reach and aggregate information from farther neighbors.
The inefficiency of shallow GCNs is exacerbated even further when the labeled nodes compared to graph size is negligible, as a shallow GCN cannot sufficiently propagate the label information to the entire graph with only a few available labels~\cite{li2018deeper}.

Although a deeper GCN is preferred to perceive more graph structure information, unlike traditional deep neural networks, it has been pointed out that deeper GCNs potentially suffer from over-smoothing~\cite{li2018deeper,oono2019graph,huang2020tackling,cai2020note,yan2021two}, vanishing/exploding gradients~\cite{li2019deepgcns}, over-squashing~\cite{alon2020bottleneck}, and training difficulties~\cite{zhou2020effective,luan2020training}, which significantly affect the performance of GCNs as the depth increases.
Among these, the most widely accepted reason is ``over-smoothing'', which is referred to as a phenomenon due to applying multiple graph convolutions such that all node embeddings converge to a single subspace (or vector) and leads to indistinguishable node representations. 

\begin{figure}
    \centering
    \includegraphics[width=0.75\textwidth]{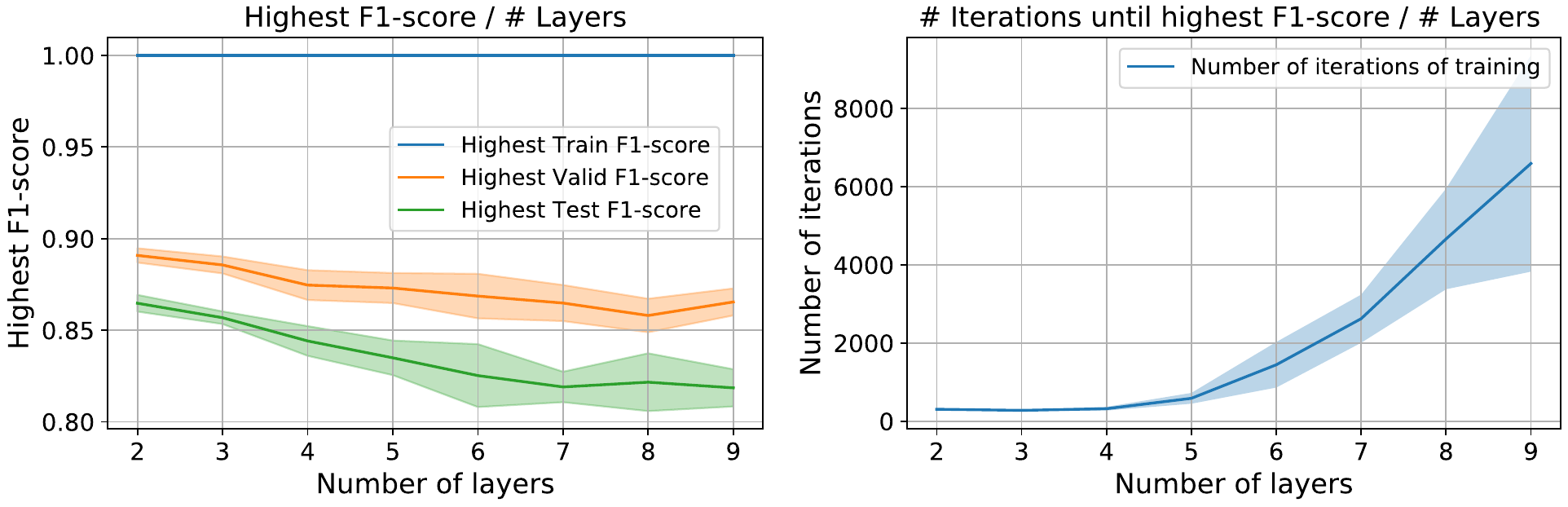}
    \vspace{-2mm}
    \caption{Comparison of F1-score for GCN with different depth on \emph{Cora} dataset, where deeper models can achieve high training accuracy, but complicate the training by requiring more iterations to converge and suffer from poor generalization.}
    \label{fig:train_acc_gcn_resgcn_appnp_gcnii}
\end{figure}

The conventional wisdom is that adding to the number of layers causes over-smoothing, which impairs the \emph{expressiveness power} of GCNs and consequently leads to a poor \emph{training accuracy}. However,  we observe that there exists a discrepancy between theoretical understanding of their inherent capabilities and practical performances. According to the definition of over-smoothing that the node representation becomes indistinguishable as GCNs goes deeper, the classifier has difficulty assigning the correct label for each node if over-smoothing happens. As a result, the  training accuracy is expected to be decreasing as the number of layers increases. 
However, as shown in Figure~\ref{fig:train_acc_gcn_resgcn_appnp_gcnii}, GCNs are capable of achieving high \emph{training accuracy} regardless of the number of layers. But, as it can be observed, deeper GCNs require more training iterations to reach a high training accuracy, and its\emph{ generalization performance} on evaluation set decreases as the number of layers increases. 
This observation suggests that the performance degradation is  likely due to  inappropriate training rather than the low expressive power caused by over-smoothing.
Otherwise, a low expressiveness model cannot achieve almost perfect training accuracy simply by proper  training tricks alone.\footnote{The results in Figure~\ref{fig:train_acc_gcn_resgcn_appnp_gcnii} can be reproduced by removing both the dropout and weight decay operation. These two augmentations are designed to improve the generalization ability (i.e., validation/testing accuracy) of the model but might hurt the training accuracy, because of the randomness and also an extra regularization term on the model parameters which are introduced during training. A simple experiment using DGL can be found \href{https://github.com/CongWeilin/DGCN/blob/master/DGL_code.ipynb}{here}.}
Indeed, recent years significant advances have been witnessed on tweaking the model architecture to overcome the training difficulties in deeper GCN models and achieve good generalization performance~\cite{luan2020training,chen2020simple,li2019deepgcns,zhou2020effective}.

\noindent\textbf{Contributions.}~
Motivated by aforementioned observation, i.e., still achieving high training accuracy when trained properly but poor generalization performance, we aim at  answering two fundamental questions in this paper:

{\sffamily{Q1:}}~\emph{Does increasing depth really impair  the expressiveness power of GCNs?\\}
In Section~\ref{section:limitation_over_smoothing}, we argue that there exists a discrepancy between over-smoothing based theoretical results and the practical capabilities of deep GCN models, demonstrating that  over-smoothing is not the key factor that leads to the performance degradation in deeper GCNs.
In particular, we mathematically show that over-smoothing~\cite{oono2019graph,huang2020tackling,li2018deeper,cai2020note} is mainly an artifact of theoretical analysis and simplifications made in analysis.
Indeed, by characterizing the representational capacity of GCNs via Weisfeiler-Lehman (WL) graph isomorphism test~\cite{morris2019weisfeiler,xu2018powerful}, we show that deeper GCN model is at least as expressive as the shallow GCN model, the deeper GCN models can distinguish nodes with a different neighborhood that the shallow GCN cannot distinguish, 
as long as the GCNs are properly trained. 
Besides, we theoretically show that more training iterations is sufficient (but \textit{not necessary} due to the assumptions made in our theoretical analysis) for a deeper model to achieve the same training error as the shallow ones, which further suggests the poor training error in deep GCN training is most likely due to inappropriate training.

{\sffamily{Q2:}}~\emph{If expressive, why then deep GCNs generalize poorly?\\} 
In Section~\ref{section:from_stability}, in order to understand the performance degradation phenomenon in deep GCNs during the evaluation phase, we give a novel generalization analysis on GCNs and its variants (e.g., ResGCN, APPNP, and GCNII) under the semi-supervised setting for the node classification task.
We show that the generalization gap of GCNs is governed by the number of training iterations, largest node degree, the largest singular value of weight matrices, and the number of layers.
In particular, our result suggests that a deeper GCN model requires more iterations of training and optimization tricks to converge (e.g., adding skip-connections), which leads to a poor generalization.
More interestingly, our generalization analysis shows that most of the so-called methods to solve over-smoothing~\cite{rong2019dropedge,zhao2019pairnorm,chen2020simple,klicpera2018predict} can greatly improve the generalization ability of the model, therefore results in a deeper model.

The aforementioned findings naturally lead to the algorithmic contribution of this paper.
In Section~\ref{section:method}, we present a novel framework, \emph{\textbf{D}ecoupled \textbf{GCN}} (DGCN), that is capable of  training deeper GCNs and can significantly improve the generalization performance. 
The main idea is to isolate the expressive power from generalization ability by decoupling the weight parameters from feature propagation. 
In Section~\ref{section:experiment}, we conduct experiments on the synthetic and real-world datasets to validate the correctness of the theoretical analysis and the advantages of DGCN over baseline methods.

\section{Related works}\label{section:related_works}

\noindent\textbf{Expressivity of GCNs.} 
Existing results on expressive power of GCNs are mixed, \cite{morris2019weisfeiler,loukas2019graph,chen2020can,chen2020graph} argue that deeper model has higher expressive power but \cite{oono2019graph,huang2020tackling,huang2021wide} have completely opposite result.
On the one hand, \cite{morris2019weisfeiler} shows a deeper GCN is as least as powerful as the shallow one in terms of distinguishing non-isomorphic graphs.
\cite{loukas2019graph} shows that deep and wide GCNs is Turing universal, however, GCNs lose power when their depth and width are restricted.
\cite{chen2020can} and \cite{chen2020graph} measure the expressive power of GCNs via its subgraph counting capability and attribute walks, and both show the expressive power of GCNs grows exponentially with the GCN depth.
On the other hand, \cite{huang2021wide} studies the infinity wide GCNs and shows that the covariance matrix of its outputs converges to a constant matrix at an exponential rate. 
\cite{oono2019graph,huang2020tackling} characterize the expressive power using the distance between node embeddings to a node feature agnostic subspace and show the distance is decreasing as the number of layers increases. 
Details are deferred to Section~\ref{section:limitation_over_smoothing} and Appendix~\ref{supp:weakness_oversmoothing}.
Such contradictory results motivate us to rethink the role of over-smoothing on expressiveness.


\noindent\textbf{Generalization analysis of GCNs.} 
In recent years, many papers are working on the generalization of GCNs using uniform stability~\cite{verma2019stability,zhou2021generalization}, Neural Tangent Kernel~\cite{du2019graph}, VC-dimension~\cite{scarselli2018vapnik}, Rademacher complexity~\cite{garg2020generalization,oono2020optimization}, algorithm alignment~\cite{xu2019can,xu2021how} and PAC-Bayesian~\cite{liao2020pac}.
Existing works only focus on a specific GCN structure, which cannot be used to understand the impact of GCN structures on its generalization ability.
The most closely related to ours is~\cite{verma2019stability}, where they analyze the stability of the single-layer GCN model, and show that the stability of GCN depends on the largest absolute eigenvalue of its Laplacian matrix. 
However, their result is under the inductive learning setting and extending the results to the multi-layer GCNs with different structures is non-trivial.

\noindent\textbf{Literature with similar observations.}
Most recently, several works have similar observations on the over-smoothing issue to ours.  
\cite{zhou2020effective} argues that the main factors to performance degradation are vanishing gradient, training instability, and over-fitting, rather than over-smoothing, and proposes a node embedding normalization heuristic to alleviate the aforementioned issues. 
\cite{luan2020training} argues that the performance degradation is mainly due to the training difficulty, and proposes a different graph Laplacian formulation, weight parameter initialization, and skip-connections to improve the training difficulty.
\cite{yang2020revisiting} argues that deep GCNs can learn to overcome the over-smoothing issue during training, and the key factor of performance degradation is over-fitting, and proposes a node embedding normalization method to help deep GCNs overcome the over-smoothing issue.
\cite{kong2020flag} improves the generalization ability of GCNs by using adversarial training and results in a consistent improvement than all GCNs baseline models without adversarial training.
All aforementioned literature only gives heuristic explanations based on the empirical results, and do not provide theoretical arguments.
\section{Preliminaries}\label{section:preliminaries}

\noindent\textbf{Notations and setup.}~
We consider the semi-supervised node classification problem, where a self-connected graph $\mathcal{G}=(\mathcal{V},\mathcal{E})$ with $N=|\mathcal{V}|$ nodes is given in which each node $i\in\mathcal{V}$ is associated with a feature vector $\mathbf{x}_i \in \mathbb{R}^{d_0}$, and only a subset of nodes $\mathcal{V}_\text{train} \subset \mathcal{V}$ are labeled, i.e., $y_i\in\{1,\ldots,|\mathcal{C}|\}$ for each $i\in\mathcal{V}_\text{train}$ and $\mathcal{C}$ is the set of all candidate classes. Let $\mathbf{A}\in\mathbb{R}^{N\times N}$ denote the adjacency matrix and $\mathbf{D}\in\mathbb{R}^{N\times N}$ denote the corresponding degree matrix with $D_{i,i} = \deg(i)$ and $D_{i,j}=0$ if $i\neq j$. Then, the propagation matrix (using the Laplacian matrix defined in~\cite{kipf2016semi}) is computed as $\mathbf{P} = \mathbf{D}^{-1/2} \mathbf{A} \mathbf{D}^{-1/2}$.
Our goal is to learn a GCN model using the node features for all nodes $\{\mathbf{x}_i\}_{i\in\mathcal{V}}$ and node labels for the training set nodes $\{y_i\}_{i\in\mathcal{V}_\text{train}}$, and expect it generalizes well on the unlabeled node set $\mathcal{V}_\text{test} = \mathcal{V} \setminus \mathcal{V}_\text{train}$.

\noindent\textbf{GCN architectures.}~
In this paper, we consider the following architectures for training GCNs:
\begin{itemize} [noitemsep,topsep=0pt,leftmargin=5mm ]
    \item \textbf{\textit{Vanilla GCN}}~\cite{kipf2016semi} computes node embeddings by $\mathbf{H}^{(\ell)} = \sigma(\mathbf{P} \mathbf{H}^{(\ell-1)} \mathbf{W}^{(\ell)})$, where $\mathbf{H}^{(\ell)} = \{ \mathbf{h}_i^{(\ell)}\}_{i=1}^N$ is the $\ell$th layer node embedding matrix, $\mathbf{h}_i^{(\ell)} \in \mathbb{R}^{d_\ell}$ is the embedding of $i$th node, $\mathbf{W}^{(\ell)} \in \mathbb{R}^{d_{\ell}\times d_{\ell-1}}$ is the $\ell$th layer weight matrix, and $\sigma(\cdot)$ is the ReLU activation. 
    \item \textbf{\textit{ResGCN}}~\cite{li2019deepgcns} solves the vanishing gradient issue by adding skip-connections between adjacency layers. 
    More specifically, ResGCN computes node embeddings by $\mathbf{H}^{(\ell)} = \sigma(\mathbf{P} \mathbf{H}^{(\ell-1)} \mathbf{W}^{(\ell)}) + \mathbf{H}^{(\ell-1)},$ where node embeddings of the previous layer is added to the output of the current layer to facilitate the training of deeper GCN models.
    \item \textbf{\textit{APPNP}}~\cite{klicpera2018predict} adds skip-connections from the input layer to each hidden layer to preserve the feature information. 
    APPNP computes node embeddings by $\mathbf{H}^{(\ell)} = \alpha_\ell \mathbf{P} \mathbf{H}^{(\ell-1)} + (1-\alpha_\ell) \mathbf{H}^{(0)} \mathbf{W},$ where $\alpha_\ell\in[0,1]$ balances the amount of information preserved at each layer. 
    By decoupling feature transformation and propagation, APPNP can aggregate information from multi-hop neighbors without significantly increasing the computation complexity.
    \item \textbf{\textit{GCNII}}~\cite{chen2020simple} improves the capacity of APPNP by adding non-linearty and weight matrix at each individual layer. 
    GCNII computes node embeddings by $\mathbf{H}^{(\ell)} = \sigma\big((\alpha_\ell \mathbf{P} \mathbf{H}^{(\ell-1)} + (1-\alpha_\ell) \mathbf{H}^{(0)}) \bar{\mathbf{W}}^{(\ell)}\big)$, where $\bar{\mathbf{W}}^{(\ell)} = \beta_\ell \mathbf{W}^{(\ell)} + (1-\beta_\ell) \mathbf{I}$, constant $\alpha_\ell$ same as APPNP, and constant $\beta_\ell \in[0,1]$ restricts the power of $\ell$th layer parameters. 
\end{itemize}

\section{On true expressiveness and optimization landscape of deep GCNs}\label{section:limitation_over_smoothing}


\begin{figure}
    \centering
    \includegraphics[width=0.78\textwidth]{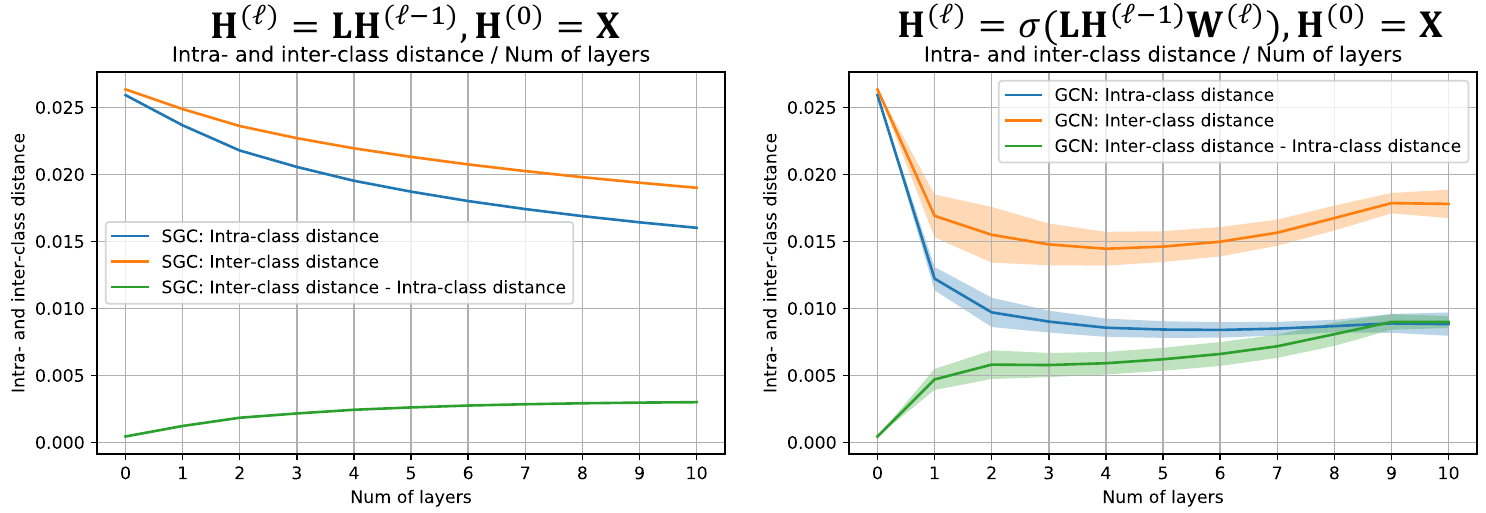}
    \vspace{-2mm}
    \caption{Comparison of intra- and inter-class normalized node embeddings $\mathbf{H}^{(\ell)}/ \|\mathbf{H}^{(\ell)}\|_\mathrm{F}$ pairwise distance on \emph{Cora} dataset. See Appendix~\ref{supp:more_result_on_linear_model} for more evidences.}
    \label{fig:pairwise_dist_num_layers}
\end{figure}

\noindent\textbf{Empirical validation of over-smoothing.}~
In this section, we aim at answering the following fundamental question: 
\emph{``Does over-smoothing really cause the performance degradation in deeper GCNs?''}
As first defined in~\cite{li2018deeper}, \emph{over-smoothing} is referred to as a phenomenon where all node embeddings converge to a single vector after applying multiple graph convolution operations to the node features. 
However, \cite{li2018deeper} only considers the graph convolution operation \emph{without} non-linearity and the per-layer weight matrices. 
To verify whether over-smoothing exists in normal GCNs, we measure the pairwise distance between the normalized node embeddings with varying the model depth.\footnote{The pairwise distances are computed on the normalized node embeddings $\mathbf{H}^{(\ell)}/\| \mathbf{H}^{(\ell)} \|_\mathrm{F}^2$ to eliminate the effect of node embedding norms.}
As shown in Figure~\ref{fig:pairwise_dist_num_layers}, without the weight matrices and non-linear activation functions, the pairwise distance between node embeddings indeed decreases as the number of layers increases. 
However, by considering the weight matrices and non-linearity, the pairwise distances are actually increasing after a certain depth which \emph{contradicts} the definition of over-smoothing that node embeddings become indistinguishable when the model becomes deeper. 
That is, graph convolution makes adjacent node embeddings get closer, then non-linearity and weight matrices help node embeddings preserve distinguishing-ability after convolution.

\begin{figure}
    \centering
    \includegraphics[width=0.8\textwidth]{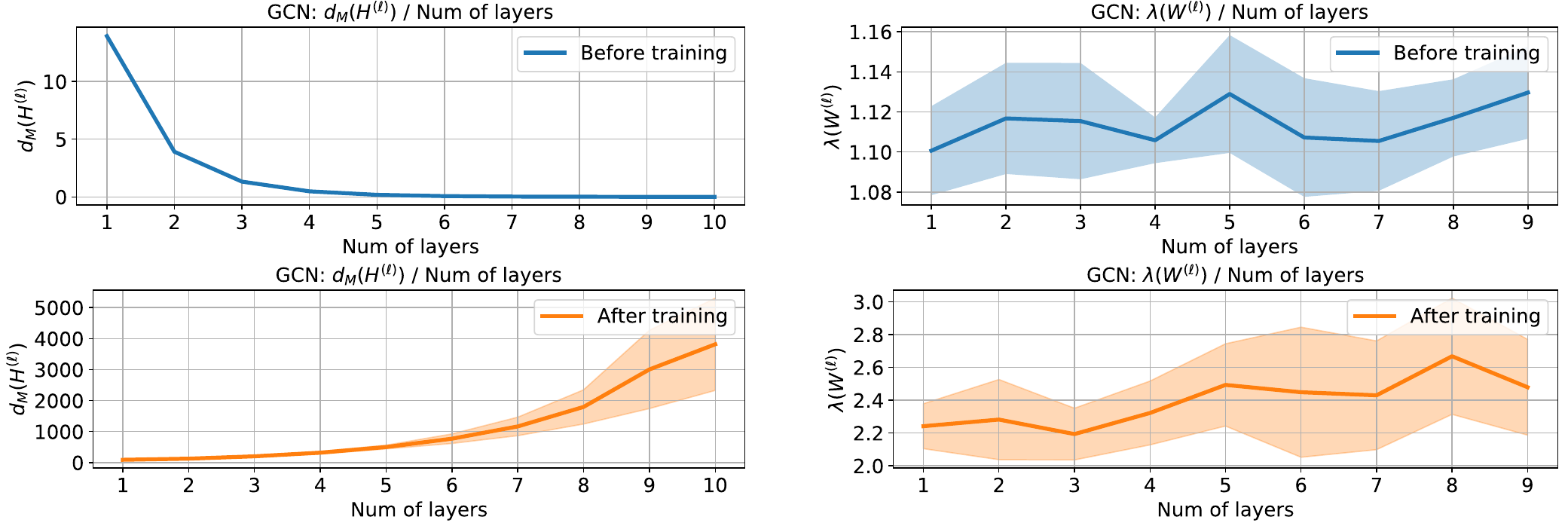}
    \caption{Compare $d_\mathcal{M}(\mathbf{H}^{(\ell)})$ and $\lambda_W(\mathbf{W}^{(\ell)})$ on both \emph{trained-} (using $500$ gradient update) and \emph{untrained-}GCN models on \emph{Cora} dataset.
    }
    \label{fig:d_m_and_weight_norm}
\end{figure}

Most recently, \cite{oono2019graph,huang2020tackling} generalize the idea of over-smoothing by taking both the non-linearity and weight matrices into considerations.
More specifically, the expressive power of the $\ell$th layer node embeddings $\mathbf{H}^{(\ell)}$ is measured using $d_\mathcal{M}(\mathbf{H}^{(\ell)})$, which is defined as the distance of node embeddings to a subspace $\mathcal{M}$ that only has node degree information.
Let denote $\lambda_L$ as the second largest eigenvalue of graph Laplacian and $\lambda_W$ as the largest singular value of weight matrices.~\cite{oono2019graph,huang2020tackling} show that the expressive power $d_\mathcal{M}(\mathbf{H}^{(\ell)})$ is bounded by $d_\mathcal{M}(\mathbf{H}^{(\ell)}) \leq (\lambda_W \lambda_L)^\ell \cdot d_\mathcal{M}(\mathbf{X})$, i.e., the expressive power of node embeddings will be exponentially decreasing or increasing as the number of layers increases, depending on whether $\lambda_W \lambda_L < 1$ or $ \lambda_W \lambda_L > 1$.
They conclude that \emph{deeper GCN exponentially loss expressive power} by assuming $\lambda_W \lambda_L < 1$. 
However, we argue that this assumption does not always hold.
To see this, let suppose $\mathbf{W}^{(\ell)}\in \mathbb{R}^{d_{\ell-1} \times d_\ell}$ is initialized by uniform distribution $\mathcal{N}(0, \sqrt{1/d_{\ell-1}})$. 
By the Gordon’s theorem for Gaussian matrices~\cite{davidson2001local}, we know its expected largest singular value is bounded by $\mathbb{E}[ \lambda_W ] \leq 1+\sqrt{d_\ell/d_{\ell-1}}$, which is strictly greater than $1$ and $\lambda_W$ usually increases during training. 
The above discussion also holds for other commonly used initialization methods~\cite{glorot2010understanding,he2015delving}.
Furthermore, since most real world graphs are sparse with $\lambda_L$ close to $1$, e.g., Cora has $\lambda_L=0.9964$, Citeseer has $\lambda_L= 0.9987$, and Pubmed has $\lambda_L= 0.9905$, making assumption on $\lambda_W \lambda_L<1$ is not realistic.
As shown in Figure~\ref{fig:d_m_and_weight_norm}, when increasing the number of layers, we observe that the distance $d_\mathcal{M}(\mathbf{H}^{(\ell)})$ is decreasing on \emph{untrained}-GCN models, however, the distance $d_\mathcal{M}(\mathbf{H}^{(\ell)})$ is increasing on \emph{trained}-GCN models, which contradicts the conclusion in~\cite{oono2019graph}.
Due to the space limit, we defer the more empirical evidence to Appendix~\ref{supp:weakness_oversmoothing}.
Through these findings, we cast doubt on the power of over-smoothing based analysis to provide a complete picture of why deep GCNs perform badly.

\begin{definition} \label{def:computation_tree}
    Let $\mathcal{T}_i^{L}$ denote the $L$-layer computation tree of node $i$, which represents the structured $L$-hop neighbors of node $i$, where the children of any node $j$ in the tree are the nodes in $\mathcal{N}(i)$.
\end{definition}

\noindent\textbf{Exponentially growing expressieness without strong assumptions.}~
Indeed, we argue that \emph{deeper GCNs have stronger expressive power than the shallow GCNs.}
To prove this, we employ the connection between WL test\footnote{WL test is a recursive algorithm where the label of a node depends on its own label and neighbors from the previous iterations, i.e., $c^{(\ell)}_i = \textrm{Hash}(c^{(\ell-1)}_i, \{c^{(\ell-1)}_j|j\in\mathcal{N}(i)\})$, where $\textrm{Hash}(\cdot)$ bijectively maps a set of values to a unique value that has not been used in the previous iterations. 
After $L$-iterations, the WL test will assign two nodes with a different label if the $L$-hop neighborhood of two nodes are non-isomorphic. }~\cite{leman1968reduction} and GCNs.
Recently, \cite{morris2019weisfeiler} shows that GCNs have the same expressiveness as the WL-test for graph isomorphism if they are appropriately trained, i.e., a properly trained $L$-layer GCN computes different node representations for two nodes if their $L$-layer \emph{computation tree} (Definition~\ref{def:computation_tree}) have different structure or different features on the corresponding nodes.
Since $L\text{-GCN}$ can encode any different computation tree into different representations, it is natural to characterize the expressiveness of $L\text{-GCN}$ by the number of computation graph it can encode.



\begin{theorem}\label{theorem:expressive_power_exp}
Suppose $\mathcal{T}^{L}$ is a computation tree with binary node features and node degree at least $d$. 
Then, by assuming the computation tree of two nodes are disjoint,
the richness (i.e., the number of computation graphs a model can encode) of the output of $L\text{-GCN}$ defined on $\mathcal{T}^{L}$ is at least 
$| L\text{-GCN} (\mathcal{T}^{L}) | \geq 2(d-1)^{L-1}$.
\end{theorem}
The proof is deferred to Appendix~\ref{supp:omitted_proof_section_3}. The above theorem implies that the richness  of $L\text{-GCN}$ grows at least exponentially with respect to the number of layers.

\noindent\textbf{Comparison of expressiveness metrics.}~
Although distance-based expressiveness metric~\cite{oono2019graph,huang2020tackling} is strong than WL-based metric in the sense that node embeddings can be \emph{distinct} but \emph{close to each other}, the distance-based metric requires explicit assumptions on the GCN structures, weight matrices, and graph structures comparing to the WL-based metric, which has been shown that are not likely hold. 
On the other hand, WL-based metric has been widely used in characterizing the expressive power of GCNs in graph-level task~\cite{morris2019weisfeiler,loukas2019graph,chen2020can,chen2020graph}. 
More details are deferred to related works (Section~\ref{section:related_works}).


Although expressive, it is still unclear why the deeper GCN requires more training iterations to achieve small training error and reach the properly trained status.
To understand this, we show in Theorem~\ref{thm:deep_wide_local_global_minima2} that under assumptions on the width of the final layer, the deeper GCN can converge to its global optimal with linear convergence rate. 
More specifically, the theorem claims that if the dimension of the last layer of GCN $d_L$ is larger than the number of data $N$\footnote{This type of over-parameterization assumptions are required and commonly used in other neural network convergence analysis to guarantee that the model parameters do not change significantly during training.}, then we can guarantee the loss $\mathcal{L}(\boldsymbol{\theta}_T) \leq \varepsilon$ after $T=\mathcal{O}(1/\varepsilon)$ iterations of the gradient updates. 
Besides, more training iterations is sufficient (but not necessary due to the assumptions made in our theoretical analysis) for a deeper model to achieve the same training error as the shallow ones.

\begin{theorem} \label{thm:deep_wide_local_global_minima2}
Let $\bm{\theta}_t = \{\mathbf{W}_t^{(\ell)} \in \mathbb{R}^{d_{\ell-1} \times d_\ell} \}_{\ell=1}^{L+1}$ be the model parameter at the $t$-th iteration and using square loss $\mathcal{L}(\bm{\theta}) = \frac{1}{2} \| \mathbf{H}^{(L)}\mathbf{W}^{(L+1)} - \mathbf{Y} \|_\mathrm{F}^2,~\mathbf{H}^{(\ell)} = \sigma(\mathbf{P} \mathbf{H}^{(\ell-1)} \mathbf{W}^{(\ell)})$ as objective function. Then, under the condition that $d_L \geq N$ we can obtain $\mathcal{L}(\bm{\theta}_T) \leq \epsilon$ if $T \geq C(L) \log( \mathcal{L}(\bm{\theta}_0)/\epsilon )$, where $\epsilon$ is the desired error and $C(L)$ is a function of GCN depth $L$ that grows as GCN becomes deeper.
\end{theorem}
A formal statement of Theorem~\ref{thm:deep_wide_local_global_minima2} and its proof are deferred to Appendix~\ref{supp:linear_convergence_rate}.
Besides, gradient stability also provides an alternative way of empirically understanding why deeper GCN requires more iterations: deeper neural networks are prone to exploding/vanishing gradient, which results in a very noisy gradient and requires small learning rate to stabilize the training. This issue can be significantly alleviated by adding skip-connections (Appendix~\ref{supp:gradient_instability}). 
When training with adaptive learning rate mechanisms, such as Adam~\cite{kingma2015adam}\footnote{In the Adam optimizer, the contribution of bias correction of moments varies exponentially over epochs completed. Although the learning rate hyper-parameter is a constant, the contribution of gradients to updated weight varies over epochs, hence adaptive. Please refer to the Chapter 8.5 of~\cite{Goodfellow-et-al-2016} for more details.}, noisy gradient will result in a much smaller update on current model compared to a stabilized gradient, therefore more training iterations are required. 



\section{A different view from generalization}\label{section:from_stability}
 
 
In the previous section, we provided evidence that a well-trained deep GCN  is at least as powerful as a shallow one. 
However, it is still unclear why a deeper GCN has worse performance than a shallow GCN during the evaluation phase. 
To answer this question, we provide a different view by analyzing the impact of GCN structures on the generalization.

\noindent\textbf{Transductive uniform stability.}~
In the following, we study the generalization ability of GCNs via \emph{transductive uniform stability}~\cite{el2006stable}, where the generalization gap is defined as the difference between the training and testing errors for the random partition of a full dataset into training and testing sets.
Transductive uniform stability is defined under the notation that the output of a classifier does not change much if the input is perturbed a bit, which is an extension of uniform stability~\cite{bousquet2002stability} from the inductive to the transductive setting.
The previous analysis on the uniform stability of GCNs~\cite{verma2019stability} only shows the result of GCN with one graph convolutional layer under inductive learning setting, which cannot explain the effect of depth, model structure, and training data size on the generalization, and its extension to multi-layer GCNs and other GCN structures are non-trivial.

\noindent\textbf{Problem setup.}~
Let $m = |\mathcal{V}_\text{train}|$ and $u=|\mathcal{V}_\text{test}|$ denote the  training and test dataset sizes, respectively. 
Under the transductive learning setting, we start with a fixed set of points $X_{m+u}=\{x_1,\ldots, x_{m+u}\}$. 
For notational convenience, we assume $X_m$ are the first $m$ data points and $X_u$ are the last $u$ data points of $X_{m+u}$.
We randomly select a subset $X_m \subset X_{m+u}$ uniformly at random and reveal the labels $Y_m$ for the selected subset for training, but the labels for the remaining $u$ data points $Y_u = Y_{m+u}\setminus Y_m$ are not available during the training phase.
Let $S_m=((x_1, y_1),\ldots,(x_m, y_m))$ denotes the labeled set and $X_u = (x_{m+1},\ldots,x_{m+u})$ denotes the unlabeled set.
Our goal is to learn a model to label the remaining unlabeled set as accurately as possible.

For the analysis purpose, we assume a binary classifier is applied to the final layer node representation $f(\mathbf{h}_i^{(L)}) = \Tilde{\sigma}(\mathbf{v}^\top \mathbf{h}_i^{(L)})$ with $\Tilde{\sigma}(\cdot)$ denotes the sigmoid function. We predict $\hat{y}_i=1$ if $f(\mathbf{h}_i^{(L)}) > 1/2$ and $\hat{y}_i=0$ otherwise, with ground truth label $y_i\in\{0,1\}$.
Let denote the perturbed dataset as $S_m^{ij} \triangleq (S_m \setminus \{(x_i,y_i)\}) \cup \{(x_j, y_j)\}$ and $X_u^{ij} \triangleq (X_u \setminus \{x_j\}) \cup \{x_i\}$, which is obtained by replacing the $i$th example in training set $S_m$ with the $j$th example from the testing set $X_u$. 
Let $\boldsymbol{\theta}$ and $\boldsymbol{\theta}^{ij}$ denote the weight parameters trained on the original dataset $(S_m, X_u)$ and the perturbed dataset $(S_m^{ij}, X_u^{ij})$, respectively.
Then, we say transductive learner $f$ is $\epsilon$-uniformly stable if the outputs change less than $\epsilon$ when we exchange two examples from the training set and testing set, i.e., for any $S_m \subset S_{m+u}$ and any $i,j \in [m+u]$ it holds that $\sup_{i\in[m+u]} |f(\mathbf{h}_i^{(L)}) - \tilde{f}(\tilde{\mathbf{h}}^{(L)}_i) )| \leq \epsilon,$ where $f(\mathbf{h}_i^{(L)}) $ and $ \tilde{f}(\tilde{\mathbf{h}}^{(L)}_i)$ denote the prediction of node $i$ using parameters $\bm{\theta}$ and $\bm{\theta}^{ij}$ respectively.

To define testing error and training error in our setting, let us introduce the difference in probability between the correct and the incorrect label as $p(z, y) \triangleq y (2 z - 1) + (1-y) (1-2z)$ with $p(z, y) \leq 0$ if there exists a classification error.
Let denote the $\gamma$-margin loss as $\Phi_\gamma(x) = \min(1, \max(0, 1-x/\gamma))$.
Then, the testing error is defined as $\mathcal{R}_u(f) = \frac{1}{u}\sum_{i=m+1}^{m+u} \mathbf{1}\big\{p(f(\mathbf{h}_i^{(L)}), y_i) \leq 0 \big\}$ and the training loss is defined as $\mathcal{R}_{m}^\gamma(f) = \frac{1}{m}\sum_{i=1}^m \Phi_\gamma( - p(f(\mathbf{h}_i^{(L)}), y_i) )$.
\footnote{Notice that a different loss function is used in convergence analysis (i.e., $\mathcal{L}(\bm{\theta})$ in Theorem~\ref{thm:wide_deep_gcn_convergence}) and generalization analysis ($\mathcal{R}_m^\gamma(f)$ in Theorem~\ref{thm:generalization}).
We do this because current understanding on the convergence of deep neural networks is still mostly limited to the square loss, but the margin loss is a more suitable and widely accepted loss function for generalization analysis.}


\begin{theorem} [Transductive uniform stability bound~\cite{el2006stable}]\label{thm:uniform_stability_base}
Let $f$ be a $\epsilon$-uniformly stable transductive learner and $\gamma, \delta >0$, and define $Q=mu/(m+u)$. 
Then, with probability at least $1-\delta$ over all training and testing partitions, we have 
$\mathcal{R}_u(f) \leq \mathcal{R}_{m}^\gamma(f) + \frac{2}{\gamma}\mathcal{O}\left( \epsilon \sqrt{Q\ln (\delta^{-1})} \right) +\mathcal{O}\left(\frac{\ln(\delta^{-1})}{\sqrt{Q}}\right).$
\end{theorem}


Recall that as we discussed in Section~\ref{section:limitation_over_smoothing}, deeper GCN is provable more expressive and can achieve very low training error $\mathcal{R}_{m}^\gamma(f)$ if properly training.
Then, if the dataset size is sufficiently large, the testing error $\mathcal{R}_u(f)$ will be dominated by the generalization gap, which is mainly controlled by uniformly stable constant $\epsilon$.
In the following, we explore the impact of GCN structures on $\epsilon$.
Our key idea is to decompose the $\epsilon$ into three terms: 
the Lipschitz continuous constant $\rho_f$, upper bound on gradient $G_f$, and the smoothness constant $L_f$ of GCNs. Please refer to Lemma~\ref{lemma:useful_lemma_for_main_theorems} for details.

\begin{lemma}\label{lemma:useful_lemma_for_main_theorems}
Suppose function $f(\mathbf{h}^{(L)})$ is $\rho_f$-Lipschitz continuous, $L_f$-smooth, and the gradient of loss w.r.t. the parameter is bounded by $G_f$.  After $T$ steps of full-batch gradient descent, we have $\epsilon = \frac{2 \eta  \rho_f G_f}{m} \sum_{t=1}^T (1+\eta L_f)^{t-1}$.
\end{lemma}
The proof is deferred to Appendix~\ref{supp:omitted_proof_uniform_stability}.
By using Lemma~\ref{lemma:useful_lemma_for_main_theorems}, we can derive constant $\epsilon$ of different model structures by comparing the Lipschitz continuity, smoothness, and gradient scale.

Before proceeding to our result, we make the following standard assumption on the node feature vectors and weight matrices,  which are previously used in generalization analysis of GCNs~\cite{garg2020generalization,liao2020pac}.

\begin{assumption} \label{assumption:norm_bound}
We assume the norm of node feature vectors, weight parameters are bounded, i.e., $\|\mathbf{x}_i\|_2 \leq B_x$, $\| \mathbf{W}^{(\ell)} \|_2 \leq B_w$, and $\|\mathbf{v}\|_2 \leq 1$.
\end{assumption}


In Theorem~\ref{thm:generalization}, we show that the generalization bounds of GCN and its variants are dominated by the following terms: maximum node degree $d$, model depth $L$, training/validation set size $(m,u)$, training iterations $T$, and spectral norm of the weight matrices $B_w$. 
The larger the aforementioned variables are, the larger the generalization gap is. 
We defer the formal statements and proofs to  Appendices~\ref{supp:proof_gcn},~\ref{supp:proof_resgcn},~\ref{supp:proof_appnp}, and ~\ref{supp:proof_gcnii}.


\begin{theorem} [Informal] \label{thm:generalization}
We say model is $\epsilon$-uniformly stable with $\epsilon = \frac{2\eta \rho_f G_f}{m}\sum_{t=1}^T (1+\eta L_f)^{t-1}$ where the result of $\rho_f, G_f, L_f$ are summarized in Table~\ref{table:theorem_results_summary}, and other related constants as 
\begin{equation}
    \begin{aligned}
    B_d^\alpha &= (1-\alpha) \textstyle\sum_{\ell=1}^L (\alpha \sqrt{d})^{\ell-1} + (\alpha \sqrt{d})^L,~
    B_w^\beta = \beta B_w + (1-\beta), \\
    B_{\ell,d}^{\alpha,\beta} &= \max\big\{ \beta \big( (1-\alpha)L + \alpha\sqrt{d} \big), (1-\alpha) L B_w^\beta + 1 \big\}.
    \end{aligned}
\end{equation}
\end{theorem}

\begin{table*}[t]
\centering
\caption{Comparison of uniform stability constant $\epsilon$ of GCN variants, where $\mathcal{O}(\cdot)$ is used to hide constants that shared between all bounds.} \label{table:theorem_results_summary}
\vspace{5pt}
\resizebox{0.96\linewidth}{!}{
\begin{tabular}{@{}llll@{}}
\toprule
                         & $\rho_f$ and $G_f$            & $L_f$                                                        & $C_1$ and $C_2$  \\ \midrule
$\epsilon_\text{GCN}$    & $\mathcal{O}(C_1^L C_2)$ & $\mathcal{O}\big(C_1^L C_2  \big( (L+2) C_1^L C_2 + 2 \big)\big)$ & $C_1 = \max\{ 1, \sqrt{d} B_w \},~ C_2 = \sqrt{d} (1+B_x)$  \\ 
$\epsilon_\text{ResGCN}$ & $\mathcal{O}(C_1^L C_2)$ & $\mathcal{O}\big(C_1^L C_2  \big( (L+2) C_1^L C_2 + 2 \big)\big)$ & $C_1 = 1 + \sqrt{d} B_w,~ C_2 = \sqrt{d} (1+B_x)$            \\ 
$\epsilon_\text{APPNP}$  & $\mathcal{O}(C_1)$ & $\mathcal{O}\big( C_1 \big( C_1 C_2 \big) + 1\big) $                              & $C_1 = B_d^\alpha B_x, C_2 = \max\{1, B_w\} $  \\ 
$\epsilon_\text{GCNII}$  & $\mathcal{O}(\beta C_1^L C_2)$                                                               & $\mathcal{O}\big( \alpha \beta C_1^L C_2 \big( (\alpha \beta L + 2) C_1^L C_2 + 2\beta \big) \big) $ & $C_1 = \max\{ 1, \alpha \sqrt{d} B_w^\beta \},~C_2 = \sqrt{d} + B_{\ell,d}^{\alpha,\beta} B_x$  \\ \midrule
$\epsilon_\text{DGCN}$ & $\mathcal{O}(C_1)$ & $\mathcal{O}(C_1(C_1 C_2) + 1)$ & $C_1 = (\sqrt{d})^L B_x, C_2 = \max\{1, B_w\} $ \\ 
\bottomrule
\end{tabular}
}
\end{table*}

In the following, we provide intuitions and discussions on the generalization bound of each algorithm: 
\begin{itemize} [noitemsep,topsep=0pt,leftmargin=5mm]
    \item  Deep \textbf{GCN} requires iterations $T$ to achieve small training error. 
    Since the generalization bound increases with $T$, more iterations significantly hurt its generalization power. 
    Note that our results considers both $B_w \leq 1$ and $B_w > 1$, where increasing model depth will not hurt the generalization if $B_w \leq 1$, and the generalization gap becomes sensitive to the model depth if $B_w > 1$. 
    Notice that $B_w > 1$ is more likely to happen during training as we discussed in Section~\ref{section:limitation_over_smoothing}.
    \item \textbf{ResGCN} resolves the training difficulties by adding skip-connections between hidden layers. 
    Although it requires less training iterations $T$, adding skip-connections enlarges the dependency on the number of layers $L$ and the spectral norm of weight matrices $B_w$, therefore results in a larger generalization gap and a poor generalization performance.. 
    \item \textbf{APPNP} alleviates the aforementioned dependency by decoupling the weight parameters and feature propagation. 
    As a result, its generalization gap does not significantly change as $L$ and $B_w$ increase. 
    The optimal $\alpha$ that minimizes the generalization gap can be obtained by finding the $\alpha$ that minimize the term $B_d^\alpha$. 
    Although APPNP can significantly reduce the generalization gap, because a single weight matrix is shared between all layers, its expressive power is not enough for large-scale challenging graph datasets~\cite{hu2020open}.
    \item To gain expressiveness, \textbf{GCNII} proposes to add the weight matrices back and add another hyper-parameter that explicitly controls the dependency on $B_w$.
    Although GCNII achieves the state-of-the-art performances on several graph datasets, the selection of hyper-parameters is non-trivial compared to APPNP because $\alpha,\beta$ are coupled with $L, B_w$, and $d$.
    In practice, \cite{chen2020simple} builds a very deep GCNII by choosing $\beta$ dynamically decreases as the number of layers and different $\alpha$ values for different datasets.
    \item By property chosen hyper-parameters, we have the following order on the generalization gap given the same training iteration $T$: APPNP $\leq$ GCNII $\leq$ GCN $\leq$ ResGCN, which exactly match our empirical evaluation on the generalization gap in Section~\ref{section:experiment} and Appendix~\ref{supp:more_empirical}.
\end{itemize}

\begin{remark} \label{remark:dropedge_pairnorm}
It is worthy to provide an alternative view of DropEdge~\cite{rong2019dropedge} and PairNorm~\cite{zhao2019pairnorm} algorithms from a generalization perspective.  To improve the generalization power of standard GCNs, DropEdge randomly drops edges in the training phase, which leads to a smaller maximum node degree $d_s < d$. PairNorm applies normalization on intermediate node embeddings to ensure that the total pairwise feature distances remain constant across layers, which leads to less dependency on $d$ and $B_w$. 
However, since deep GCN requires significantly more iterations to achieve low training error than shallow one, the performance of applying DropEdge and PairNorm on GCNs is still degrading as the number of layers increases. 
Most importantly, our empirical results in Appendix~\ref{supp:dropegde} and~\ref{supp:pairnorm} suggest that applying Dropout and PairNorm is hurting the training accuracy (i.e., not alleviating over-smoothing) but reducing the generalization gap.
\end{remark}

\section{ Decoupled GCN}\label{section:method}

We propose to decouple the expressive power from generalization ability with \textbf{D}ecoupled \textbf{GCN} (DGCN).
The DGCN model can be mathematically formulated as $\mathbf{Z} = \sum_{\ell=1}^L \alpha_\ell f^{(\ell)}(\mathbf{X})$ and $f^{(\ell)}(\mathbf{X}) = \mathbf{P}^{\ell} \mathbf{X} \big( \beta_\ell \mathbf{W}^{(\ell)} + (1-\beta_\ell) \mathbf{I} \big)$,
where $\mathbf{W}^{(\ell)}$, $\alpha_\ell$ and $\beta_\ell$ are the learnable weights for $\ell$th layer function $f^{(\ell)}(\mathbf{X})$.
The design of DGCN has the following key ingredients:
\begin{itemize} [noitemsep,nolistsep,leftmargin=5mm]
    \item (Decoupling) The generalization gap in GCN grows exponentially with the number of layers. 
    To overcome this issue, we propose to decouple the weight matrices from propagation by assigning weight $\mathbf{W}^{(\ell)}$ to each individual layerwise function $f^{(\ell)}(\mathbf{X})$. 
    DGCN can be thought of as an ensemble of multiple SGCs~\cite{wu2019simplifying} with depth from $1$ to $L$. 
    By doing so, the generalization gap has less dependency on the number of weight matrices, and deep models with large receptive fields can incorporate information of the global graph structure. 
    Please refer to Theorem~\ref{thm:dgcn_result} for the details.
    \item (Learnable $\alpha_\ell$) After decoupling the weight matrices from feature propagation, the layerwise function $f^{(\ell)}(\mathbf{X})$ with more propagation steps can suffer from less expressive power.
    Therefore, we propose to assign a learnable weight $\alpha_\ell$ for each step of feature propagation. 
    Intuitively, DGCN assigns smaller weight $\alpha_\ell$ to each layerwise function $f^{(\ell)}(\mathbf{X})$ with more propagation steps at the beginning of training. 
    Throughout the training, DGCN gradually adjusts the weight to leverage more useful large receptive field information. 
    \item (Learnable $\beta_\ell$) A learnable weight $\beta_\ell\in[0,1]$ is assigned to each weight matrix to balance the expressiveness with model complexity, which guarantees a better generalization ability. 
\end{itemize}


\begin{theorem}\label{thm:dgcn_result}
Let suppose $\alpha_\ell$ and $\beta_\ell$ are pre-selected and fixed during training.
We say DGCN is $\epsilon_\texttt{DGCN}$-uniformly stable with $\epsilon_\texttt{DGCN} = \frac{2\eta \rho_f G_f}{m}\sum_{t=1}^T (1+\eta L_f)^{t-1}$ where 
\begin{equation}
    \rho_f = G_f = \mathcal{O}\Big( (\sqrt{d})^L B_x \Big), 
    L_f = \mathcal{O}\Big(( \sqrt{d})^L B_x \big( (\sqrt{d})^L B_x \max\{1, B_w\} + 1)  \big) \Big). 
\end{equation}
\end{theorem}

The details are deferred to Appendix~\ref{supp:proof_dgcn}, and comparison of bound to other GCN variants are summarized in Table~\ref{table:theorem_results_summary}. Depending on the automatic selection of $\alpha_\ell, \beta_\ell$, the generalization bound of DGCN is between APPNP and GCN. In the following, we make connection to many GCN structures:
\begin{itemize} [noitemsep,topsep=0pt,leftmargin=5mm ]
    \item \noindent\textbf{Connections to APPNP:}
    APPNP can be thought of as a variant of DGCN.
    More specifically, the layerwise weight in APPNP is computed as $\alpha_\ell = \alpha(1-\alpha)^\ell$ for $\ell < L$ and $\alpha_\ell = (1-\alpha)^\ell$ for $\ell=L$ given some constant $\alpha\in(0,1)$, and the weight matrix is shared between all layers.
    Although DGCN has $L$ weight matrices, its generalization is independent of the number of weight matrices, and thus enjoys a low generalization error with high expressiveness.
    \item \noindent\textbf{Connections to GCNII:}
    GCNII can be regarded as a variant of DGCN. 
    Compared to GCNII, the decoupled propagation of DGCN significantly reduces the dependency of generalization error to the weight matrices. 
    Besides, the learnable weights $\alpha_\ell$ and $\beta_\ell$ allow DGCN to automatically adapt to challenging large-scale datasets without time-consuming hyper-parameter selection.
    \item \noindent\textbf{Connections to ResGCN:}
    By expanding the forward computation of ResGCN, we know that ResGCN can be think of as training an ensemble of GCNs from $1$ to $L$ layer, i.e., $\smash{\mathbf{H}^{(L)} = \sum_{\ell=1}^L \alpha_\ell \sigma(\mathbf{P} \mathbf{H}^{(\ell-1)} \mathbf{W}^{(\ell)})}$ with $\alpha_\ell=1$. 
    In other word,
    ResNet can be regarded as the ``\emph{summation} of the model complexity'' of $L$-layer. However, DGCN is using $\smash{\sum_{\ell=1}^L \alpha_\ell = 1}$, which can be thought of as a ``\emph{weighted average} of model complexity''. 
    Therefore, ResGCN is a special case of DGCN with equal weights $\alpha_\ell$ on each layerwise function. With just a simple change on the ResNet structure, our model DGCN is both easy to train and good to generalize.
\end{itemize}



\section{Experiments}\label{section:experiment}

\begin{figure*}[t]
    \centering
    \includegraphics[width=0.99\textwidth]{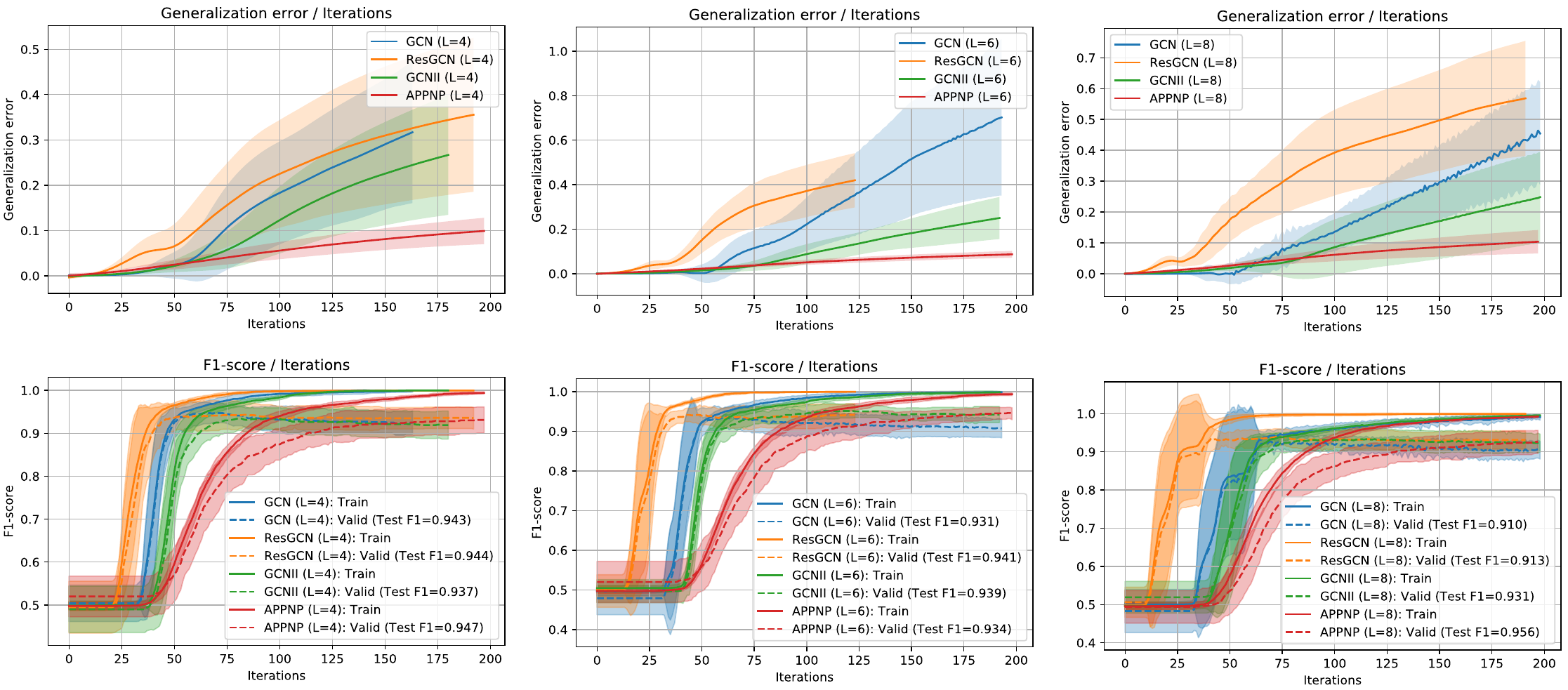}
    \vspace{-4mm}
    \caption{Comparison of generalization error  on synthetic dataset.
    The curve early stopped at the largest training accuracy iteration.  }
    \label{fig:synthetic_dataset}
\end{figure*}

\noindent\textbf{Synthetic dataset.}
We empirically compare the generalization error of different GCN structures on the  synthetic dataset. 
In particular, we create the synthetic dataset by contextual stochastic block model (CSBM)~\cite{deshpande2018contextual} with two equal-size classes. 
CSBM is a graph generation algorithm that adds Gaussian random vectors as node features on top of classical SBM.
CSBM allows for smooth control over the information ratio between node features and graph topology by a pair of hyper-parameter $(\mu,\lambda)$, where $\mu$ controls the diversity of the Gaussian distribution and $\lambda$ controls the number of edges between intra- and inter-class nodes. 
We generate random graphs with $1000$ nodes, average node degree as $5$, and each node has a Gaussian random vector of dimension $1000$ as node features. 
We chose $75\%$ nodes as training set, $15\%$ of nodes as validation set for hyper-parameter tuning, and the remaining nodes as testing set.
We conduct an experiment $20$ times by randomly selecting $(\mu,\lambda)$ such that both node feature and graph topology are equally informative. 

As shown in Figure~\ref{fig:synthetic_dataset}, we have the following order on the generalization gap given the same training iteration $T$: APPNP $\leq$ GCNII $\leq$ GCN $\leq$ ResGCN, which exactly match the theoretical result in Theorem~\ref{thm:generalization}. More specifically,
ResGCN has the largest generalization gap due to the skip-connections, APPNP has the smallest generalization gap by removing the weight matrices in each individual layer. 
GCNII achieves a good balance between GCN and APPNP by balancing the expressive and generalization power.
Finally, DGCN enjoys a small generalization error by using the decoupled GCN structure.

\noindent\textbf{Open graph benchmark dataset.}
As pointed out by Hu et al.~\cite{hu2020open}, the traditional commonly-used graph datasets are unable to provide a reliable evaluation due to various factors including dataset size, leakage of node features and no consensus on data splitting. 
To truly evaluate the expressive and the generalization power of existing methods, we evaluate on the open graph benchmark (OGB) dataset.
Experiment setups are based on the default setting for GCN implementation on the leaderboard. 
We choose the hidden dimension as $128$, learning rate as $0.01$, dropout ratio as $0.5$ for \textit{Arxiv} dataset, and no dropout for \textit{Products} and \textit{Protein} datasets. 
We train $300/1000/500$ epochs for \textit{Products}, \textit{Proteins}, and \textit{Arxiv} dataset respectively. 
Due to limited GPU memory, the number of layers is selected as the one with the best performance between $2$ to $16$ layers for \textit{Arxiv} dataset, $2$ to $8$ layers for \textit{Protein} dataset, and $2$ to $4$ for \textit{Products} dataset. 
We choose $\alpha_\ell$ from $\{0.9, 0.8, 0.5\}$ for APPNP and GCNII, and use $\beta_\ell = 0.5/\ell$ for GCNII, and select the setup with the best validation result for comparison.

As shown in Table~\ref{table:OBG result}, DGCN achieves a compatible performance to GCNII\footnote{$\dagger$ GCNII underfits with the default hyper-parameters. $\ddagger$ GCNII achieves $72.74 \pm 0.16 $ by using hidden dimension as $256$ and a different design of graph convolution layer. Refer  \hyperlink{https://github.com/chennnM/GCNII/blob/ca91f5686c4cd09cc1c6f98431a5d5b7e36acc92/PyG/ogbn-arxiv/layer.py}{here} for details.} without the need of manually tuning the hyper-parameters for all settings, and it significantly outperform APPNP and ResGCN.
Due to the space limit, the detailed setups and more results can be found in Appendix~\ref{supp:more_empirical}.
Notice that generalization bounds are more valuable when comparing two models with same training accuracy (therefore we first show in Section~\ref{section:limitation_over_smoothing} that deeper model can also achieve low training error before our discussion on generalization in Section~\ref{section:from_stability}). 
\begin{wraptable}[10]{r}{0.56\textwidth} 
\vspace{-12pt}
\caption{Comparison of F1-score on OGB dataset.} \label{table:OBG result}
\scalebox{0.85}{
\begin{tabular}{llll}
\toprule
\textbf{$\%$}                & \textbf{Products}    & \textbf{Proteins}            & \textbf{Arvix}    \\ \midrule
\textbf{{GCN}}    & $75.39 \pm 0.21$     & $71.66 \pm 0.48$             & $71.56 \pm 0.19$  \\ \midrule
\textbf{{ResGCN}} & $75.53 \pm 0.12$     & $74.50 \pm 0.41$             & $72.56 \pm 0.31$  \\ \midrule
\textbf{{APPNP}}  & $66.35 \pm 0.10$     & $71.78 \pm 0.29$             & $68.02 \pm 0.55$  \\ \midrule
\textbf{{GCNII}}  & $71.93 \pm 0.35^\dagger$     & $75.60 \pm 0.47$             & $72.57 \pm 0.23^\ddagger$  \\ \midrule
\textbf{{DGCN}}   & $76.09 \pm 0.29$     & $75.45 \pm 0.24$             & $72.63 \pm 0.12$  \\ \bottomrule
\end{tabular}}
\end{wraptable} 
In Table 2, because ResGCN has no restriction on the weight matrices, it can achieve lower training error and its test performance is mainly restricted by its generalization error. 
However, because GCNII and APPNP have restrictions on the weight matrices, their performance is mainly restricted by their training error. 
A model with small generalization error and no restriction on the weight (e.g., DGCN) is preferred as it has higher potential to reach a better test accuracy by reducing its training error.


\section{Conclusion}
In this work, we show that there exists a discrepancy between over-smoothing based theoretical results and the practical behavior of deep GCNs. Our theoretical result shows that a deeper GCN can be as expressive as a shallow GCN, if it is properly trained. 
To truly understand the performance decay issue of deep GCNs, we provide the first transductive uniform stability-based generalization analysis of GCNs and other GCN structures. 
To improve the optimization issue and benefit from depth, we propose DGCN that enjoys a provable high expressive power and generalization power.
We conduct empirical evaluations on various synthetic and real-world datasets to validate the correctness of our theory and advantages over the baselines.

\section*{Acknowledgements}
This work was supported in part by NSF grant 2008398.

\clearpage
\bibliographystyle{plain}
\bibliography{reference}

\begin{thebibliography}{10}

\bibitem{alon2020bottleneck}
Uri Alon and Eran Yahav.
\newblock On the bottleneck of graph neural networks and its practical
  implications.
\newblock In {\em International Conference on Learning Representations}, 2020.

\bibitem{berg2017graph}
Rianne van~den Berg, Thomas~N Kipf, and Max Welling.
\newblock Graph convolutional matrix completion.
\newblock {\em arXiv preprint arXiv:1706.02263}, 2017.

\bibitem{bousquet2002stability}
Olivier Bousquet and Andr{\'e} Elisseeff.
\newblock Stability and generalization.
\newblock {\em Journal of machine learning research}, 2(Mar):499--526, 2002.

\bibitem{cai2020note}
Chen Cai and Yusu Wang.
\newblock A note on over-smoothing for graph neural networks.
\newblock {\em arXiv preprint arXiv:2006.13318}, 2020.

\bibitem{chen2020graph}
Lei Chen, Zhengdao Chen, and Joan Bruna.
\newblock On graph neural networks versus graph-augmented mlps.
\newblock In {\em International Conference on Learning Representations}, 2020.

\bibitem{chen2020simple}
Ming Chen, Zhewei Wei, Zengfeng Huang, Bolin Ding, and Yaliang Li.
\newblock Simple and deep graph convolutional networks.
\newblock In {\em Proceedings of the 37th International Conference on Machine
  Learning, {ICML} 2020, 13-18 July 2020, Virtual Event}, volume 119 of {\em
  Proceedings of Machine Learning Research}, pages 1725--1735. {PMLR}, 2020.

\bibitem{chen2020can}
Zhengdao Chen, Lei Chen, Soledad Villar, and Joan Bruna.
\newblock Can graph neural networks count substructures?
\newblock In {\em Advances in Neural Information Processing Systems 33: Annual
  Conference on Neural Information Processing Systems 2020, NeurIPS 2020,
  December 6-12, 2020, virtual}, 2020.

\bibitem{cong2020minimal}
Weilin Cong, Rana Forsati, Mahmut~T. Kandemir, and Mehrdad Mahdavi.
\newblock Minimal variance sampling with provable guarantees for fast training
  of graph neural networks.
\newblock In {\em {KDD} '20: The 26th {ACM} {SIGKDD} Conference on Knowledge
  Discovery and Data Mining, Virtual Event, CA, USA, August 23-27, 2020}, pages
  1393--1403. {ACM}, 2020.

\bibitem{cui2019traffic}
Zhiyong Cui, Kristian Henrickson, Ruimin Ke, and Yinhai Wang.
\newblock Traffic graph convolutional recurrent neural network: A deep learning
  framework for network-scale traffic learning and forecasting.
\newblock {\em IEEE Transactions on Intelligent Transportation Systems}, 2019.

\bibitem{davidson2001local}
Kenneth~R Davidson and Stanislaw~J Szarek.
\newblock Local operator theory, random matrices and banach spaces.
\newblock {\em Handbook of the geometry of Banach spaces}, 1(317-366):131,
  2001.

\bibitem{deng2019learning}
Songgaojun Deng, Huzefa Rangwala, and Yue Ning.
\newblock Learning dynamic context graphs for predicting social events.
\newblock In {\em Proceedings of the 25th {ACM} {SIGKDD} International
  Conference on Knowledge Discovery {\&} Data Mining, {KDD} 2019, Anchorage,
  AK, USA, August 4-8, 2019}, pages 1007--1016. {ACM}, 2019.

\bibitem{deshpande2018contextual}
Yash Deshpande, Subhabrata Sen, Andrea Montanari, and Elchanan Mossel.
\newblock Contextual stochastic block models.
\newblock In {\em Advances in Neural Information Processing Systems 31: Annual
  Conference on Neural Information Processing Systems 2018, NeurIPS 2018,
  December 3-8, 2018, Montr{\'{e}}al, Canada}, pages 8590--8602, 2018.

\bibitem{do2019graph}
Kien Do, Truyen Tran, and Svetha Venkatesh.
\newblock Graph transformation policy network for chemical reaction prediction.
\newblock In {\em Proceedings of the 25th {ACM} {SIGKDD} International
  Conference on Knowledge Discovery {\&} Data Mining, {KDD} 2019, Anchorage,
  AK, USA, August 4-8, 2019}, pages 750--760. {ACM}, 2019.

\bibitem{du2019graph}
Simon~S. Du, Kangcheng Hou, Ruslan Salakhutdinov, Barnab{\'{a}}s P{\'{o}}czos,
  Ruosong Wang, and Keyulu Xu.
\newblock Graph neural tangent kernel: Fusing graph neural networks with graph
  kernels.
\newblock In {\em Advances in Neural Information Processing Systems 32: Annual
  Conference on Neural Information Processing Systems 2019, NeurIPS 2019,
  December 8-14, 2019, Vancouver, BC, Canada}, pages 5724--5734, 2019.

\bibitem{duvenaud2015convolutional}
David Duvenaud, Dougal Maclaurin, Jorge Aguilera{-}Iparraguirre, Rafael
  G{\'{o}}mez{-}Bombarelli, Timothy Hirzel, Al{\'{a}}n Aspuru{-}Guzik, and
  Ryan~P. Adams.
\newblock Convolutional networks on graphs for learning molecular fingerprints.
\newblock In {\em Advances in Neural Information Processing Systems 28: Annual
  Conference on Neural Information Processing Systems 2015, December 7-12,
  2015, Montreal, Quebec, Canada}, pages 2224--2232, 2015.

\bibitem{el2006stable}
Ran El{-}Yaniv and Dmitry Pechyony.
\newblock Stable transductive learning.
\newblock In {\em Learning Theory, 19th Annual Conference on Learning Theory,
  {COLT} 2006, Pittsburgh, PA, USA, June 22-25, 2006, Proceedings}, volume 4005
  of {\em Lecture Notes in Computer Science}, pages 35--49. Springer, 2006.

\bibitem{garg2020generalization}
Vikas~K. Garg, Stefanie Jegelka, and Tommi~S. Jaakkola.
\newblock Generalization and representational limits of graph neural networks.
\newblock In {\em Proceedings of the 37th International Conference on Machine
  Learning, {ICML} 2020, 13-18 July 2020, Virtual Event}, volume 119 of {\em
  Proceedings of Machine Learning Research}, pages 3419--3430. {PMLR}, 2020.

\bibitem{glorot2010understanding}
Xavier Glorot and Yoshua Bengio.
\newblock Understanding the difficulty of training deep feedforward neural
  networks.
\newblock In {\em Proceedings of the thirteenth international conference on
  artificial intelligence and statistics}, pages 249--256. JMLR Workshop and
  Conference Proceedings, 2010.

\bibitem{gong2020geometrically}
Shunwang Gong, Mehdi Bahri, Michael~M. Bronstein, and Stefanos Zafeiriou.
\newblock Geometrically principled connections in graph neural networks.
\newblock In {\em 2020 {IEEE/CVF} Conference on Computer Vision and Pattern
  Recognition, {CVPR} 2020, Seattle, WA, USA, June 13-19, 2020}, pages
  11412--11421. {IEEE}, 2020.

\bibitem{hamilton2017inductive}
William~L. Hamilton, Zhitao Ying, and Jure Leskovec.
\newblock Inductive representation learning on large graphs.
\newblock In {\em Advances in Neural Information Processing Systems 30: Annual
  Conference on Neural Information Processing Systems 2017, December 4-9, 2017,
  Long Beach, CA, {USA}}, pages 1024--1034, 2017.

\bibitem{he2015delving}
Kaiming He, Xiangyu Zhang, Shaoqing Ren, and Jian Sun.
\newblock Delving deep into rectifiers: Surpassing human-level performance on
  imagenet classification.
\newblock In {\em 2015 {IEEE} International Conference on Computer Vision,
  {ICCV} 2015, Santiago, Chile, December 7-13, 2015}, pages 1026--1034. {IEEE}
  Computer Society, 2015.

\bibitem{he2016deep}
Kaiming He, Xiangyu Zhang, Shaoqing Ren, and Jian Sun.
\newblock Deep residual learning for image recognition.
\newblock In {\em 2016 {IEEE} Conference on Computer Vision and Pattern
  Recognition, {CVPR} 2016, Las Vegas, NV, USA, June 27-30, 2016}, pages
  770--778. {IEEE} Computer Society, 2016.

\bibitem{hu2020open}
Weihua Hu, Matthias Fey, Marinka Zitnik, Yuxiao Dong, Hongyu Ren, Bowen Liu,
  Michele Catasta, and Jure Leskovec.
\newblock Open graph benchmark: Datasets for machine learning on graphs.
\newblock In {\em Advances in Neural Information Processing Systems 33: Annual
  Conference on Neural Information Processing Systems 2020, NeurIPS 2020,
  December 6-12, 2020, virtual}, 2020.

\bibitem{huang2021wide}
Wei Huang, Yayong Li, Weitao Du, Richard~Yi Da~Xu, Jie Yin, and Ling Chen.
\newblock Wide graph neural networks: Aggregation provably leads to
  exponentially trainability loss.
\newblock {\em arXiv preprint arXiv:2103.03113}, 2021.

\bibitem{huang2020tackling}
Wenbing Huang, Yu~Rong, Tingyang Xu, Fuchun Sun, and Junzhou Huang.
\newblock Tackling over-smoothing for general graph convolutional networks.
\newblock {\em arXiv e-prints}, pages arXiv--2008, 2020.

\bibitem{kingma2015adam}
Diederik~P. Kingma and Jimmy Ba.
\newblock Adam: {A} method for stochastic optimization.
\newblock In {\em 3rd International Conference on Learning Representations,
  {ICLR} 2015, San Diego, CA, USA, May 7-9, 2015, Conference Track
  Proceedings}, 2015.

\bibitem{kipf2016semi}
Thomas~N. Kipf and Max Welling.
\newblock Semi-supervised classification with graph convolutional networks.
\newblock In {\em 5th International Conference on Learning Representations,
  {ICLR} 2017, Toulon, France, April 24-26, 2017, Conference Track
  Proceedings}. OpenReview.net, 2017.

\bibitem{klicpera2018predict}
Johannes Klicpera, Aleksandar Bojchevski, and Stephan G{\"{u}}nnemann.
\newblock Predict then propagate: Graph neural networks meet personalized
  pagerank.
\newblock In {\em 7th International Conference on Learning Representations,
  {ICLR} 2019, New Orleans, LA, USA, May 6-9, 2019}. OpenReview.net, 2019.

\bibitem{kong2020flag}
Kezhi Kong, Guohao Li, Mucong Ding, Zuxuan Wu, Chen Zhu, Bernard Ghanem, Gavin
  Taylor, and Tom Goldstein.
\newblock Flag: Adversarial data augmentation for graph neural networks.
\newblock {\em arXiv preprint arXiv:2010.09891}, 2020.

\bibitem{kumar2019predicting}
Srijan Kumar, Xikun Zhang, and Jure Leskovec.
\newblock Predicting dynamic embedding trajectory in temporal interaction
  networks.
\newblock In {\em Proceedings of the 25th {ACM} {SIGKDD} International
  Conference on Knowledge Discovery {\&} Data Mining, {KDD} 2019, Anchorage,
  AK, USA, August 4-8, 2019}, pages 1269--1278. {ACM}, 2019.

\bibitem{leman1968reduction}
AA~Leman and B~Weisfeiler.
\newblock A reduction of a graph to a canonical form and an algebra arising
  during this reduction.
\newblock {\em Nauchno-Technicheskaya Informatsiya}, 2(9):12--16, 1968.

\bibitem{li2019deepgcns}
Guohao Li, Matthias M{\"{u}}ller, Ali~K. Thabet, and Bernard Ghanem.
\newblock Deepgcns: Can gcns go as deep as cnns?
\newblock In {\em 2019 {IEEE/CVF} International Conference on Computer Vision,
  {ICCV} 2019, Seoul, Korea (South), October 27 - November 2, 2019}, pages
  9266--9275. {IEEE}, 2019.

\bibitem{li2019predicting}
Jia Li, Zhichao Han, Hong Cheng, Jiao Su, Pengyun Wang, Jianfeng Zhang, and
  Lujia Pan.
\newblock Predicting path failure in time-evolving graphs.
\newblock In {\em Proceedings of the 25th {ACM} {SIGKDD} International
  Conference on Knowledge Discovery {\&} Data Mining, {KDD} 2019, Anchorage,
  AK, USA, August 4-8, 2019}, pages 1279--1289. {ACM}, 2019.

\bibitem{li2018deeper}
Qimai Li, Zhichao Han, and Xiao{-}Ming Wu.
\newblock Deeper insights into graph convolutional networks for semi-supervised
  learning.
\newblock In {\em Proceedings of the Thirty-Second {AAAI} Conference on
  Artificial Intelligence, (AAAI-18), the 30th innovative Applications of
  Artificial Intelligence (IAAI-18), and the 8th {AAAI} Symposium on
  Educational Advances in Artificial Intelligence (EAAI-18), New Orleans,
  Louisiana, USA, February 2-7, 2018}, pages 3538--3545. {AAAI} Press, 2018.

\bibitem{liao2020pac}
Renjie Liao, Raquel Urtasun, and Richard Zemel.
\newblock A pac-bayesian approach to generalization bounds for graph neural
  networks.
\newblock {\em arXiv preprint arXiv:2012.07690}, 2020.

\bibitem{loukas2019graph}
Andreas Loukas.
\newblock What graph neural networks cannot learn: depth vs width.
\newblock In {\em 8th International Conference on Learning Representations,
  {ICLR} 2020, Addis Ababa, Ethiopia, April 26-30, 2020}. OpenReview.net, 2020.

\bibitem{luan2020training}
Sitao Luan, Mingde Zhao, Xiao-Wen Chang, and Doina Precup.
\newblock Training matters: Unlocking potentials of deeper graph convolutional
  neural networks.
\newblock {\em arXiv preprint arXiv:2008.08838}, 2020.

\bibitem{morris2019weisfeiler}
Christopher Morris, Martin Ritzert, Matthias Fey, William~L. Hamilton, Jan~Eric
  Lenssen, Gaurav Rattan, and Martin Grohe.
\newblock Weisfeiler and leman go neural: Higher-order graph neural networks.
\newblock In {\em The Thirty-Third {AAAI} Conference on Artificial
  Intelligence, {AAAI} 2019, The Thirty-First Innovative Applications of
  Artificial Intelligence Conference, {IAAI} 2019, The Ninth {AAAI} Symposium
  on Educational Advances in Artificial Intelligence, {EAAI} 2019, Honolulu,
  Hawaii, USA, January 27 - February 1, 2019}, pages 4602--4609. {AAAI} Press,
  2019.

\bibitem{nguyen2021proof}
Quynh Nguyen.
\newblock On the proof of global convergence of gradient descent for deep relu
  networks with linear widths.
\newblock {\em arXiv preprint arXiv:2101.09612}, 2021.

\bibitem{oono2019graph}
Kenta Oono and Taiji Suzuki.
\newblock Graph neural networks exponentially lose expressive power for node
  classification.
\newblock In {\em 8th International Conference on Learning Representations,
  {ICLR} 2020, Addis Ababa, Ethiopia, April 26-30, 2020}. OpenReview.net, 2020.

\bibitem{oono2020optimization}
Kenta Oono and Taiji Suzuki.
\newblock Optimization and generalization analysis of transduction through
  gradient boosting and application to multi-scale graph neural networks.
\newblock {\em Advances in Neural Information Processing Systems}, 33, 2020.

\bibitem{park2019estimating}
Namyong Park, Andrey Kan, Xin~Luna Dong, Tong Zhao, and Christos Faloutsos.
\newblock Estimating node importance in knowledge graphs using graph neural
  networks.
\newblock In {\em Proceedings of the 25th {ACM} {SIGKDD} International
  Conference on Knowledge Discovery {\&} Data Mining, {KDD} 2019, Anchorage,
  AK, USA, August 4-8, 2019}, pages 596--606. {ACM}, 2019.

\bibitem{qiu2018deepinf}
Jiezhong Qiu, Jian Tang, Hao Ma, Yuxiao Dong, Kuansan Wang, and Jie Tang.
\newblock Deepinf: Modeling influence locality in large social networks.
\newblock In {\em KDD}, 2018.

\bibitem{rahimi2018semi}
Afshin Rahimi, Trevor Cohn, and Timothy Baldwin.
\newblock Semi-supervised user geolocation via graph convolutional networks.
\newblock In {\em Proceedings of the 56th Annual Meeting of the Association for
  Computational Linguistics (Volume 1: Long Papers)}, pages 2009--2019.
  Association for Computational Linguistics, 2018.

\bibitem{ramezani2020gcn}
Morteza Ramezani, Weilin Cong, Mehrdad Mahdavi, Anand Sivasubramaniam, and
  Mahmut Kandemir.
\newblock Gcn meets gpu: Decoupling “when to sample” from “how to
  sample”.
\newblock {\em Advances in Neural Information Processing Systems}, 33, 2020.

\bibitem{rong2019dropedge}
Yu~Rong, Wenbing Huang, Tingyang Xu, and Junzhou Huang.
\newblock Dropedge: Towards deep graph convolutional networks on node
  classification.
\newblock In {\em 8th International Conference on Learning Representations,
  {ICLR} 2020, Addis Ababa, Ethiopia, April 26-30, 2020}. OpenReview.net, 2020.

\bibitem{Goodfellow-et-al-2016}
Ruslan Salakhutdinov.
\newblock Deep learning.
\newblock In {\em The 20th {ACM} {SIGKDD} International Conference on Knowledge
  Discovery and Data Mining, {KDD} '14, New York, NY, {USA} - August 24 - 27,
  2014}, page 1973. {ACM}, 2014.

\bibitem{scarselli2018vapnik}
Franco Scarselli, Ah~Chung Tsoi, and Markus Hagenbuchner.
\newblock The vapnik--chervonenkis dimension of graph and recursive neural
  networks.
\newblock {\em Neural Networks}, 108:248--259, 2018.

\bibitem{verma2019stability}
Saurabh Verma and Zhi{-}Li Zhang.
\newblock Stability and generalization of graph convolutional neural networks.
\newblock In {\em Proceedings of the 25th {ACM} {SIGKDD} International
  Conference on Knowledge Discovery {\&} Data Mining, {KDD} 2019, Anchorage,
  AK, USA, August 4-8, 2019}, pages 1539--1548. {ACM}, 2019.

\bibitem{wang2019mcne}
Hao Wang, Tong Xu, Qi~Liu, Defu Lian, Enhong Chen, Dongfang Du, Han Wu, and Wen
  Su.
\newblock {MCNE:} an end-to-end framework for learning multiple conditional
  network representations of social network.
\newblock In {\em Proceedings of the 25th {ACM} {SIGKDD} International
  Conference on Knowledge Discovery {\&} Data Mining, {KDD} 2019, Anchorage,
  AK, USA, August 4-8, 2019}, pages 1064--1072. {ACM}, 2019.

\bibitem{wang2019knowledge}
Hongwei Wang, Fuzheng Zhang, Mengdi Zhang, Jure Leskovec, Miao Zhao, Wenjie Li,
  and Zhongyuan Wang.
\newblock Knowledge-aware graph neural networks with label smoothness
  regularization for recommender systems.
\newblock In {\em Proceedings of the 25th {ACM} {SIGKDD} International
  Conference on Knowledge Discovery {\&} Data Mining, {KDD} 2019, Anchorage,
  AK, USA, August 4-8, 2019}, pages 968--977. {ACM}, 2019.

\bibitem{wang2019kgat}
Xiang Wang, Xiangnan He, Yixin Cao, Meng Liu, and Tat{-}Seng Chua.
\newblock {KGAT:} knowledge graph attention network for recommendation.
\newblock In {\em Proceedings of the 25th {ACM} {SIGKDD} International
  Conference on Knowledge Discovery {\&} Data Mining, {KDD} 2019, Anchorage,
  AK, USA, August 4-8, 2019}, pages 950--958. {ACM}, 2019.

\bibitem{wu2019simplifying}
Felix Wu, Amauri H.~Souza Jr., Tianyi Zhang, Christopher Fifty, Tao Yu, and
  Kilian~Q. Weinberger.
\newblock Simplifying graph convolutional networks.
\newblock In {\em Proceedings of the 36th International Conference on Machine
  Learning, {ICML} 2019, 9-15 June 2019, Long Beach, California, {USA}},
  volume~97 of {\em Proceedings of Machine Learning Research}, pages
  6861--6871. {PMLR}, 2019.

\bibitem{xu2018powerful}
Keyulu Xu, Weihua Hu, Jure Leskovec, and Stefanie Jegelka.
\newblock How powerful are graph neural networks?
\newblock In {\em 7th International Conference on Learning Representations,
  {ICLR} 2019, New Orleans, LA, USA, May 6-9, 2019}. OpenReview.net, 2019.

\bibitem{xu2019can}
Keyulu Xu, Jingling Li, Mozhi Zhang, Simon~S. Du, Ken{-}ichi Kawarabayashi, and
  Stefanie Jegelka.
\newblock What can neural networks reason about?
\newblock In {\em 8th International Conference on Learning Representations,
  {ICLR} 2020, Addis Ababa, Ethiopia, April 26-30, 2020}. OpenReview.net, 2020.

\bibitem{xu2021how}
Keyulu Xu, Mozhi Zhang, Jingling Li, Simon~Shaolei Du, Ken-Ichi Kawarabayashi,
  and Stefanie Jegelka.
\newblock How neural networks extrapolate: From feedforward to graph neural
  networks.
\newblock In {\em International Conference on Learning Representations}, 2021.

\bibitem{yan2021two}
Yujun Yan, Milad Hashemi, Kevin Swersky, Yaoqing Yang, and Danai Koutra.
\newblock Two sides of the same coin: Heterophily and oversmoothing in graph
  convolutional neural networks.
\newblock {\em arXiv preprint arXiv:2102.06462}, 2021.

\bibitem{yang2020revisiting}
Chaoqi Yang, Ruijie Wang, Shuochao Yao, Shengzhong Liu, and Tarek Abdelzaher.
\newblock Revisiting" over-smoothing" in deep gcns.
\newblock {\em arXiv preprint arXiv:2003.13663}, 2020.

\bibitem{ying2018graph}
Rex Ying, Ruining He, Kaifeng Chen, Pong Eksombatchai, William~L. Hamilton, and
  Jure Leskovec.
\newblock Graph convolutional neural networks for web-scale recommender
  systems.
\newblock In {\em Proceedings of the 24th {ACM} {SIGKDD} International
  Conference on Knowledge Discovery {\&} Data Mining, {KDD} 2018, London, UK,
  August 19-23, 2018}, pages 974--983. {ACM}, 2018.

\bibitem{zhao2019pairnorm}
Lingxiao Zhao and Leman Akoglu.
\newblock Pairnorm: Tackling oversmoothing in gnns.
\newblock In {\em 8th International Conference on Learning Representations,
  {ICLR} 2020, Addis Ababa, Ethiopia, April 26-30, 2020}. OpenReview.net, 2020.

\bibitem{zhou2020effective}
Kuangqi Zhou, Yanfei Dong, Wee~Sun Lee, Bryan Hooi, Huan Xu, and Jiashi Feng.
\newblock Effective training strategies for deep graph neural networks.
\newblock {\em arXiv preprint arXiv:2006.07107}, 2020.

\bibitem{zhou2021generalization}
Xianchen Zhou and Hongxia Wang.
\newblock The generalization error of graph convolutional networks may enlarge
  with more layers.
\newblock {\em Neurocomputing}, 424:97--106, 2021.

\end{thebibliography}

\clearpage


\appendix

\clearpage
\makeappendixtitle
\noindent\textbf{Organization.}
In Section~\ref{supp:more_result_on_linear_model}, we provide additional empirical evaluations on the change of pairwise distances for node embeddings as the number of layers increases.
In Section~\ref{supp:weakness_oversmoothing}, we summarize the existing theoretical results on over-smoothing and provide empirical validation on whether over-smoothing happens in practice.
In Sections~\ref{supp:omitted_proof_section_3} and~\ref{supp:linear_convergence_rate}, we provide the proof of Theorem~\ref{theorem:expressive_power_exp} (expressive power) and Theorem~\ref{thm:deep_wide_local_global_minima2} (convergence to global optimal and characterization of number of iterations), respectively.
In Section~\ref{supp:more_empirical}, we provide more empirical results on the effectiveness of the proposed algorithm.
In Sections~\ref{supp:proof_gcn},~\ref{supp:proof_resgcn},~\ref{supp:proof_appnp},~\ref{supp:proof_gcnii}, and~\ref{supp:proof_dgcn} we provide the generalization analysis of GCN, ResGCN, APPNP, GCNII, and DGCN, respectively. Code can be found at the following repository: 
\begin{center}
\url{https://github.com/CongWeilin/DGCN}.
\end{center}

\section{Empirical results on the pairwise distance of node embeddings}\label{supp:more_result_on_linear_model}

In this section, we provide additional empirical evaluations on the change of pairwise distance for node embeddings. 
First we introduce a concrete definition of the intra- and inter-class pairwise distances, then provide the experimental setups, and finally illustrate the results.

\noindent\textbf{Definition of pairwise distance.}
To verify whether over-smoothing exists in GCNs, we define the intra-class pairwise distance of $\mathbf{H}^{(\ell)}$ as the average pairwise Euclidean distance of two-node embeddings if they have the same ground truth label. 
Similarly, we define the inter-class pairwise distance of $\mathbf{H}^{(\ell)}$ as the average pairwise distance of two node embeddings if they have different ground truth labels. 
\begin{equation*}
    \begin{aligned}
    \text{intra}(\mathbf{H}^{(\ell)}) &\triangleq \frac{\sum_{i,j=1}^N \mathbf{1}{\{y_i = y_j\}} \cdot \| \bm{h}_i^{(\ell)} - \bm{h}_j^{(\ell)} \|_2^2}{ \|\mathbf{H}^{(\ell)}\|_F \cdot \sum_{i,j=1}^N \mathbf{1}{\{y_i = y_j\}}}, \\
    \text{inter}(\mathbf{H}^{(\ell)}) &\triangleq \frac{\sum_{i,j=1}^N \mathbf{1}{\{y_i \neq y_j\}} \cdot \| \bm{h}_i^{(\ell)} - \bm{h}_j^{(\ell)} \|_2^2}{ \|\mathbf{H}^{(\ell)}\|_F \cdot \sum_{i,j=1}^N \mathbf{1}{\{y_i \neq y_j\}}}.
    \end{aligned}
\end{equation*}
Both intra- and inter-class distances are normalized using the Frobenius norm to eliminate the difference caused by the scale of the node embedding matrix. 
By the definition of over-smoothing in~\cite{li2018deeper}, the distance between intra- and inter-class pairwise distance should be decreasing sharply as the number of GCN layers increases.

\noindent\textbf{Setup.}
In Figure~\ref{fig:cora_pairwise_dist_gcn_sgc_appnp} and Figure~\ref{fig:citeseer_pairwise_dist_gcn_sgc_appnp}, we plot the training error in a $10$-layer GCN~\cite{kipf2016semi}, SGC~\cite{wu2019simplifying}, and APPNP~\cite{klicpera2018predict} models until convergence, and choose the model with the best validation score for pairwise distance computation.
The pairwise distance at the $\ell$th layer is computed by the node embedding matrix $\mathbf{H}^{(\ell)}$ generated by the selected model.
We repeat each experiment $10$ times and plot the mean and standard deviation of the results. 

\noindent\textbf{Results.}
As shown in Figure~\ref{fig:cora_pairwise_dist_gcn_sgc_appnp} and Figure~\ref{fig:citeseer_pairwise_dist_gcn_sgc_appnp}, when ignoring the weight matrices, i.e., sub-figures in box (b), the intra- and inter-class pairwise distance are decreasing but the gap between intra- and inter-class pairwise distance is increasing as the model goes deeper. 
That is, the difference between intra- and inter-class nodes' embeddings is increasing and becoming more discriminative. 
On the other hand, when considering the weight matrices, i.e., sub-figures in box (a), the difference between intra- and inter-class node embeddings is large and increasing as the model goes deeper.
In other words, weight matrices learn to make node embeddings discriminative and generate expressive node embeddings during training.

\begin{figure}[h]
    \centering
    \includegraphics[width=1.0\textwidth]{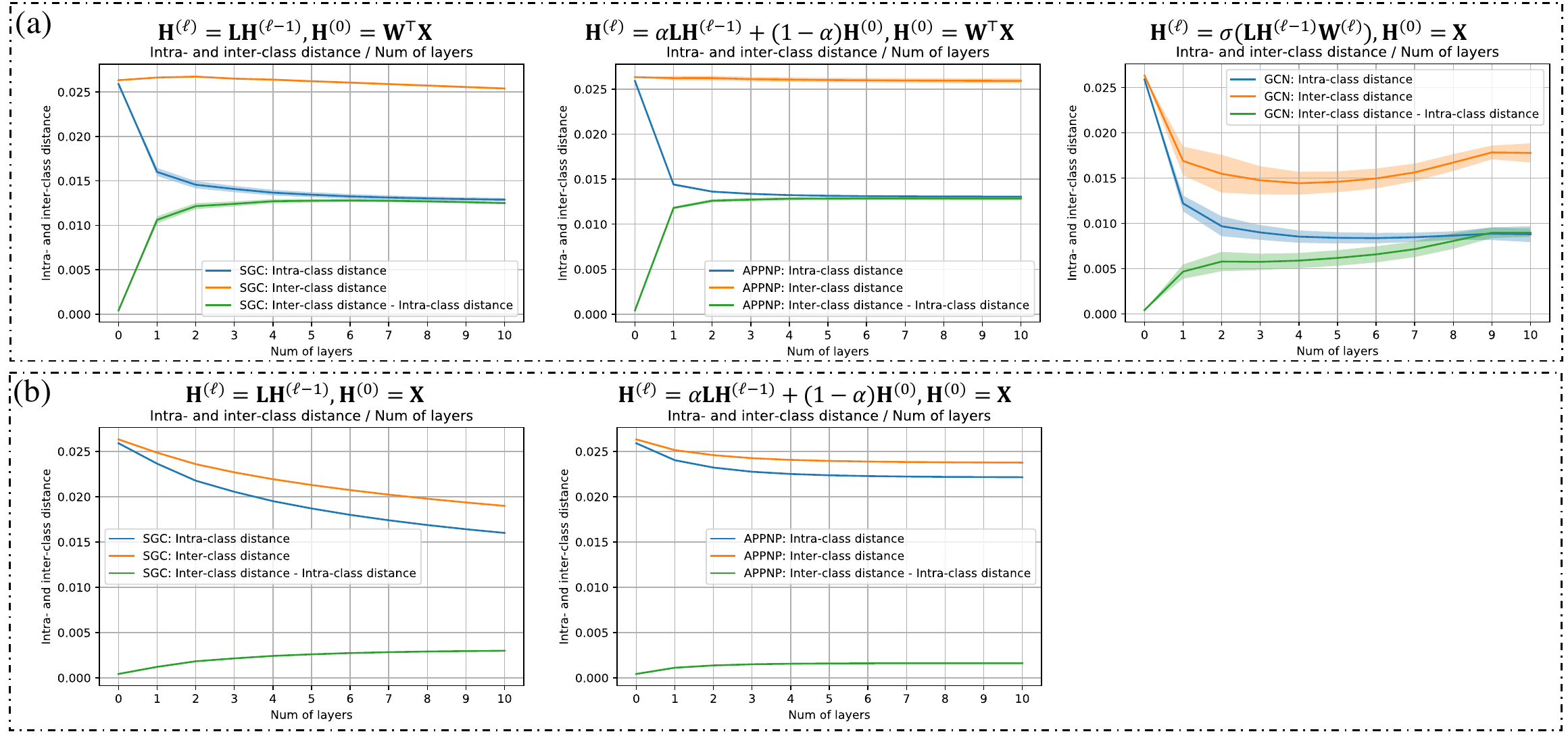}
    \caption{Comparison of the pairwise distance for intra- and inter-class node embeddings on \emph{Cora} dataset for different models by increasing the number of layers in the proposed architecture.}
    \label{fig:cora_pairwise_dist_gcn_sgc_appnp}
\end{figure}

\begin{figure}[h]
    \centering
    \includegraphics[width=1.0\textwidth]{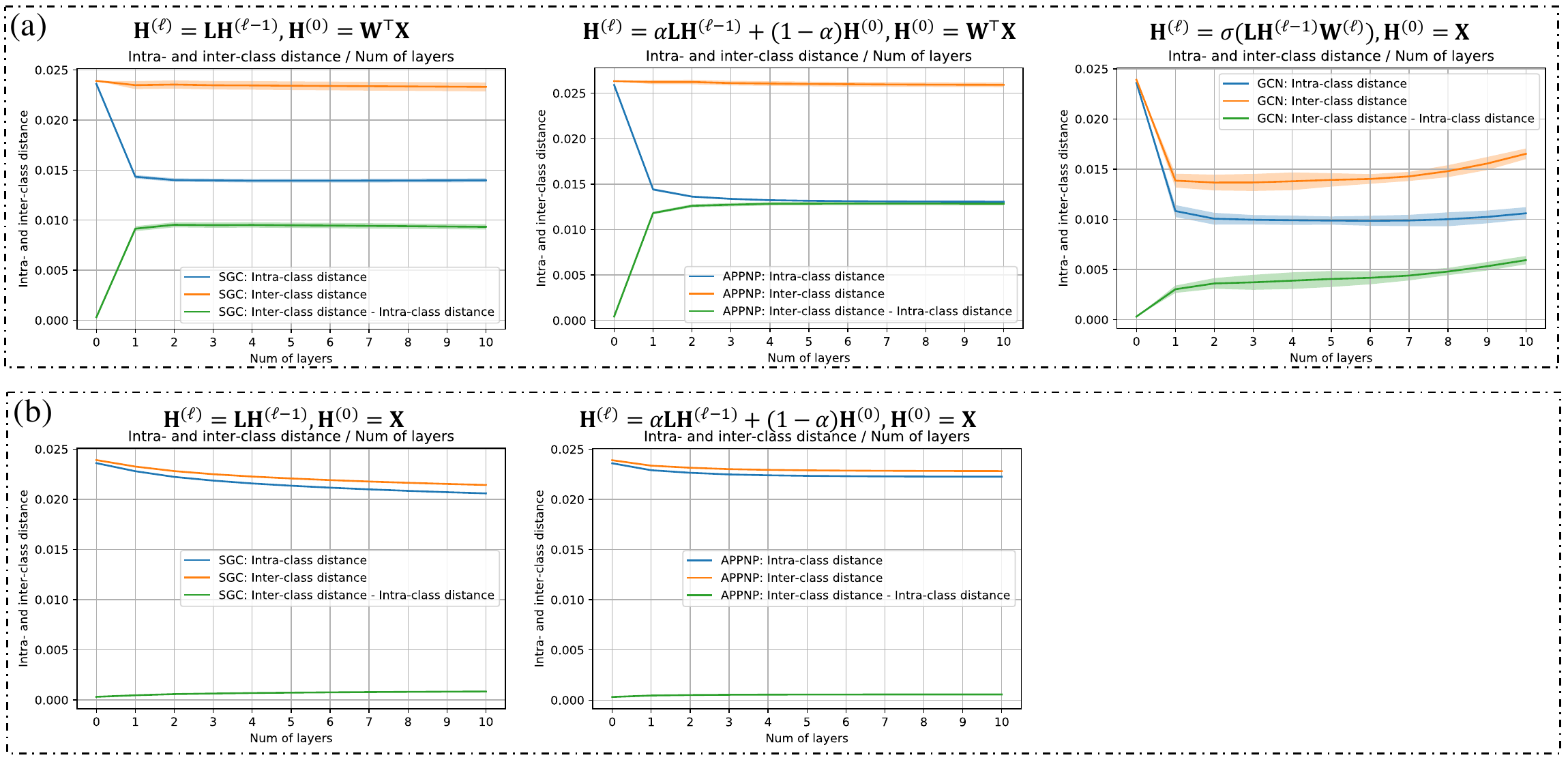}
    \caption{Comparison of the pairwise distance for intra- and inter-class node embeddings on \emph{Citeseer} dataset for different models by increasing the number of layers in the proposed architecture.}
    \label{fig:citeseer_pairwise_dist_gcn_sgc_appnp}
\end{figure}

\section{A summary of theoretical results on over-smoothing}\label{supp:weakness_oversmoothing}


In this section, we survey the existing theories on understanding over-smoothing in GCNs from~\cite{huang2020tackling,oono2019graph,cai2020note} and illustrate why the underlying assumptions may not hold in practice. 
Finally, we  discuss conditions where over-smoothing and over-fitting might  happen simultaneously.

Before processing, let first recall the notation we defined in Section~\ref{section:preliminaries}.
Recall that  $\mathbf{A} \in \mathbb{R}^{N\times N}$ denotes  the adjacency matrix of the self-connected graph,  i.e.,  $A_{i,j} = 1$ if edge $(i,j)\in\mathcal{E}$ and $A_{i,j} = 0$ otherwise,  $\mathbf{D} \in \mathbb{R}^{N\times N}$ denotes the corresponding degree matrix, i.e., $D_{i,i} = \deg(i)$ and $D_{i,j}=0$ if $i\neq j$, and the symmetric Laplacian matrix is computed as $\mathbf{L} = \mathbf{D}^{-1/2} \mathbf{A} \mathbf{D}^{-1/2}$.

\subsection{Expressive power based analysis~\cite{oono2020optimization,huang2020tackling}}

\begin{proposition} [Proposition~1 in \cite{oono2019graph}, Theorem~1 of~\cite{huang2020tackling}]
    Let $\lambda_1 \geq \ldots \geq \lambda_N$ denote the eigenvalues of the Laplacian matrix $\mathbf{L}$ in descending order, and let $\mathbf{e}_i$ denote the eigenvector associated with the eigenvalue $\lambda_i$. 
    Suppose graph has $M\leq N$ connected components, then we have $\lambda_1 = \ldots = \lambda_M = 1$ and $1 > \lambda_{M+1} > \ldots > \lambda_N$. 
    Let $\mathbf{E} = \{ \mathbf{e}_i \}_{i=1}^M \in \mathbb{R}^{M\times N}$ denote the stack of eigenvectors that associated to eigenvalues $\lambda_1,\ldots,\lambda_M$. Then for all $i\in[M]$, the eigenvector $\mathbf{e}_i\in \mathbb{R}^N$ is defined as
    \begin{equation}
        \mathbf{e}_i = \{e_i(j)\}_{j=1}^N,~
        e_i(j) \propto \begin{cases}
        \sqrt{\deg(j)} & \text{ if $j$th node is in the $i$th connected component} \\ 
        0 & \text{ otherwise } 
        \end{cases}.
    \end{equation}
\end{proposition}
\begin{proof}
The proof can be found in Section~B of~\cite{oono2019graph}.
\end{proof}

\cite{oono2019graph} proposes to measure the expressive power of node embeddings using its distance to a subspace $\mathcal{M}$ that only has node degree information, i.e., regardless of the node feature information.

\begin{definition} [Subspace and distance to subspace] 
We define subspace $\mathcal{M}$ as
\begin{equation}
    \mathcal{M} := \{ \mathbf{E}^\top \mathbf{R}\in \mathbb{R}^{N\times d}~|~\mathbf{R} \in \mathbb{R}^{M \times d} \},
\end{equation}
where $\mathbf{E} = \{ \mathbf{e}_i \}_{i=1}^M \in \mathbb{R}^{M\times N}$ is the orthogonal basis of Laplacian matrix $\mathbf{L}$, and $\mathbf{R}$ is any random matrix. 
The distance of a node embedding matrix $\mathbf{H}^{(\ell)}\in\mathbb{R}^{N\times d}$ to the subspace can be computed as $d_{\mathcal{M}}(\mathbf{H}^{(\ell)}) = \inf_{\mathbf{Y} \in \mathcal{M}} \| \mathbf{H}^{(\ell)} - \mathbf{Y}\|_{\mathrm{F}}$, where $\|\cdot\|_{\mathrm{F}}$ denotes the Frobenius norm of the matrix.
\end{definition}

Subspace $\mathcal{M}$ is a $d$-dimensional subspace defined by orthogonal vectors $\{ \mathbf{e}_i \}_{i=1}^M$ that only captures node degree information.
Intuitively, a smaller distance $d_{\mathcal{M}}(\mathbf{H}^{(\ell)})$ means that $\mathbf{H}^{(\ell)}$ is losing its feature information but only carries the degree information. 

In the following, we summarize the theoretical results of over-smoothing in~\cite{huang2020tackling} and extend their results to SGC and GCNII, which will be used in our discussion below.   
\begin{theorem} [Theorem 2 of~\cite{huang2020tackling}] \label{thm:oversmoothing_thm_1}
Let $\lambda=\max_{i\in[M+1, N]} |\lambda_i|$ denote the largest absolute eigenvalue of Laplacian matrix $\mathbf{L}$ which is bounded by $1$, let $s$ denote the largest singular-value of weight parameter $\mathbf{W}^{(\ell)},~\ell\in[L]$, then we have
\begin{equation}
    d_{\mathcal{M}}(\mathbf{H}^{(\ell)}) - \varepsilon \leq \gamma \Big( d_{\mathcal{M}}(\mathbf{H}^{(\ell-1)}) - \varepsilon \Big),
\end{equation}
where $\gamma$ and $\varepsilon$ are functions of $\lambda$ and $s$ that depend on the model structure.
\end{theorem}


\noindent\textbf{GCN.} For vanilla \emph{GCN} model with the graph convolutional operation defined as
\begin{equation}
    \mathbf{H}^{(\ell)} = \sigma(\mathbf{L} \mathbf{H}^{(\ell-1)} \mathbf{W}^{(\ell)}),~
    \mathbf{H}^{(0)} = \mathbf{X},
\end{equation}
we have $\gamma_\mathsf{GCN} = \lambda s $ and $\epsilon_\mathsf{GCN}=0$. 

Under the assumption that $\gamma_\mathsf{GCN} \leq 1$, i.e., the graph is densely connected and the largest singular value of weight matrices is small, we have $d_{\mathcal{M}}(\mathbf{H}^{(\ell+1)}) \leq \gamma_\mathsf{GCN} \cdot d_{\mathcal{M}}(\mathbf{H}^{(\ell)}) $. 
In other words, the vanilla \emph{GCN} model is losing expressive power as the number of layers increases.

However, if we suppose there exist a trainable bias parameter $\mathbf{B}^{(\ell)}=\{\mathbf{b}^{(\ell)} \}_{i=1}^N \in \mathbb{R}^{N\times d_\ell}$ in each graph convolutional layer
\begin{equation}
    \mathbf{H}^{(\ell)} = \sigma(\mathbf{L} \mathbf{H}^{(\ell-1)} \mathbf{W}^{(\ell)} + \mathbf{B}^{(\ell)}),~
    \mathbf{H}^{(0)} = \mathbf{X},
\end{equation}
we have $\gamma_\mathsf{GCN} = \lambda s$ and $\epsilon_\mathsf{GCN}=d_{\mathcal{M}}(\mathbf{B}^{(\ell)})$. 
When $d_{\mathcal{M}}(\mathbf{B}^{(\ell)})$ is considerably large, there is no guarantee that \emph{GCN-bias} suffers from losing expressive power issue. 

\noindent\textbf{ResGCN.} For GCN with residual connection, the graph convolution operation is defined as
\begin{equation}
    \mathbf{H}^{(\ell)} = \sigma(\mathbf{L} \mathbf{H}^{(\ell)} \mathbf{W}^{(\ell)}) + \mathbf{H}^{(\ell-1)},~
    \mathbf{H}^{(0)} = \mathbf{X} \mathbf{W}^{(0)},
\end{equation}
where we have $\gamma_\mathsf{ResGCN} = 1 + \lambda s$ and $\epsilon_\mathsf{ResGCN}=0$. Since $\gamma_\mathsf{ResGCN} \geq 1$, there is no guarantee that \emph{ResGCN} suffers from losing expressive power. 

\noindent\textbf{APPNP.} For \emph{APPNP} with graph convolution operation defined as
\begin{equation}
    \mathbf{H}^{(\ell)} = \alpha \mathbf{L} \mathbf{H}^{(\ell)} + (1-\alpha) \mathbf{H}^{(0)},~
    \mathbf{H}^{(0)} = \mathbf{X} \mathbf{W}^{(0)},
\end{equation}
we have $\gamma_\mathsf{APPNP} = \alpha \lambda$ and $\epsilon_\mathsf{APPNP}=\frac{(1-\alpha) d_{\mathcal{M}}(\mathbf{H}^{(0)})}{1-\gamma_\mathsf{APPNP}}$. 
Although $\gamma_\mathsf{APPNP} < 1$, because $\epsilon_\mathsf{APPNP}$ can be large, there is no guarantee that \emph{APPNP} suffers from losing expressive power. 

\noindent\textbf{GCNII.}
For \emph{GCNII} with graph convolution operation defined as
\begin{equation}
    \mathbf{H}^{(\ell)} = \sigma\Big( \big( \alpha\mathbf{L} \mathbf{H}^{(\ell)} + (1-\alpha) \mathbf{H}^{(0)} \big) \big( \beta \mathbf{W}^{(\ell)} + (1-\beta) \mathbf{I}_N \big)\Big),~
    \mathbf{H}^{(0)} = \mathbf{X} \mathbf{W}^{(0)},
\end{equation}
we have $\gamma_\mathsf{GCNII} = \Big(1-(1-\beta)(1- s) \Big) \alpha \lambda$, and $\epsilon_\mathsf{GCNII} =\frac{(1-\alpha) d_\mathcal{M}(\mathbf{H}^{(0)})}{1-\gamma_\mathsf{GCNII}}$. Although we have $\gamma_\mathsf{GCNII} < 1$, because $\epsilon_\mathsf{GCNII}$ can be considerably large, there is no guarantee that \emph{GCNII} suffers from losing expressive power.

\noindent\textbf{SGC.}
The result can be also extended to the linear model \emph{SGC}~\cite{wu2019simplifying}, where the graph convolution is defined as
\begin{equation}
    \mathbf{H}^{(\ell)} = \mathbf{L} \mathbf{H}^{(\ell-1)},~
    \mathbf{H}^{(0)} = \mathbf{X} \mathbf{W}^{(0)},
\end{equation}
we have $\gamma_\mathsf{SGC} = \lambda$, and $\epsilon_\mathsf{SGC} =0$, which guarantees losing expressive power as the number of layers increases. 

\noindent\textbf{Discussion on the result of Theorem~\ref{thm:oversmoothing_thm_1}.}
In summary, the linear model \emph{SGC} always suffers from losing expressive power without the assumption on the trainable parameters.
Under the assumption that the multiplication of the \emph{largest singular-value of trainable parameter} and \emph{the largest absolute eigenvalue of Laplacian matrix smaller than $1$}, i.e., $\lambda s \leq 1$, we can only guarantee that \emph{GCN} suffers from losing expressive power issue, but cannot have the same guarantee on \emph{GCN-bias}, \emph{ResGCN}, \emph{APPNP}, and \emph{GCNII}.
However, as we show in Figure~\ref{fig:d_m_and_weight_norm}, the assumption is not going to hold in practice, and the distances are not decreasing for most cases.


\begin{remark} \label{remark:connectivity_eigenvalue}
\cite{oono2019graph} conducts experiments on Erd\H{o}s-R\'enyi graph to show that when the graph is sufficiently dense and large, vanilla GCN suffers from expressive lower loss.
Erd\H{o}s-R\'enyi graph is constructed by connecting nodes randomly. Each edge is included in the graph with probability $p$ independent from every other edge.
To guarantee a small $\lambda$, a dense graph with larger $p$ is required. 
For example, in the Section 6.2 of \cite{oono2019graph}, they choose $p=0.5$ (one node is connected to $50\%$ of other nodes) such that $\lambda=0.063$ and choose $p=0.1$ (one node is connected to $10\%$ of other nodes) such that $\lambda=0.195$. However, real world datasets are sparse and have a $\lambda$ that is closer to $1$. For example, Cora has $\lambda= 0.9964$, Citeseer has $\lambda= 0.9987$, and Pubmed has $\lambda= 0.9905$.
\end{remark}

In Figure~\ref{fig:test_loss_dm_before_after_cora} and Figure~\ref{fig:test_loss_dm_before_after_citeseer}, we compare the testing set F1-score, expressive power metric $d_{\mathcal{M}}(\mathbf{H}^{(\ell)})$  before and after training.  $d_{\mathcal{M}}(\mathbf{H}^{(\ell)})$ is computed on \emph{the final output of a $\ell$-layer GCN model}. 
The ``after training results'' are computed on the model with the best testing score.
We repeat each experiment $10$ times and report the average values.
As shown in Figure~\ref{fig:test_loss_dm_before_after_cora} and Figure~\ref{fig:test_loss_dm_before_after_citeseer}, the expressive power metric $d_{\mathcal{M}}(\mathbf{H}^{(\ell)})$ of the \emph{untrained} vanilla GCN is decreasing as the number of layers increases. 
However, the expressive power metric $d_{\mathcal{M}}(\mathbf{H}^{(\ell)})$ of the \emph{trained} vanilla GCN increases with more layers. 
Besides, we observe that there is no obvious connection between the testing F1-score to $d_{\mathcal{M}}(\mathbf{H}^{(\ell)})$ for other GCN variant. 
For example, the testing F1-score of APPNP is increasing while its $d_{\mathcal{M}}(\mathbf{H}^{(\ell)})$ is not changing much.

\begin{figure}[h]
    \centering
    \includegraphics[width=1.0\textwidth]{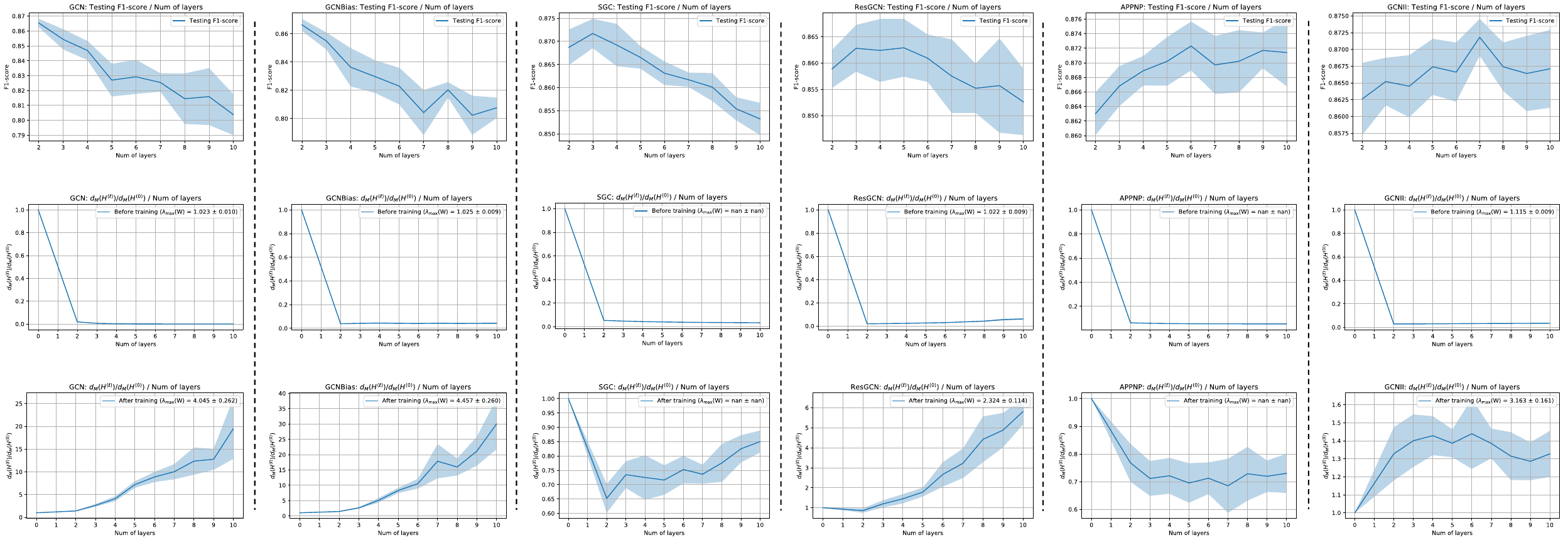}
    \caption{Comparison of F1-score, expressive power metric $d_{\mathcal{M}}(\mathbf{H}^{(\ell)})$ of \emph{GCN}, \emph{GCN-bias}, \emph{SGC}, \emph{ResGCN}, \emph{APPNP}, and \emph{GCNII} before and after training on \textit{Cora} dataset. The average ``largest sigular value'' of weight matrices in graph convolutional layers are reported, where ``nan'' stands for no weight matrices inside graph convolution layers. We only consider GCN model with depth from $2$-layers to $10$-layers. (Better viewed in PDF version)}
    \label{fig:test_loss_dm_before_after_cora}
\end{figure}

\begin{figure}[h]
    \centering
    \includegraphics[width=1.0\textwidth]{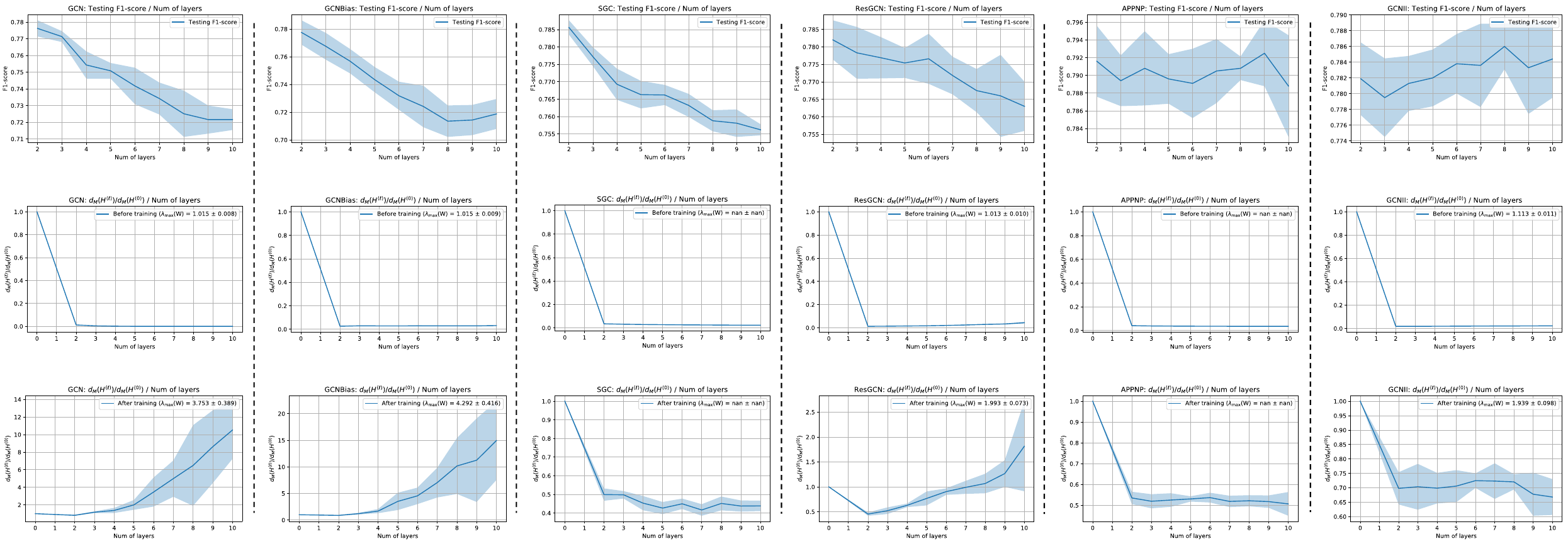}
    \caption{Comparison of F1-score, expressive power metric $d_{\mathcal{M}}(\mathbf{H}^{(\ell)})$ of \emph{GCN}, \emph{GCN-bias}, \emph{SGC}, \emph{ResGCN}, \emph{APPNP}, and \emph{GCNII} before and after training on \textit{Citeseer} dataset. The average ``largest singular value'' of weight matrices in graph convolutional layers are reported, where ``nan'' stands for no weight matrices inside graph convolution layers. We only consider GCN model with depth from $2$-layers to $10$-layers. (Better viewed in PDF version)}
    \label{fig:test_loss_dm_before_after_citeseer}
\end{figure}


In Figure~\ref{fig:test_loss_dm_before_before_cora_one_model} and Figure~\ref{fig:test_loss_dm_before_before_citeseer_one_model}, we compare the testing set F1-score, expressive power metric $d_{\mathcal{M}}(\mathbf{H}^{(\ell)})$  before and after training. 
$d_{\mathcal{M}}(\mathbf{H}^{(\ell)})$ is computed on \emph{the $\ell$th layer intermediate node embeddings of a $10$-layer GCN model}. 
The ``after training results'' are computed on the model with the best testing score.
We repeat experiment $10$ times and report the average values.
As shown in Figure~\ref{fig:test_loss_dm_before_before_cora_one_model} and Figure~\ref{fig:test_loss_dm_before_before_citeseer_one_model}, the expressive power metric $d_{\mathcal{M}}(\mathbf{H}^{(\ell)})$ of the \emph{untrained} vanilla GCN is decreasing as the number of layers increases. However, the expressive power metric $d_{\mathcal{M}}(\mathbf{H}^{(\ell)})$ of the \emph{trained} vanilla GCN increases as the number of layers increases. 

\begin{figure}[h]
    \centering
    \includegraphics[width=1.0\textwidth]{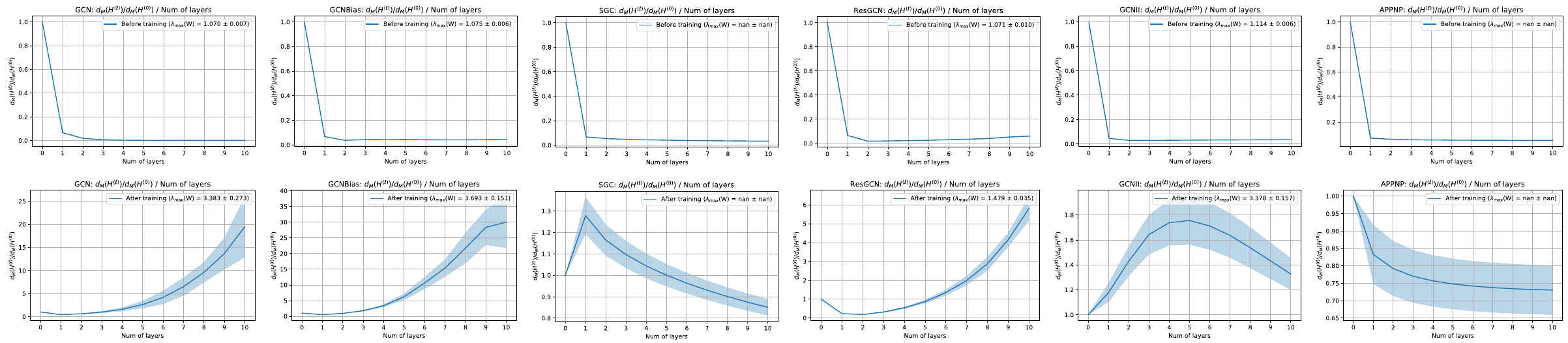}
    \caption{Comparison of F1-score, expressive power metric $d_{\mathcal{M}}(\mathbf{H}^{(\ell)})$ of $10$-layer \emph{GCN}, \emph{GCN-bias}, \emph{SGC}, \emph{ResGCN}, \emph{APPNP}, and \emph{GCNII} before and after training on \textit{Cora} dataset. The average ``largest singular value'' of weight matrices in graph convolutional layers are reported, where ``nan'' stands for no weight matrices inside graph convolution layers. (Better viewed in PDF version)}
    \label{fig:test_loss_dm_before_before_cora_one_model}
\end{figure}

\begin{figure}[h]
    \centering
    \includegraphics[width=1.0\textwidth]{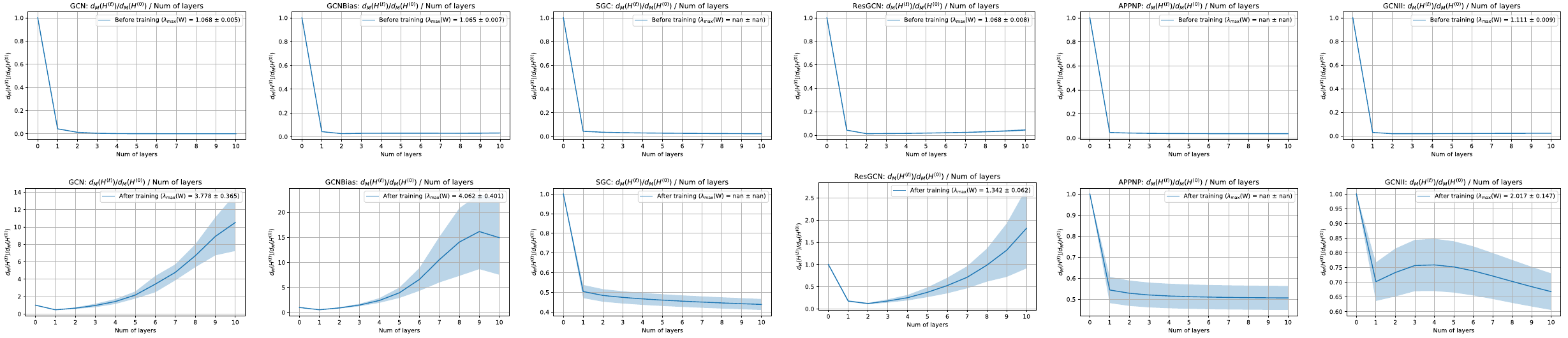}
    \caption{Comparison of F1-score, expressive power metric $d_{\mathcal{M}}(\mathbf{H}^{(\ell)})$ of $10$-layer \emph{GCN}, \emph{GCN-bias}, \emph{SGC}, \emph{ResGCN}, \emph{APPNP}, and \emph{GCNII} before and after training on \textit{Citeseer} dataset. The average ``largest singular value'' of weight matrices in graph convolutional layers are reported, where ``nan'' stands for no weight matrices inside graph convolution layers. (Better viewed in PDF version)}
    \label{fig:test_loss_dm_before_before_citeseer_one_model}
\end{figure}

\subsection{Dirichlet energy based analysis~\cite{cai2020note}}
\begin{definition} [Dirichlet energy]
The Dirichlet energy of node embedding matrix $\mathbf{H}^{(\ell)}$ is defined as
\begin{equation}
    \texttt{DE}(\mathbf{H}^{(\ell)}) = \frac{1}{2} \sum_{i,j=1}^N L_{i,j} \Big\| \frac{\mathbf{h}_i^{(\ell)}}{\sqrt{\deg(i)}} - \frac{\mathbf{h}_j^{(\ell)}}{\sqrt{\deg(j)}} \Big\|^2_2.
\end{equation}
\end{definition}
Intuitively, the Dirichlet energy measures the smoothness of node embeddings. Based on the definition of Dirichlet energy, they obtain a similar result as shown in~\cite{oono2020optimization} that
\begin{equation}
    \texttt{DE}(\mathbf{H}^{(\ell+1)}) \leq \lambda s \cdot \texttt{DE}(\mathbf{H}^{(\ell)}),
\end{equation}
where $\lambda$ is \textit{largest absolute eigenvalue of Laplacian matrix $\mathbf{L}$ that is less than $1$}, and $s$ is \textit{the largest singular-value of trainable parameter $\mathbf{W}^{(\ell)}$}. 

\begin{remark}
The Dirichlet energy used in~\cite{cai2020note} is closely related to the pairwise distance as shown in Figure~\ref{fig:pairwise_dist_num_layers}. 
In Figure~\ref{fig:pairwise_dist_num_layers}, we measure the ``smoothness'' or ``expressive power'' of node embeddings using a normalized Dirichlet energy function, i.e., $\frac{\texttt{DE}(\mathbf{H}^{(\ell)})}{\|\mathbf{H}^{(\ell)}\|_F^2}$ for inner- and cross-class nodes respectively, which refers to the real smoothness metric of graph signal as pointed out in the footnote 1 of ~\cite{cai2020note}. 
\end{remark}
Intuitively, under the assumption that the largest singular-value of weight matrices less than $1$, as the number of layers increase, i.e., $\ell\rightarrow \infty$, the Dirichlet energy of node embeddings $\texttt{DE}(\mathbf{H}^{(\ell)})$ is decreasing.
However, recall that this assumption is not likely to hold in practice because the real-world graphs are usually sparse and the largest singular value of weight matrices is usually larger than $1$.

Besides, generalizing the result to other GCN variants (e.g., ResGCN, GCNII) is non-trivial. 
For example in ResGCN, we have
\begin{equation}
    \begin{aligned}
    \texttt{DE}(\mathbf{H}^{(\ell)}) &= \texttt{DE}(\sigma(\mathbf{L} \mathbf{H}^{(\ell-1)} \mathbf{W}^{(\ell)}) + \mathbf{H}^{(\ell-1)}) \\
    &\leq 2\texttt{DE}(\sigma(\mathbf{L} \mathbf{H}^{(\ell-1)} \mathbf{W}^{(\ell)})) + 2\texttt{DE}(\mathbf{H}^{(\ell-1)}) \\
    &\leq 2(\lambda s + 1) \texttt{DE}(\mathbf{H}^{(\ell-1)}).
    \end{aligned}
\end{equation}
The coefficient $2$ on the right hand side of the inequality makes the result less meaningful.

\subsection{A condition to have over-smoothing and over-fitting  happen simultaneously}
As it has alluded to before, in this paper, we argue that the performance degradation issue is mainly due to over-fitting, and over-smoothing is not likely to happen in practice.
One might doubt whether over-smoothing can happen simultaneously with over-fitting.
In the following, we show that over-smoothing and over-fitting might happen at the same time if a model can distinguish all nodes by only using graph structure information (e.g., node degree) and ignoring node feature information. 
However, such a condition is not likely in the practice.

For simplicity, let us assume that the graph contains only a single connected component.
Let first consider the over-smoothing notation as defined in \cite{oono2019graph}. 
We can write the node embedding of node $i$ at the $\ell$th layer as $\mathbf{h}_i^{(\ell)} = \mathbf{e} \mathbf{R}_{i,\ell} + \bm{\varepsilon}_{i, \ell} \in \mathbb{R}^{d}$, where $\mathbf{e}$ is the eigen vector associated with the largest eigenvalue, $\mathbf{R}_{i,\ell} \in \mathbb{R}^{N \times d}$ is the random projection matrix, and $\bm{\varepsilon}_{i, \ell} \in \mathbb{R}^d$ is a $d$-dimensional random vector such that $d_\mathcal{M}(\mathbf{h}_i^{(\ell)}) = \| \bm{\varepsilon}_{i, \ell} \|_2$ corresponds to its distance to subspace $\mathcal{M}$. 
According to the notation in \cite{oono2019graph}, we have $\| \bm{\varepsilon}_{i, \ell} \|_2 \rightarrow 0$ as $\ell \rightarrow +\infty$ for any $i \in \mathcal{V}$.
The minimum training loss can be achieved by choosing the final layer weight vector $\mathbf{w} \in \mathbb{R}^d$ such that $\sum_{i\in\mathcal{V}} \text{Loss}\big(\mathbf{w}^\top (\mathbf{e} \mathbf{R}_{i,\ell} + \bm{\varepsilon}_{i, \ell}),  y_i \big) $ is small.
A small training error is achieved when $\mathbf{e}\mathbf{R}_{i,\ell}$ is discriminative. In other words, a model is under both over-smoothing and over-fitting conditions only if it can make predictions based on the graph structure, without leveraging node feature information.
On the other hand, considering the over-smoothing notation as defined in \cite{cai2020note}, we know that over-smoothing happens if $\frac{\mathbf{h}_i^{(\ell)}}{\sqrt{\deg(i)}} \approx \frac{\mathbf{h}_j^{(\ell)}}{\sqrt{\deg(j)}}$ for any $i,j\in\mathcal{V}$, i.e., a classifier can distinguish any nodes based on its node degree (only based on graph information without leveraging its feature information).

\clearpage

\section{Proof of Theorem~\ref{theorem:expressive_power_exp}} \label{supp:omitted_proof_section_3}

Given a $L$-layer computation tree $\mathcal{T}^L$, let $v_i^\ell$ denote the $i$th node in the $\ell$th layer of a $L$-layer computation tree $\mathcal{T}^L$. Let suppose each node has at least $d$ neighbors and each node has binary feature. 

Let $P(v_i^\ell)$ denote the parent of node $v_i^\ell$ and $C(v_i^\ell)$ denote the set of children of node $v_i^\ell$ with $|C(v_i^\ell)| \geq d$.
By the definition of computation tree, we know that $P(v_i^\ell) \in C(v_i^\ell)$ for any $\ell>1$.
As a result, we know that $C(v_i^\ell)$ has at least $(d-1)$ different choices since one of the node in $C(v_i^\ell)$ must be the same as $P(v_i^\ell)$.

Therefore, we know that a $L$-layer computation tree $\mathcal{T}^L$ has at least $2(d-1)^{L-1}$ different choices, where the constant $2$ is because each root node has at least two different choices.

\section{Proof of Theorem~\ref{thm:deep_wide_local_global_minima2}} \label{supp:linear_convergence_rate}
In this section, we study the convergence of $L$-layer deep GCN model, and show that a deeper model requires more iterations $T$ to achieve $\epsilon$ training error. Before preceding, let first introduce the notations used for the convergence analysis of deep GCN. Let $\mathcal{G}(\mathcal{V}, \mathcal{E})$ as the graph of $N=|\mathcal{V}|$ nodes and each node in the graph is associated with a node feature and target label pair $(\mathbf{x}_i, \mathbf{y}_i)$, where $\mathbf{x}_i \in \mathbb{R}^{d_0}$ and $\mathbf{y}_i \in \mathbb{R}^{d_{L+1}}$.
Let $\mathbf{X} = \{\mathbf{x}_i \}_{i=1}^N \in \mathbb{R}^{N \times d_0}$ and  $\mathbf{Y} = \{\mathbf{y}_i \}_{i=1}^N \in \mathbb{R}^{N \times d_{L+1}}$ be the stack of node vector and target vector of the training data set of size $N$. 
Given a set of $L$ matrices $\bm{\theta} = \{ \mathbf{W}^{(\ell)} \}_{\ell=1}^L$,  we consider the following training objective  with squared loss function to learn the parameters of the model:
\begin{equation}\label{eq:deeper_better_obj}
    \underset{\bm{\theta}}{\text{minimize}}~\mathcal{L}(\bm{\theta}) = \frac{1}{2} \| \widehat{\mathbf{Y}}(\bm{\theta}) - \mathbf{Y}\|_\mathrm{F}^2,~ \widehat{\mathbf{Y}}(\bm{\theta}) = \mathbf{H}^{(L)}\mathbf{W}^{(L+1)},~
    \mathbf{H}^{(\ell)} = \sigma(\mathbf{L} \mathbf{H}^{(\ell-1)} \mathbf{W}^{(\ell)}),
\end{equation}
where $\sigma(\cdot)$ is ReLU activation.

\subsection{Useful lemmas}
The following lemma analyzes the Lipschitz continuity of deep GCNs, which characterizes the change of inner layer representations by slight change of the weight parameters.

\begin{lemma}\label{lemma:convergence_width_lip}

Let $\bm{\theta}_1 = \{ \mathbf{W}^{(\ell)}_1 \}_{\ell \in [L+1]}$ and $\bm{\theta}_2 = \{ \mathbf{W}^{(\ell)}_2 \}_{\ell \in [L+1]}$ be two set of parameters.
Let $\lambda_\ell \geq  \max\{ \| \mathbf{W}_1^{(\ell)}\|_2, \| \mathbf{W}_2^{(\ell)}\|_2 \}$, $s \geq \| \mathbf{L} \|_2$,
and $\lambda_{i\rightarrow j} = \prod_{\ell=i}^j \lambda_\ell$,  we have
\begin{equation}
    \begin{aligned}
    \| \mathbf{H}_1^{(\ell)} - \mathbf{H}_2^{(\ell)} \|_\mathrm{F}  
    &\leq s^\ell \lambda_{1\rightarrow \ell} \| \mathbf{X} \|_\mathrm{F} \sum_{j=1}^\ell \lambda_j^{-1} \| \mathbf{W}^{(j)}_1 - \mathbf{W}^{(j)}_2 \|_2.
    \end{aligned}
\end{equation}

\end{lemma}
\begin{proof} [Proof of Lemma~\ref{lemma:convergence_width_lip}]
Our proof relies on the following inequality
\begin{equation}
    \forall \mathbf{A} \in \mathbb{R}^{a\times b}, \mathbf{B}\in \mathbb{R}^{b\times c},~ \| \mathbf{A} \mathbf{B} \|_\mathrm{F} \leq \| \mathbf{A}\|_2 \| \mathbf{B}\|_\mathrm{F}~\text{and}~\| \mathbf{A} \mathbf{B} \|_\mathrm{F} \leq \| \mathbf{A}\|_\mathrm{F} \| \mathbf{B}\|_2.
\end{equation}

By definition of $\mathbf{H}^{(\ell)}$, we have for $\ell\in[L]$
\begin{equation} \label{eq:convergence_width_lip_eq1}
    \begin{aligned}
    \| \mathbf{H}^{(\ell)} \|_\mathrm{F}
    &= \| \sigma(\mathbf{L} \mathbf{H}^{(\ell-1)} \mathbf{W}^{(\ell)}) \|_\mathrm{F} \\
    &\leq \| \mathbf{L} \mathbf{H}^{(\ell-1)} \mathbf{W}^{(\ell)} \|_\mathrm{F} \\
    &\leq  \| \mathbf{L}\|_2 \| \mathbf{H}^{(\ell-1)} \mathbf{W}^{(\ell)} \|_\mathrm{F} \\
    & \leq  \| \mathbf{L}\|_2 \| \mathbf{W}^{(\ell)} \|_2 \| \mathbf{H}^{(\ell-1)} \|_\mathrm{F} \\
    &\leq \| \mathbf{X} \|_\mathrm{F} \prod_{j=1}^\ell \Big( \| \mathbf{L}\|_2 \| \mathbf{W}^{(\ell)} \|_2 \Big).
    \end{aligned}
\end{equation}

Let $\bm{\theta}_1 = \{ \mathbf{W}^{(\ell)}_1 \}_{\ell \in [L+1]}$ and $\bm{\theta}_2 = \{ \mathbf{W}^{(\ell)}_2 \}_{\ell \in [L+1]}$ be two set of parameters, we have
\begin{equation}
    \begin{aligned}
    \| \mathbf{H}_1^{(\ell)} - \mathbf{H}_2^{(\ell)} \|_\mathrm{F} 
    &= \| \sigma(\mathbf{L} \mathbf{H}^{(\ell-1)}_1 \mathbf{W}^{(\ell)}_1) - \sigma(\mathbf{L} \mathbf{H}^{(\ell-1)}_2 \mathbf{W}^{(\ell)}_2) \|_\mathrm{F} \\
    &\underset{(a)}{\leq} \| \mathbf{L} \mathbf{H}^{(\ell-1)}_1 \mathbf{W}^{(\ell)}_1 - \mathbf{L} \mathbf{H}^{(\ell-1)}_2 \mathbf{W}^{(\ell)}_2 \|_\mathrm{F} \\
    &\leq \| \mathbf{L} \|_2 \| \mathbf{H}^{(\ell-1)}_1 \mathbf{W}^{(\ell)}_1 - \mathbf{H}^{(\ell-1)}_2 \mathbf{W}^{(\ell)}_2 \|_\mathrm{F} \\
    &= \| \mathbf{L} \|_2 \left( \| \mathbf{H}^{(\ell-1)}_1 (\mathbf{W}^{(\ell)}_1 - \mathbf{W}^{(\ell)}_2) + ( \mathbf{H}^{(\ell-1)}_1 - \mathbf{H}^{(\ell-1)}_2) \mathbf{W}^{(\ell)}_2 \|_\mathrm{F} \right) \\
    &\leq \| \mathbf{L} \|_2 \left( \| \mathbf{H}^{(\ell-1)}_1 \|_\mathrm{F} \| \mathbf{W}^{(\ell)}_1 - \mathbf{W}^{(\ell)}_2 \|_2 + \| \mathbf{H}^{(\ell-1)}_1 - \mathbf{H}^{(\ell-1)}_2 \|_\mathrm{F} \| \mathbf{W}^{(\ell)}_2 \|_2 \right)\\
    &\underset{(b)}{\leq} \| \mathbf{L} \|_2 \| \mathbf{X} \|_\mathrm{F} \left[ \prod_{j=1}^{\ell-1} \Big( \| \mathbf{L}\|_2 \| \mathbf{W}^{(\ell)}_1 \|_2 \Big) \right] \| \mathbf{W}^{(\ell)}_1 - \mathbf{W}^{(\ell)}_2 \|_2 \\
    &\quad + \| \mathbf{L} \|_2  \| \mathbf{W}^{(\ell)}_2 \|_2 \| \mathbf{H}^{(\ell-1)}_1 - \mathbf{H}^{(\ell-1)}_2 \|_\mathrm{F},
    \end{aligned}
\end{equation}
where $(a)$ is due to $\sigma(\cdot)$ is $1$-Lipschitz continuous, $(b)$ is due to Eq.~\ref{eq:convergence_width_lip_eq1}.

Let $\lambda_\ell \geq  \max\{ \| \mathbf{W}_1^{(\ell)}\|_2, \| \mathbf{W}_2^{(\ell)}\|_2 \}$, $s \geq \| \mathbf{L} \|_2$, and
$\lambda_{i\rightarrow j} = \prod_{\ell=i}^j \lambda_\ell$, , we have
\begin{equation}
    \begin{aligned}
    \| \mathbf{H}_1^{(\ell)} - \mathbf{H}_2^{(\ell)} \|_\mathrm{F} 
    &\leq s^\ell \lambda_{1\rightarrow (\ell-1)} \| \mathbf{X} \|_\mathrm{F} \| \mathbf{W}^{(\ell)}_1 - \mathbf{W}^{(\ell)}_2 \|_2 + s \lambda_\ell \| \mathbf{H}_1^{(\ell-1)} - \mathbf{H}_2^{(\ell-1)} \|_\mathrm{F} \\
    &\leq s^\ell \lambda_{1\rightarrow \ell} \| \mathbf{X} \|_\mathrm{F} \left(\sum_{j=1}^\ell \lambda_j^{-1} \| \mathbf{W}^{(j)}_1 - \mathbf{W}^{(j)}_2 \|_2 \right).
    \end{aligned}
\end{equation}

\end{proof}

The following lemma derives the upper bound of the gradient with respect to the weight parameters, which plays an important role in the convergence analysis of deep GCN.

\begin{lemma} \label{lemma:convergence_grad_bound}
Let $\bm{\theta} = \{ \mathbf{W}^{(\ell)} \}_{\ell \in [L+1]}$ be the set of parameters, $\lambda_\ell \geq  \max\{ \| \mathbf{W}_1^{(\ell)}\|_2, \| \mathbf{W}_2^{(\ell)}\|_2 \}$, $s \geq \| \mathbf{L} \|_2$,
and $\lambda_{i\rightarrow j} = \prod_{\ell=i}^j \lambda_\ell$,  we have
\begin{equation}
    \left\| \frac{\partial \mathcal{L}(\bm{\theta})}{\partial \mathbf{W}^{(\ell)}}\right\|_\mathrm{F} \leq \frac{s^L \lambda_{1\rightarrow (L+1)}}{\lambda_\ell} \| \mathbf{X} \|_\mathrm{F} \| \widehat{\mathbf{Y}} - \mathbf{Y} \|_\mathrm{F}.
\end{equation}
\end{lemma}
\begin{proof} [Proof of Lemma~\ref{lemma:convergence_grad_bound}]

By the definition of $\frac{\partial \mathcal{L}(\bm{\theta})}{\partial \mathbf{W}^{(\ell)}}$, we have
\begingroup
\allowdisplaybreaks
    \begin{align*}
    \left\| \frac{\partial \mathcal{L}(\bm{\theta})}{\partial \mathbf{W}^{(\ell)}}\right\|_\mathrm{F}
    &= \left\| \left( \frac{\partial \widehat{\mathbf{Y}}}{\partial \mathbf{W}^{(\ell)}} \right)^\top ( \widehat{\mathbf{Y}} - \mathbf{Y})  \right \|_\mathrm{F} \\
    &= \left\|  \frac{\partial [\mathbf{I}_N \otimes (\mathbf{W}^{(L+1)})^\top ]\text{vec}(\mathbf{H}^{(L)} ) }{\partial \text{vec}(\mathbf{W}^{(\ell)}) } \text{vec}( \widehat{\mathbf{Y}} - \mathbf{Y})  \right \|_2 \\
    &= \left\|  [\mathbf{I}_N \otimes (\mathbf{W}^{(L+1)})^\top] \frac{\partial \text{vec}(\mathbf{H}^{(L)} ) }{\partial \text{vec}(\mathbf{H}^{(\ell)})} \frac{\partial \text{vec}(\mathbf{H}^{(\ell)})}{\partial \text{vec}(\mathbf{W}^{(\ell)}) }  \text{vec}( \widehat{\mathbf{Y}} - \mathbf{Y})  \right \|_2 \\
    &\underset{(a)}{\leq} \left\|  [\mathbf{I}_N \otimes (\mathbf{W}^{(L+1)})^\top] \frac{\partial \text{vec}(\mathbf{L}\mathbf{H}^{(L-1)} \mathbf{W}^{(L)}) }{\partial \text{vec}(\mathbf{H}^{(\ell)})} \frac{\partial \text{vec}(\mathbf{L} \mathbf{H}^{(\ell-1)} \mathbf{W}^{(\ell)}) }{\partial \text{vec}(\mathbf{W}^{(\ell)}) } \text{vec}( \widehat{\mathbf{Y}} - \mathbf{Y})  \right \|_2 \\
    & =  \left\|  [\mathbf{I}_N \otimes (\mathbf{W}^{(L+1)})^\top] \frac{\partial \text{vec}( \mathbf{L} \otimes (\mathbf{W}^{(L)})^\top ) \text{vec}(\mathbf{H}^{(L-1)}) }{\partial \text{vec}(\mathbf{H}^{(\ell)})} \frac{\partial [\mathbf{L} \mathbf{H}^{(\ell-1)} \otimes \mathbf{I}_{d_\ell}] \text{vec}(\mathbf{W}^{(\ell)}) }{\partial \text{vec}(\mathbf{W}^{(\ell)}) }  \text{vec}( \widehat{\mathbf{Y}} - \mathbf{Y})  \right \|_2 \\
    &= \left\|  [\mathbf{I}_N \otimes (\mathbf{W}^{(L+1)})^\top]  \text{vec}\left(\mathbf{L}^{L-\ell} \otimes (\mathbf{W}^{(\ell+1)} \ldots \mathbf{W}^{(L)})^\top \right) [\mathbf{L} \mathbf{H}^{(\ell-1)} \otimes \mathbf{I}_{d_\ell}]  \text{vec}( \widehat{\mathbf{Y}} - \mathbf{Y})  \right \|_2 \\
    &= \left\| \underbrace{(\mathbf{L}^{L-\ell+1} \mathbf{H}^{(\ell-1)}) \otimes ( \mathbf{W}^{(\ell+1)} \ldots \mathbf{W}^{(L+1)})^\top  }_{\mathbb{R}^{N d_{L+1} \times d_{\ell-1} d_\ell}} \text{vec}( \widehat{\mathbf{Y}} - \mathbf{Y})  \right \|_\mathrm{F} \\
    &= \left\| (\mathbf{W}^{(\ell+1)} \ldots \mathbf{W}^{(L+1)}) ( \widehat{\mathbf{Y}} - \mathbf{Y})^\top \mathbf{L}^{L-\ell+1} \mathbf{H}^{(\ell-1)} \right\|_2 \\
    &\leq s^{L-\ell+1} \lambda_{(\ell+1) \rightarrow (L+1)} \| \widehat{\mathbf{Y}} - \mathbf{Y} \|_\mathrm{F} \| \mathbf{H}^{(\ell-1)} \|_\mathrm{F},
    \end{align*}
\endgroup
where $(a)$ is due to $\| \sigma(x) \|_2 \leq \| x \|_2$.

Using similar proof strategy, we can upper bound $\| \mathbf{H}^{(\ell-1)} \|_\mathrm{F}$ by
\begin{equation}
    \begin{aligned}
    \| \mathbf{H}^{(\ell-1)} \|_\mathrm{F} 
    &= \| \sigma( \mathbf{L} \mathbf{H}^{(\ell-2)} \mathbf{W}^{(\ell-1)})\|_\mathrm{F} \\
    &\leq \|  \mathbf{L} \mathbf{H}^{(\ell-2)} \mathbf{W}^{(\ell-1)} \|_\mathrm{F} \\
    &\leq s^{\ell-1} \lambda_{1\rightarrow(\ell-1)} \| \mathbf{X} \|_\mathrm{F}.
    \end{aligned}
\end{equation}

By combining the above two equations, we have
\begin{equation}
    \left\| \frac{\partial \mathcal{L}(\bm{\theta})}{\partial \mathbf{W}^{(\ell)}}\right\|_\mathrm{F} \leq \frac{s^L \lambda_{1\rightarrow (L+1)}}{\lambda_\ell} \| \mathbf{X} \|_\mathrm{F} \| \widehat{\mathbf{Y}} - \mathbf{Y} \|_\mathrm{F}.
\end{equation}

\end{proof}

\subsection{Main result}

In the following, we provide the convergence analysis on deep GCNs. 
In particular, Theorem~\ref{thm:wide_deep_gcn_convergence} shows that under assumption on the weight initialization (Eq.~\ref{eq:wide_deep_gcn_convergence_assump}) and the width of the last layer's input dimension ($d_L \geq N$), the deep GCN enjoys linear convergence rate. 
Recall our discussion in Section~\ref{section:limitation_over_smoothing} that $\lambda_\ell \geq 1$ (by Gordon's Theorem for Gaussian matrices) and $\| \mathbf{L} \|_2 = 1$ for $\mathbf{L} = \mathbf{D}^{-1/2} \mathbf{A} \mathbf{D}^{-1/2}$, we know that $s \lambda_\ell \geq 1$, thus deeper model requires a smaller learning rate $\eta$ according to Eq.~\ref{eq:learning_rate}.
Since the number of training iteration $T$ is reverse proportional to learning rate $\eta$, which explains why a deeper model requires more training iterations than the shallow one.

\begin{theorem} \label{thm:wide_deep_gcn_convergence}
Consider a deep GCN with ReLU activation (defined in Eq.~\ref{eq:deeper_better_obj}) where the width of the last hidden layer satisfies $d_L \geq N$. 
Let $\{C_\ell\}_{\ell\in[L+1]}$ be any sequence of positive numbers and define $\alpha_0 = \lambda_{\min}(\mathbf{H}^{(L)}_0)$,~$\lambda_\ell = \| \mathbf{W}^{(\ell)}_0\|_2 + C_\ell $,~
$\lambda_{i \rightarrow j} = \prod_{\ell=i}^j \lambda_\ell$.

Assume the following conditions are satisfied at the initialization:
\begin{equation}\label{eq:wide_deep_gcn_convergence_assump}
    \begin{aligned}
    \alpha_0^2 &\geq 16 s^{2L} \| \mathbf{X}\|_\mathrm{F} \max_{\ell \in [L+1]} \frac{\lambda_{1\rightarrow (L+1)}}{\lambda_\ell C_\ell} \sqrt{2 \mathcal{L}(\bm{\theta}_0)}, \\
    \alpha^3_0 &\geq 32 s^{2L} \| \mathbf{X} \|_\mathrm{F}^2 \lambda_{L+1} \sum_{\ell=1}^{L} \frac{\lambda^2_{1\rightarrow L}}{\lambda^2_\ell} \sqrt{2\mathcal{L}(\bm{\theta}_0)}, \\
    \alpha^2_0 &\geq 16 s^{2L} \| \mathbf{X} \|_\mathrm{F}^2 \lambda_{L+1}^2 \sum_{\ell=1}^{L} \frac{\lambda^2_{1\rightarrow L}}{\lambda^2_\ell} .
    \end{aligned}
\end{equation}

Let the learning rate satisfies
\begin{equation}\label{eq:learning_rate}
    \eta \leq \min\left( \frac{8}{\alpha_0^2}, \frac{\left(\sum_{\ell=1}^L \lambda_\ell^{-2}\right)}{s^{2L} \lambda_{1 \rightarrow (L+1)}^2 \| \mathbf{X}\|_\mathrm{F}^2 \left(\sum_{\ell=1}^{L+1} \lambda_\ell^{-2} \right)^2 } \right).
\end{equation}

Then deep GCN can achieve $\mathcal{L}(\bm{\theta}_T) \leq \epsilon$ for any $T$ satisfied
\begin{equation}
    T \geq \frac{8}{\eta \alpha_0^2} \log\left( \frac{\mathcal{L}(\bm{\theta}_0)}{\epsilon}\right).
\end{equation}
\end{theorem}

\begin{proof} [Proof of Theorem~\ref{thm:wide_deep_gcn_convergence}]

The proof follows the proof of Theorem 2.2 in~\cite{nguyen2021proof} by extending result from MLP to GCN. 
The key idea of the proof is to show the followings hold by induction:
\begin{equation}\label{eq:wide_deep_gcn_convergence_eq1}
    \begin{cases}
    \| \mathbf{W}_r^{(\ell)}\|_2 \leq \lambda_\ell & \text{ for all } \ell \in [L],~r \in [0, t], \\ 
    \lambda_{\min}(\mathbf{H}_r^{(L)}) \geq \frac{\alpha_0}{2} & \text{ for all } r \in [0, t], \\ 
    \mathcal{L}(\bm{\theta}_r) \leq (1 - \eta \alpha_0^2 / 8) \mathcal{L}(\bm{\theta}_0)  & \text{ for all } r \in [0, t] .
    \end{cases}
\end{equation}

Clearly, Eq.~\ref{eq:wide_deep_gcn_convergence_eq1} holds for $t=0$. Let assume Eq.~\ref{eq:wide_deep_gcn_convergence_eq1} holds up to iteration $t$, and let show it also holds for $(t+1)$th iteration.

Let first prove the first inequality holds in Eq.~\ref{eq:wide_deep_gcn_convergence_eq1}.
By the triangle inequality, we can decompose the distance between $\mathbf{W}_{t+1}^{(\ell)}$ to $\mathbf{W}_{0}^{(\ell)}$ by
\begin{equation}\label{eq:wide_deep_gcn_convergence_eq2}
    \begin{aligned}
    \| \mathbf{W}_{t+1}^{(\ell)} - \mathbf{W}^{(\ell)}_0 \|_\mathrm{F}
    &\leq \sum_{k=0}^{t} \| \mathbf{W}_{k+1}^{(\ell)} - \mathbf{W}^{(\ell)}_k \|_\mathrm{F} \\
    &= \eta \sum_{k=0}^{t} \left\| \frac{\partial \mathcal{L}(\bm{\theta})}{\partial \mathbf{W}^{(\ell)}_k} \right\|_\mathrm{F} \\
    &\leq \eta s^L \lambda_{1\rightarrow (L+1)}\lambda_\ell^{-1} \| \mathbf{X} \|_\mathrm{F}  \sum_{k=0}^t \| \widehat{\mathbf{Y}}_k - \mathbf{Y} \|_\mathrm{F} \\
    &\leq \eta s^L \lambda_{1\rightarrow (L+1)}\lambda_\ell^{-1} \| \mathbf{X} \|_\mathrm{F}  \sum_{k=0}^t \sqrt{2 \mathcal{L}(\bm{\theta}_k)} \\
    &\underset{(a)}{\leq} \eta s^L \lambda_{1\rightarrow (L+1)}\lambda_\ell^{-1} \| \mathbf{X} \|_\mathrm{F}  \sum_{k=0}^t \left( 1 - \eta \frac{\alpha_0^2}{8}\right)^{k/2} \sqrt{2 \mathcal{L}(\bm{\theta}_0)},
    \end{aligned}
\end{equation}
where inequality $(a)$ follows the induction condition.

Let $u = \sqrt{1 - \eta \alpha_0^2 / 8}$ we have $\eta = 8(1 - u^2) / \alpha_0^2$ and
\begin{equation}
    \sum_{k=0}^t \left( 1 - \eta \frac{\alpha_0^2}{8}\right)^{k/2} = \sum_{k=0}^t u^k = \frac{1-u^{t+1}}{1-u}.
\end{equation}

Then, plugging the result into Eq.~\ref{eq:wide_deep_gcn_convergence_eq2}, we have
\begin{equation} \label{eq:wide_deep_gcn_convergence_eq3}
    \begin{aligned}
    \| \mathbf{W}_{t+1}^{(\ell)} - \mathbf{W}^{(\ell)}_0 \|_\mathrm{F} &\leq 
     s^L \lambda_{1\rightarrow (L+1)}\lambda_\ell^{-1} \| \mathbf{X} \|_\mathrm{F} \frac{8(1 - u^2)}{\alpha_0^2} \frac{1-u^{t+1}}{1-u} \sqrt{2 \mathcal{L}(\bm{\theta}_0)}, \\
     &\underset{(a)}{\leq} \frac{16}{\alpha_0^2} s^L \lambda_{1\rightarrow (L+1)}\lambda_\ell^{-1} \| \mathbf{X} \|_\mathrm{F} \sqrt{2 \mathcal{L}(\bm{\theta}_0)} ,
    \end{aligned}
\end{equation}
where $(a)$ is due to $u \in (0,1)$.
Let $C_\ell$ as a positive constant that
\begin{equation}
    C_\ell \geq \frac{16}{\alpha_0^2} s^L \lambda_{1\rightarrow (L+1)}\lambda_\ell^{-1} \| \mathbf{X} \|_\mathrm{F} \sqrt{2 \mathcal{L}(\bm{\theta}_0)}.
\end{equation}
By Wely's inequality, we have
\begin{equation}
    \| \mathbf{W}_{t+1}^{(\ell)} \|_2 \leq \| \mathbf{W}_{0}^{(\ell)} \|_2 + C_\ell = \lambda_\ell.
\end{equation}

To prove the second inequality holds in Eq.~\ref{eq:wide_deep_gcn_convergence_eq1}, we first upper bound $\| \mathbf{H}_{t+1}^{(L)} - \mathbf{H}_0^{(L)} \|_\mathrm{F} $ by
\begin{equation}
    \begin{aligned}
    \| \mathbf{H}_{t+1}^{(L)} - \mathbf{H}_0^{(L)} \|_\mathrm{F} 
    &\underset{(a)}{\leq} s^L \lambda_{1\rightarrow L} \| \mathbf{X} \|_\mathrm{F} \sum_{\ell=1}^L\lambda_\ell^{-1} \| \mathbf{W}^{(\ell)}_{t+1} - \mathbf{W}^{(\ell)}_0 \|_2 \\
    &\underset{(b)}{\leq} s^L \lambda_{1\rightarrow L} \| \mathbf{X} \|_\mathrm{F} \sum_{\ell=1}^L \lambda_\ell^{-1} \left( \frac{16}{\alpha_0^2} s^L \lambda_{1\rightarrow (L+1)}\lambda_\ell^{-1} \| \mathbf{X} \|_\mathrm{F} \sqrt{2 \mathcal{L}(\bm{\theta}_0)} \right) \\
    &= \frac{16}{\alpha_0^2} s^{2L} \lambda_{1\rightarrow L} \lambda_{1\rightarrow (L+1)}  \| \mathbf{X} \|_\mathrm{F}^2 \sqrt{2\mathcal{L}(\bm{\theta}_0)} \left( \sum_{\ell=1}^L \lambda_\ell^{-2} \right) \\
    &\underset{(c)}{\leq} \frac{\alpha_0}{2},
    \end{aligned}
\end{equation}
where $(a)$ is due to Lemma~\ref{lemma:convergence_width_lip}, $(b)$ is due to Eq.~\ref{eq:wide_deep_gcn_convergence_eq3}, and $(c)$ is due to the second condition in Eq.~\ref{eq:wide_deep_gcn_convergence_assump}.

By Weyl's inequality, this implies $| \lambda_{\min}(\mathbf{H}_{t+1}^{(L)}) - \lambda_{\min}(\mathbf{H}_{0}^{(L)}) | = | \lambda_{\min}(\mathbf{H}_{t+1}^{(L)}) - \alpha_0 | \leq \frac{\alpha_0}{2}$ and
\begin{equation}\label{eq:wide_deep_gcn_convergence_eq4}
    \lambda_{\min}(\mathbf{H}_{t+1}^{(L)}) \geq \frac{\alpha_0}{2}.
\end{equation}

Let $\mathbf{G} = \mathbf{H}_t^{(L)} \mathbf{W}_{t+1}^{(L+1)}$, then we have
\begin{equation} \label{eq:decompose_loss_convergence}
    \begin{aligned}
    2 \mathcal{L}(\bm{\theta}_{t+1}) &= \| \widehat{\mathbf{Y}}_{t+1} - \mathbf{Y} \|_\mathrm{F}^2 \\
    &= \| \widehat{\mathbf{Y}}_{t+1} - \widehat{\mathbf{Y}}_{t} + \widehat{\mathbf{Y}}_{t}   - \mathbf{Y} \|_\mathrm{F}^2 \\
    &= \| \widehat{\mathbf{Y}}_{t}   - \mathbf{Y} \|_\mathrm{F}^2 + \| \widehat{\mathbf{Y}}_{t+1} - \widehat{\mathbf{Y}}_{t} \|_\mathrm{F}^2 + 2 \langle \widehat{\mathbf{Y}}_{t}   - \mathbf{Y}, \widehat{\mathbf{Y}}_{t+1} - \widehat{\mathbf{Y}}_{t} \rangle_\mathrm{F} \\
    &= 2 \mathcal{L}(\bm{\theta}_{t}) + \| \widehat{\mathbf{Y}}_{t+1} - \widehat{\mathbf{Y}}_{t} \|_\mathrm{F}^2 + 2 \langle \widehat{\mathbf{Y}}_{t}   - \mathbf{Y}, \widehat{\mathbf{Y}}_{t+1} - \mathbf{G} \rangle_\mathrm{F} + 2 \langle \widehat{\mathbf{Y}}_{t}   - \mathbf{Y}, \mathbf{G} - \widehat{\mathbf{Y}}_{t} \rangle_\mathrm{F}.
    \end{aligned}
\end{equation}

We can upper bound $\| \widehat{\mathbf{Y}}_{t+1} - \widehat{\mathbf{Y}}_{t} \|_\mathrm{F}^2$ in Eq.~\ref{eq:decompose_loss_convergence} by 
\begin{equation}
    \begin{aligned}
    \| \widehat{\mathbf{Y}}_{t+1} - \widehat{\mathbf{Y}}_{t} \|_\mathrm{F}
    &= \| \mathbf{H}^{(L)}_{t+1} \mathbf{W}^{(L+1)}_{t+1} - \mathbf{H}^{(L)}_{t} \mathbf{W}^{(L+1)}_{t} \|_\mathrm{F} \\
    &= \| \mathbf{H}^{(L)}_{t+1} (\mathbf{W}^{(L+1)}_{t+1} - \mathbf{W}^{(L+1)}_{t} ) + (\mathbf{H}^{(L)}_{t+1} - \mathbf{H}^{(L)}_{t} ) \mathbf{W}^{(L+1)}_{t} \|_\mathrm{F} \\
    &\leq \| \mathbf{H}^{(L)}_{t+1}  \|_\mathrm{F} \| \mathbf{W}^{(L+1)}_{t+1} - \mathbf{W}^{(L+1)}_{t}  \|_2 + \| \mathbf{H}^{(L)}_{t+1} - \mathbf{H}^{(L)}_{t} \|_\mathrm{F} \| \mathbf{W}^{(L+1)}_{t} \|_2 \\
    &\underset{(a)}{\leq} s^L\lambda_{1\rightarrow (L+1)} \| \mathbf{X} \|_\mathrm{F} \sum_{\ell=1}^L \lambda_\ell^{-1} \| \mathbf{W}^{(\ell)}_{t+1} - \mathbf{W}^{(\ell)}_t \|_2 + s^L \lambda_{1\rightarrow L}\| \mathbf{X} \|_\mathrm{F} \| \mathbf{W}^{(L+1)}_{t+1} - \mathbf{W}^{(L+1)}_{t}  \|_2 \\
    &= s^L\lambda_{1\rightarrow (L+1)} \| \mathbf{X} \|_\mathrm{F} \sum_{\ell=1}^{L+1} \lambda_\ell^{-1} \| \mathbf{W}^{(\ell)}_{t+1} - \mathbf{W}^{(\ell)}_t \|_2 \\
    &\underset{(b)}{\leq} \eta s^L\lambda_{1\rightarrow (L+1)} \| \mathbf{X} \|_\mathrm{F} \sum_{\ell=1}^{L+1} \lambda_\ell^{-1} \left( \frac{s^L \lambda_{1\rightarrow (L+1)}}{\lambda_\ell} \| \mathbf{X} \|_\mathrm{F} \| \widehat{\mathbf{Y}}_t - \mathbf{Y} \|_\mathrm{F} \right) \\
    &= \eta s^{2L}\lambda_{1\rightarrow (L+1)}^2 \| \mathbf{X} \|_\mathrm{F}^2 \left( \sum_{\ell=1}^{L+1} \lambda_\ell^{-2} \right) \| \widehat{\mathbf{Y}}_t - \mathbf{Y} \|_\mathrm{F} ,
    \end{aligned}
\end{equation}
where $(a)$ is due to Eq.~\ref{eq:convergence_width_lip_eq1} and Lemma~\ref{lemma:convergence_width_lip}, and $(b)$ follows from Lemma~\ref{lemma:convergence_grad_bound}.

We can upper bound $\langle \widehat{\mathbf{Y}}_{t}   - \mathbf{Y}, \widehat{\mathbf{Y}}_{t+1} - \mathbf{G} \rangle_\mathrm{F}$ in Eq.~\ref{eq:decompose_loss_convergence} by 
\begin{equation}
    \begin{aligned}
    &\langle \widehat{\mathbf{Y}}_{t}   - \mathbf{Y}, \widehat{\mathbf{Y}}_{t+1} - \mathbf{G} \rangle_\mathrm{F} \\
    &\leq \|  \widehat{\mathbf{Y}}_{t}   - \mathbf{Y} \|_\mathrm{F} \| \widehat{\mathbf{Y}}_{t+1} - \mathbf{G} \|_\mathrm{F} \\
    &= \|  \widehat{\mathbf{Y}}_{t}   - \mathbf{Y} \|_\mathrm{F} \| (\mathbf{H}^{(L)}_{t+1} - \mathbf{H}_t^{(L)} ) \mathbf{W}_{t+1}^{(L+1)} \|_\mathrm{F} \\
    &\underset{(a)}{\leq} \lambda_{L+1} \|  \widehat{\mathbf{Y}}_{t}   - \mathbf{Y} \|_\mathrm{F} \left( s^L \lambda_{1\rightarrow L} \| \mathbf{X} \|_\mathrm{F} \sum_{\ell=1}^L \lambda_\ell^{-1} \| \mathbf{W}^{(\ell)}_{t+1} - \mathbf{W}^{(\ell)}_t \|_2 \right) \\
    &= s^L \lambda_{1\rightarrow (L+1)} \|  \widehat{\mathbf{Y}}_{t}   - \mathbf{Y} \|_\mathrm{F} \| \mathbf{X} \|_\mathrm{F} \left( \sum_{\ell=1}^L \lambda_\ell^{-1} \| \mathbf{W}^{(\ell)}_{t+1} - \mathbf{W}^{(\ell)}_t \|_2  \right) \\
    &\underset{(b)}{\leq} \eta s^L \lambda_{1\rightarrow (L+1)} \|  \widehat{\mathbf{Y}}_{t}   - \mathbf{Y} \|_\mathrm{F} \| \mathbf{X} \|_\mathrm{F} \left( \sum_{\ell=1}^L \lambda_\ell^{-1} \frac{s^L \lambda_{1\rightarrow (L+1)}}{\lambda_\ell} \| \mathbf{X} \|_\mathrm{F} \| \widehat{\mathbf{Y}} - \mathbf{Y} \|_\mathrm{F}  \right) \\
    &= \eta s^{2L} \lambda_{1\rightarrow (L+1)}^2 \|  \widehat{\mathbf{Y}}_{t}   - \mathbf{Y} \|_\mathrm{F}^2 \| \mathbf{X} \|_\mathrm{F}^2 \left( \sum_{\ell=1}^L \lambda_\ell^{-2} \right),
    \end{aligned}
\end{equation}
where $(a)$ is due to Lemma~\ref{lemma:convergence_width_lip}, and $(b)$ follows from Lemma~\ref{lemma:convergence_grad_bound}.

We can upper bound $\langle \widehat{\mathbf{Y}}_{t}   - \mathbf{Y}, \mathbf{G} - \widehat{\mathbf{Y}}_{t} \rangle_\mathrm{F}$ by 
\begin{equation}
    \begin{aligned}
    \langle \widehat{\mathbf{Y}}_{t}   - \mathbf{Y}, \mathbf{G} - \widehat{\mathbf{Y}}_{t} \rangle_\mathrm{F}
    &= \langle \widehat{\mathbf{Y}}_{t}   - \mathbf{Y}, \mathbf{H}_t^{(L)} (\mathbf{W}_{t+1}^{(L+1)} - \mathbf{W}_t^{(L+1)} ) \rangle_\mathrm{F} \\
    &= \eta \left\langle \widehat{\mathbf{Y}}_{t}   - \mathbf{Y}, \mathbf{H}_t^{(L)} \frac{\partial \mathcal{L}(\bm{\theta})}{\partial \mathbf{W}^{(L+1)}_t} \right\rangle_\mathrm{F} \\
    &= - \eta \left\langle \widehat{\mathbf{Y}}_{t}   - \mathbf{Y}, \mathbf{H}_t^{(L)} \left((\mathbf{H}_t^{(L)})^\top(\widehat{\mathbf{Y}}_{t}   - \mathbf{Y})\right) \right\rangle_\mathrm{F} \\
    &\underset{(a)}{=} -\eta ~\mathrm{tr}\left( (\widehat{\mathbf{Y}}_t - \mathbf{Y})^\top \mathbf{H}_t^{(L)}  (\mathbf{H}_t^{(L)})^\top (\widehat{\mathbf{Y}}_t - \mathbf{Y})\right) \\
    &\underset{(b)}{=} -\eta  ~\mathrm{tr}\left(  \mathbf{H}_t^{(L)}  (\mathbf{H}_t^{(L)})^\top (\widehat{\mathbf{Y}}_t - \mathbf{Y}) (\widehat{\mathbf{Y}}_t - \mathbf{Y})^\top \right) \\
    &\underset{(c)}{\leq} -\eta \lambda_{\min}\left(\mathbf{H}_t^{(L)}  (\mathbf{H}_t^{(L)})^\top\right) \mathrm{tr}\left((\widehat{\mathbf{Y}}_t - \mathbf{Y}) (\widehat{\mathbf{Y}}_t - \mathbf{Y})^\top\right)\\
    &\underset{(d)}{=} -\eta  \lambda_{\min}\left(\mathbf{H}_t^{(L)}  (\mathbf{H}_t^{(L)})^\top\right) \| \widehat{\mathbf{Y}}_t - \mathbf{Y} \|_\mathrm{F}^2 \\
    &\underset{(e)}{=} -\eta  \lambda_{\min}^2 \left(\mathbf{H}_t^{(L)} \right) \| \widehat{\mathbf{Y}}_t - \mathbf{Y} \|_\mathrm{F}^2, \\
    &\underset{(f)}{\leq} - \eta \frac{\alpha_0^2}{4} \| \widehat{\mathbf{Y}}_{t}   - \mathbf{Y} \|_\mathrm{F}^2,
    \end{aligned}
\end{equation}
where $(a)$ is due to $\langle \mathbf{A}, \mathbf{B} \rangle = \mathrm{tr}(\mathbf{A}^\top \mathbf{B})$, $(b)$ is because of the cyclic property of the 
trace, $(c)$ is due to $\mathrm{tr}(\mathbf{A} \mathbf{B}) \geq \lambda_{\min}(\mathbf{A}) \mathrm{tr}( \mathbf{B}) $ for any symmetric matrices $\mathbf{A},\mathbf{B} \in \mathbb{S}$ and $\mathbf{B}$ is positive semi-definite, $(d)$ is due to $\mathrm{tr}(\mathbf{A} \mathbf{A}^\top)=\| \mathbf{A} \|^2_\mathrm{F}$, $(e)$ requires the $d_L \geq N$ assumption\footnote{
The equality $(e)$ follows from the fact that given a fat matrix $\mathbf{A} \in \mathbb{R}^{N\times d}$ with $d > N$, all the eigenvalues of $\mathbf{A} \mathbf{A}^\top$ are just the square of the singular values of $\mathbf{A}$. However, note that this is no longer true when $d < N$ since in this case we have $\mathbf{A}\mathbf{A}^\top$ being a low rank matrix, hence its smallest eigenvalue is zero, whereas the smallest singular value of $\mathbf{A}$ (i.e., the $\min(N,d)$-th singular value of $\mathbf{A}$) can be strictly positive if $\mathbf{A}$ is full rank, hence (e) does not hold.},  and $(f)$ is due to Eq.~\ref{eq:wide_deep_gcn_convergence_eq4}.

Let $A = s^{2L} \lambda_{1\rightarrow (L+1)}^2 \| \mathbf{X} \|_\mathrm{F}^2 \left( \sum_{\ell=1}^{L+1} \lambda_\ell^{-2} \right)$ and $B = s^{2L} \lambda_{1\rightarrow (L+1)}^2 \| \mathbf{X} \|_\mathrm{F}^2 \left( \sum_{\ell=1}^L \lambda_\ell^{-2} \right)$, we have
\begin{equation}
    \begin{aligned}
    \mathcal{L}(\bm{\theta}_{t+1}) &\leq \left( 1 - \eta \frac{\alpha_0^2}{4} + \eta^2 A^2 + \eta B \right)\mathcal{L}(\bm{\theta}_{t}) \\
    &\underset{(a)}{\leq} \left( 1 - \eta \left( \frac{\alpha_0^2}{4} - 2 B \right)\right) \mathcal{L}(\bm{\theta}_{t}) \\
    &\underset{(b)}{\leq} \left( 1 - \eta \frac{\alpha_0^2}{8} \right) \mathcal{L}(\bm{\theta}_{t}),
    \end{aligned}
\end{equation}
where $(a)$ requires choosing $\eta$ such that $\eta \leq B/A^2$,
and $(b)$ holds due to the first condition in Eq.~\ref{eq:wide_deep_gcn_convergence_assump}.

Therefore, we have \begin{equation}
    \mathcal{L}(\bm{\theta}_T) \leq \left( 1 - \eta \frac{\alpha_0^2}{8} \right)^T \mathcal{L}(\bm{\theta}_0) \leq \exp\left( - \eta T \frac{\alpha_0^2}{8} \right) \mathcal{L}(\bm{\theta}_0)  = \epsilon.
\end{equation}
By taking log on both side, we have
\begin{equation}
    T \geq \frac{8}{\eta \alpha_0^2} \log\left( \frac{\mathcal{L}(\bm{\theta}_0)}{\epsilon}\right).
\end{equation}

\end{proof}

\clearpage
\section{Additional empirical results}\label{supp:more_empirical}

In this section, we report additional empirical results to illustrate the correctness of our theoretical analysis and the benefits of decoupled GCN structure.

\subsection{Comparison of generalization error on the real-world datasets.}
We empirically compare the generalization error of different GCN structures on the real-world datasets, where the generalization error is measured by the difference between validation loss and training loss. 

\noindent\textbf{Setups.}
We select hidden dimension as $64$, $\alpha_\ell=0.9$ for APPNP and GCNII, $\beta_\ell \approx \frac{0.5}{\ell}$ for GCNII, and without any weight decay ($\ell_2$ regularization on weight matrices) or dropout operators.
Note that although in our theoretical analysis we suppose $\beta_\ell$ is a constant that does not change with respect to the layers, in this experiment we follow the configuration in~\cite{chen2020simple} by selecting $\beta_\ell = \log( \frac{0.5}{\ell} + 1)$, which guarantees a small generalization error on most small scale datasets. 
The same setup is also used for the results on synthetic dataset in Figure~\ref{fig:synthetic_dataset}. 

\noindent\textbf{Results.}
As shown in Figure~\ref{fig:generalization_gap_cora_468} and Figure~\ref{fig:generalization_gap_citeseer_468}, ResGCN has the largest generalization
gap due to the skip-connections, APPNP has the smallest generalization error by removing the weight matrices in each individual layers. 
On the other hand, GCNII achieves a good balance between GCN and APPNP by balancing the expressive and generalization power. 
Finally, DGCN enjoys a small generalization error by using the decoupled GCN structure.

\begin{figure}[h]
    \centering
    \includegraphics[width=1.0\textwidth]{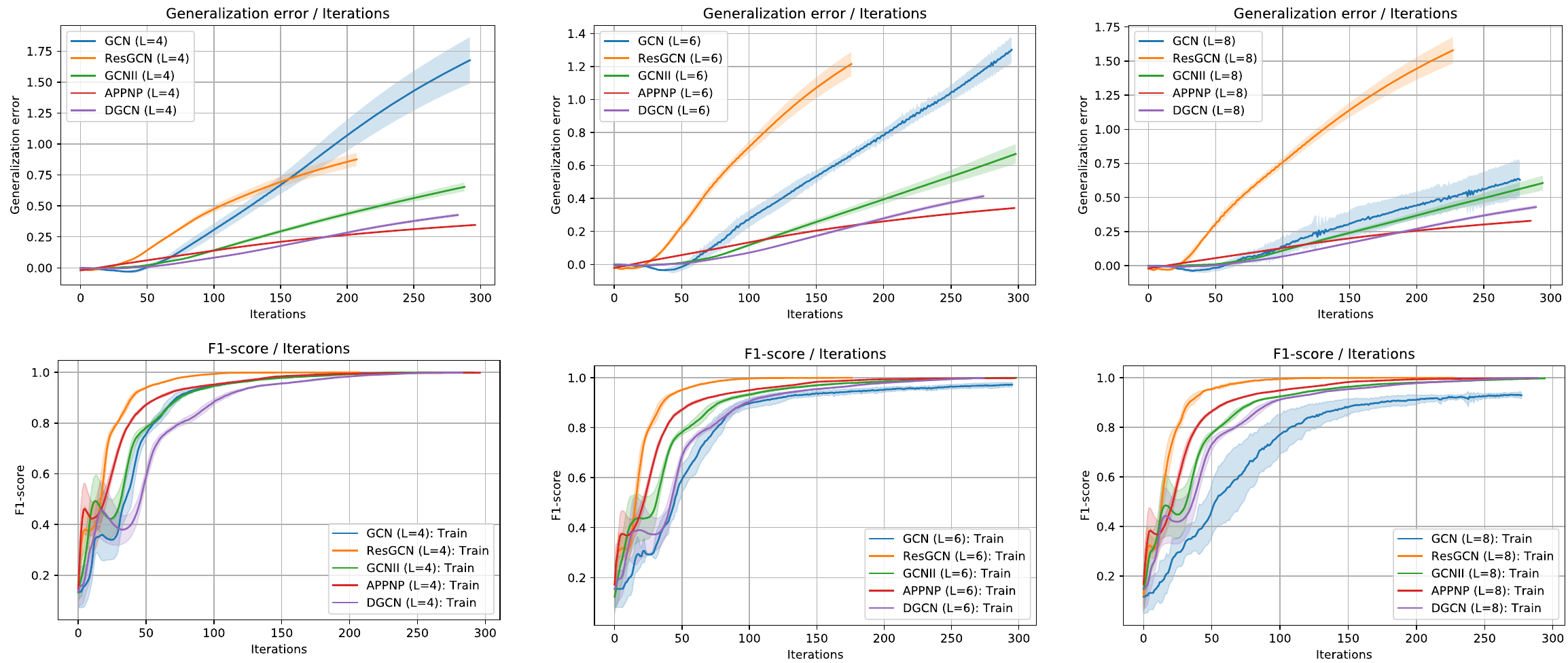}
    \caption{Comparison of generalization error and training F1-score on the \textit{Cora} dataset. The curve stops early at the largest training accuracy iteration.}
    \label{fig:generalization_gap_cora_468}
\end{figure}

\begin{figure}[h]
    \centering
    \includegraphics[width=1.0\textwidth]{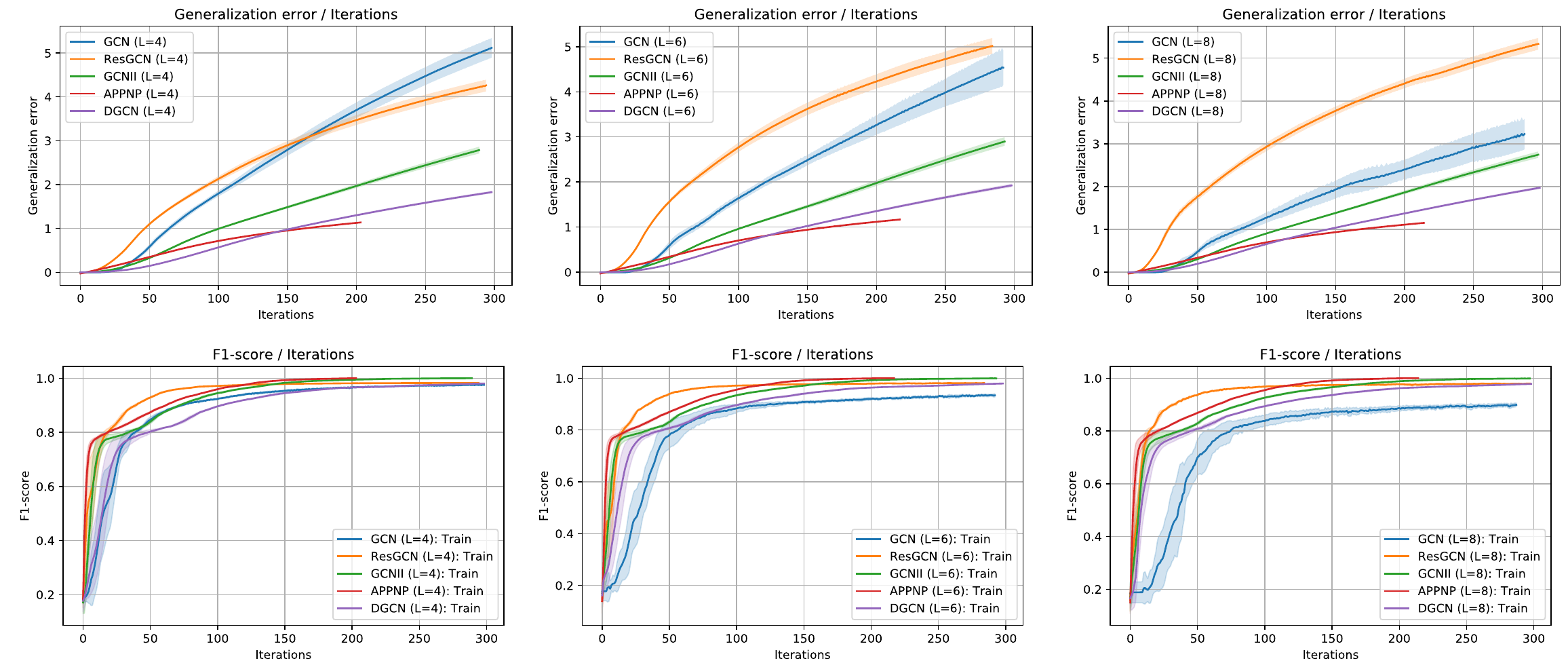}
    \caption{Comparison of generalization error and training F1-score on the \textit{Citeseer} dataset. The curve stops early at the largest training accuracy iteration.}
    \label{fig:generalization_gap_citeseer_468}
\end{figure}

\subsection{Effect of hyper-parameters in GCNII.}

In the following, we first compare the effect of $\alpha_\ell$ on the generalization error of GCNII.
We select the hidden dimension as $64$, and no dropout or weight decay is used for the training.
As shown in Figure~\ref{fig:generalization_gap_GCNII_alpha_4816_cora} and Figure~\ref{fig:generalization_gap_GCNII_alpha_4816_citeseer}, increasing $\alpha_\ell$ leads to a smaller generalization error but a slower convergence speed, i.e., GCNII trades off the expressiveness for generalization power by increasing $\alpha_\ell$ from $0$ to $1$. 
In practice, $\alpha_\ell=0.9$ is utilized in GCNII~\cite{chen2020simple} for empirical evaluations.\footnote{Note that the definition of $\alpha_\ell$ we use in this paper is different from the definition in~\cite{chen2020simple}, where $\alpha_\ell=0.9$ in this paper stands for the selection of alpha as $1-\alpha_\ell=0.1$ in~\cite{chen2020simple}.}

\begin{figure}[h]
    \centering
    \includegraphics[width=1.0\textwidth]{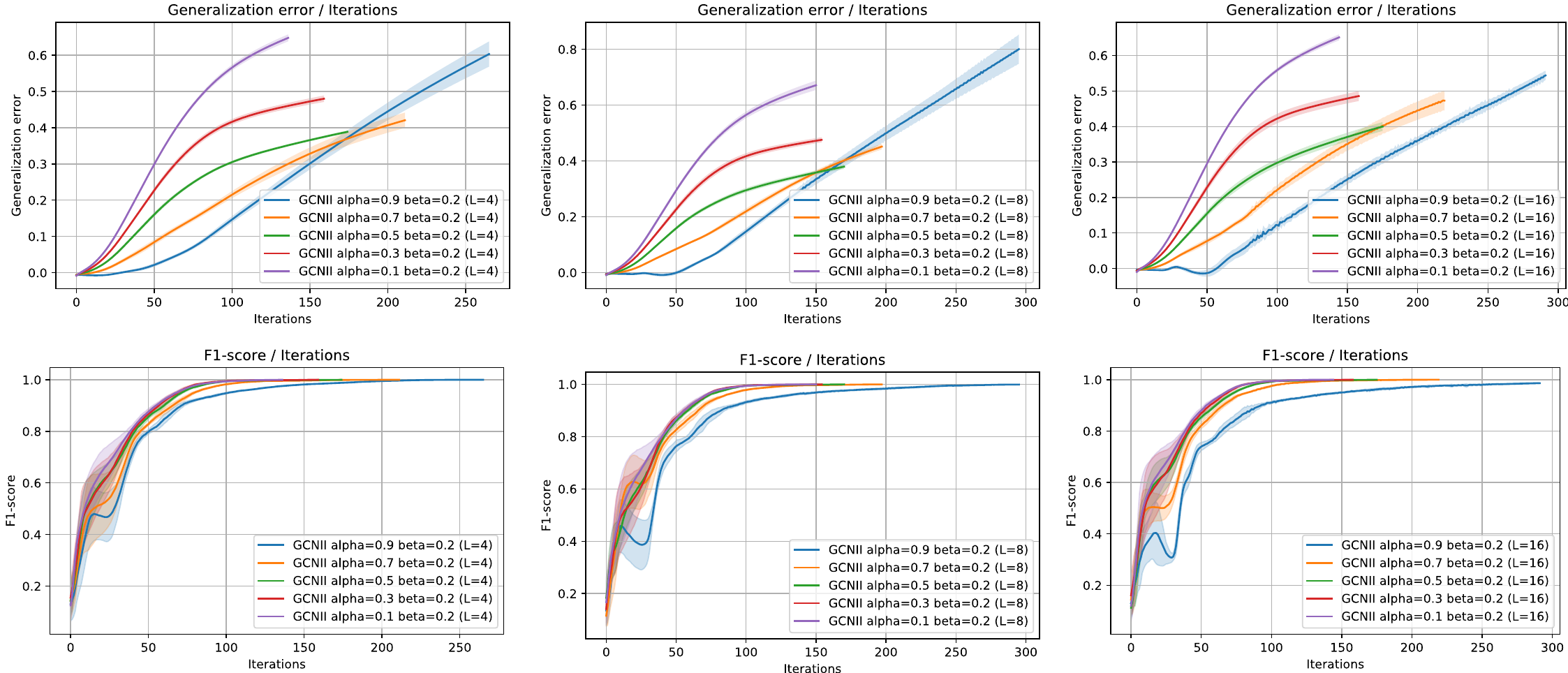}
    \caption{Comparison of $\alpha_\ell$ on the generalization error on \textit{Cora} dataset. The curve stops early at the largest training accuracy iteration.}
    \label{fig:generalization_gap_GCNII_alpha_4816_cora}
\end{figure}

\begin{figure}[h]
    \centering
    \includegraphics[width=1.0\textwidth]{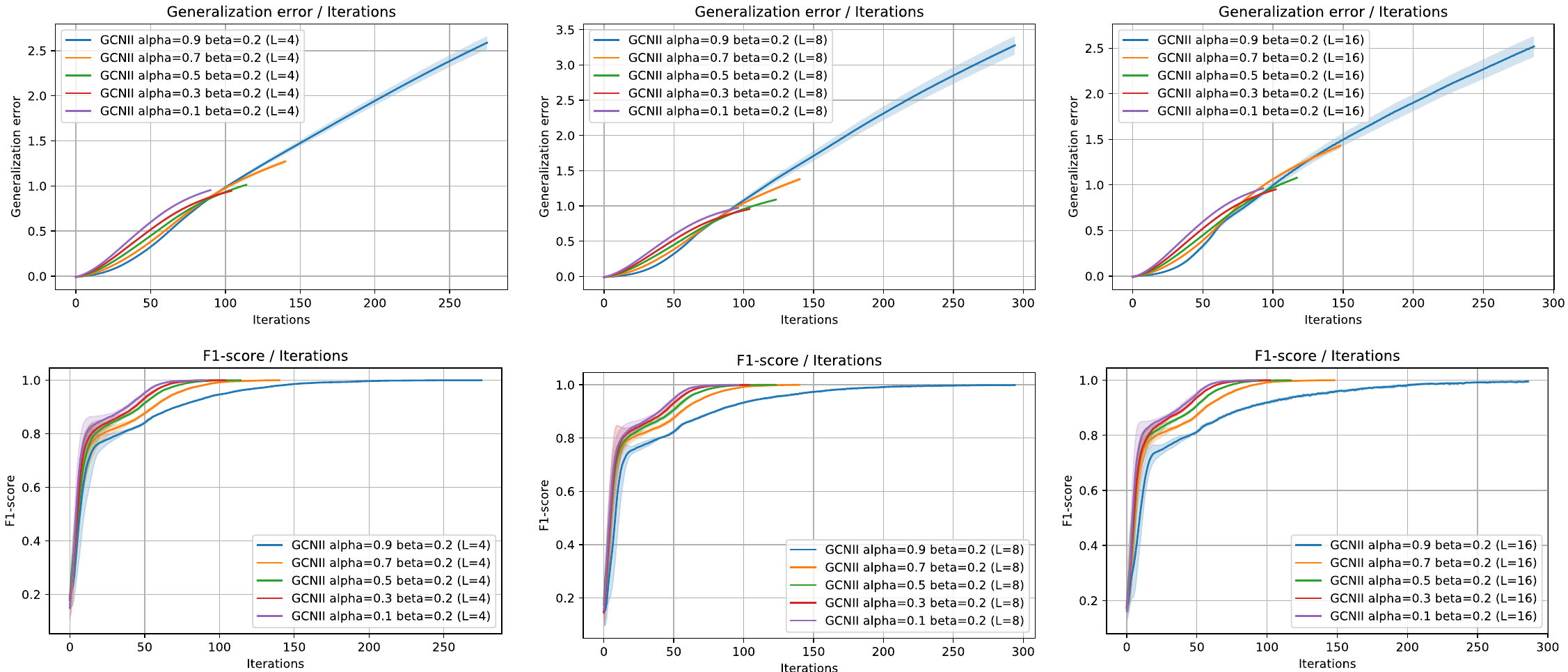}
    \caption{Comparison of $\alpha_\ell$ on the generalization error on \textit{Citeseer} dataset. The curve stops early at the largest training accuracy iteration.}
    \label{fig:generalization_gap_GCNII_alpha_4816_citeseer}
\end{figure}

Then, we compare the effect of $\beta_\ell$ on the generalization error of GCNII. 
Similar to the previous experiments, we choose hidden dimension as $64$, without applying dropout or weight decay during training.
As shown in Figure~\ref{fig:generalization_gap_GCNII_beta_4816_cora} and Figure~\ref{fig:generalization_gap_GCNII_beta_4816_citeseer}, decreasing $\beta_\ell$ leads to a smaller generalization error. 
In practice, $\beta_\ell = \log(\frac{0.5}{\ell} + 1)$ is utilized in GCNII~\cite{chen2020simple} for empirical evaluations, which guarantees a small generalization error.

\begin{figure}[h]
    \centering
    \includegraphics[width=1.0\textwidth]{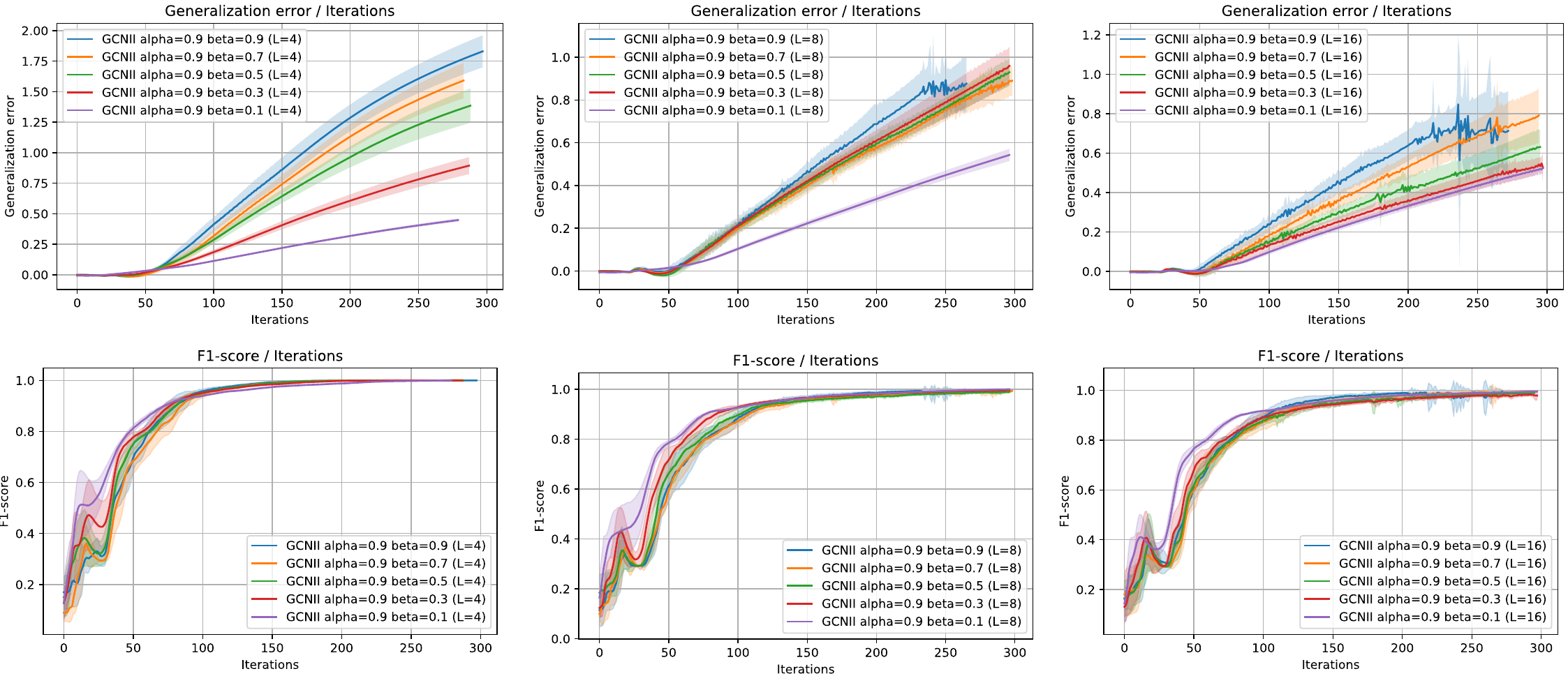}
    \caption{Comparison of $\beta_\ell$ on the generalization error on \textit{Cora} dataset. The curve stops early at the largest training accuracy iteration.}
    \label{fig:generalization_gap_GCNII_beta_4816_cora}
\end{figure}

\begin{figure}[h]
    \centering
    \includegraphics[width=1.0\textwidth]{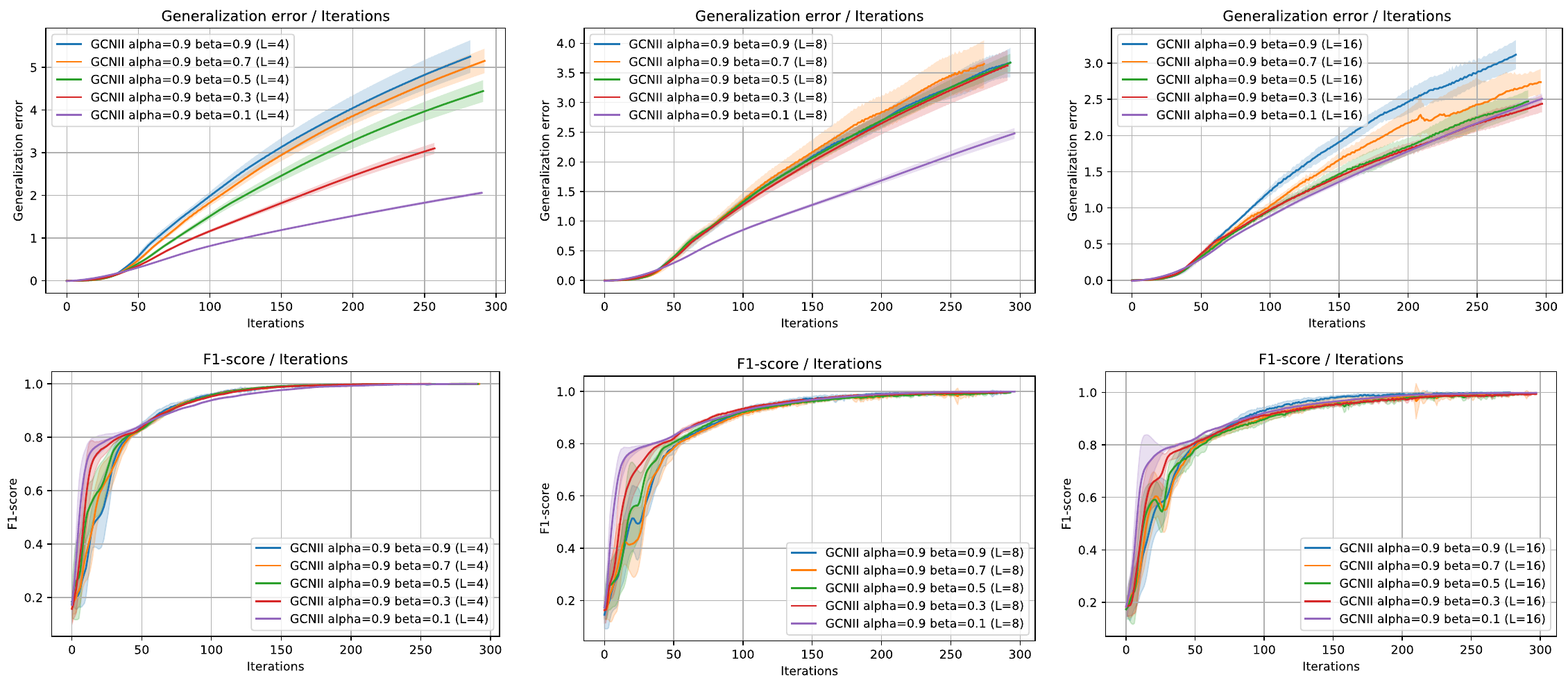}
    \caption{Comparison of $\beta_\ell$ on the generalization error on \textit{Citeseer} dataset. The curve stops early at the largest training accuracy iteration.}
    \label{fig:generalization_gap_GCNII_beta_4816_citeseer}
\end{figure}

\subsection{Effect of DropEdge on generalization} \label{supp:dropegde}
In the following, we explore the effect of node embedding augmentation technique DropEdge~\cite{rong2019dropedge} on the generalization of GCNs.
Recall that DropEdge proposes to randomly drop a certain rate of edges in the input graph at each iteration and compute node embedding based on the sparser graph. 
The forward propagation rule of this technique can be formulated as $\mathbf{H}^{(\ell-1)} = \sigma(\widetilde{\mathbf{L}} \mathbf{H}^{(\ell-1)} \mathbf{W}^{(\ell)}))$ where $\widetilde{\mathbf{L}}$ is constructed by the adjacency matrix of the sparser graph with $\text{supp}(\widetilde{\mathbf{L}}) \ll \text{supp}(\mathbf{L})$. 

Again, we choose hidden dimension as $64$, without applying dropout or weight decay during training.
As shown in Figure~\ref{fig:dropedge_cora_456} and Figure~\ref{fig:dropedge_citeseer_456}, DropEdge reduces the generalization error by restricting the number of nodes used during training.
Besides, we observe that both training accuracy and generalization error decrease when the fraction of remaining edges in the graph decreases, which implies that edge dropping is impacting the generalization of GCN rather than the discriminativeness of node representations.

\begin{figure}[h]
    \centering
    \includegraphics[width=1.0\textwidth]{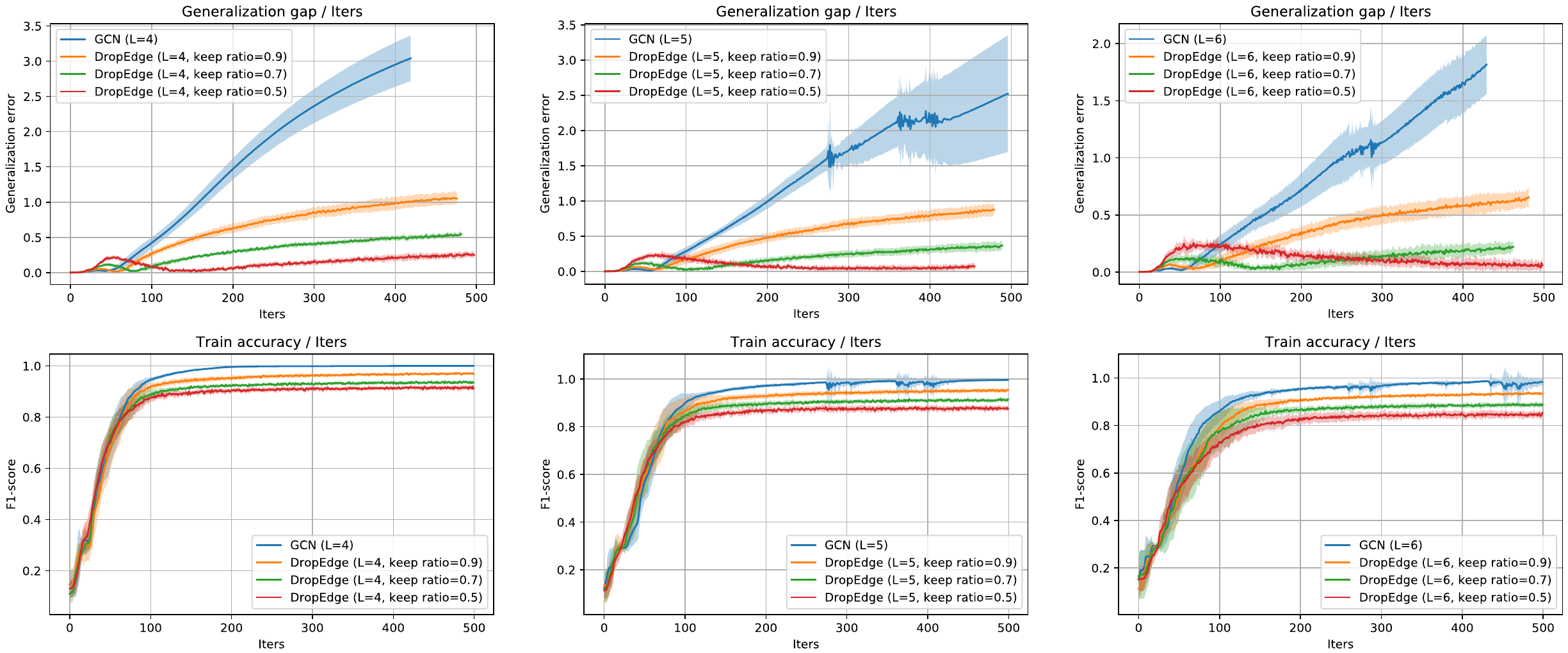}
    \caption{Comparison of generalization error of \textit{DropEdge} on \textit{Cora} dataset. The curve stops early at the largest training accuracy iteration.}
    \label{fig:dropedge_cora_456}
\end{figure}

\begin{figure}[h]
    \centering
    \includegraphics[width=1.0\textwidth]{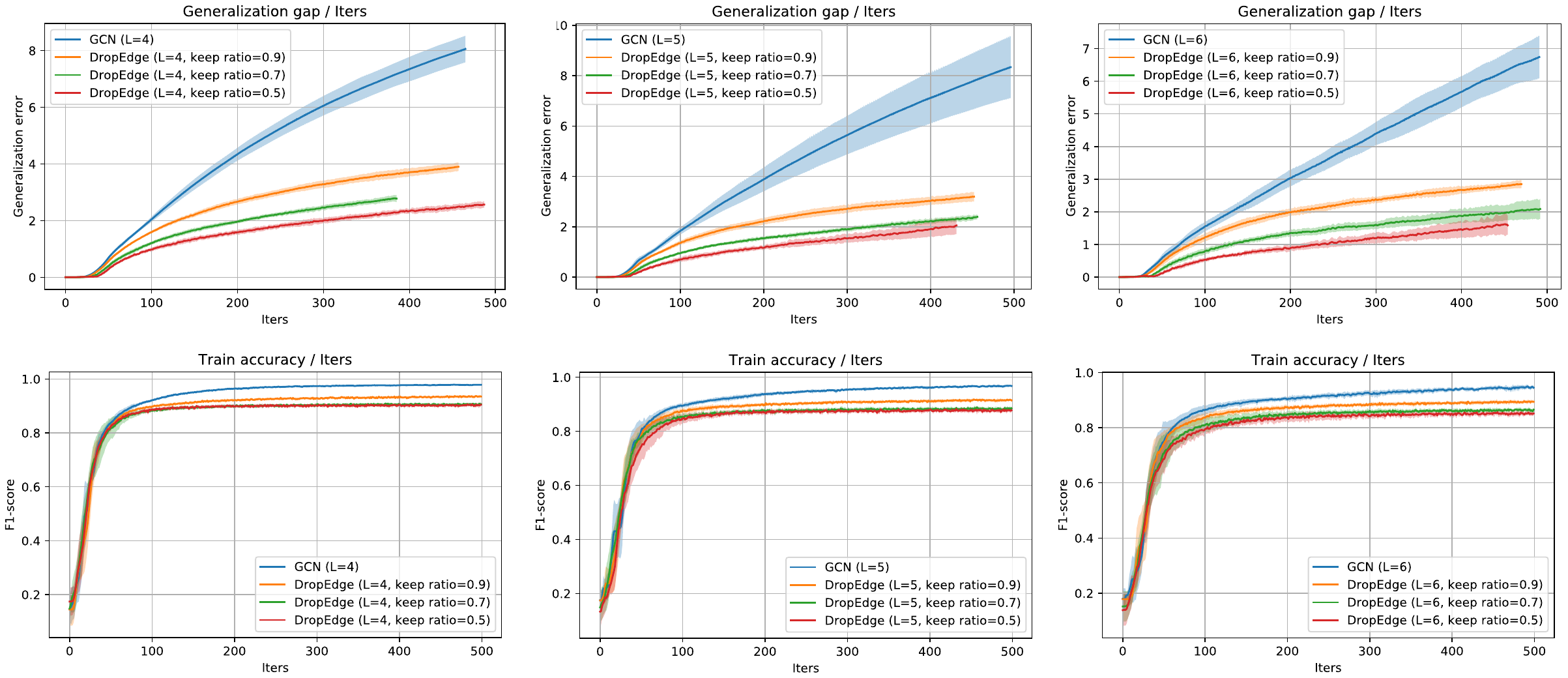}
    \caption{Comparison of generalization error of \textit{DropEdge} on \textit{Citeseer} dataset. The curve stops early at the largest training accuracy iteration.}
    \label{fig:dropedge_citeseer_456}
\end{figure}

\subsection{Effect of PairNorm on generalization} \label{supp:pairnorm}
In the following, we explore the effect of node embedding augmentation technique PairNorm~\cite{zhao2019pairnorm} on the generalization of GCNs.
PairNorm proposes to normalized node embeddings by
$\mathbf{H}^{(\ell)} = \texttt{PN}(\mathbf{H}^{(\ell)}) = \gamma \frac{\mathbf{H}^{(\ell)} - \bm{\mu}(\mathbf{H}^{(\ell)})}{\sigma(\mathbf{H}^{(\ell)})}$ where the average node embedding is computed as $ \bm{\mu}(\mathbf{H}^{(\ell)})= \frac{1}{N} \sum_{i=1}^N \mathbf{h}_i^{(\ell)}$, the variance of node embeddings is computed as $\sigma^2(\mathbf{H}^{(\ell)}) = \frac{1}{N}\sum_{i=1}^N \| \mathbf{h}_i^{(\ell)} - \bm{\mu}(\mathbf{H}^{(\ell)}) \|_2^2$, and $\gamma\geq 0$ controls the scale of node embeddings.

We choose hidden dimension as $64$, and no dropout or weight decay is used during training.
As shown in Figure~\ref{fig:pairnorm_cora_456} and Figure~\ref{fig:pairnorm_citeseer_456}, a larger scale ratio $\gamma$ can improve the discriminativeness of node embeddings, but will hurt the generalization error. 
A smaller scale ratio leads to a small generalization error, but it makes the node embeddings harder to discriminate, therefore over-smoothing happens (e.g., using $\gamma=0.1$ can be think of as creating over-smoothing effect on node embeddings).

\begin{figure}[h]
    \centering
    \includegraphics[width=1.0\textwidth]{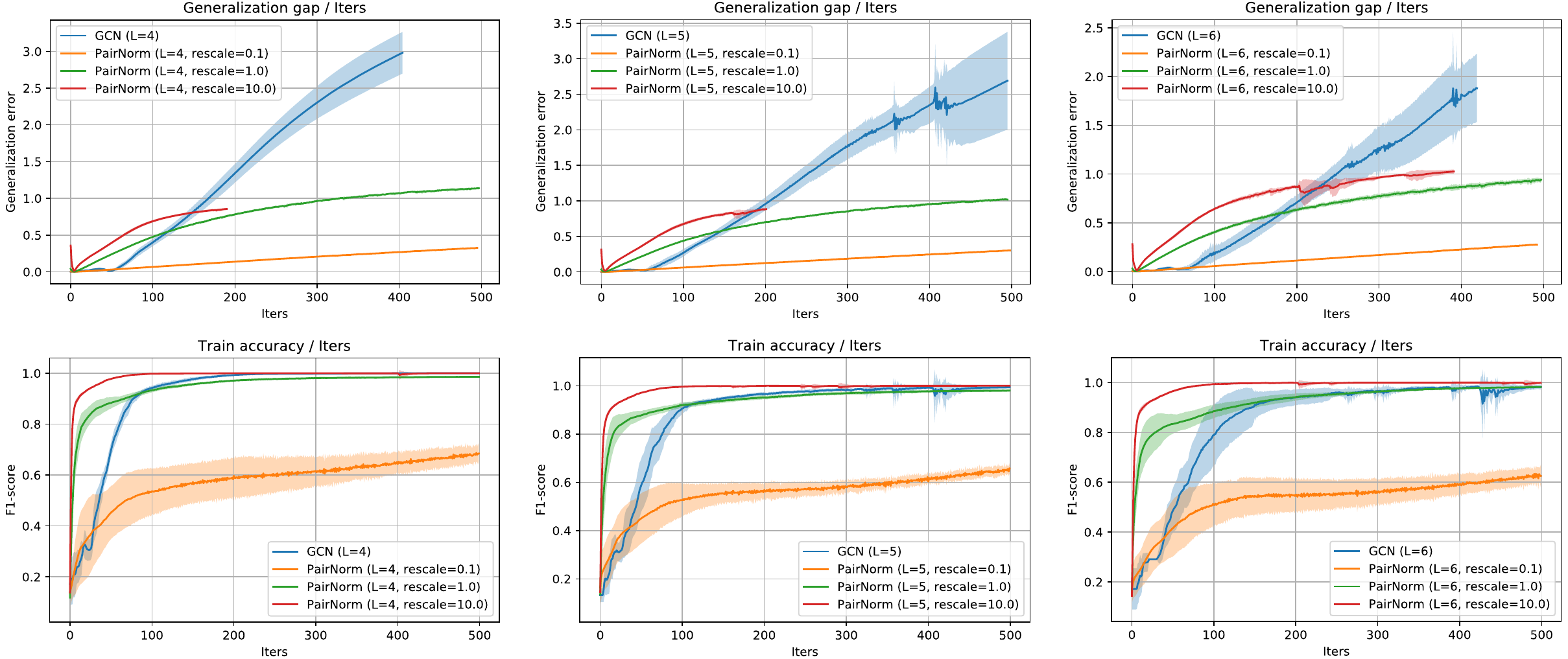}
    \caption{Comparison of generalization error of \textit{PairNorm} on \textit{Cora} dataset. The curve stops early at the largest training accuracy iteration.}
    \label{fig:pairnorm_cora_456}
\end{figure}

\begin{figure}[h]
    \centering
    \includegraphics[width=1.0\textwidth]{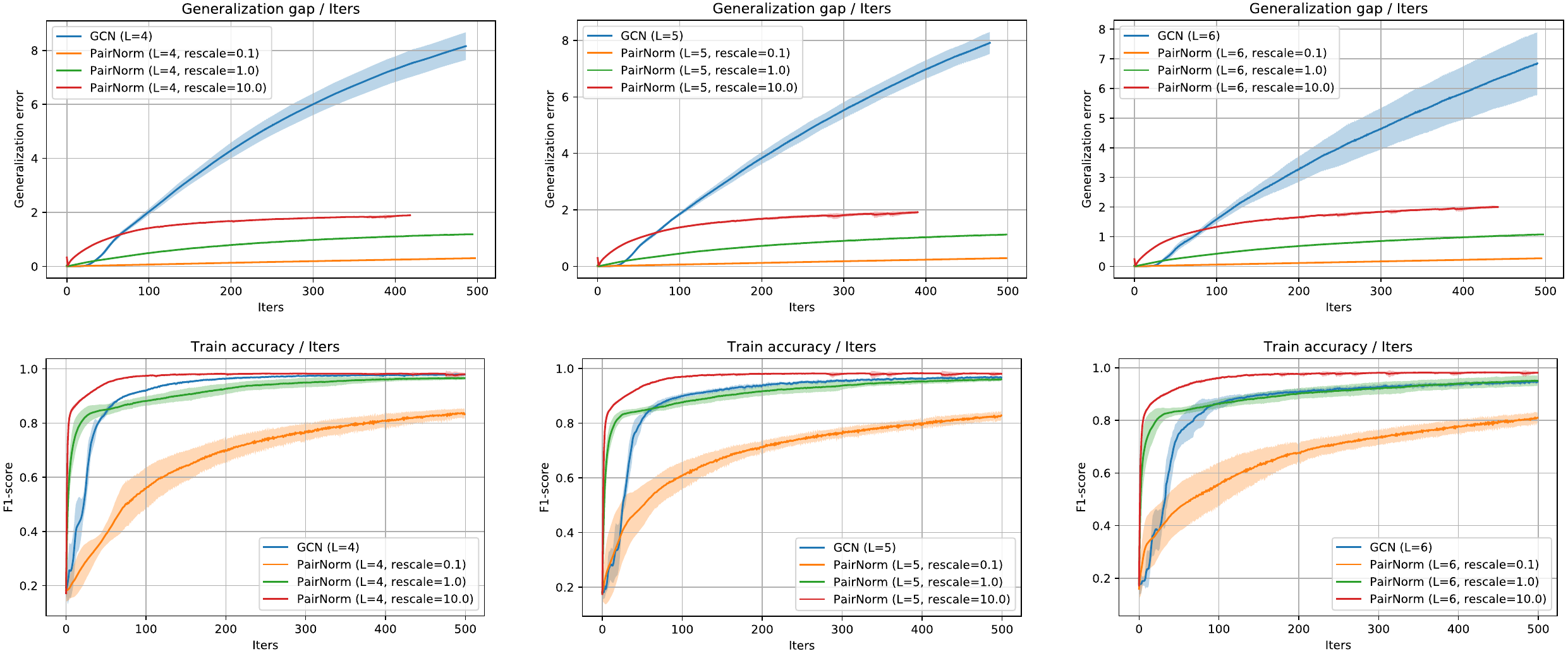}
    \caption{Comparison of generalization error of \textit{PairNorm} on \textit{Citeseer} dataset. The curve stops early at the largest training accuracy iteration.}
    \label{fig:pairnorm_citeseer_456}
\end{figure}

\subsection{Illustrating the gradient instability issue during GCN training}\label{supp:gradient_instability}

Gradient instability refers to a phenomenon that the gradient changes significantly at every iteration.
The main reason for gradient instability is because the scale of weight matrices is large, which causes the calculated gradients become large.
Notice that the gradient instability issue is more significant on vanilla GCN than other sequential GCN structures such as ResGCN and GCNII, which is one of the key factors that impacts the training phase of vanilla GCNs.

To see this, let $\mathbf{W}^{(\ell)} \leftarrow \mathbf{W}^{(\ell)} - \eta \mathbf{G}^{(\ell)}$ denote the gradient descent update of the $\ell$th layer weight matrix, where $\mathbf{G}^{(\ell)}$ is the gradient with respect to the $\ell$th layer weight matrix $\mathbf{W}^{(\ell)}$.
The upper bounds of the gradient for GCN, ResGCN, and GCNII are computed as 
\begin{equation}\label{eq:exploding_grad_gcn}
    \begin{aligned}
    \|\mathbf{G}^{(\ell)}\|_2 &= \mathcal{O} \Big( (\max\{1, \sqrt{d} B_w\})^L\Big), && \text{GCN~(Eq.~\ref{eq:gradient_norm_GCN})}\\
    \|\mathbf{G}^{(\ell)}\|_2 &= \mathcal{O} \Big( (1+\sqrt{d} B_w)^L\Big), && \text{ResGCN~(Eq.~\ref{eq:gradient_norm_resgcn})}\\
    \|\mathbf{G}^{(\ell)}\|_2 &= \mathcal{O} \Big( \beta (\max\{1, \alpha \sqrt{d} B_w\})^L\Big), && \text{GCNII~(Eq.~\ref{eq:gradient_norm_GCNII})}\\
    \end{aligned}
\end{equation}
where $d$ is the largest number of node degree, $L$ is the number of layers, and $\|\mathbf{W}^{(\ell)}\|_2 \leq B_w$ is the largest singular value of weight matrix. Details please refer to the derivative of Eq.~\ref{eq:gradient_norm_GCN}, Eq.~\ref{eq:gradient_norm_resgcn} , and Eq.~\ref{eq:gradient_norm_GCNII}.

From Eq.~\ref{eq:exploding_grad_gcn}, we know that the largest singular value of weight matrices is the key factor that affects the scale of the gradient. 
However, upper bound can be vacuous if we simply ignore the impact of network structure on $B_w$. 

\begin{figure}[h]
    \centering
    \includegraphics[width=1.0\textwidth]{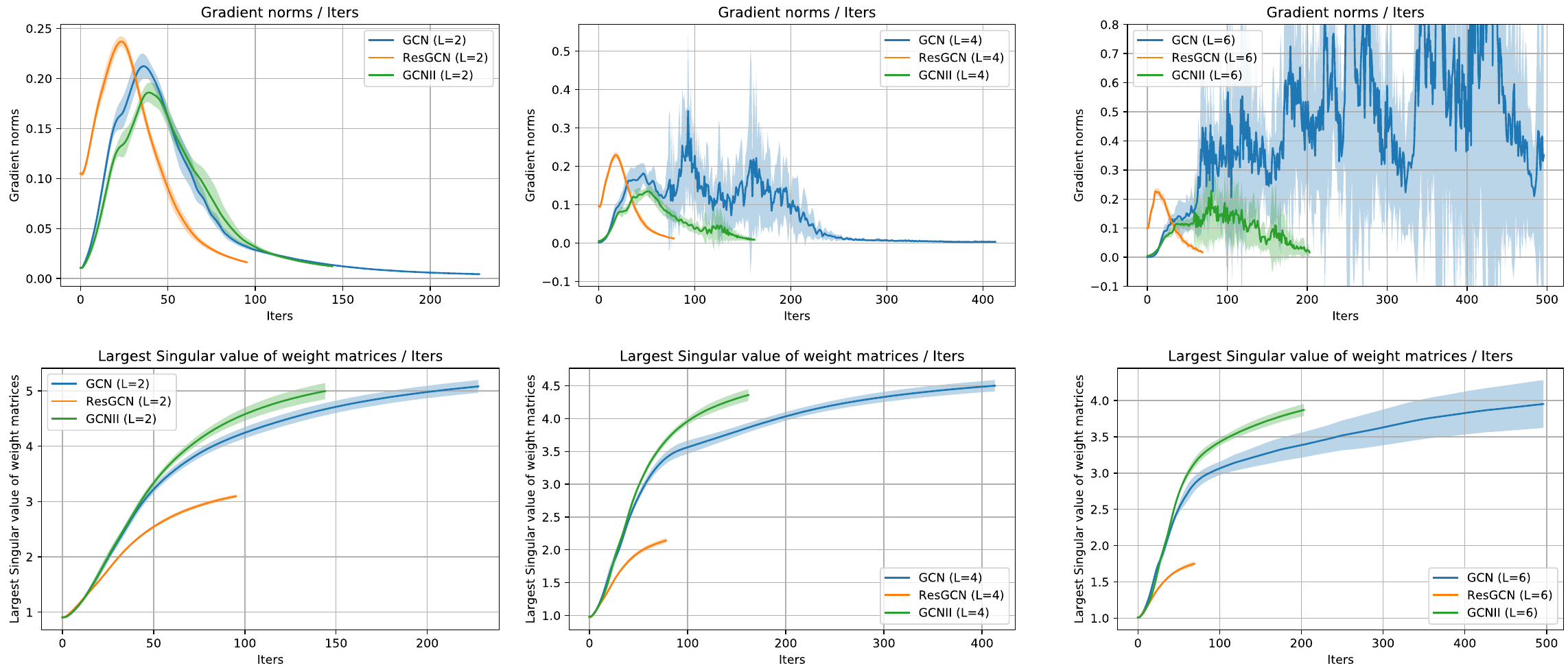}
    \caption{Comparison of gradient norm on \textit{Cora} dataset. The curve stops early at the largest training accuracy iteration.}
    \label{fig:cora_gradient_norm}
\end{figure}

\begin{figure}[h]
    \centering
    \includegraphics[width=1.0\textwidth]{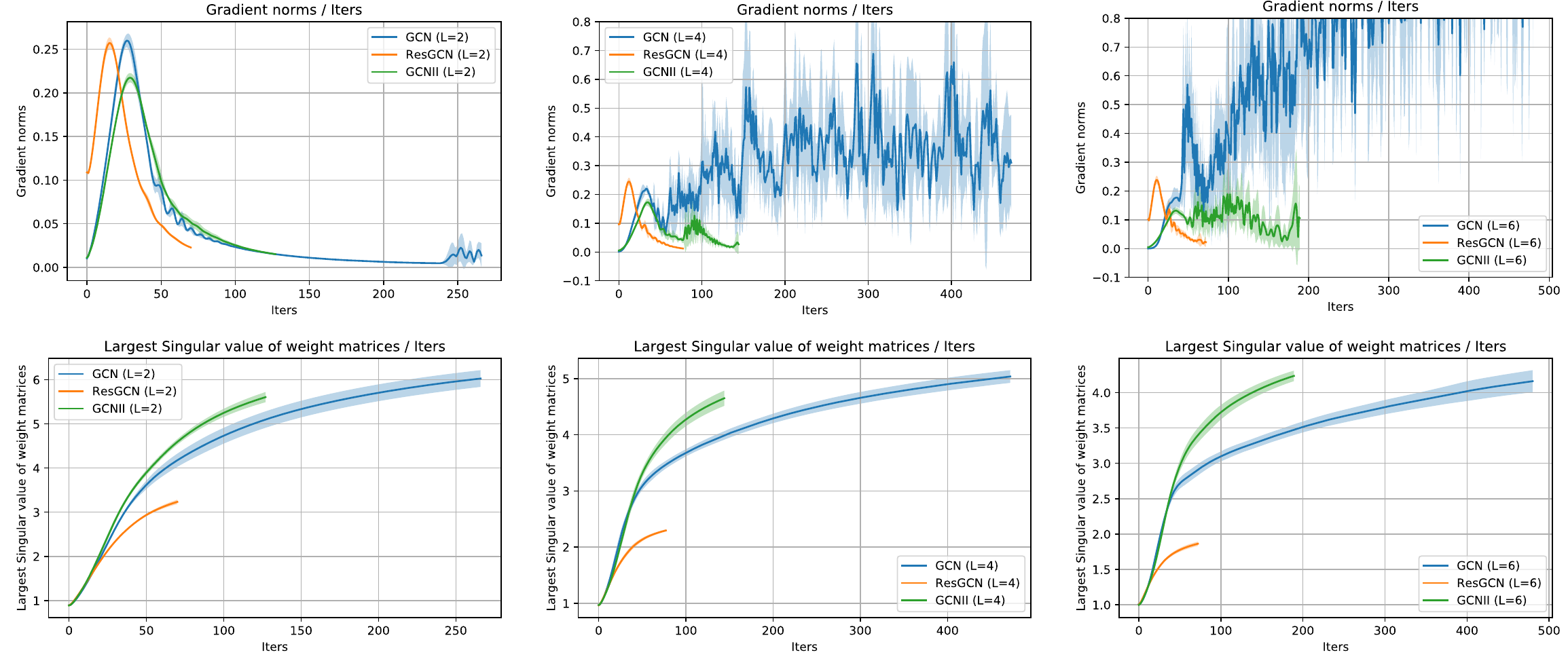}
    \caption{Comparison of gradient norm on \textit{Citeseer} dataset. The curve stops early at the largest training accuracy iteration.}
    \label{fig:citeseer_gradient_norm}
\end{figure}

From Figure~\ref{fig:cora_gradient_norm} and Figure~\ref{fig:citeseer_gradient_norm}, we can observe that the residual connection has implicit regularization on the weight matrices, which makes the weight matrices in ResGCN has a smaller largest singular values than GCN.
As a result, the ResGCN does not suffer from gradient instability even its gradient norm upper bound in Eq.~\ref{eq:exploding_grad_gcn} is larger than GCN.
Furthermore, although the largest singular value of the weight matrices for GCNII is larger than GCN, 
by selecting a small enough $\beta_\ell$, GCNII can be less impacted by gradient instability than vanilla GCN.

\subsection{Illustrating how more training leads to high training F1-score}

As a compliment to Figure~\ref{fig:train_acc_gcn_resgcn_appnp_gcnii}, we provide training and validation F1-score of the baseline models. 
During training, we chose hidden dimension as $64$, Adam optimizer with learning rate $0.001$, without any dropout or weight decay. 
Please note that removing dropout and weight decay is necessary because both operations are designed to prevent neural networks from overfitting, and will hurt the best training accuracy that a model can achieve.
As shown in Figure~\ref{fig:f1_score_iters_cora} and Figure~\ref{fig:f1_score_iters_citeseer}, all methods can achieve high training F1-score regardless the number of layers, which indicates node embeddings are distinguishable. 

\begin{figure}[h]
    \centering
    \includegraphics[width=0.8\textwidth]{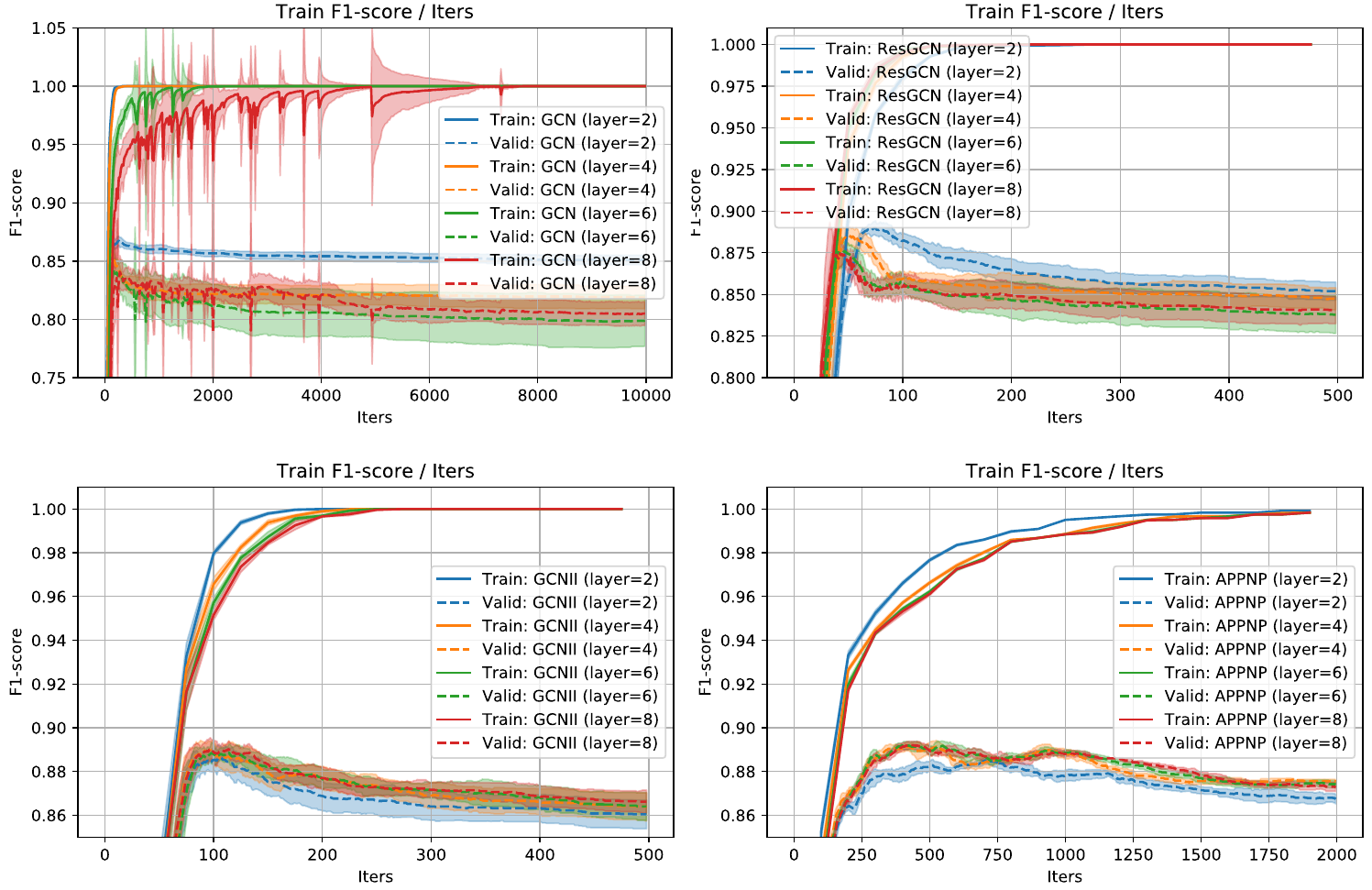}
    \caption{Comparison of training F1-score and number of iterations on \textit{Cora} dataset.}
    \label{fig:f1_score_iters_cora}
\end{figure}

\begin{figure}[h]
    \centering
    \includegraphics[width=0.8\textwidth]{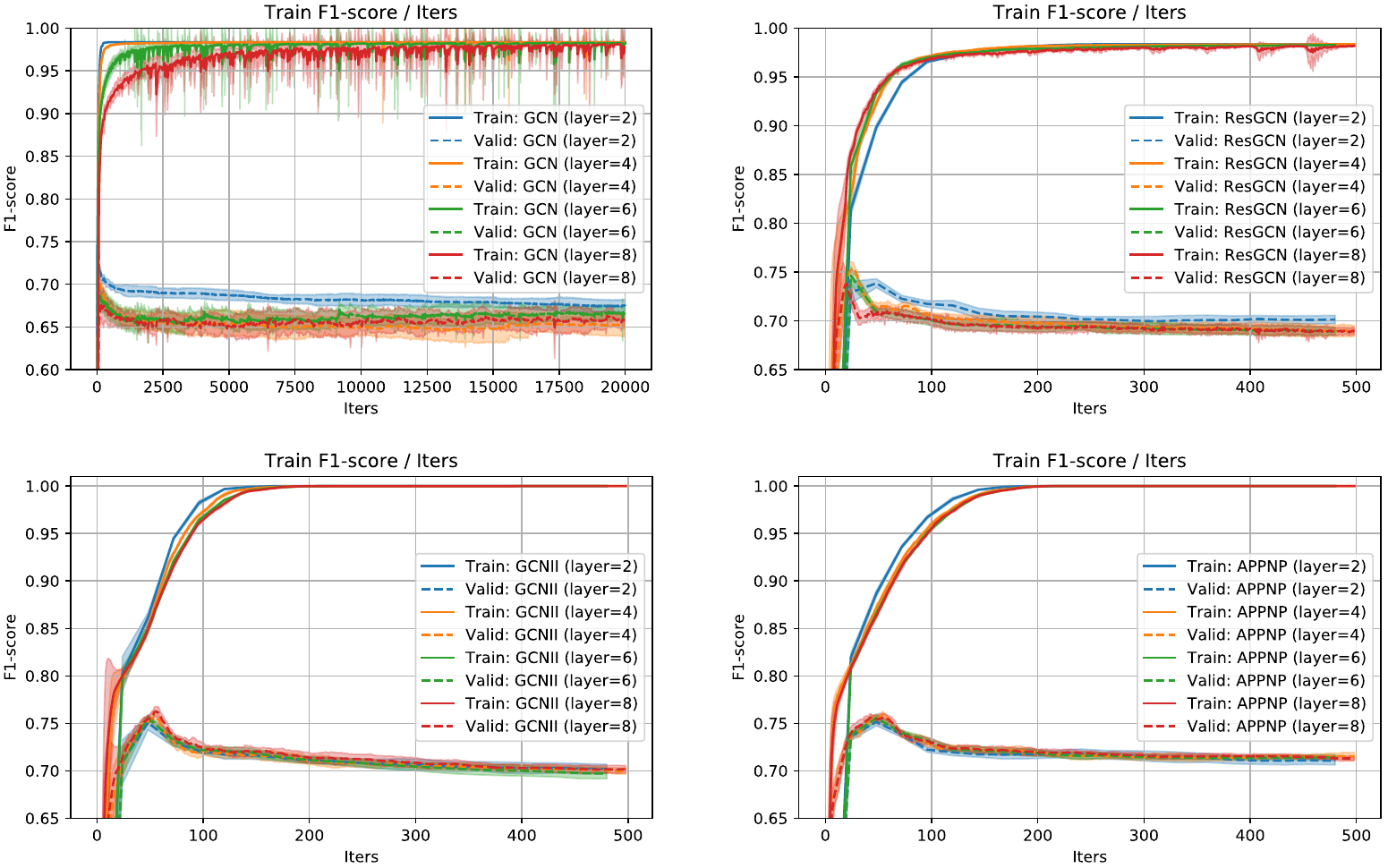}
    \caption{Comparison of training F1-score and number of iterations on \textit{Citeseer} dataset. }
    \label{fig:f1_score_iters_citeseer}
\end{figure}

\subsection{Effect of number of layers on real-word datasets}

In the following, we demonstrate the effect of the number of layers and hyper-parameters on the performance of the model on OGB Arxiv~\cite{hu2020open}.
We follow the default hyper-parameter setup of GCN on the leaderboard,\footnote{\url{https://github.com/snap-stanford/ogb/blob/master/examples/nodeproppred/arxiv/gnn.py}. 
Due to the memory limitation, we choose hidden dimension as $128$ instead of the default $256$ hidden dimension for all models.} i.e., we choose hidden dimension as $128$, dropout ratio as $0.5$, Adam optimizer with learning rate as $0.01$, and applying batch normalization after each graph convolutional layer.
As shown in Table~\ref{table:ogbn_arxiv_layers}, the number of layers and the choice of hyper-parameters can largely impact the performance of the models. 
Since DGCN can automatically adjust the $\alpha_\ell$ and $\beta_\ell$ to better adapt to the change of model depth,
it achieves a comparable and more stable performance than most baseline models.

\begin{table}[h]
\centering
\caption{Comparison of F1-score on OGB-Arxiv dataset for different number of layers}
\label{table:ogbn_arxiv_layers}
\scalebox{0.9}{
\begin{tabular}{@{}lcccccc@{}}
\toprule
\textbf{Model}  & $\alpha$  & \textbf{2 Layers}    & \textbf{4 Layers} & \textbf{8 Layers}  &  \textbf{12 Layers}    & \textbf{16 Layers}\\ \midrule
\textbf{GCN} & $-$ & $71.02\% \pm 0.14$ 	& $71.56\% \pm 0.19$ 	& $71.28\% \pm 0.33$ 	& $70.28\% \pm 0.23$ 	& $69.37\% \pm 0.46$ \\
\textbf{ResGCN} & $-$	& $70.66\% \pm 0.48$ 	& $72.41\% \pm 0.31$ 	& $72.56\% \pm 0.31$ 	& $72.46\% \pm 0.23$ 	& $72.11\% \pm 0.28$ \\
\textbf{GCNII} & $0.9$	& $71.35\% \pm 0.21$ 	& $72.57\% \pm 0.23$ 	& $72.06\% \pm 0.42$ 	& $71.31\% \pm 0.62$ 	& $69.99\% \pm 0.80$ \\
\textbf{GCNII} & $0.8$	& $71.14\% \pm 0.27$ 	& $72.32\% \pm 0.19$ 	& $71.90\% \pm 0.41$ 	& $71.21\% \pm 0.23$ 	& $70.56\% \pm 0.72$ \\
\textbf{GCNII} & $0.5$	& $70.54\% \pm 0.30$ 	& $72.09\% \pm 0.25$ 	& $71.92\% \pm 0.32$ 	& $71.24\% \pm 0.47$ 	& $71.02\% \pm 0.58$ \\
\textbf{APPNP} & $0.9$	& $67.38\% \pm 0.34$ 	& $68.02\% \pm 0.55$ 	& $66.62\% \pm 0.48$ 	& $67.43\% \pm 0.50$ 	& $67.42\% \pm 1.00$ \\
\textbf{APPNP} & $0.8$	& $66.71\% \pm 0.32$ 	& $68.25\% \pm 0.43$ 	& $66.40\% \pm 0.89$ 	& $66.51\% \pm 2.09$ 	& $66.56\% \pm 0.74$ \\
\textbf{DGCN} & $-$	& $71.21\% \pm 0.25$ 	& $72.29\% \pm 0.18$ 	& $72.39\% \pm 0.21$ 	& $\mathbf{72.63\% \pm 0.12}$ 	& $72.41\% \pm 0.07$ \\ \bottomrule
\end{tabular}
}
\end{table}

\section{Generalization bound for \textbf{GCN}}\label{supp:proof_gcn}

In this section, we provide detailed proof on the generalization bound of GCN.
Recall that the update rule of GCN is defined as
\begin{equation}
    \mathbf{H}^{(\ell)} = \sigma(\mathbf{L} \mathbf{H}^{(\ell-1)} \mathbf{W}^{(\ell)}),
\end{equation}
where $\sigma(\cdot)$ is the ReLU activation function. 
Note that although ReLU function $\sigma(x)$ is not differentiable when $x=0$, for analysis purpose we suppose the $\sigma^\prime(0) = 0$.\footnote{Widely used deep learning frameworks, including PyTorch and Tensorflow, also set the subgradient of ReLU as zero when its input is zero.}

The training of GCN is an empirical risk minimization with respect to a set of parameters $\bm{\theta} = \{ \mathbf{W}^{(1)}, \ldots, \mathbf{W}^{(L)}, \mathbf{v} \}$, i.e.,
\begin{equation}
    \mathcal{L}(\bm{\theta}) = \frac{1}{m} \sum_{i=1}^m \Phi_\gamma (-p(f(\mathbf{h}_i^{(L)}), y_i)),~
    f(\mathbf{h}_i^{(L)}) = \tilde{\sigma}(\mathbf{v}^\top \mathbf{h}_i^{(L)}),~
\end{equation}
 
where $\mathbf{h}_i^{(L)}$ is the node representation of the $i$th node at the final layer,  $f(\mathbf{h}_i^{(L)})$ is the predicted label for the $i$th node,
$\tilde{\sigma}(x) = \frac{1}{\exp(-x) + 1}$ is the sigmoid function, and loss function $\Phi_\gamma (-p(z, y))$ is $\frac{2}{\gamma}$-Lipschitz continuous with respect to its first input $z$ with $p(z, y)$ as defined in Section~\ref{section:from_stability}.
For simplification, we will use $\text{Loss}(z,y)$ which represents $\Phi_\gamma (-p(z, y))$ in the proof.

To establish the generalization of GCN as stated in Theorem~\ref{thm:uniform_stability_base}, we utilize  the following result on transductive uniform stability from~\cite{el2006stable}.

\begin{theorem} [Transductive uniform stability bound~\cite{el2006stable}]\label{thm:uniform_stability_base_supp}
Let $f$ be a $\epsilon$-uniformly stable transductive learner and $\gamma, \delta >0$, and define $Q=mu/(m+u)$. Then, with probability at least $1-\delta$ over all training and testing partitions, we have 
\begin{equation*}
    \mathcal{R}_u(f) \leq \mathcal{R}_{m}^\gamma(f) + \frac{2}{\gamma}\mathcal{O}\Big( \epsilon \sqrt{Q\ln (\delta^{-1})} \Big) +\mathcal{O}\Big(\frac{\ln(\delta^{-1})}{\sqrt{Q}}\Big).
\end{equation*}
\end{theorem}



Then, in Lemma~\ref{lemma:uniform_stable_gcn}, we derive the uniform stability constant for GCN, i.e., $\epsilon_\text{GCN}$. 

\begin{lemma}\label{lemma:uniform_stable_gcn} 
The uniform stability constant  for GCN is computed as $\epsilon_\texttt{GCN} = \frac{2 \eta  \rho_f G_f}{m} \sum_{t=1}^T (1+\eta L_F)^{t-1} $ where
\begin{equation}
    \begin{aligned}
    \rho_f &= C_1^L C_2,~
    G_f =  \frac{2}{\gamma} (L+1) C_1^L C_2,~
    L_f =  \frac{2}{\gamma} (L+1) C_1^L C_2  \Big( (L+2) C_1^L C_2 + 2 \Big), \\
    C_1 &= \max\{ 1, \sqrt{d} B_w \},~ C_2 = \sqrt{d} (1+B_x).
    \end{aligned}
\end{equation}
\end{lemma}
By plugging the result in Lemma~\ref{lemma:uniform_stable_gcn} back to Theorem~\ref{thm:uniform_stability_base_supp}, we establish the generalization bound for GCN.

The key idea of the proof is to decompose the change of the GCN output into two terms (in Lemma~\ref{lemma:universal_boound_on_model_output}) which depend on 
\begin{itemize}
    \item (Lemma~\ref{lemma:gcn_h_max_bound_1}) The maximum change of node embeddings, i.e., $\Delta h_{\max}^{(\ell)} = \max_i \| [\mathbf{H}^{(\ell)} - \tilde{\mathbf{H}}^{(\ell)}]_{i,:}\|_2$,
    \item (Lemma~\ref{lemma:gcn_delta_h_max_bound_2}) The maximum node embeddings, i.e., $h_{\max}^{(\ell)} = \max_i \|[ \mathbf{H}^{(\ell)}]_{i,:}\|_2$.
\end{itemize}

\begin{lemma} \label{lemma:universal_boound_on_model_output}
Let $f(\mathbf{h}_i^{(L)}) = \tilde{\sigma}(\mathbf{v}^\top \mathbf{h}^{(L)}_i), \tilde{f}(\tilde{\mathbf{h}}^{(L)}_i) = \tilde{\sigma}(\tilde{\mathbf{v}}^\top \tilde{\mathbf{h}}^{(L)}_i)$ denote the prediction of node $i$ using parameters $\bm{\theta} = \{ \mathbf{W}^{(1)}, \ldots, \mathbf{W}^{(L)}, \mathbf{v} \}, \tilde{\bm{\theta}} = \{ \tilde{\mathbf{W}}^{(1)}, \ldots, \tilde{\mathbf{W}}^{(L)}, \tilde{\mathbf{v}} \}$ (i.e., the two set of parameters trained on the original and the perturbed dataset) respectively. Then we have
\begin{equation}
    \begin{aligned}
    \max_i |f(\mathbf{h}_i^{(L)}) - \tilde{f}(\tilde{\mathbf{h}}_i^{(L)})| &\leq  \Delta h_{\max}^{(L)} + h_{\max}^{(L)} \| \Delta \mathbf{v} \|_2, \\
    \max_i \Big\|\frac{\partial f(\mathbf{h}_i^{(L)}) }{\partial \mathbf{h}_i^{(L)}} - \frac{\partial \tilde{f}(\tilde{\mathbf{h}}_i^{(L)})}{\partial \tilde{\mathbf{h}}_i^{(L)}} \Big\|_2 &\leq \Delta h_{\max}^{(L)} + (  h_{\max}^{(L)} + 1) \| \Delta \mathbf{v} \|_2,
    \end{aligned}
\end{equation}
where $\Delta \mathbf{v} = \mathbf{v} - \tilde{\mathbf{v}}$.
\end{lemma}

\begin{lemma} [Upper bound of $h_{\max}^{(\ell)}$ for \textbf{GCN}] \label{lemma:gcn_h_max_bound_1} 
Let suppose Assumption~\ref{assumption:norm_bound} hold. Then, the maximum node embeddings for any node at the $\ell$th layer is bounded by
\begin{equation}
    h_{\max}^{(\ell)} 
    \leq  B_x (\max\{1, \sqrt{d} B_w\})^\ell.
\end{equation}
\end{lemma}
\begin{lemma} [Upper bound of $\Delta h_{\max}^{(\ell)}$ for \textbf{GCN}] \label{lemma:gcn_delta_h_max_bound_2}
Let suppose Assumption~\ref{assumption:norm_bound} hold. 
Then, the maximum change between the node embeddings  on two different set of weight parameters for any node at the $\ell$th layer is bounded by
\begin{equation}
    \begin{aligned}
    \Delta h_{\max}^{(\ell)} 
    &\leq \sqrt{d} B_x (\max\{1,  \sqrt{d} B_w\} )^{\ell-1} (\|\Delta \mathbf{W}^{(1)}\|_2 + \ldots + \|\Delta \mathbf{W}^{(\ell)}\|_2),
    \end{aligned}
\end{equation}
where $\Delta \mathbf{W}^{(\ell)} = \mathbf{W}^{(\ell)} - \tilde{\mathbf{W}}^{(\ell)}$.
\end{lemma}

Besides, in Lemma~\ref{lemma:GCN_delta_z_5}, we derive the upper bound on the maximum change of node embeddings before the activation function, i.e., $\Delta z_{\max}^{(\ell)} = \max_i \| [\mathbf{Z}^{(\ell)} - \tilde{\mathbf{Z}}^{(\ell)}]_{i,:}\|_2 $. $\Delta z_{\max}^{(\ell)}$ will be used in the computation of the gradient related upper bounds.
\begin{lemma} [Upper bound of $\Delta z_{\max}^{(\ell)}$ for \textbf{GCN}] \label{lemma:GCN_delta_z_5}
Let suppose Assumption~\ref{assumption:norm_bound} hold. 
Then, the maximum change between the node embeddings before the activation function  on two different set of weight parameters for any node at the $\ell$th layer is bounded by
\begin{equation}
    \Delta z_{\max}^{(\ell)} \leq \sqrt{d} B_x (\max\{1, \sqrt{d} B_w\})^{\ell-1} \Big( \| \Delta \mathbf{W}^{(1)} \|_2 + \ldots + \| \Delta \mathbf{W}^{(\ell)} \|_2\Big),
\end{equation}
where $\Delta \mathbf{W}^{(\ell)} = \mathbf{W}^{(\ell)} - \tilde{\mathbf{W}}^{(\ell)}$.
\end{lemma}

Then, in Lemma~\ref{lemma:gcn_d_max_bound_3}, we decompose the change of the model parameters into two terms which depend on
\begin{itemize}
    \item The maximum change of gradient passing from the $(\ell+1)$th layer to the $\ell$th layer $\Delta d_{\max}^{(\ell)} = \max_i \Big\| \Big[\frac{\partial \sigma(\mathbf{L} \mathbf{H}^{(\ell-1)} \mathbf{W}^{(\ell)})}{\partial \mathbf{H}^{(\ell-1)}} - \frac{\partial \sigma(\mathbf{L} \tilde{\mathbf{H}}^{(\ell-1)} \tilde{\mathbf{W}}^{(\ell)})}{\partial \tilde{\mathbf{H}}^{(\ell-1)}} \Big]_{i,:} \Big\|_2$,
    \item The maximum gradient passing from the $(\ell+1)$th layer to the $\ell$th layer $d_{\max}^{(\ell)} = \max_i \Big\| \Big[\frac{\partial \sigma(\mathbf{L} \mathbf{H}^{(\ell-1)} \mathbf{W}^{(\ell)})}{\partial \mathbf{H}^{(\ell-1)}} \Big]_{i,:} \Big\|_2$.
\end{itemize}

\begin{lemma} [Upper bound of $d_{\max}^{(\ell)}, \Delta d_{\max}^{(\ell)}$ for \textbf{GCN}] \label{lemma:gcn_d_max_bound_3}
Let suppose Assumption~\ref{assumption:norm_bound} hold. 
Then, the upper bound on the maximum gradient passing from layer $\ell+1$ to layer $\ell$  for any node is
\begin{equation}
    d_{\max}^{(\ell)} \leq \frac{2}{\gamma} (\max\{1, \sqrt{d} B_w\})^{L-\ell+1},
\end{equation}
and the upper bound on the maximum change between the gradient passing from layer $\ell+1$ to layer $\ell$  on two different set of weight parameters for any node is
\begin{equation}
    \begin{aligned}
    \Delta d_{\max}^{(\ell)} &\leq  (\max\{1, \sqrt{d} B_w\} )^{L-\ell} \frac{2}{\gamma} \Big(  (L+1) \sqrt{d} (B_x+1) (\max\{1, \sqrt{d} B_w\} )^L  + 1  \Big) \|\Delta \bm{\theta}\|_2,
    \end{aligned}
\end{equation}
where $\|\Delta \bm{\theta}\|_2 = \|\mathbf{v} - \tilde{\mathbf{v}}\|_2 + \sum_{\ell-1}^L \| \mathbf{W}^{(\ell)} - \tilde{\mathbf{W}}^{(\ell)}\|_2$ denotes the change of two set of parameters.
\end{lemma} 

Finally, based on the previous result, in Lemma~\ref{lemma:gcn_G_bound_4}, we decompose the change of the model parameters into two terms which depend on
\begin{itemize}
    \item The change of gradient with respect to the $\ell$th layer weight parameters $\|\Delta \mathbf{G}^{(\ell)}\|_2$,
    \item The gradient with respect to the $\ell$th layer weight parameters $\|\mathbf{G}^{(\ell)} \|_2$,
\end{itemize}
where $\mathbf{G}^{(L+1)}$ denotes the gradient with respect to the weight $\mathbf{v}$ of the binary classifier and $\mathbf{G}^{(\ell)}$ denotes the gradient with respect to the weight $\mathbf{W}^{(\ell)}$ of the $\ell$th graph convolutional layer.
Notice that $\|\Delta \mathbf{G}^{(\ell)}\|_2$ reflects the smoothness of GCN model and $\|\mathbf{G}^{(\ell)} \|_2$ corresponds to the upper bound of gradient.

\begin{lemma} [Upper bound of $\| \mathbf{G}^{(\ell)} \|_2, \|\Delta \mathbf{G}^{(\ell)} \|_2$  for \textbf{GCN}] \label{lemma:gcn_G_bound_4}
Let suppose Assumption~\ref{assumption:norm_bound} hold and let $C_1 = \max\{1, \sqrt{d}B_w\}$ and $C_2 = \sqrt{d}(B_x+1)$.
Then, the gradient and the maximum change between gradients computed on two different set of weight parameters are bounded by
\begin{equation}
    \begin{aligned}
    \sum_{\ell=1}^{L+1} \| \mathbf{G}^{(\ell)} \|_2
    &\leq  \frac{2}{\gamma} (L+1) C_1^L C_2 \|\Delta \bm{\theta}\|_2, \\ 
    \sum_{\ell=1}^{L+1} \|\Delta \mathbf{G}^{(\ell)} \|_2 
    &\leq  \frac{2}{\gamma} (L+1) C_1^L C_2  \Big( (L+2) C_1^L C_2 + 2 \Big) \|\Delta \bm{\theta}\|_2,
    \end{aligned}
\end{equation}
where $\|\Delta \bm{\theta}\|_2 = \|\mathbf{v} - \tilde{\mathbf{v}}\|_2 + \sum_{\ell-1}^L \| \mathbf{W}^{(\ell)} - \tilde{\mathbf{W}}^{(\ell)}\|_2$ denotes the change of two set of parameters.
\end{lemma}

Equipped with above intermediate results, we now proceed to prove  Lemma~\ref{lemma:uniform_stable_gcn}.
\begin{proof}[Proof of Lemma~\ref{lemma:uniform_stable_gcn}]
Recall that our goal is to explore the impact of different GCN structures on the uniform stability constant $\epsilon_\text{GCN}$, which is a function of $\rho_f$, $G_f$, and $L_f$. 
Let $C_1 = \max\{1, \sqrt{d}B_w\}$ and $C_2 = \sqrt{d}(B_x+1)$. Firstly, by plugging Lemma~\ref{lemma:gcn_delta_h_max_bound_2} and Lemma~\ref{lemma:gcn_h_max_bound_1} into Lemma~\ref{lemma:universal_boound_on_model_output}, we have
\begin{equation}
    \begin{aligned}
    \max_i |f(\mathbf{h}_i^{(L)}) - \tilde{f}(\tilde{\mathbf{h}}_i^{(L)})| &\leq \Delta h_{\max}^{(L)} + h_{\max}^{(L)} \| \Delta \mathbf{v} \|_2 \\
    &\leq \sqrt{d} B_x \cdot (\max\{1, \sqrt{d} B_w\} )^{L} \| \Delta \bm{\theta} \|_2 \\
    &\leq C_1^L C_2 \| \Delta \bm{\theta} \|_2.
    \end{aligned}
\end{equation}
Therefore, we know that the function $f$ is $\rho_f$-Lipschitz continuous, with $\rho_f = C_1^L C_2 $.
Then, by Lemma~\ref{lemma:gcn_G_bound_4}, we know that the function $f$ is $L_f$-smoothness, and the gradient of each weight matrix is bounded by $G_f$, with
\begin{equation}
    G_f =  \frac{2}{\gamma} (L+1) C_1^L C_2,~
    L_f =  \frac{2}{\gamma} (L+1) C_1^L C_2  \Big( (L+2) C_1^L C_2 + 2 \Big).
\end{equation}


By plugging $\epsilon_\text{GCN}$ into Theorem~\ref{thm:uniform_stability_base}, we obtain the generalization bound of GCN.
\end{proof}

\subsection{Proof of Lemma~\ref{lemma:universal_boound_on_model_output}}
By the definition of $f(\mathbf{h}_i^{(L)})$ and $\tilde{f}(\tilde{\mathbf{h}}_i^{(L)})$, we have
\begin{equation}
    \begin{aligned}
    \max_i |f(\mathbf{h}_i^{(L)}) - \tilde{f}(\tilde{\mathbf{h}}_i^{(L)})|
    &= \max_i |\tilde{\sigma}(\mathbf{v}^\top \mathbf{h}_i^{(L)} ) - \tilde{\sigma}(\tilde{\mathbf{v}}^\top \tilde{\mathbf{h}}_i^{(L)} )| \\
    &\underset{(a)}{\leq} \max_i \| \mathbf{v}^\top \mathbf{h}_i^{(L)} - \tilde{\mathbf{v}}^\top \tilde{\mathbf{h}}_i^{(L)} \|_2 \\
    &\leq \max_i \| \mathbf{v}^\top (\mathbf{h}_i^{(L)} - \tilde{\mathbf{h}}_i^{(L)} ) \|_2 + \max_i \| \tilde{\mathbf{h}}_i^{(L)} (\mathbf{v} - \tilde{\mathbf{v}})\|_2 \\
    &\leq  \Delta h_{\max}^{(L)} + h_{\max}^{(L)} \Delta \mathbf{v},
    \end{aligned}
\end{equation}
where $(a)$ is due to the fact that sigmoid function is $1$-Lipschitz continuous.

\begin{equation}
    \begin{aligned}
    \max_i \left\|\frac{\partial f(\mathbf{h}_i^{(L)}) }{\partial \mathbf{h}_i^{(L)}} - \frac{\partial \tilde{f}(\tilde{\mathbf{h}}_i^{(L)})}{\partial \tilde{\mathbf{h}}_i^{(L)}} \right\|_2 
    &= \max_i \| \tilde{\sigma}^\prime( \mathbf{v}^\top \mathbf{h}_i^{(L)}) \mathbf{v}^\top - \tilde{\sigma}^\prime( \tilde{\mathbf{v}}^\top \tilde{\mathbf{h}}_i^{(L)}) \tilde{\mathbf{v}}^\top \|_2\\
    &\leq \max_i \| \tilde{\sigma}^\prime( \mathbf{v}^\top \mathbf{h}_i^{(L)}) \mathbf{v}^\top - \tilde{\sigma}^\prime( \mathbf{v}^\top \mathbf{h}_i^{(L)}) \tilde{\mathbf{v}}^\top \|_2 \\
    &\quad + \max_i \| \tilde{\sigma}^\prime( \mathbf{v}^\top \mathbf{h}_i^{(L)}) \tilde{\mathbf{v}}^\top  - \tilde{\sigma}^\prime( \tilde{\mathbf{v}}^\top \tilde{\mathbf{h}}_i^{(L)}) \tilde{\mathbf{v}}^\top\|_2 \\
    &\underset{(a)}{\leq} \Delta \mathbf{v} + \Big( \Delta h_{\max}^{(L)} + h_{\max}^{(L)} \Delta \mathbf{v} \Big) \\
    &= \Delta h_{\max}^{(L)} + (  h_{\max}^{(L)} + 1)\Delta \mathbf{v},
    \end{aligned}
\end{equation}
where $(a)$ is due to the fact that sigmoid function and its gradient are $1$-Lipschitz continuous.

\subsection{Proof of Lemma~\ref{lemma:gcn_h_max_bound_1}}
By the definition of $h_{\max}^{(\ell)}$, we have
\begin{equation}
    \begin{aligned}
    h_{\max}^{(\ell)} 
    &= \max_i \| [\sigma(\mathbf{L} \mathbf{H}^{(\ell-1)} \mathbf{W}^{(\ell)}) ]_{i,:}\|_2 \\
    &\underset{(a)}{\leq}  \max_i \| [\mathbf{L} \mathbf{H}^{(\ell-1)} \mathbf{W}^{(\ell)} ]_{i,:}\|_2 \\
    &\leq \max_i \| [ \mathbf{L} \mathbf{H}^{(\ell-1)} ]_{i,:}\|_2 \|\mathbf{W}^{(\ell)} \|_2 \\
    &=  \max_i \left\| \sum_{j=1}^N L_{i,j} \mathbf{h}_j^{(\ell-1)} \right\|_2 \|\mathbf{W}^{(\ell)} \|_2 \\
    &\leq \max_i \|\sum_{j=1}^N L_{i,j} \|_2  \cdot \max_j \| \mathbf{h}_j^{(\ell-1)} \|_2 \cdot \|\mathbf{W}^{(\ell)} \|_2 \\
    &\underset{(b)}{\leq} \sqrt{d} \|\mathbf{W}^{(\ell)} \|_2 \cdot h_{\max}^{(\ell-1)} \\
    &\leq \sqrt{d} B_w \cdot h_{\max}^{(\ell-1)} \leq (\max\{ 1, \sqrt{d} B_w))^\ell B_x,
    \end{aligned}
\end{equation}
where $(a)$ is due to $\| \sigma(\mathbf{x}) \|_2 \leq \|\mathbf{x}\|_2$ and $(b)$ is due to Lemma~\ref{lemma:laplacian_1_norm_bound}.

\subsection{Proof of Lemma~\ref{lemma:gcn_delta_h_max_bound_2}}
By the definition of $\Delta h_{\max}^{(\ell)}$, we have
\begin{equation}
    \begin{aligned}
    \Delta h_{\max}^{(\ell)}
    &= \max_i \| \mathbf{h}_i^{(\ell)} - \tilde{\mathbf{h}}_i^{(\ell)}\|_2 \\
    &= \max_i \| [\sigma(\mathbf{L} \mathbf{H}^{(\ell-1)} \mathbf{W}^{(\ell)}) - \sigma(\mathbf{L} \tilde{\mathbf{H}}^{(\ell-1)} \tilde{\mathbf{W}}^{(\ell)}) ]_{i,:} \|_2 \\
    &= \max_i \| [ \mathbf{L} \mathbf{H}^{(\ell-1)} \mathbf{W}^{(\ell)} - \mathbf{L} \tilde{\mathbf{H}}^{(\ell-1)} \tilde{\mathbf{W}}^{(\ell)} ]_{i,:}\|_2 \\
    &\leq \max_i \| [ \mathbf{L} \mathbf{H}^{(\ell-1)} (\mathbf{W}^{(\ell)} - \tilde{\mathbf{W}}^{(\ell)}) + \mathbf{L} (\mathbf{H}^{(\ell-1)} - \tilde{\mathbf{H}}^{(\ell-1)}) \tilde{\mathbf{W}}^{(\ell)} ]_{i,:}\|_2 \\
    &\leq \max_i \left\|\sum_{j=1}^N L_{i,j} \mathbf{h}_j^{(\ell-1)} \right\|_2 \|\Delta \mathbf{W}^{(\ell)}\|_2 + \max_i \left\| \sum_{j=1}^N L_{i,j} \Delta \mathbf{h}_j^{(\ell-1)} \right\|_2 \|\tilde{\mathbf{W}}^{(\ell)}\|_2 \\
    &\leq \sqrt{d} \Delta h_{\max}^{(\ell-1)} \|\tilde{\mathbf{W}}^{(\ell)}\|_2 + \sqrt{d} h_{\max}^{(\ell-1)} \|\Delta\mathbf{W}^{(\ell)}\|_2 \\
    &\leq \sqrt{d}B_w \Delta h_{\max}^{(\ell-1)} + \sqrt{d} h_{\max}^{(\ell-1)} \|\Delta\mathbf{W}^{(\ell)}\|_2.
    \end{aligned}
\end{equation}

By induction, we have
\begin{equation}
    \begin{aligned}
    \Delta h_{\max}^{(\ell)} 
    &\leq \sqrt{d} B_w \Delta h_{\max}^{(\ell-1)} + \sqrt{d} h_{\max}^{(\ell-1)} \| \Delta \mathbf{W}^{(\ell)}\|_2 \\
    &\leq (\sqrt{d} B_w)^{2} \Delta h_{\max}^{(\ell-2)} + \sqrt{d} \Big( h_{\max}^{(\ell-1)} \| \Delta \mathbf{W}^{(\ell)}\|_2 +  (\sqrt{d} B_w) h_{\max}^{(\ell-2)} \| \Delta \mathbf{W}^{(\ell-1)}\|_2 \Big) \\
    &\ldots\\
    &\leq (\sqrt{d} B_w)^{\ell} \Delta h_{\max}^{(0)} + \sqrt{d}\Big( h_{\max}^{(\ell-1)} \| \Delta \mathbf{W}^{(\ell)}\|_2 +  (\sqrt{d} B_w) h_{\max}^{(\ell-2)} \| \Delta \mathbf{W}^{(\ell-1)}\|_2 +\ldots + (\sqrt{d} B_w)^{\ell-1} h_{\max}^{(0)} \| \Delta \mathbf{W}^{(1)}\|_2 \Big) \\
    &\underset{(a)}{=} \sqrt{d}\Big( h_{\max}^{(\ell-1)} \| \Delta \mathbf{W}^{(\ell)}\|_2 +  (\sqrt{d} B_w) h_{\max}^{(\ell-2)} \| \Delta \mathbf{W}^{(\ell-1)}\|_2 +\ldots + (\sqrt{d} B_w)^{\ell-1} h_{\max}^{(0)} \| \Delta \mathbf{W}^{(1)}\|_2 \Big),
    \end{aligned}
\end{equation}
where $(a)$ is due to $\Delta h_{\max}^{(0)} = 0$.
Plugging in the upper bound of $h_{\max}^{(\ell)}$ in Lemma~\ref{lemma:gcn_h_max_bound_1}, yields
\begin{equation}
    \begin{aligned}
    \Delta h_{\max}^{(\ell)} 
    &\leq \sqrt{d} B_x(\sqrt{d} B_w)^{\ell-1} \Big( \| \Delta \mathbf{W}^{(\ell)}\|_2 + \ldots +  \| \Delta \mathbf{W}^{(1)}\|_2 ) \\
    &\leq \sqrt{d} B_x(\max\{1, \sqrt{d} B_w\} )^{\ell-1} \Big( \| \Delta \mathbf{W}^{(\ell)}\|_2 + \ldots +  \| \Delta \mathbf{W}^{(1)}\|_2 ). 
    \end{aligned}
\end{equation}

\subsection{Proof of Lemma~\ref{lemma:GCN_delta_z_5}}
By the definition of $\mathbf{Z}^{(\ell)}$, we have
\begin{equation}
    \begin{aligned}
    \mathbf{Z}^{(\ell)} - \Tilde{\mathbf{Z}}^{(\ell)} 
    &= \mathbf{L} \mathbf{H}^{(\ell-1)} \mathbf{W}^{(\ell)} -   \mathbf{L} \Tilde{\mathbf{H}}^{(\ell-1)} \Tilde{\mathbf{W}}^{(\ell)}  \\
    &= \mathbf{L} \Big( \mathbf{H}^{(\ell-1)} \mathbf{W}^{(\ell)} - \Tilde{\mathbf{H}}^{(\ell-1)} \mathbf{W}^{(\ell)} + \Tilde{\mathbf{H}}^{(\ell-1)} \mathbf{W}^{(\ell)} - \Tilde{\mathbf{H}}^{(\ell-1)} \Tilde{\mathbf{W}}^{(\ell)} \Big) \\
    &= \mathbf{L} \big( \mathbf{H}^{(\ell-1)}  - \Tilde{\mathbf{H}}^{(\ell-1)} \big) \mathbf{W}^{(\ell)}  + \mathbf{L} \Tilde{\mathbf{H}}^{(\ell-1)} \big( \mathbf{W}^{(\ell)}  - \tilde{\mathbf{W}}^{(\ell)} \big).
    \end{aligned}
\end{equation}

By taking norm on the both side of the equation, we have
\begin{equation}
    \begin{aligned}
    \Delta z_{\max}^{(\ell)}
    &= \max_i \| [\mathbf{Z}^{(\ell)} - \Tilde{\mathbf{Z}}^{(\ell)}]_{i,:} \|_2 \\
    &\leq \sqrt{d} B_w \cdot \max_i \|[\mathbf{Z}^{(\ell-1)} - \Tilde{\mathbf{Z}}^{(\ell-1)}]_{i,:}  \|_2 + \sqrt{d} h_{\max}^{(\ell-1)} \| \Delta \mathbf{W}^{(\ell)} \|_2 \\
    &= \sqrt{d} B_w \cdot z_{\max}^{(\ell-1)} + \sqrt{d} h_{\max}^{(\ell-1)} \| \Delta \mathbf{W}^{(\ell)} \|_2. \\
    \end{aligned}
\end{equation}

By induction, we have
\begin{equation}
    \begin{aligned}
    \Delta z_{\max}^{(\ell)} 
    &\leq \sqrt{d} B_w \cdot \Delta z_{\max}^{(\ell-1)} + \sqrt{d} h_{\max}^{(\ell-1)} \| \Delta \mathbf{W}^{(\ell)} \|_2 \\
    &\underset{(a)}{\leq} \sqrt{d} B_w \cdot \Delta z_{\max}^{(\ell-1)} + \sqrt{d} B_x (\sqrt{d} B_w)^{\ell-1} \| \Delta \mathbf{W}^{(\ell)} \|_2 \\
    &\leq \sqrt{d} B_w \Big( \sqrt{d} B_w \cdot \Delta z_{\max}^{(\ell-2)} + \sqrt{d} B_x (\sqrt{d} B_w)^{\ell-2} \| \Delta \mathbf{W}^{(\ell-1)} \|_2 \Big) + \sqrt{d} B_x (\sqrt{d} B_w)^{\ell-1} \| \Delta \mathbf{W}^{(\ell)} \|_2 \\
    &=(\sqrt{d} B_w)^2  \cdot \Delta z_{\max}^{(\ell-2)} + \sqrt{d} B_x (\sqrt{d} B_w)^{\ell-1} \Big( \| \Delta \mathbf{W}^{(\ell-1)} \|_2 + \| \Delta \mathbf{W}^{(\ell)} \|_2\Big) \\
    &\leq \ldots \\
    &\leq \sqrt{d} B_x (\sqrt{d} B_w)^{\ell-1} \Big( \| \Delta \mathbf{W}^{(1)} \|_2 + \ldots + \| \Delta \mathbf{W}^{(\ell)} \|_2\Big),
    \end{aligned}
\end{equation}
where (a) is due to $h_{\max}^{(\ell-1)} \leq (\sqrt{d} B_w)^{\ell-1} B_x$.

\subsection{Proof of Lemma~\ref{lemma:gcn_d_max_bound_3}}
For notation simplicity, let $\mathbf{D}^{(\ell)}$ denote the gradient passing from $\ell$th to $(\ell-1)$th layer.
By the definition of $d_{\max}^{(\ell)}$, we have
\begin{equation}
    \begin{aligned}
    d_{\max}^{(\ell)} 
    &= \max_i \|\mathbf{d}_i^{(\ell)}\|_2 \\
    &= \max_i \| [\mathbf{L}^\top \sigma^\prime(\mathbf{Z}^{(\ell)}) \odot \mathbf{D}^{(\ell+1)} \mathbf{W}^{(\ell)}  ]_{i,:} \|_2  \\
    &\underset{(a)}{\leq} \max_i \| [ \mathbf{L}^\top \mathbf{D}^{(\ell+1)} ]_{i,:} \|_2 \|\mathbf{W}^{(\ell)}\|_2 \\
    &\leq \max_i \left\|\sum_{j=1}^N L_{i,j}^\top \mathbf{d}_j^{(\ell+1)} \right\|_2 \|\mathbf{W}^{(\ell)}\|_2 \\
    &\leq \sqrt{d} \|\mathbf{W}^{(\ell)}\|_2 \cdot d_{\max}^{(\ell+1)} \\
    &\leq \sqrt{d} B_w \cdot d_{\max}^{(\ell+1)} \\
    &\underset{(b)}{\leq} \frac{2}{\gamma} (\sqrt{d} B_w)^{L-\ell+1} \leq \frac{2}{\gamma} (\max\{1, \sqrt{d} B_w\} )^{L-\ell+1}
    \end{aligned}
\end{equation}
where $(a)$ is due to each element in $\sigma^\prime(\mathbf{Z}^{(\ell)})$ is either $0$ or $1$ depending on $\mathbf{Z}^{(\ell)}$, and $(b)$ is due to $\| \frac{\partial \text{Loss}(f(\mathbf{h}_i^{(L)}), y_i)}{\partial \mathbf{h}_i^{(L)}} \|_2 \leq 2/\gamma$.
By taking norm on the both sides, we have
\begin{equation}
    \begin{aligned}
    \Delta d_{\max}^{(\ell)} 
    &= \max_i \| [\mathbf{D}^{(\ell)} - \tilde{\mathbf{D}}^{(\ell)}]_{i,:} \|_2 \\
    &= \max_i \| [\mathbf{L}^\top (\mathbf{D}^{(\ell+1)} \odot \sigma^\prime(\mathbf{Z}^{(\ell)})) [\mathbf{W}^{(\ell)}]^\top - \mathbf{L}^\top (\tilde{\mathbf{D}}^{(\ell+1) } \odot \sigma^\prime(\tilde{\mathbf{Z}}^{(\ell)})) [\tilde{\mathbf{W}}^{(\ell)}]^\top ]_{i,:}\|_2 \\
    &\leq \max_i \| [\mathbf{L}^\top (\mathbf{D}^{(\ell+1)} \odot \sigma^\prime(\mathbf{Z}^{(\ell)}) ) [\mathbf{W}^{(\ell)}]^\top - \mathbf{L}^\top (\tilde{\mathbf{D}}^{(\ell+1)} \odot \sigma^\prime(\mathbf{Z}^{(\ell)}) ) [\mathbf{W}^{(\ell)}]^\top]_{i,:} \|_2 \\
    &\quad + \max_i \| [ \mathbf{L}^\top (\tilde{\mathbf{D}}^{(\ell+1)} \odot \sigma^\prime(\mathbf{Z}^{(\ell)}) ) [\mathbf{W}^{(\ell)}]^\top - \mathbf{L}^\top (\tilde{\mathbf{D}}^{(\ell+1)} \odot \sigma^\prime(\tilde{\mathbf{Z}}^{(\ell)}) ) [\mathbf{W}^{(\ell)}]^\top ]_{i,:} \|_2 \\
    &\quad + \max_i \| [ \mathbf{L}^\top (\tilde{\mathbf{D}}^{(\ell+1)} \odot \sigma^\prime(\tilde{\mathbf{Z}}^{(\ell)}) ) [\mathbf{W}^{(\ell)}]^\top - \mathbf{L}^\top (\tilde{\mathbf{D}}^{(\ell+1)} \odot \sigma^\prime(\tilde{\mathbf{Z}}^{(\ell)}) ) [\tilde{\mathbf{W}}^{(\ell)}]^\top ]_{i,:} \\
    &\underset{(a)}{\leq} \max_i \| [\mathbf{L}^\top ((\mathbf{D}^{(\ell+1)} - \tilde{\mathbf{D}}^{(\ell+1)}) ) [\mathbf{W}^{(\ell)}]^\top]_{i,:} \|_2 \\
    &\quad + \max_i \| [\mathbf{L}^\top (\tilde{\mathbf{D}}^{(\ell+1)}  [\mathbf{W}^{(\ell)} - \tilde{\mathbf{W}}^{(\ell)}]^\top]_{i,:} \|_2 \\
    &\quad + \max_i \| [\mathbf{L} \tilde{\mathbf{D}}^{(\ell+1)} (\mathbf{Z}^{(\ell)} - \tilde{\mathbf{Z}}^{(\ell)}) \mathbf{W}^{(\ell)} ]_{i,:} \|_2 \\
    &\leq \sqrt{d} B_w \Delta d_{\max}^{(\ell+1)} + \sqrt{d} d_{\max}^{(\ell+1)} \|\Delta \mathbf{W}^{(\ell)}\|_2 + \sqrt{d} B_w d_{\max}^{(\ell+1)} \Delta z_{\max}^{(\ell)} \\
    &\leq \sqrt{d} B_w \Delta d_{\max}^{(\ell+1)} + \underbrace{d_{\max}^{(\ell+1)} \Big( \sqrt{d} \|\Delta \mathbf{W}^{(\ell)}\|_2 + \sqrt{d} B_w \Delta z_{\max}^{(\ell)} \Big) }_{(A)},
    \end{aligned}
\end{equation}
where inequality $(a)$ is due to the gradient of ReLU activation function is either $1$ or $0$.

Knowing that $d_{\max}^{(\ell+1)} \leq \frac{2}{\gamma} (\sqrt{d} B_w)^{L-\ell}$ and using the upper bound of $z_{\max}^{(\ell)}$ as derived in Lemma~\ref{lemma:GCN_delta_z_5}, we can upper bound $(A)$ as
\begin{equation}
    \begin{aligned}
    & d_{\max}^{(\ell+1)} \Big( \sqrt{d} \|\Delta \mathbf{W}^{(\ell)}\|_2 + \sqrt{d} B_w \Delta z_{\max}^{(\ell)} \Big) \\
    &\quad \leq \frac{2}{\gamma} (\sqrt{d} B_w)^{L-\ell} \Big( \sqrt{d} \|\Delta \mathbf{W}^{(\ell)}\|_2 + \sqrt{d} B_x (\sqrt{d} B_w)^\ell \big( \| \Delta \mathbf{W}^{(1)} \|_2 + \ldots + \| \Delta \mathbf{W}^{(\ell)} \|_2 \big)  \Big) \\
    &\quad \leq  \frac{2}{\gamma}  (\sqrt{d} B_w)^{L-\ell} \Big( \sqrt{d} + \sqrt{d} B_x (\sqrt{d} B_w)^\ell  \Big) \big( \| \Delta \mathbf{W}^{(1)} \|_2 + \ldots + \| \Delta \mathbf{W}^{(\ell)} \|_2 \big) \\
    &\quad \leq \frac{2}{\gamma}  (\sqrt{d} B_w)^{L-\ell} \sqrt{d} ( 1 + B_x) (\max\{1, \sqrt{d} B_w\} )^\ell  \| \Delta \bm{\theta} \|_2 \\
    &\quad \leq \frac{2}{\gamma} \sqrt{d} (1+B_x) (\max\{1, \sqrt{d} B_w\})^{L} \| \Delta \bm{\theta} \|_2.
    \end{aligned}
\end{equation}
Plugging it back, we have
\begin{equation}
    \begin{aligned}
    \Delta d_{\max}^{(\ell)} 
    &\leq \sqrt{d} B_w \Delta d_{\max}^{(\ell+1)} + \frac{2}{\gamma} \sqrt{d} (B_x+1) (\max\{1, \sqrt{d} B_w\} )^{L} \| \Delta \bm{\theta} \|_2 \\
    &\leq \sqrt{d} B_w \Big( \sqrt{d} B_w \Delta d_{\max}^{(\ell+2)} + \frac{2}{\gamma} \sqrt{d} (B_x+1) (\max\{1, \sqrt{d} B_w\} )^{L} \| \Delta \bm{\theta} \|_2 \Big) \\
    &\quad + \frac{2}{\gamma} \sqrt{d} (B_x+1) (\max\{1, \sqrt{d} B_w\} )^{L} \| \Delta \bm{\theta} \|_2 \\
    &\leq \Big( 1 + \sqrt{d} B_w + \ldots + (\max\{1, \sqrt{d} B_w\} )^{L-\ell-1} \Big) \cdot \frac{2}{\gamma} \sqrt{d} (B_x+1) (\max\{1, \sqrt{d} B_w\} )^{L} \| \Delta \bm{\theta} \|_2 \\
    &\quad + (\max\{1, \sqrt{d} B_w\} )^{L-\ell+1} \Delta d_{\max}^{L+1} \\
    &\leq (\max\{1, \sqrt{d} B_w\} )^{L-\ell} \cdot \frac{2}{\gamma} L \sqrt{d} (B_x+1) (\max\{1, \sqrt{d} B_w\} )^L \| \Delta \bm{\theta} \|_2 + (\max\{1, \sqrt{d} B_w\} )^{L-\ell} \Delta d_{\max}^{L+1} \\
    &= (\max\{1, \sqrt{d} B_w\} )^{L-\ell} \Big(  \frac{2}{\gamma} L \sqrt{d} (B_x+1) (\max\{1, \sqrt{d} B_w\} )^L \| \Delta \bm{\theta} \|_2 + \Delta d_{\max}^{L+1}  \Big).
    \end{aligned}
\end{equation}
Then, our next step is to explicit upper bound $\Delta d_{\max}^{(L+1)}$. By definition, we can write $\Delta d_{\max}^{(L+1)}$ as \begin{equation}
    \begin{aligned}
    \Delta d_{\max}^{(L+1)} 
    &= \max_i \|\mathbf{d}_i^{(L+1)} - \tilde{\mathbf{d}}_i^{(L+1)}\|_2 \\
    &= \max_i \left\|\frac{\partial \text{Loss}(f(\mathbf{h}_i^{(L)}), y_i)}{\partial \mathbf{h}_i^{(L)}} - \frac{\partial  \text{Loss}(\tilde{f}(\tilde{\mathbf{h}}_i^{(L)}), y_i)}{\partial \tilde{\mathbf{h}}_i^{(L)}} \right\|_2 \\
    &\underset{(a)}{\leq} \frac{2}{\gamma} \max_i \left\| \frac{\partial f(\mathbf{h}_i^{(L)})}{\partial \mathbf{h}_i^{(L)}} - \frac{\partial f(\tilde{\mathbf{h}}_i^{(L)})}{\partial \tilde{\mathbf{h}}_i^{(L)}} \right\|_2 \\
    &\underset{(b)}{\leq} \frac{2}{\gamma}\Big( \sqrt{d} B_x (\max\{1, \sqrt{d} B_w\} )^{L-1} (\|\Delta \mathbf{W}^{(1)}\|_2 + \ldots + \|\Delta \mathbf{W}^{(L)}\|_2) + \big( B_x (\max\{1, \sqrt{d} B_w\} )^L+1 \big) \Delta \mathbf{v} \Big) \\
    &\leq \frac{2}{\gamma} \Big( \sqrt{d} B_x (\max\{1, \sqrt{d} B_w\} )^L +1 \Big) \Big( \|\Delta\mathbf{W}^{(1)}\|_2 + \ldots + \|\Delta\mathbf{W}^{(L)}\|_2 + \|\Delta\mathbf{v}\|_2 \Big) \\
    &= \frac{2}{\gamma} \Big( \sqrt{d} B_x (\max\{1, \sqrt{d} B_w\} )^L +1 \Big) \|\Delta \bm{\theta}\|_2,
    \end{aligned}
\end{equation}
where $(a)$ is due to fact that $\nabla \text{Loss}(z,y)$ is $2/\gamma$-Lipschitz continuous with respect to $z$, $(b)$ follows from Lemma~\ref{lemma:gcn_h_max_bound_1} and Lemma~\ref{lemma:gcn_delta_h_max_bound_2}.

Therefore, we have
\begin{equation}
    \begin{aligned}
    \Delta d_{\max}^{(\ell)} 
    &\leq (\max\{1, \sqrt{d} B_w\} )^{L-\ell} \Big(  \frac{2}{\gamma} L \sqrt{d} (B_x+1) (\max\{1, \sqrt{d} B_w\} )^L \| \Delta \bm{\theta} \|_2 + \Delta d_{\max}^{L+1}  \Big) \\
    &\leq (\max\{1, \sqrt{d} B_w\} )^{L-\ell} \frac{2}{\gamma} \Big(  L \sqrt{d} (B_x+1) (\max\{1, \sqrt{d} B_w\} )^L + \sqrt{d} B_x (\max\{1, \sqrt{d} B_w\} )^L +1  \Big) \|\Delta \bm{\theta}\|_2 \\
    &\leq (\max\{1, \sqrt{d} B_w\} )^{L-\ell} \frac{2}{\gamma} \Big(  (L+1) \sqrt{d} (B_x+1) (\max\{1, \sqrt{d} B_w\} )^L  + 1  \Big) \|\Delta \bm{\theta}\|_2.
    \end{aligned}
\end{equation}

\subsection{Proof of Lemma~\ref{lemma:gcn_G_bound_4}}
By the definition of $\mathbf{G}^{(\ell)},~\ell\in[L]$, we have
\begin{equation}
    \begin{aligned}
    \| \mathbf{G}^{(\ell)} \|_2 
    &= \frac{1}{m}\|[\mathbf{L} \mathbf{H}^{(\ell-1)}]^\top \mathbf{D}^{(\ell)} \odot \sigma^\prime(\mathbf{Z}^{(\ell)}) \|_2 \\
    &= \frac{1}{m} \|[\mathbf{L} \mathbf{H}^{(\ell-1)}]^\top \mathbf{D}^{(\ell)} \|_2 \\
    &\leq \sqrt{d} h_{\max}^{(\ell-1)} d_{\max}^{(\ell)}.
    \end{aligned}
\end{equation}

By plugging in the result from Lemma~\ref{lemma:gcn_h_max_bound_1} and Lemma~\ref{lemma:gcn_d_max_bound_3}, we have
\begin{equation}\label{eq:gradient_norm_GCN}
    \begin{aligned}
    \| \mathbf{G}^{(\ell)} \|_2 \leq \frac{2}{\gamma}\sqrt{d} (B_x+1) (\max\{ 1, \sqrt{d} B_w\} )^L.
    \end{aligned}
\end{equation}
Besides, recall that $\mathbf{G}^{(L+1)}$ denotes the gradient with respect to the weight of binary classifier.
Therefore, we have
\begin{equation}
    \| \mathbf{G}^{(L+1)} \|_2 \leq \frac{2}{\gamma} h_{\max}^{(L)} \leq \frac{2}{\gamma} (\max\{1, \sqrt{d} B_w\})^L B_x,
\end{equation}
which is smaller than the right hand side of the Eq.~\ref{eq:gradient_norm_GCN}. Therefore, we have
\begin{equation}
    \sum_{\ell=1}^{L+1}  \| \mathbf{G}^{(\ell)} \|_2 \leq (L+1) \frac{2}{\gamma}\sqrt{d} (B_x+1) (\max\{ 1, \sqrt{d} B_w\} )^L.
\end{equation}
Similarly, we can upper bound the difference between gradients computed on two different set of weight parameters for the $\ell$th layer as 
\begin{equation}
    \begin{aligned}
    \|\mathbf{G}^{(\ell)} - \tilde{\mathbf{G}}^{(\ell)}\|_2
    &= \frac{1}{m} \| [\mathbf{L} \mathbf{H}^{(\ell-1)}]^\top \mathbf{D}^{(\ell+1)} \odot \sigma^\prime(\mathbf{Z}^{(\ell)})  - [\mathbf{L} \tilde{\mathbf{H}}^{(\ell-1)}]^\top \tilde{\mathbf{D}}^{(\ell+1)} \odot \sigma^\prime(\tilde{\mathbf{Z}}^{(\ell)}) \|_2\\
    &\leq \frac{1}{m} \| [\mathbf{L} \mathbf{H}^{(\ell-1)}]^\top \mathbf{D}^{(\ell+1)} \odot \sigma^\prime(\mathbf{Z}^{(\ell)}) - [\mathbf{L} \tilde{\mathbf{H}}^{(\ell-1)}]^\top \mathbf{D}^{(\ell+1)} \odot \sigma^\prime(\mathbf{Z}^{(\ell)}) \|_2 \\
    &\quad + \frac{1}{m} \| [\mathbf{L} \tilde{\mathbf{H}}^{(\ell-1)}]^\top \mathbf{D}^{(\ell+1)} \odot \sigma^\prime(\mathbf{Z}^{(\ell)}) - [\mathbf{L} \tilde{\mathbf{H}}^{(\ell-1)}]^\top \tilde{\mathbf{D}}^{(\ell+1)} \odot \sigma^\prime(\mathbf{Z}^{(\ell)}) \|_2\\
    &\quad + \frac{1}{m} \| [\mathbf{L} \tilde{\mathbf{H}}^{(\ell-1)}]^\top \tilde{\mathbf{D}}^{(\ell+1)} \odot \sigma^\prime(\mathbf{Z}^{(\ell)}) - [\mathbf{L} \tilde{\mathbf{H}}^{(\ell-1)}]^\top \tilde{\mathbf{D}}^{(\ell+1)} \odot \sigma^\prime(\tilde{\mathbf{Z}}^{(\ell)}) \|_2\\
    &\underset{(a)}{\leq} \frac{1}{m} \| [\mathbf{L} (\mathbf{H}^{(\ell-1)} - \tilde{\mathbf{H}}^{(\ell-1)})]^\top \mathbf{D}^{(\ell+1)}  \|_2 + \frac{1}{m} \| [\mathbf{L} \tilde{\mathbf{H}}^{(\ell-1)}]^\top (\mathbf{D}^{(\ell+1)} - \tilde{\mathbf{D}}^{(\ell+1)})  \|_2\\
    &\quad + \frac{1}{m} \| [\mathbf{L} \tilde{\mathbf{H}}^{(\ell-1)}]^\top \tilde{\mathbf{D}}^{(\ell+1)} \odot \big( \sigma^\prime(\mathbf{Z}^{(\ell)}) - \sigma^\prime(\tilde{\mathbf{Z}}^{(\ell)}) \big)  \|_2 \\
    &\underset{(b)}{\leq}   \max_i \|[ [\mathbf{L} (\mathbf{H}^{(\ell-1)} - \tilde{\mathbf{H}}^{(\ell-1)})]^\top \mathbf{D}^{(\ell+1)} ]_{i,:}\|_2 \\
    &\quad +   \max_i \|[ [\mathbf{L} \tilde{\mathbf{H}}^{(\ell-1)}]^\top (\mathbf{D}^{(\ell+1)} - \tilde{\mathbf{D}}^{(\ell+1)}) ]_{i,:}\|_2 \\
    &\quad +  \max_i \| [\mathbf{L} \tilde{\mathbf{H}}^{(\ell-1)}]^\top \tilde{\mathbf{D}}^{(\ell+1)} \|_2 \|\mathbf{Z}^{(\ell)} - \tilde{\mathbf{Z}}^{(\ell)} \|_2   \\
    &\leq  \sqrt{d} \Big( \underbrace{\Delta h_{\max}^{(\ell-1)} d_{\max}^{(\ell+1)}}_{(A)} +  \underbrace{h_{\max}^{(\ell-1)} \Delta d_{\max}^{(\ell+1)}}_{(B)} + \underbrace{h_{\max}^{(\ell-1)} d_{\max}^{(\ell+1)} \Delta z_{\max}^{(\ell)}}_{(C)} \Big),
    \end{aligned}
\end{equation}
where inequalities $(a)$ and $(b)$ are due to the fact that the gradient of ReLU is either $0$ or $1$.

We now proceed to upper bound the three terms on the right hand side.  By plugging the result from Lemma~\ref{lemma:gcn_h_max_bound_1}, Lemma~\ref{lemma:gcn_delta_h_max_bound_2}, Lemma~\ref{lemma:gcn_d_max_bound_3}, and letting $C_1 = \max\{1, \sqrt{d} B_w\}$ and $C_2 = \sqrt{d}(B_x + 1)$, we can upper bound the term $(A)$ as
\begin{equation}
    \begin{aligned}
    \Delta h_{\max}^{(\ell-1)} d_{\max}^{(\ell+1)} 
    &\leq \sqrt{d} B_x (\max\{1, \sqrt{d}B_w\})^{\ell-2} \cdot \frac{2}{\gamma} (\max\{1, \sqrt{d}B_w\})^{L-\ell} \\
    & \leq \frac{2}{\gamma} \sqrt{d}B_x (\max\{1, \sqrt{d}B_w\})^L \|\Delta \bm{\theta}  \|_2 \\
    &\leq \frac{2}{\gamma} C_1^L C_2 \|\Delta \bm{\theta}  \|_2,
    \end{aligned}
\end{equation}
upper bound $(B)$ as
\begin{equation}
    \begin{aligned}
    h_{\max}^{(\ell-1)} \Delta d_{\max}^{(\ell+1)} 
    &\leq B_x (\max\{1, \sqrt{d} B_w\} )^L \frac{2}{\gamma} \Big(  (L+1) \sqrt{d} (B_x+1) (\max\{1, \sqrt{d} B_w\} )^L  + 1  \Big) \|\Delta \bm{\theta}\|_2 \\
    &\leq \frac{2}{\gamma} C_1^L C_2 \Big( (L+1) C_1^L C_2 + 1 \Big) \|\Delta \bm{\theta}\|_2,
    \end{aligned}
\end{equation}
and upper bound $(C)$ as
\begin{equation}
    \begin{aligned}
    h_{\max}^{(\ell-1)} d_{\max}^{(\ell+1)} \Delta z_{\max}^{(\ell)} 
    &\leq B_x (\max\{1, \sqrt{d}B_w\})^{\ell-1} \cdot \frac{2}{\gamma} (\max\{1, \sqrt{d}B_w\})^{L-\ell} \cdot \sqrt{d} B_x (\max\{1, \sqrt{d}B_w\})^{\ell-1} \|\Delta \bm{\theta}  \|_2 \\
    &\leq \frac{2}{\gamma} \sqrt{d} B_x^2 (\max\{1, \sqrt{d}B_w\})^{2L} \|\Delta \bm{\theta}  \|_2 \\
    &\leq \frac{2}{\gamma} (C_1^L C_2)^2 \|\Delta \bm{\theta}\|_2.
    \end{aligned}
\end{equation}
By combining the results above, we can upper bound $\Delta \mathbf{G}^{(\ell)}$ as
\begin{equation}
    \begin{aligned}
    \|\mathbf{G}^{(\ell)} - \tilde{\mathbf{G}}^{(\ell)} \|_2 
    &\leq  \sqrt{d} \Big( \Delta h_{\max}^{(\ell-1)} d_{\max}^{(\ell+1)} + h_{\max}^{(\ell-1)} \Delta d_{\max}^{(\ell+1)} + h_{\max}^{(\ell-1)} d_{\max}^{(\ell+1)} \cdot \Delta z_{\max}^{(\ell)} \Big) \\
    &\leq  \frac{2}{\gamma} C_1^L C_2  \Big( (L+2) C_1^L C_2 + 2 \Big) \|\Delta \bm{\theta}\|_2.
    \end{aligned}
\end{equation}
Similarly, we can upper bound the difference between gradient for the weight of binary classifier as
\begin{equation}
    \begin{aligned}
    \|\mathbf{G}^{(L+1)} - \tilde{\mathbf{G}}^{(L+1)}\|_2 
    &\leq \frac{2}{\gamma} \Delta h_{\max}^{(L)} \\
    &\leq \frac{2}{\gamma} \sqrt{d}B_x (\{1,\sqrt{d} B_w\})^{L-1} \Big( \|\mathbf{W}^{(1)}\|_2 + \ldots + \|\mathbf{W}^{(L)}\|_2 \Big),
    \end{aligned}
\end{equation}
which is smaller than the right hand side of the previous equation.
Therefore, we have
\begin{equation}
    \sum_{\ell=1}^{L+1} \|\mathbf{G}^{(\ell)} - \tilde{\mathbf{G}}^{(\ell)} \|_2 \leq (L+1) \frac{2}{\gamma} C_1^L C_2  \Big( (L+2) C_1^L C_2 + 2 \Big) \|\Delta \bm{\theta}\|_2.
\end{equation}

\section{Generalization bound for ResGCN} \label{supp:proof_resgcn}

In the following, we provide detailed proof on the generalization bound of ResGCN.
Recall that the update rule of ResGCN is defined as
\begin{equation}
    \mathbf{H}^{(\ell)} = \sigma(\mathbf{L} \mathbf{H}^{(\ell-1)} \mathbf{W}^{(\ell)}) + \mathbf{H}^{(\ell-1)},
\end{equation}
where $\sigma(\cdot)$ is ReLU activation function. Please notice that although ReLU function $\sigma(x)$ is not differentiable when $x=0$, for analysis purpose, we suppose the $\sigma^\prime(0) = 0$. 

The training of ResGCN is an empirical risk minimization with respect to a set of parameters $\bm{\theta} = \{ \mathbf{W}^{(1)}, \ldots, \mathbf{W}^{(L)}, \mathbf{v}\}$, i.e.,
\begin{equation}
    \mathcal{L}(\bm{\theta}) = \frac{1}{m} \sum_{i=1}^m \Phi_\gamma (-p(f(\mathbf{h}_i^{(L)}), y_i)),~
    f(\mathbf{h}_i^{(L)}) = \tilde{\sigma}(\mathbf{v}^\top \mathbf{h}_i^{(L)}),
\end{equation}
where $\mathbf{h}_i^{(L)}$ is the node representation of the $i$th node at the final layer,
$f(\mathbf{h}_i^{(L)})$ is the prediction of the $i$th node,
$\tilde{\sigma}(x) = \frac{1}{\exp(-x) + 1}$ is the sigmoid function, and loss function $\Phi_\gamma (-p(z, y))$ is $\frac{2}{\gamma}$-Lipschitz continuous with respect to its first input $z$. For simplification, we will use $\text{Loss}(z,y)$ denote $\Phi_\gamma (-p(z, y))$ in the proof.

To establish the generalization of ResGCN as stated in Theorem~\ref{thm:uniform_stability_base}, we utilize the result on transductive uniform stability from~\cite{el2006stable} (Theorem~\ref{thm:uniform_stability_base_supp} in Appendix~\ref{supp:proof_gcn}).
Then, in Lemma~\ref{lemma:uniform_stable_resgcn}, we derive the uniform stability constant for ResGCN, i.e., $\epsilon_\text{ResGCN}$. 
\begin{lemma}\label{lemma:uniform_stable_resgcn} 
The uniform stability constant for ResGCN is computed as $\epsilon_\texttt{ResGCN} = \frac{2 \eta  \rho_f G_f}{m} \sum_{t=1}^T (1+\eta L_F)^{t-1} $ where
\begin{equation}
    \begin{aligned}
    \rho_f &= C_1^L C_2,~
    G_f =  \frac{2}{\gamma} (L+1) C_1^{L} C_2,~
    L_f = \frac{2}{\gamma} (L+1) C_1^{L} C_2 \Big( (L+2) C_1^L C_2 + 2 \Big) , \\
    C_1 &= 1 + \sqrt{d} B_w ,~ C_2 = \sqrt{d} (B_x+1).
    \end{aligned}
\end{equation}
\end{lemma}

By plugging the result in Lemma~\ref{lemma:uniform_stable_resgcn} back to Theorem~\ref{thm:uniform_stability_base_supp}, we establish the generalization bound for ResGCN.

Similar to the proof on the generalization bound of GCN in Section~\ref{supp:proof_gcn}, the key idea of the proof is to decompose the change if the ResGCN output into two terms (in Lemma~\ref{lemma:universal_boound_on_model_output}) which depend on 
\begin{itemize}
    \item (Lemma~\ref{lemma:resgcn_h_max_bound_1}) The maximum change of node representation, i.e., $\Delta h_{\max}^{(\ell)} = \max_i \| [\mathbf{H}^{(\ell)} - \tilde{\mathbf{H}}^{(\ell)}]_{i,:}\|_2$, 
    \item (Lemma~\ref{lemma:resgcn_delta_h_max_bound_2}) The maximum node representation, i.e., $h_{\max}^{(\ell)} = \max_i \|[ \mathbf{H}^{(\ell)}]_{i,:}\|_2$. 
\end{itemize}

\begin{lemma} [Upper bound of $h_{\max}^{(\ell)}$ for \textbf{ResGCN}] \label{lemma:resgcn_h_max_bound_1}
Let suppose Assumptoon~\ref{assumption:norm_bound} hold. Then, the maximum node embeddings for any node at the $\ell$th layer is bounded by
\begin{equation}
    h_{\max}^{(\ell)} 
    \leq B_x (1+\sqrt{d}B_w)^\ell. 
\end{equation}
\end{lemma}
\begin{lemma} [Upper bound of $\Delta h_{\max}^{(\ell)}$ for \textbf{ResGCN}] \label{lemma:resgcn_delta_h_max_bound_2} 
Let suppose Assumption~\ref{assumption:norm_bound} hold. 
Then, the maximum change between the node embeddings  on two different set of weight parameters for any node at the $\ell$th layer is bounded by
\begin{equation}
    \begin{aligned}
    \Delta h_{\max}^{(\ell)} 
    &\leq \sqrt{d} B_x (1+\sqrt{d} B_w)^{\ell-1} (\|\Delta \mathbf{W}^{(1)}\|_2 + \ldots + \|\Delta \mathbf{W}^{(\ell)}\|_2), \\
    \end{aligned}
\end{equation}
where $\Delta \mathbf{W}^{(\ell)} = \mathbf{W}^{(\ell)} - \tilde{\mathbf{W}}^{(\ell)}$.
\end{lemma}

Besides, in Lemma~\ref{lemma:resgcn_delta_z_max_bound_5}, we derive the upper bound on the maximum change of node embeddings before the activation function, i.e., $\Delta z_{\max}^{(\ell)} = \max_i \| [\mathbf{Z}^{(\ell)} - \tilde{\mathbf{Z}}^{(\ell)} ]_{i,:} \|_2$, which will be used for the proof of gradient related upper bounds.
\begin{lemma} [Upper bound of $\Delta z_{\max}^{(\ell)}$ for \textbf{ResGCN}] \label{lemma:resgcn_delta_z_max_bound_5}
Let suppose Assumption~\ref{assumption:norm_bound} hold. 
Then, the maximum change between the node embeddings before the activation function  on two different set of weight parameters for any node at the $\ell$th layer is bounded by
\begin{equation}
    \begin{aligned}
    \Delta z_{\max}^{(\ell)} 
    &\leq \sqrt{d} B_x (1+\sqrt{d} B_w)^{\ell-1} (\|\Delta \mathbf{W}^{(1)}\|_2 + \ldots + \|\Delta \mathbf{W}^{(\ell)}\|_2),
    \end{aligned}
\end{equation}
where $\Delta \mathbf{W}^{(\ell)} = \mathbf{W}^{(\ell)} - \tilde{\mathbf{W}}^{(\ell)}$.
\end{lemma}

Then, in Lemma~\ref{lemma:resgcn_d_max_bound_3}, we decompose the change of the model parameters into two terms which depend on
\begin{itemize}
    \item The maximum change of gradient passing from the $(\ell+1)$th layer to the $\ell$th layer
    \begin{equation*}
        \Delta d_{\max}^{(\ell)} = \max_i \left\| \left[\frac{\partial ( \sigma(\mathbf{L} \mathbf{H}^{(\ell-1)} \mathbf{W}^{(\ell)}) + \mathbf{H}^{(\ell-1)} ) }{\partial \mathbf{H}^{(\ell-1)}} - \frac{\partial ( \sigma(\mathbf{L} \tilde{\mathbf{H}}^{(\ell-1)} \tilde{\mathbf{W}}^{(\ell)}) + \tilde{\mathbf{H}}^{(\ell-1)}) }{\partial \tilde{\mathbf{H}}^{(\ell-1)}} \right]_{i,:} \right\|_2.
    \end{equation*}
    \item The maximum gradient passing from the $(\ell+1)$th layer to the $\ell$th layer 
    \begin{equation*}
        d_{\max}^{(\ell)} = \max_i \left\| \left[\frac{\partial ( \sigma(\mathbf{L} \mathbf{H}^{(\ell-1)} \mathbf{W}^{(\ell)}) + \mathbf{H}^{(\ell-1)} ) }{\partial \mathbf{H}^{(\ell-1)}}\right]_{i,:} \right\|_2.
    \end{equation*}
\end{itemize}

\begin{lemma} [Upper bound of $d_{\max}^{(\ell)}, \Delta d_{\max}^{(\ell)}$ for \textbf{ResGCN}] \label{lemma:resgcn_d_max_bound_3}
Let suppose Assumption~\ref{assumption:norm_bound} hold. 
Then, the maximum gradient passing from layer $\ell+1$ to layer $\ell$  for any node is bounded by
\begin{equation}
    d_{\max}^{(\ell)} \leq \frac{2}{\gamma} (1+\sqrt{d} B_w)^{L-\ell+1},
\end{equation}
and the maximum change between the gradient passing from layer $\ell+1$ to layer $\ell$  on two different set of weight parameters for any node is bounded by
\begin{equation}
    \begin{aligned}
    \Delta d_{\max}^{(\ell)} &\leq  \frac{2}{\gamma} (1+\sqrt{d}B_w)^{L-\ell} \Big( (L+1) \sqrt{d} (B_x+1) (1+\sqrt{d} B_w)^L + 1 \Big) \| \Delta \bm{\theta}\|_2,
    \end{aligned}
\end{equation}
where $\|\Delta \bm{\theta}\|_2 = \|\mathbf{v} - \tilde{\mathbf{v}}\|_2 + \sum_{\ell-1}^L \| \mathbf{W}^{(\ell)} - \tilde{\mathbf{W}}^{(\ell)}\|_2$ denotes the change of two set of parameters.
\end{lemma} 

Finally, based on the previous result, in Lemma~\ref{lemma:resgcn_G_bound_4}, we decompose the change of the model parameters into two terms which depend on
\begin{itemize}
    \item The change of gradient with respect to the $\ell$th layer weight matrix $\|\Delta \mathbf{G}^{(\ell)}\|_2$, 
    \item The gradient with respect to the $\ell$th layer weight matrix $\|\mathbf{G}^{(\ell)} \|_2$, 
\end{itemize}
where $\mathbf{G}^{(L+1)}$ denotes the gradient with respect to the weight $\mathbf{v}$ of the binary classifier and $\mathbf{G}^{(\ell)}$ denotes the gradient with respect to the weight $\mathbf{W}^{(\ell)}$ of the $\ell$th layer graph convolutional layer.
Notice that $\|\Delta \mathbf{G}^{(\ell)}\|_2$ reflect the smoothness of ResGCN model and $\|\mathbf{G}^{(\ell)} \|_2$ correspond the upper bound of gradient.

\begin{lemma} [Upper bound of $\|\mathbf{G}^{(\ell)}\|_2, \| \Delta \mathbf{G}^{(\ell)}\|_2$ for \textbf{ResGCN}] \label{lemma:resgcn_G_bound_4}
Let suppose Assumption~\ref{assumption:norm_bound} hold and let $C_1 = 1 + \sqrt{d}B_w$ and $C_2 = \sqrt{d}(B_x+1)$.
Then, the gradient and the maximum change between gradients on two different set of weight parameters are bounded by
\begin{equation}
    \begin{aligned}
    \sum_{\ell=1}^{L+1} \| \mathbf{G}^{(\ell)} \|_2 
    &\leq  \frac{2}{\gamma} (L+1) C_1^L C_2 \|\Delta \bm{\theta}\|_2, \\ 
    \sum_{\ell=1}^{L+1} \|\Delta \mathbf{G}^{(\ell)} \|_2, 
    &\leq  \frac{2}{\gamma} (L+1) C_1^L C_2  \Big( (L+2) C_1^L C_2 + 2 \Big) \|\Delta \bm{\theta}\|_2,
    \end{aligned}
\end{equation}
where $\|\Delta \bm{\theta}\|_2 = \|\mathbf{v} - \tilde{\mathbf{v}}\|_2 + \sum_{\ell-1}^L \| \mathbf{W}^{(\ell)} - \tilde{\mathbf{W}}^{(\ell)}\|_2$ denotes the change of two set of parameters.
\end{lemma}

Equipped with above intermediate results, we now proceed to prove  Lemma~\ref{lemma:uniform_stable_resgcn}.

\begin{proof}
Recall that our goal is to explore the impact of different GCN structures on the uniform stability constant $\epsilon_\text{ResGCN}$, which is a function of $\rho_f$, $G_f$, and $L_f$. 
Let $C_1 = 1 + \sqrt{d} B_w $, $C_2 = \sqrt{d} (B_x+1)$.
Firstly, by plugging Lemma~\ref{lemma:resgcn_delta_h_max_bound_2} and Lemma~\ref{lemma:resgcn_h_max_bound_1} into Lemma~\ref{lemma:universal_boound_on_model_output}, we have that
\begin{equation}
    \begin{aligned}
    \max_i |f(\mathbf{h}_i^{(L)}) - \tilde{f}(\tilde{\mathbf{h}}_i^{(L)})| &\leq \Delta h_{\max}^{(L)} + h_{\max}^{(L)} \| \Delta \mathbf{v} \|_2 \\
    &\leq \sqrt{d} B_x \cdot (\max\{1, \sqrt{d} B_w\} )^{L} \| \Delta \bm{\theta} \|_2 \\
    &\leq C_1^L C_2 \| \Delta \bm{\theta} \|_2.
    \end{aligned}
\end{equation}
Therefore, we know that the function $f$ is $\rho_f$-Lipschitz continuous, with $\rho_f = C_1^L C_2 $.
Then, by Lemma~\ref{lemma:resgcn_G_bound_4}, we know that the function $f$ is $L_f$-smoothness, and the gradient of each weight matrices is bounded by $G_f$, with
\begin{equation}
    G_f =  \frac{2}{\gamma} (L+1) C_1^L C_2,~
    L_f =  \frac{2}{\gamma} (L+1) C_1^L C_2  \Big( (L+2) C_1^L C_2 + 2 \Big).
\end{equation}
By plugging $\epsilon_\text{ResGCN}$ into Theorem~\ref{thm:uniform_stability_base}, we obtain the generalization bound of ResGCN.

\end{proof}

\subsection{Proof of Lemma~\ref{lemma:resgcn_h_max_bound_1}}
By the definition of $h_{\max}^{(\ell)}$, we have
\begin{equation}
    \begin{aligned}
    h_{\max}^{(\ell)} 
    &= \max_i \| [\sigma(\mathbf{L} \mathbf{H}^{(\ell-1)} \mathbf{W}^{(\ell)}) + \mathbf{H}^{(\ell-1)}]_{i,:}\|_2 \\
    &\leq  \max_i \| [\mathbf{L} \mathbf{H}^{(\ell-1)} \mathbf{W}^{(\ell)} ]_{i,:}\|_2  + \max_i \|\mathbf{h}_i^{(\ell-1)}\|_2\\
    &\leq \max_i \| [ \mathbf{L} \mathbf{H}^{(\ell-1)} ]_{i,:}\|_2 \|\mathbf{W}^{(\ell)} \|_2 + h_{\max}^{(\ell-1)}\\
    &=  \max_i \left\| \sum_{j=1}^N L_{i,j} \mathbf{h}_j^{(\ell-1)} \right\|_2 \|\mathbf{W}^{(\ell)} \|_2 + h_{\max}^{(\ell-1)}\\
    &\leq \max_i \left\|\sum_{j=1}^N L_{i,j} \right\|_2  \cdot \max_j \| \mathbf{h}_j^{(\ell-1)} \|_2 \cdot \|\mathbf{W}^{(\ell)} \|_2 + h_{\max}^{(\ell-1)}\\
    &\underset{(a)}{\leq} (1+\sqrt{d} \|\mathbf{W}^{(\ell)} \|_2) \cdot h_{\max}^{(\ell-1)} \\
    &\leq (1+\sqrt{d} B_w) \cdot h_{\max}^{(\ell-1)} \leq B_x (1+\sqrt{d}B_w)^\ell
    \end{aligned}
\end{equation}
where inequality $(a)$ follows from Lemma~\ref{lemma:laplacian_1_norm_bound}.

\subsection{Proof of Lemma~\ref{lemma:resgcn_delta_h_max_bound_2}}
By the definition of $\Delta h_{\max}^{(\ell)}$, we have
\begin{equation}
    \begin{aligned}
    \Delta h_{\max}^{(\ell)}
    &= \max_i \| \mathbf{h}_i^{(\ell)} - \tilde{\mathbf{h}}_i^{(\ell)}\|_2 \\
    &= \max_i \| [\sigma(\mathbf{L} \mathbf{H}^{(\ell-1)} \mathbf{W}^{(\ell)}) - \sigma(\mathbf{L} \tilde{\mathbf{H}}^{(\ell-1)} \tilde{\mathbf{W}}^{(\ell)}) + \mathbf{H}^{(\ell)} - \tilde{\mathbf{H}}^{(\ell)} ]_{i,:} \|_2 \\
    &\leq \max_i \| [\mathbf{L} \mathbf{H}^{(\ell-1)} \mathbf{W}^{(\ell)} - \mathbf{L} \tilde{\mathbf{H}}^{(\ell-1)} \tilde{\mathbf{W}}^{(\ell)} ]_{i,:}\|_2 + \max_i \| [\mathbf{H}^{(\ell)} - \tilde{\mathbf{H}}^{(\ell)}]_{i,:} \|_2\\
    &\leq \max_i \| [ \mathbf{L} \mathbf{H}^{(\ell-1)} (\mathbf{W}^{(\ell)} - \tilde{\mathbf{W}}^{(\ell)}) - \mathbf{L} (\mathbf{H}^{(\ell-1)} - \tilde{\mathbf{H}}^{(\ell-1)}) \tilde{\mathbf{W}}^{(\ell)} ]_{i,:}\|_2 + \Delta h_{\max}^{(\ell-1)} \\
    &\underset{(a)}{\leq} \max_i \left\|\sum_{j=1}^N L_{i,j} \mathbf{h}_j^{(\ell-1)} \right\|_2 \|\Delta \mathbf{W}^{(\ell)}\|_2 + \max_i \left\| \sum_{j=1}^N L_{i,j} \Delta \mathbf{h}_j^{(\ell-1)} \right\|_2 \|\tilde{\mathbf{W}}^{(\ell)}\|_2 + \Delta h_{\max}^{(\ell-1)}\\
    &\leq (1 + \sqrt{d} B_w ) \cdot \Delta h_{\max}^{(\ell-1)}  + \sqrt{d} h_{\max}^{(\ell-1)} \|\Delta\mathbf{W}^{(\ell)}\|_2,
    \end{aligned}
\end{equation}
where inequality $(a)$ follows from Lemma~\ref{lemma:laplacian_1_norm_bound}.

By induction, we have
\begin{equation}
    \begin{aligned}
    \Delta h_{\max}^{(\ell)}
    &\leq (1 + \sqrt{d} B_w ) \cdot \Delta h_{\max}^{(\ell-1)}  + \sqrt{d} h_{\max}^{(\ell-1)} \|\Delta\mathbf{W}^{(\ell)}\|_2 \\
    &\leq (1 + \sqrt{d} B_w )^2 \cdot \Delta h_{\max}^{(\ell-2)} \\
    &\quad + \sqrt{d}\Big( h_{\max}^{(\ell-1)} \|\Delta\mathbf{W}^{(\ell)}\|_2 + (1+\sqrt{d} B_w) h_{\max}^{(\ell-2)} \| \Delta \mathbf{W}^{(\ell-1)}\|_2 \Big) \\
    &\ldots \\
    &\leq (1 + \sqrt{d} B_w )^\ell \cdot \Delta h_{\max}^{(0)} \\
    &\quad + \sqrt{d} \Big( h_{\max}^{(\ell-1)} \| \Delta \mathbf{W}^{(\ell)}\|_2 + (1+\sqrt{d} B_w) h_{\max}^{(\ell-2)} \| \Delta \mathbf{W}^{(\ell-1)}\|_2 + \ldots + (1+\sqrt{d} B_w)^{\ell-1} h_{\max}^{(0)} \| \Delta \mathbf{W}^{(1)}\|_2 ) \\
    &= \sqrt{d} \Big( h_{\max}^{(\ell-1)} \| \Delta \mathbf{W}^{(\ell)}\|_2 + (1+\sqrt{d} B_w) h_{\max}^{(\ell-2)} \| \Delta \mathbf{W}^{(\ell-1)}\|_2 + \ldots + (1+\sqrt{d} B_w)^{\ell-1} h_{\max}^{(0)} \| \Delta \mathbf{W}^{(1)}\|_2 ),
    \end{aligned}
\end{equation}
where the last equality is due to $\Delta h_{\max}^{(0)} = 0$.

Plugging in the upper bound of $h_{\max}^{(\ell)}$ in Lemma~\ref{lemma:resgcn_delta_h_max_bound_2}, we have
\begin{equation}
    \Delta h_{\max}^{(\ell)} \leq \sqrt{d} B_x (1+\sqrt{d} B_w)^{\ell-1} (\|\mathbf{W}^{(1)}\|_2 + \ldots + \|\mathbf{W}^{(\ell)}\|_2).
\end{equation}

\subsection{Proof of Lemma~\ref{lemma:resgcn_delta_z_max_bound_5}}
By the definition of $\mathbf{Z}^{(\ell)}$, we have
\begin{equation}
    \begin{aligned}
    \mathbf{Z}^{(\ell)} - \Tilde{\mathbf{Z}}^{(\ell)} 
    &= \mathbf{L} \mathbf{H}^{(\ell-1)} \mathbf{W}^{(\ell)} -   \mathbf{L} \Tilde{\mathbf{H}}^{(\ell-1)} \Tilde{\mathbf{W}}^{(\ell)}  \\
    &= \mathbf{L} \Big( \mathbf{H}^{(\ell-1)} \mathbf{W}^{(\ell)} - \Tilde{\mathbf{H}}^{(\ell-1)} \mathbf{W}^{(\ell)} + \Tilde{\mathbf{H}}^{(\ell-1)} \mathbf{W}^{(\ell)} - \Tilde{\mathbf{H}}^{(\ell-1)} \Tilde{\mathbf{W}}^{(\ell)} \Big) \\
    &= \mathbf{L} \big( \mathbf{H}^{(\ell-1)}  - \Tilde{\mathbf{H}}^{(\ell-1)} \big) \mathbf{W}^{(\ell)}  + \mathbf{L} \Tilde{\mathbf{H}}^{(\ell-1)} \big( \mathbf{W}^{(\ell)}  - \tilde{\mathbf{W}}^{(\ell)} \big).
    \end{aligned}
\end{equation}
Then, by taking the norm of both sides, we have
\begin{equation}
    \begin{aligned}
    \Delta z_{\max}^{(\ell)}
    &= \max_i \| [\mathbf{Z}^{(\ell)} - \Tilde{\mathbf{Z}}^{(\ell)}]_{i,:} \|_2 \\
    &\leq \sqrt{d} B_w \cdot \max_i \|[\mathbf{Z}^{(\ell-1)} - \Tilde{\mathbf{Z}}^{(\ell-1)}]_{i,:}  \|_2 + \sqrt{d} h_{\max}^{(\ell-1)} \| \Delta \mathbf{W}^{(\ell)} \|_2 \\
    &= \sqrt{d} B_w \cdot z_{\max}^{(\ell-1)} + \sqrt{d} h_{\max}^{(\ell-1)} \| \Delta \mathbf{W}^{(\ell)} \|_2. \\
    \end{aligned}
\end{equation}
By induction, we have
\begin{equation}
    \begin{aligned}
    \Delta z_{\max}^{(\ell)} 
    &\leq \sqrt{d} B_w \cdot \Delta z_{\max}^{(\ell-1)} + \sqrt{d} h_{\max}^{(\ell-1)} \| \Delta \mathbf{W}^{(\ell)} \|_2 \\
    &\underset{(a)}{\leq} \sqrt{d} B_w \cdot \Delta z_{\max}^{(\ell-1)} + \sqrt{d} B_x (\sqrt{d} B_w)^{\ell-1} \| \Delta \mathbf{W}^{(\ell)} \|_2 \\
    &\leq \sqrt{d} B_w \Big( \sqrt{d} B_w \cdot \Delta z_{\max}^{(\ell-2)} + \sqrt{d} B_x (\sqrt{d} B_w)^{\ell-2} \| \Delta \mathbf{W}^{(\ell-1)} \|_2 \Big) + \sqrt{d} B_x (1+\sqrt{d} B_w)^{\ell-1} \| \Delta \mathbf{W}^{(\ell)} \|_2 \\
    &=(\sqrt{d} B_w)^2  \cdot \Delta z_{\max}^{(\ell-2)} + \sqrt{d} B_x (\sqrt{d} B_w)^{\ell-1} \Big( \| \Delta \mathbf{W}^{(\ell-1)} \|_2 + \| \Delta \mathbf{W}^{(\ell)} \|_2\Big) \\
    &\leq \ldots \\
    &\leq \sqrt{d} B_x (\sqrt{d} B_w)^{\ell-1} \Big( \| \Delta \mathbf{W}^{(1)} \|_2 + \ldots + \| \Delta \mathbf{W}^{(\ell)} \|_2\Big),
    \end{aligned}
\end{equation}
where (a) is due to $h_{\max}^{(\ell-1)} \leq (\sqrt{d} B_w)^{\ell-1} B_x$.

\subsection{Proof of Lemma~\ref{lemma:resgcn_d_max_bound_3}}
For notation simplicity, let $\mathbf{D}^{(\ell)}$ denote the gradient passing from the $\ell$th to the $(\ell-1)$th layer.  
By the definition of $d_{\max}^{(\ell)}$, we have
\begin{equation}
    \begin{aligned}
    d_{\max}^{(\ell)} 
    &= \max_i \|\mathbf{d}_i^{(\ell)}\|_2 \\
    &= \max_i [ \| \mathbf{L}^\top \sigma^\prime(\mathbf{Z}^{(\ell)}) \odot \mathbf{D}^{(\ell+1)} \mathbf{W}^{(\ell)} + \mathbf{D}^{(\ell+1)}\| ]_{i,:} \\
    &\leq \max_i [\| \mathbf{L}^\top \mathbf{D}^{(\ell+1)} \|_2]_{i,:} \|\mathbf{W}^{(\ell)}\|_2 + d_{\max}^{(\ell+1)}\\
    &\leq \max_i \left\|\sum_{j=1}^N L_{i,j}^\top \mathbf{d}_j^{(\ell+1)} \right\|_2 \|\mathbf{W}^{(\ell)}\|_2 + d_{\max}^{(\ell+1)}\\
    &\leq (1+\sqrt{d} \|\mathbf{W}^{(\ell)}\|_2) d_{\max}^{(\ell+1)} \\
    &\underset{(a)}{\leq} (1+\sqrt{d} B_w) \cdot d_{\max}^{(\ell+1)} \leq \frac{2}{\gamma} (1+\sqrt{d} B_w)^{L-\ell+1},
    \end{aligned}
\end{equation}
where inequality $(a)$ follows from Lemma~\ref{lemma:laplacian_1_norm_bound}.

By the definition of $\Delta d_{\max}^{(\ell)}$, we have
\begin{equation}
    \begin{aligned}
    \Delta d_{\max}^{(\ell)} 
    &= \max_i \| [\mathbf{D}^{(\ell)} - \tilde{\mathbf{D}}^{(\ell)}]_{i,:} \|_2 \\
    &= \max_i \| [\mathbf{L}^\top (\mathbf{D}^{(\ell+1)} \odot \sigma^\prime(\mathbf{Z}^{(\ell)})) [\mathbf{W}^{(\ell)}]^\top + \mathbf{D}^{(\ell+1)} - \mathbf{L}^\top (\tilde{\mathbf{D}}^{(\ell+1) } \odot \sigma^\prime(\tilde{\mathbf{Z}}^{(\ell)})) [\tilde{\mathbf{W}}^{(\ell)}]^\top - \tilde{\mathbf{D}}^{(\ell+1)}]_{i,:}\|_2 \\
    &\leq \max_i \| [\mathbf{L}^\top (\mathbf{D}^{(\ell+1)} \odot \sigma^\prime(\mathbf{Z}^{(\ell)}) ) [\mathbf{W}^{(\ell)}]^\top - \mathbf{L}^\top (\tilde{\mathbf{D}}^{(\ell+1)} \odot \sigma^\prime(\mathbf{Z}^{(\ell)}) ) [\mathbf{W}^{(\ell)}]^\top]_{i,:} \|_2 \\
    &\quad + \max_i \| [\mathbf{L}^\top (\tilde{\mathbf{D}}^{(\ell+1)} \odot \sigma^\prime(\mathbf{Z}^{(\ell)}) )  [\mathbf{W}^{(\ell)}]^\top - \mathbf{L}^\top (\tilde{\mathbf{D}}^{(\ell+1)} \odot \sigma^\prime(\mathbf{Z}^{(\ell)}) ) [\tilde{\mathbf{W}}^{(\ell)}]^\top]_{i,:} \|_2 \\
    &\quad + \max_i \| [ \mathbf{L}^\top (\tilde{\mathbf{D}}^{(\ell+1)} \odot \sigma^\prime(\mathbf{Z}^{(\ell)}) ) [\tilde{\mathbf{W}}^{(\ell)}]^\top - \mathbf{L}^\top (\tilde{\mathbf{D}}^{(\ell+1)} \odot \sigma^\prime(\tilde{\mathbf{Z}}^{(\ell)}) ) [\tilde{\mathbf{W}}^{(\ell)}]^\top ]_{i,:}\|_2 + \max_i \|[ \mathbf{D}^{(\ell+1)} - \tilde{\mathbf{D}}^{(\ell+1)} ]_{i,:}\|_2\\
    &\underset{(a)}{\leq} \max_i \| [\mathbf{L}^\top ((\mathbf{D}^{(\ell+1)} - \tilde{\mathbf{D}}^{(\ell+1)}) ) [\mathbf{W}^{(\ell)}]^\top]_{i,:} \|_2 + \max_i \| [\mathbf{L}^\top (\tilde{\mathbf{D}}^{(\ell+1)}  [\mathbf{W}^{(\ell)} - \tilde{\mathbf{W}}^{(\ell)}]^\top]_{i,:} \|_2 \\
    &\quad + \max_i \| [ \mathbf{L}^\top \tilde{\mathbf{D}}^{(\ell+1)} [\tilde{\mathbf{W}}^{(\ell)}]^\top ]_{i,:} \|_2 \max_i \|[\mathbf{Z}^{(\ell)} - \tilde{\mathbf{Z}}^{(\ell)}]_{i,:} \|+ \max_i \|[ \mathbf{D}^{(\ell+1)} - \tilde{\mathbf{D}}^{(\ell+1)} ]_{i,:}\|_2\\
    &\leq (1+\sqrt{d} B_w) \Delta d_{\max}^{(\ell+1)} + \sqrt{d} d_{\max}^{(\ell+1)} \|\Delta \mathbf{W}^{(\ell)}\|_2 + \sqrt{d} B_w d_{\max}^{(\ell+1)} \Delta z_{\max}^{(\ell)} \\
    &= (1+\sqrt{d} B_w) \Delta d_{\max}^{(\ell+1)} + \underbrace{d_{\max}^{(\ell+1)} \Big( \sqrt{d} \|\Delta \mathbf{W}^{(\ell)}\|_2 + \sqrt{d} B_w \Delta z_{\max}^{(\ell)} \Big) }_{(A)},
    \end{aligned}
\end{equation}
where inequality $(a)$ is due to the fact that the gradient of ReLU is either $1$ or $0$.

Knowing that $d_{\max}^{(\ell+1)} \leq \frac{2}{\gamma} (1+\sqrt{d} B_w)^{L-\ell}$ and $z_{\max}^{(\ell)} \leq \sqrt{d} B_x (\sqrt{d} B_w)^{\ell-1} \Big( \|\mathbf{W}^{(1)}\|_2 + \ldots + \|\mathbf{W}^{(\ell)}\|_2 \Big)$, we can upper bound $(A)$ by
\begin{equation}
    \begin{aligned}
    & d_{\max}^{(\ell+1)} \Big( \sqrt{d} \|\Delta \mathbf{W}^{(\ell)}\|_2 + \sqrt{d} B_w \Delta z_{\max}^{(\ell)} \Big) \\
    &\leq \frac{2}{\gamma} (1+\sqrt{d} B_w)^{L-\ell} \Big( \sqrt{d} \|\Delta \mathbf{W}^{(\ell)}\|_2 + \sqrt{d} B_x (\sqrt{d} B_w)^\ell \big( \|\mathbf{W}^{(1)}\|_2 + \ldots + \|\mathbf{W}^{(\ell)}\|_2 \big) \Big) \\
    &\leq \frac{2}{\gamma} (1+\sqrt{d} B_w)^{L-\ell} \Big( \sqrt{d} + \sqrt{d} B_x (\sqrt{d} B_w)^\ell \Big) \big( \|\mathbf{W}^{(1)}\|_2 + \ldots + \|\mathbf{W}^{(\ell)}\|_2 \big) \\
    &\leq \frac{2}{\gamma} \sqrt{d} (1+B_x) (1+\sqrt{d} B_w)^L \| \Delta \bm{\theta} \|_2.
    \end{aligned}
\end{equation}
By plugging it back, we have
\begin{equation}
    \begin{aligned}
    \Delta d_{\max}^{(\ell)} 
    &\leq \sqrt{d} B_w \Delta d_{\max}^{(\ell+1)} + \frac{2}{\gamma} \sqrt{d} (B_x+1) (1+\sqrt{d} B_w)^{L} \| \Delta \bm{\theta} \|_2 \\
    &\leq \sqrt{d} B_w \Big( \sqrt{d} B_w \Delta d_{\max}^{(\ell+2)} + \frac{2}{\gamma} \sqrt{d} (B_x+1) (1+\sqrt{d} B_w)^{L} \| \Delta \bm{\theta} \|_2 \Big) \\
    &\quad + \frac{2}{\gamma} \sqrt{d} (B_x+1) (1+\sqrt{d} B_w)^{L} \| \Delta \bm{\theta} \|_2 \\
    &\leq \Big( 1 + \sqrt{d} B_w + \ldots + (1+\sqrt{d} B_w)^{L-\ell-1} \Big) \cdot \frac{2}{\gamma} \sqrt{d} (B_x+1) (1+\sqrt{d} B_w)^{L} \| \Delta \bm{\theta} \|_2 \\
    &\quad + (1+\sqrt{d} B_w)^{L-\ell+1} \Delta d_{\max}^{L+1} \\
    &\leq (1+\sqrt{d} B_w)^{L-\ell} \cdot \frac{2}{\gamma} L \sqrt{d} (B_x+1) (1+\sqrt{d} B_w)^L \| \Delta \bm{\theta} \|_2 + (1+\sqrt{d} B_w)^{L-\ell} \Delta d_{\max}^{L+1} \\
    &= (1+\sqrt{d} B_w)^{L-\ell} \Big(  \frac{2}{\gamma} L \sqrt{d} (B_x+1) (1+\sqrt{d} B_w)^L \| \Delta \bm{\theta} \|_2 + \Delta d_{\max}^{L+1}  \Big).
    \end{aligned}
\end{equation}

Let first explicit upper bound $\Delta d_{\max}^{(L+1)}$. By definition, we can write $\Delta d_{\max}^{(L+1)}$ as \begin{equation}
    \begin{aligned}
    \Delta d_{\max}^{(L+1)} 
    &= \max_i \|\mathbf{d}_i^{(L+1)} - \tilde{\mathbf{d}}_i^{(L+1)}\|_2 \\
    &= \max_i \left\|\frac{\partial \text{Loss}(f(\mathbf{h}_i^{(L)}), y_i)}{\partial \mathbf{h}_i^{(L)}} - \frac{\partial  \text{Loss}(\tilde{f}(\tilde{\mathbf{h}}_i^{(L)}), y_i)}{\partial \tilde{\mathbf{h}}_i^{(L)}}\right\|_2 \\
    &\underset{(a)}{\leq} \frac{2}{\gamma} \max_i \left\| \frac{\partial f(\mathbf{h}_i^{(L)})}{\partial \mathbf{h}_i^{(L)}} - \frac{\partial f(\tilde{\mathbf{h}}_i^{(L)})}{\partial \tilde{\mathbf{h}}_i^{(L)}} \right\|_2 \\
    &\underset{(b)}{\leq} \frac{2}{\gamma}\Big( \sqrt{d} B_x (1+\sqrt{d} B_w)^{L-1} (\|\Delta \mathbf{W}^{(1)}\|_2 + \ldots + \|\Delta \mathbf{W}^{(L)}\|_2) + \big( B_x (1+\sqrt{d} B_w)^L+1 \big) \Delta \mathbf{v} \Big) \\
    &\leq \frac{2}{\gamma} \Big( \sqrt{d} B_x (1+\sqrt{d} B_w)^L +1 \Big) \Big( \|\Delta\mathbf{W}^{(1)}\|_2 + \ldots + \|\Delta\mathbf{W}^{(L)}\|_2 + \|\Delta\mathbf{v}\|_2 \Big) \\
    &= \frac{2}{\gamma} \Big( \sqrt{d} B_x (1+\sqrt{d} B_w)^L +1 \Big) \|\Delta \bm{\theta}\|_2,
    \end{aligned}
\end{equation}
where inequality $(a)$ is due to the fact that $\nabla \text{Loss}(z,y)$ is $\frac{2}{\gamma}$-Lipschitz continuous with respect to $z$, and inequality $(b)$ follows from Lemma~\ref{lemma:resgcn_h_max_bound_1} and Lemma~\ref{lemma:resgcn_delta_h_max_bound_2}.
Therefore, we have
\begin{equation}
    \begin{aligned}
    \Delta d_{\max}^{(\ell)} 
    &\leq (1+\sqrt{d} B_w)^{L-\ell} \Big(  \frac{2}{\gamma} L \sqrt{d} (B_x+1) (1+\sqrt{d} B_w)^L \| \Delta \bm{\theta} \|_2 + \Delta d_{\max}^{L+1}  \Big) \\
    &\leq (1+\sqrt{d} B_w)^{L-\ell} \frac{2}{\gamma} \Big(  L \sqrt{d} (B_x+1) (1+\sqrt{d} B_w)^L + \sqrt{d} B_x (1+\sqrt{d} B_w)^L +1  \Big) \|\Delta \bm{\theta}\|_2 \\
    &\leq (1+\sqrt{d} B_w)^{L-\ell} \frac{2}{\gamma} \Big(  (L+1) \sqrt{d} (B_x+1) (1+\sqrt{d} B_w)^L  + 1  \Big) \|\Delta \bm{\theta}\|_2.
    \end{aligned}
\end{equation}

\subsection{Proof of Lemma~\ref{lemma:resgcn_G_bound_4}}
First By the definition of $\mathbf{G}^{(\ell)},~\ell\in[L]$, we have
\begin{equation}\label{eq:gradient_norm_resgcn}
    \begin{aligned}
    \| \mathbf{G}^{(\ell)} \|_2 
    &= \frac{1}{m} \|[\mathbf{L} \mathbf{H}^{(\ell-1)}]^\top \mathbf{D}^{(\ell)} \odot \sigma^\prime(\bm{Z}^{(\ell)}) \|_2 \\
    &\leq \frac{1}{m}\|[\mathbf{L} \mathbf{H}^{(\ell-1)}]^\top \mathbf{D}^{(\ell)} \|_2 \\
    &\leq \sqrt{d} h_{\max}^{(\ell-1)} d_{\max}^{(\ell)} \\
    &\underset{(a)}{\leq}  \frac{2}{\gamma} \sqrt{d} (B_x+1) (1+\sqrt{d}B_w)^L,
    \end{aligned}
\end{equation}
where inequality $(a)$ follows from Lemma~\ref{lemma:resgcn_h_max_bound_1} and Lemma~\ref{lemma:resgcn_d_max_bound_3}, and the fact that loss function is $\frac{2}{\gamma}$-Lipschitz continuous.

Similarly, by the definition of $\mathbf{G}^{(L+1)}$, we have
\begin{equation}
    \|\mathbf{G}^{(L+1)}\|_2 \underset{(a)}{\leq} \frac{2}{\gamma} h_{\max}^{(L)} \underset{(b)}{\leq} \frac{2}{\gamma} (B_x+1) (1+\sqrt{d} B_w)^L,
\end{equation}
which is smaller then the right hind side of Eq.~\ref{eq:gradient_norm_resgcn}, and inequality $(a)$ is due to the fact that loss function is $\frac{2}{\gamma}$-Lipschitz continuous and inequality $(b)$ follows from Lemma~\ref{lemma:resgcn_h_max_bound_1}.

By combining the above two inequalities together, we have
\begin{equation}
    \sum_{\ell=1}^{L+1}  \| \mathbf{G}^{(\ell)} \|_2 \leq \frac{2}{\gamma} (L+1) \sqrt{d} (B_x+1) (1+\sqrt{d}B_w)^L.
\end{equation}
Furthermore, we can upper bound the difference between gradients computed on two different set of weight parameters for the $\ell$th layer as
\begin{equation}
    \begin{aligned}
    \|\mathbf{G}^{(\ell)} - \tilde{\mathbf{G}}^{(\ell)}\|_2
    &= \frac{1}{m} \| [\mathbf{L} \mathbf{H}^{(\ell-1)}]^\top \mathbf{D}^{(\ell+1)} \odot \sigma^\prime(\mathbf{Z}^{(\ell)})  - [\mathbf{L} \tilde{\mathbf{H}}^{(\ell-1)}]^\top \tilde{\mathbf{D}}^{(\ell+1)} \odot \sigma^\prime(\tilde{\mathbf{Z}}^{(\ell)}) \|_2\\
    &\leq \frac{1}{m} \| [\mathbf{L} \mathbf{H}^{(\ell-1)}]^\top \mathbf{D}^{(\ell+1)} \odot \sigma^\prime(\mathbf{Z}^{(\ell)}) - [\mathbf{L} \tilde{\mathbf{H}}^{(\ell-1)}]^\top \mathbf{D}^{(\ell+1)} \odot \sigma^\prime(\mathbf{Z}^{(\ell)}) \|_2 \\
    &\quad + \frac{1}{m} \| [\mathbf{L} \tilde{\mathbf{H}}^{(\ell-1)}]^\top \mathbf{D}^{(\ell+1)} \odot \sigma^\prime(\mathbf{Z}^{(\ell)}) - [\mathbf{L} \tilde{\mathbf{H}}^{(\ell-1)}]^\top \tilde{\mathbf{D}}^{(\ell+1)} \odot \sigma^\prime(\mathbf{Z}^{(\ell)}) \|_2\\
    &\quad + \frac{1}{m} \| [\mathbf{L} \tilde{\mathbf{H}}^{(\ell-1)}]^\top \tilde{\mathbf{D}}^{(\ell+1)} \odot \sigma^\prime(\mathbf{Z}^{(\ell)}) - [\mathbf{L} \tilde{\mathbf{H}}^{(\ell-1)}]^\top \tilde{\mathbf{D}}^{(\ell+1)} \odot \sigma^\prime(\tilde{\mathbf{Z}}^{(\ell)}) \|_2\\
    &\underset{(a)}{\leq} \frac{1}{m} \| [\mathbf{L} (\mathbf{H}^{(\ell-1)} - \tilde{\mathbf{H}}^{(\ell-1)})]^\top \mathbf{D}^{(\ell+1)}  \|_2 + \frac{1}{m} \| [\mathbf{L} \tilde{\mathbf{H}}^{(\ell-1)}]^\top (\mathbf{D}^{(\ell+1)} - \tilde{\mathbf{D}}^{(\ell+1)})  \|_2\\
    &\quad + \frac{1}{m} \| [\mathbf{L} \tilde{\mathbf{H}}^{(\ell-1)}]^\top \tilde{\mathbf{D}}^{(\ell+1)} \odot \big( \sigma^\prime(\mathbf{Z}^{(\ell)}) - \sigma^\prime(\tilde{\mathbf{Z}}^{(\ell)}) \big)  \|_2 \\
    &\underset{(b)}{\leq}   \max_i \|[ [\mathbf{L} (\mathbf{H}^{(\ell-1)} - \tilde{\mathbf{H}}^{(\ell-1)})]^\top \mathbf{D}^{(\ell+1)} ]_{i,:}\|_2 \\
    &\quad +   \max_i \|[ [\mathbf{L} \tilde{\mathbf{H}}^{(\ell-1)}]^\top (\mathbf{D}^{(\ell+1)} - \tilde{\mathbf{D}}^{(\ell+1)}) ]_{i,:}\|_2 \\
    &\quad +  \max_i \| [\mathbf{L} \tilde{\mathbf{H}}^{(\ell-1)}]^\top \tilde{\mathbf{D}}^{(\ell+1)} \|_2 \|\mathbf{Z}^{(\ell)} - \tilde{\mathbf{Z}}^{(\ell)} \|_2   \\
    &\leq  \sqrt{d} \Big( \underbrace{\Delta h_{\max}^{(\ell-1)} d_{\max}^{(\ell+1)}}_{(A)} +  \underbrace{h_{\max}^{(\ell-1)} \Delta d_{\max}^{(\ell+1)}}_{(B)} + \underbrace{h_{\max}^{(\ell-1)} d_{\max}^{(\ell+1)} \Delta z_{\max}^{(\ell)}}_{(C)} \Big), 
    \end{aligned}
\end{equation}
where inequality $(a)$ and $(b)$ is due to the fact that the gradient of ReLU activation function is element-wise either $0$ or $1$.

By plugging the result from Lemma~\ref{lemma:resgcn_h_max_bound_1}, Lemma~\ref{lemma:resgcn_delta_h_max_bound_2}, Lemma~\ref{lemma:resgcn_d_max_bound_3}, and letting $C_1 = 1 + \sqrt{d} B_w$ and $C_2 = \sqrt{d}(B_x + 1)$ we can upper bound $(A)$ as
\begin{equation}
    \begin{aligned}
    \Delta h_{\max}^{(\ell-1)} d_{\max}^{(\ell+1)} 
    &\leq \sqrt{d} B_x (1+ \sqrt{d}B_w)^{\ell-2} \cdot \frac{2}{\gamma} (1+ \sqrt{d}B_w)^{L-\ell} \\
    & \leq \frac{2}{\gamma} \sqrt{d}B_x (1+ \sqrt{d}B_w)^L \|\Delta \bm{\theta}  \|_2 \\
    &\leq \frac{2}{\gamma} C_1^L C_2 \|\Delta \bm{\theta}  \|_2,
    \end{aligned}
\end{equation}
upper bound $(B)$ as
\begin{equation}
    \begin{aligned}
    h_{\max}^{(\ell-1)} \Delta d_{\max}^{(\ell+1)} 
    &\leq B_x (\max\{1, \sqrt{d} B_w\} )^L \frac{2}{\gamma} \Big(  (L+1) \sqrt{d} (B_x+1) (\max\{1, \sqrt{d} B_w\} )^L  + 1  \Big) \|\Delta \bm{\theta}\|_2 \\
    &\leq \frac{2}{\gamma} C_1^L C_2 \Big( (L+1) C_1^L C_2 + 1 \Big) \|\Delta \bm{\theta}\|_2,
    \end{aligned}
\end{equation}
and upper bound $(C)$ as
\begin{equation}
    \begin{aligned}
    h_{\max}^{(\ell-1)} d_{\max}^{(\ell+1)} \Delta z_{\max}^{(\ell)} 
    &\leq B_x (1+ \sqrt{d}B_w)^{\ell-1} \cdot \frac{2}{\gamma} (1+ \sqrt{d}B_w)^{L-\ell} \cdot \sqrt{d} B_x (1+ \sqrt{d}B_w)^{\ell-1} \|\Delta \bm{\theta}  \|_2 \\
    &\leq \frac{2}{\gamma} \sqrt{d} B_x^2 (1+ \sqrt{d}B_w)^{2L} \|\Delta \bm{\theta}  \|_2 \\
    &\leq \frac{2}{\gamma} (C_1^L C_2)^2 \|\Delta \bm{\theta}\|_2.
    \end{aligned}
\end{equation}
By combining the results above, we can upper bound $\Delta \mathbf{G}^{(\ell)}$ as
\begin{equation}
    \begin{aligned}
    \|\mathbf{G}^{(\ell)} - \tilde{\mathbf{G}}^{(\ell)} \|_2 
    &\leq  \sqrt{d} \Big( \Delta h_{\max}^{(\ell-1)} d_{\max}^{(\ell+1)} + h_{\max}^{(\ell-1)} \Delta d_{\max}^{(\ell+1)} + h_{\max}^{(\ell-1)} d_{\max}^{(\ell+1)} \cdot \Delta z_{\max}^{(\ell)} \Big) \\
    &\leq  \frac{2}{\gamma} C_1^L C_2  \Big( (L+2) C_1^L C_2 + 2 \Big) \|\Delta \bm{\theta}\|_2.
    \end{aligned}
\end{equation}
By plugging the result from Lemma~\ref{lemma:resgcn_h_max_bound_1}, Lemma~\ref{lemma:resgcn_delta_h_max_bound_2}, Lemma~\ref{lemma:resgcn_d_max_bound_3}, we can upper bound $\Delta \mathbf{G}^{(\ell)}$ as
\begin{equation}
    \begin{aligned}
    \|\Delta \mathbf{G}^{(\ell)} \|_2 
    &\leq  \sqrt{d} \Big( \Delta h_{\max}^{(\ell-1)} d_{\max}^{(\ell+1)} + h_{\max}^{(\ell-1)} \Delta d_{\max}^{(\ell+1)} \Big) \\
    &\leq  \frac{2}{\gamma} \Big( \sqrt{d} B_x (1+\sqrt{d} B_w)^{L} + (1+\sqrt{d} B_w) + \sqrt{d} \Big) \|\Delta \bm{\theta}\|_2 \cdot \sqrt{d} B_x (1+\sqrt{d} B_w)^{L}.
    \end{aligned}
\end{equation}
Similarly, we can upper bound the difference between gradient for the weight parameters of the binary classifier as
\begin{equation}
    \begin{aligned}
    \| \mathbf{G}^{(L+1)} - \tilde{\mathbf{G}}^{(L+1)}\|_2 
    &\underset{(a)}{\leq} \frac{2}{\gamma} \Delta h_{\max}^{(L)} \\
    &\leq \frac{2}{\gamma} \sqrt{d} B_x (1+\sqrt{d} B_w)^{L-1} (\|\mathbf{W}^{(1)}\|_2 + \ldots + \|\mathbf{W}^{(L)}\|_2),
    \end{aligned}
\end{equation}
where $(a)$ is due to the fact that loss function is $\frac{2}{\gamma}$-Lipschitz continuous.

Therefore, by combining the above inequalites, we have
\begin{equation}
    \sum_{\ell=1}^L  \|\Delta \mathbf{G}^{(\ell)} \|_2 \leq \frac{2}{\gamma} (L+1) \Big( \sqrt{d} B_x (1+\sqrt{d} B_w)^{L} + (1+\sqrt{d} B_w) + \sqrt{d} \Big) \|\Delta \bm{\theta}\|_2 \cdot \sqrt{d} B_x (1+\sqrt{d} B_w)^{L}.
\end{equation}
\section{Generalization bound for \textbf{APPNP}} \label{supp:proof_appnp}

In the following, we provide detailed proof on the generalization bound of APPNP.
Recall that the update rule of APPNP is defined as
\begin{equation}
    \mathbf{H}^{(\ell)} = \alpha \mathbf{L} \mathbf{H}^{(\ell-1)} + (1-\alpha) \mathbf{H}^{(0)},~\mathbf{H}^{(0)}=\mathbf{W} \mathbf{X}.
\end{equation}
The training of APPNP is an empirical risk minimization with respect to weight parameters $\bm{\theta} = \{ \mathbf{W}^{(1)}, \ldots, \mathbf{W}^{(L)}, \bm{v} \}$, i.e.,
\begin{equation}
    \mathcal{L}(\bm{\theta}) = \frac{1}{m} \sum_{i=1}^m \Phi_\gamma (-p(f(\mathbf{h}_i^{(L)}), y_i)),~
    f(\mathbf{h}_i^{(L)}) = \tilde{\sigma}(\bm{v}^\top \mathbf{h}_i^{(L)}),
\end{equation}
where $\mathbf{h}_i^{(L)}$ is the node representation of the $i$th node at the final layer,  $f(\mathbf{h}_i^{(L)})$ is the predicted label for the $i$th node,
$\tilde{\sigma}(x) = \frac{1}{\exp(-x) + 1}$ is the sigmoid function, and loss function $\Phi_\gamma (-p(z, y))$ is $\frac{2}{\gamma}$-Lipschitz continuous with respect to its first input $z$. For simplification, we will use $\text{Loss}(z,y)$ denote $\Phi_\gamma (-p(z, y))$ in the proof.

To establish the generalization of APPNP as stated in Theorem~\ref{thm:uniform_stability_base}, we utilize the result on transductive uniform stability from~\cite{el2006stable} (Theorem~\ref{thm:uniform_stability_base_supp} in Appendix~\ref{supp:proof_gcn}).
Then, in Lemma~\ref{lemma:uniform_stable_appnp}, we derive the uniform stability constant for APPNP, i.e., $\epsilon_\text{APPNP}$. 

\begin{lemma}\label{lemma:uniform_stable_appnp} 
The uniform stability constant for APPNP is computed as $\epsilon_\texttt{APPNP} = \frac{2 \eta  \rho_f G_f}{m} \sum_{t=1}^T (1+\eta L_F)^{t-1} $ where
\begin{equation}
    \begin{aligned}
    \rho_f &= C_1 B_w,~
    G_f =  \frac{4}{\gamma} C_1, L_f~ =  \frac{4}{\gamma} C_1 \big( C_1 C_2 + 1 \big), \\
    C_1 &= B_d^\alpha B_x ,~C_2 = \max\{1, B_w\} , ~B_d^\alpha = (1-\alpha) \sum_{\ell=1}^L (\alpha \sqrt{d})^{\ell-1} + (\alpha \sqrt{d})^L.
    \end{aligned}
\end{equation}
\end{lemma}

By plugging the result in Lemma~\ref{lemma:uniform_stable_appnp} back to Theorem~\ref{thm:uniform_stability_base_supp}, we establish the generalization bound for APPNP.

The key idea of the proof is to decompose the change of the APPNP output into two terms (in Lemma~\ref{lemma:universal_boound_on_model_output}) which depend on 
\begin{itemize}
    \item (Lemma~\ref{lemma:appnp_h_max_bound_1}) The maximum change of node representation $\Delta h_{\max}^{(L)} = \max_i \| [\mathbf{H}^{(L)} - \tilde{\mathbf{H}}^{(L)}]_{i,:}\|_2$, 
    \item (Lemma~\ref{lemma:appnp_delta_h_max_bound_2}) The maximum node representation $h_{\max}^{(L)} = \max_i \|[ \mathbf{H}^{(L)}]_{i,:}\|_2$. 
\end{itemize}

\begin{lemma} [Upper bound of $h_{\max}^{(L)}$ for \textbf{APPNP}] \label{lemma:appnp_h_max_bound_1} 
Let suppose Assumption~\ref{assumption:norm_bound} hold. Then, the maximum node embeddings for any node at the $\ell$th layer is bounded by
\begin{equation}
    h_{\max}^{(L)} 
    \leq  B_d^\alpha B_x B_w,
\end{equation}
where $B_d^\alpha = (1-\alpha) \sum_{\ell=1}^L (\alpha \sqrt{d})^{\ell-1} + (\alpha \sqrt{d})^L $.
\end{lemma}
\begin{lemma} [Upper bound of $\Delta h_{\max}^{(L)}$ for \textbf{APPNP}] \label{lemma:appnp_delta_h_max_bound_2} 
Let suppose Assumption~\ref{assumption:norm_bound} hold. 
Then, the maximum change between the node embeddings  on two different set of weight parameters for any node at the $\ell$th layer is bounded by
\begin{equation}
    \Delta h_{\max}^{(L)} \leq B_d^\alpha B_x \| \Delta \mathbf{W} \|_2,
\end{equation}
where $B_d^\alpha = (1-\alpha) \sum_{\ell=1}^L (\alpha \sqrt{d})^{\ell-1} + (\alpha \sqrt{d})^L $ and $\Delta \mathbf{W}^{(\ell)} = \mathbf{W}^{(\ell)} - \tilde{\mathbf{W}}^{(\ell)}$.
\end{lemma}

Then, in Lemma~\ref{lemma:appnp_G_bound_4}, we decompose the change of the model parameters into two terms which depend on
\begin{itemize}
    \item The change of gradient with respect to the $\ell$th layer weight matrix $\|\Delta \mathbf{G}^{(\ell)}\|_2$, 
    \item The gradient with respect to the $\ell$th layer weight matrix $\|\mathbf{G}^{(\ell)}\|_2$, 
\end{itemize}
where $\|\Delta \mathbf{G}^{(\ell)}\|_2$ reflects the smoothness of APPNP model and $\|\mathbf{G}^{(\ell)} \|_2$ corresponds the upper bound of gradient.

\begin{lemma} [Upper bound of $\| \mathbf{G}^{(\ell)} \|_2, \|\Delta \mathbf{G}^{(\ell)} \|_2$  for \textbf{APPNP}] \label{lemma:appnp_G_bound_4}
Let suppose Assumption~\ref{assumption:norm_bound} hold.
Then, the gradient and the maximum change between gradients on two different set of weight parameters are bounded by
\begin{equation}
    \begin{aligned}
    \sum_{\ell=1}^{2} \| \mathbf{G}^{(\ell)} \|_2 &\leq \frac{4}{\gamma} B_d^\alpha B_x, \\
    \sum_{\ell=1}^{2} \|\Delta \mathbf{G}^{(1)} \|_2 &\leq \frac{4}{\gamma}  B_d^\alpha B_x \Big( B_d^\alpha B_x \max\{1, B_w\} + 1 \Big) \| \Delta \bm{\theta} \|_2, \\
    \end{aligned}
\end{equation}
where $\|\Delta \bm{\theta}\|_2 = \|\bm{v} - \tilde{\bm{v}}\|_2 + \| \mathbf{W} - \tilde{\mathbf{W}}\|_2$ denotes the change of two set of parameters.
\end{lemma}

Equipped with above intermediate results, we now proceed to prove  Lemma~\ref{lemma:uniform_stable_appnp}.

\begin{proof}
Recall that our goal is to explore the impact of different GCN structures on the uniform stability constant $\epsilon_\text{APPNP}$, which is a function of $\rho_f$, $G_f$, and $L_f$. 
By Lemma~\ref{lemma:appnp_delta_h_max_bound_2}, we know that the function $f$ is $\rho_f$-Lipschitz continuous, with $\rho_f = B_d^\alpha B_x$.

By Lemma~\ref{lemma:appnp_G_bound_4}, we know that the function $f$ is $L_f$-smoothness, and the gradient of each weight matrices is bounded by $G_f$, with 
\begin{equation*}
    G_f \leq \frac{4}{\gamma} B_d^\alpha B_x,~
    L_f \leq \frac{4}{\gamma}  B_d^\alpha B_x \Big( B_d^\alpha B_x \max\{1, B_w\} + 1 \Big).
\end{equation*}
By plugging $\epsilon_\text{APPNP}$ into Theorem~\ref{thm:uniform_stability_base}, we obtain the generalization bound of APPNP.
\end{proof}

\subsection{Proof of Lemma~\ref{lemma:appnp_h_max_bound_1}}
By the definition of $h_{\max}^{(L)}$, we have
\begin{equation}
    \begin{aligned}
    h_{\max}^{(L)} 
    &= \max_i \| [ \alpha \mathbf{L} \mathbf{H}^{(L-1)} + (1-\alpha) \mathbf{H}^{(0)} ]_{i,:}\|_2 \\
    &\leq \alpha \max_i \| [ \mathbf{L} \mathbf{H}^{(L-1)} ]_{i,:}\|_2 + (1-\alpha) \max_i \| [\mathbf{H}^{(0)}]_{i,:} \|_2 \\
    &=  \alpha \max_i \left\| \sum_{j=1}^N L_{i,j} \mathbf{h}_j^{(L-1)} \right\|_2 + (1-\alpha) B_w B_x \\
    &\underset{(a)}{\leq} \alpha \sqrt{d} h_{\max}^{(L-1)} + (1-\alpha) B_w B_x \\
    &\underset{(b)}{\leq} \Big( (1-\alpha) \sum_{\ell=1}^L (\alpha \sqrt{d})^{\ell-1} + (\alpha \sqrt{d})^L \Big) B_x B_w,
    \end{aligned}
\end{equation}
where inequality $(a)$ is follows from Lemma~\ref{lemma:laplacian_1_norm_bound} and inequality $(b)$ can be obtained by recursively applying $(a)$.

Let denote $B_d^\alpha = (1-\alpha) \sum_{\ell=1}^L (\alpha \sqrt{d})^{\ell-1} + (\alpha \sqrt{d})^L $, then we have
\begin{equation}
    h_{\max}^{(L)} \leq B_d^\alpha B_x B_w.
\end{equation}

\subsection{Proof of Lemma~\ref{lemma:appnp_delta_h_max_bound_2}}
By the definition of $\Delta h_{\max}^{(L)}$, we have
\begin{equation}
    \begin{aligned}
    \Delta h_{\max}^{(L)}
    &= \max_i \| \mathbf{h}_i^{(L)} - \tilde{\mathbf{h}}_i^{(L)}\|_2 \\
    &= \max_i \| [\big( \alpha \mathbf{L} \mathbf{H}^{(L-1)} + (1-\alpha) \mathbf{H}^{(0)} \big) - \big(\alpha \mathbf{L} \tilde{\mathbf{H}}^{(L-1)} + (1-\alpha) \tilde{\mathbf{H}}^{(0)} \big) ]_{i,:} \|_2 \\
    &\leq \alpha \max_i  \| [\mathbf{L}(\mathbf{H}^{(L-1)} - \tilde{\mathbf{H}}^{(L-1)}) ]_{i,:}\|_2 + (1-\alpha) \max_i \| [\mathbf{H}^{(0)} - \tilde{\mathbf{H}}^{(0)}]_{i,:}\|_2 \\
    &\underset{(a)}{\leq} \alpha \sqrt{d} \Delta h_{\max}^{(L-1)} + (1-\alpha) B_x \| \Delta \mathbf{W} \|_2 \\
    &\underset{(b)}{\leq} \Big(  (1-\alpha) \sum_{L=1}^L (\alpha \sqrt{d})^{\ell-1} + (\alpha \sqrt{d})^\ell \Big) B_x \| \Delta \mathbf{W} \|_2,
    \end{aligned}
\end{equation}
where inequality $(a)$ is follows from Lemma~\ref{lemma:laplacian_1_norm_bound} and inequality $(b)$ can be obtained by recursively applying $(a)$.

Let denote $B_d^\alpha = (1-\alpha) \sum_{\ell=1}^L (\alpha \sqrt{d})^{\ell-1} + (\alpha \sqrt{d})^L $, then we have
\begin{equation}
    \Delta h_{\max}^{(L)} \leq B_d^\alpha B_x \| \Delta \mathbf{W} \|_2.
\end{equation}

\subsection{Proof of Lemma~\ref{lemma:appnp_G_bound_4}}
By definition, we know that $\mathbf{H}^{(L)}$ is computed as
\begin{equation}
    \mathbf{H}^{(L)} = \Big( (\alpha \mathbf{L})^L +(1-\alpha) \big( 1+ (\alpha \mathbf{L}) + \ldots + (\alpha \mathbf{L})^{L-1}\big)\Big) \mathbf{X} \mathbf{W}.
\end{equation}

Therefore, the gradient with respect to weight $\mathbf{W}$ is computed as
\begin{equation}
    \begin{aligned}
    \| \mathbf{G}^{(1)} \|_2 
    &\underset{(a)}{\leq} \frac{2}{\gamma}  \max_i \left\| \frac{\partial f(\mathbf{h}_i)}{\partial \mathbf{W}} \right\|_2  \\
    &= \frac{2}{\gamma} \max_i \left\| \left[\Big( (\alpha \mathbf{L})^L +(1-\alpha) \big( 1+ (\alpha \mathbf{L}) + \ldots + (\alpha \mathbf{L})^{L-1}\big)\Big) \mathbf{X} \right]_{i,:} \right\|_2 \\
    &\leq \frac{2}{\gamma}  \Big( (\alpha \sqrt{d})^L + (1-\alpha)\big( 1 + (\alpha\sqrt{d}) + \ldots + (\alpha\sqrt{d})^{L-1} \big) \Big) B_x \\
    &\leq \frac{2}{\gamma}  B_d^\alpha B_x,
    \end{aligned}
\end{equation}
where inequality $(a)$ is due to the fact that the loss function is $\frac{2}{\gamma}$-Lipschitz continuous, and $B_d^\alpha = (1-\alpha) \sum_{\ell=1}^L (\alpha \sqrt{d})^{\ell-1} + (\alpha \sqrt{d})^L $.
Besides, the gradient with respect to weight $\bm{v}$ is computed as
\begin{equation}
    \begin{aligned}
    \| \mathbf{G}^{(2)} \|_2 &= \left\| \frac{\partial \Phi_\gamma (-p(f(\mathbf{h}_i^{(L)}), y_i)) }{\partial \bm{v}} \right\|_2 \\
    &\leq \frac{2}{\gamma} h_{\max}^{(L)} \\
    &\leq \frac{2}{\gamma} B_d^\alpha B_x B_w.
    \end{aligned}
\end{equation}

Combine the result above, we have $\sum_{\ell=1}^2 \| \mathbf{G}^{(\ell)} \|_2 \leq \frac{4}{\gamma}  B_d^\alpha B_x B_w $.
Then, the difference between two gradient computed on two different set of weight parameters is computed as
\begin{equation}
    \begin{aligned}
    \|\mathbf{G}^{(1)}- \tilde{\mathbf{G}}^{(1)} \|_2 &\leq \frac{2}{\gamma} \max_i \left\| \frac{\partial f(\mathbf{h}^{(L)}_i)}{\partial \mathbf{h}^{(L)}_i } \frac{\partial \mathbf{h}^{(L)}_i }{\partial \mathbf{W}}-  \frac{\partial f(\tilde{\mathbf{h}}^{(L)}_i ) }{\partial \tilde{\mathbf{h}}^{(L)}_i } \frac{\partial \tilde{\mathbf{h}}_i^{(L)}}{\partial \tilde{\mathbf{W}}} \right\|_2 \\
    &= \frac{2}{\gamma}  B_x B_d^\alpha  \left\| \frac{\partial f(\mathbf{h}^{(L)}_i)}{\partial \mathbf{h}^{(L)}_i } -  \frac{\partial f(\tilde{\mathbf{h}}^{(L)}_i ) }{\partial \tilde{\mathbf{h}}^{(L)}_i } \right\|_2 \\
    &\underset{(a)}{\leq} \frac{2}{\gamma}  B_x B_d^\alpha \Big(\Delta h_{\max}^{(L)} + (h_{\max}^{(L)} + 1) \| \Delta \bm{v} \|_2 \Big) \\
    &\underset{(b)}{\leq} \frac{2}{\gamma}  B_x B_d^\alpha \Big( B_x B_d^\alpha \| \Delta \mathbf{W} \|_2 + (B_x B_d^\alpha B_w + 1) \| \Delta \bm{v} \|_2 \Big) \\
    &\leq \frac{2}{\gamma}  B_x B_d^\alpha \Big(B_x B_d^\alpha \max\{1, B_w\} + 1\Big) \| \Delta \bm{\theta} \|_2,
    \end{aligned}
\end{equation}
where inequality $(a)$ follows from Lemma~\ref{lemma:universal_boound_on_model_output} and inequality $(b)$ follows from Lemma~\ref{lemma:appnp_delta_h_max_bound_2} and Lemma~\ref{lemma:appnp_h_max_bound_1}.

Similarly, the difference of gradient w.r.t. weight $\bm{v}$ is computed as
\begin{equation}
    \begin{aligned}
    \| \mathbf{G}^{(2)} - \tilde{\mathbf{G}}^{(2)} \|_2
    &= \left\| \frac{\partial \Phi_\gamma (-p(f(\mathbf{h}_i^{(L)}), y_i)) }{\partial \bm{v}} - \frac{\partial \Phi_\gamma (-p(\tilde{f}(\tilde{\mathbf{h}}_i^{(L)}), y_i)) }{\partial \tilde{\bm{v}}} \right\|_2 \\
    &\leq \frac{2}{\gamma} \Delta h_{\max}^{(L)} \\
    &\leq \frac{2}{\gamma} B_d^\alpha B_x \|\Delta \mathbf{W}\|_2.
    \end{aligned}
\end{equation}
Combining result above yields $\sum_{\ell=1}^2 \|\mathbf{G}^{(\ell)} - \tilde{\mathbf{G}}^{(\ell)}\|_2 \leq \frac{4}{\gamma}  B_x B_d^\alpha \Big(B_x B_d^\alpha \max\{1, B_w\} + 1\Big) \| \Delta \bm{\theta} \|_2$ as desired.
\section{Generalization bound for GCNII}\label{supp:proof_gcnii}

In the following, we provide generalization bound of the GCNII.
Recall that the update rule of GCNII is defined as
\begin{equation}
    \mathbf{H}^{(\ell)} = \sigma\Big( \big( \alpha\mathbf{L} \mathbf{H}^{(\ell)} + (1-\alpha) \mathbf{X} \big) \big( \beta \mathbf{W}^{(\ell)} + (1-\beta) \mathbf{I}_N \big)\Big),
\end{equation}
where $\sigma(\cdot)$ is ReLU activation function. Although ReLU function $\sigma(x)$ is not differentiable when $x=0$, for analysis purpose, we suppose the $\sigma^\prime(0) = 0$ which is also used in optimization.

The training of GCNII is an empirical risk minimization with respect to a set of parameters $\bm{\theta} = \{ \mathbf{W}^{(1)}, \ldots, \mathbf{W}^{(L)}, \bm{v} \}$, i.e.,
\begin{equation}
    \mathcal{L}(\bm{\theta}) = \frac{1}{m} \sum_{i=1}^m \Phi_\gamma (-p(f(\mathbf{h}_i^{(L)}), y_i)),~
    f(\mathbf{h}_i^{(L)}) = \tilde{\sigma}(\bm{v}^\top \mathbf{h}_i^{(L)}),
\end{equation}
where $\mathbf{h}_i^{(L)}$ is the node representation of the $i$th node at the final layer,  $f(\mathbf{h}_i^{(L)})$ is the predicted label for the $i$th node,
$\tilde{\sigma}(x) = \frac{1}{\exp(-x) + 1}$ is the sigmoid function, and loss function $\Phi_\gamma (-p(z, y))$ is $\frac{2}{\gamma}$-Lipschitz continuous with respect to its first input $z$. For simplification, let $\text{Loss}(z,y)$ denote $\Phi_\gamma (-p(z, y))$ in the proof.

To establish the generalization of GCNII as stated in Theorem~\ref{thm:uniform_stability_base}, we utilize the result on transductive uniform stability from~\cite{el2006stable} (Theorem~\ref{thm:uniform_stability_base_supp} in Appendix~\ref{supp:proof_gcn}).
Then, in Lemma~\ref{lemma:uniform_stable_gcnii}, we derive the uniform stability constant for GCNII, i.e., $\epsilon_\text{GCNII}$. 

\begin{lemma}\label{lemma:uniform_stable_gcnii} 
The uniform stability constant for GCNII is computed as $\epsilon_\texttt{GCNII} = \frac{2 \eta  \rho_f G_f}{m} \sum_{t=1}^T (1+\eta L_F)^{t-1} $ where
\begin{equation}
    \begin{aligned}
    \rho_f &= \beta C_1^L C_2,~
    G_f = \beta  \frac{2}{\gamma} (L+1) C_1^{L} C_2, \\
    L_f &= \alpha \beta  \frac{2}{\gamma} (L+1) \sqrt{d} C_1^L C_2 \Big( (\alpha \beta L + \beta B_w + 2) C_1^L C_2 + 2\beta \Big),  \\
    C_1 &= \max\{ 1, \alpha \sqrt{d} B_w^\beta \},~
    C_2 = \sqrt{d} + B_{\ell,d}^{\alpha,\beta} B_x ,~\\
    B_w^\beta&=\beta B_w + (1-\beta),~
    B_{\ell,d}^{\alpha,\beta} = \max\Big\{ \beta \big( (1-\alpha)L + \alpha\sqrt{d} \big), (1-\alpha) L B_w^\beta + 1 \Big\}.
    \end{aligned}
\end{equation}
\end{lemma}

By plugging the result in Lemma~\ref{lemma:uniform_stable_gcnii} back to Theorem~\ref{thm:uniform_stability_base_supp}, we establish the generalization bound for GCNII.

The key idea of the proof is to decompose the change of the GCNII output into two terms (in Lemma~\ref{lemma:universal_boound_on_model_output}) which depend on 
\begin{itemize}
    \item (Lemma~\ref{lemma:gcnii_h_max_bound_1}) The maximum change of node representation $\Delta h_{\max}^{(\ell)} = \max_i \| [\mathbf{H}^{(\ell)} - \tilde{\mathbf{H}}^{(\ell)}]_{i,:}\|_2$, 
    \item (Lemma~\ref{lemma:gcnii_delta_h_max_bound_2}) The maximum node representation $h_{\max}^{(\ell)} = \max_i \|[ \mathbf{H}^{(\ell)}]_{i,:}\|_2$.
\end{itemize}

\begin{lemma} [Upper bound of $h_{\max}^{(\ell)}$ for \textbf{GCNII}] \label{lemma:gcnii_h_max_bound_1}
Let suppose Assumption~\ref{assumption:norm_bound} hold. Then, the maximum node embeddings for any node at the $\ell$th layer is bounded by
\begin{equation}
    h_{\max}^{(\ell)} \leq \big((1-\alpha)\ell B_w^\beta +1\big)B_x  \cdot \big( \max\{1, \alpha \sqrt{d} B_w^\beta \} \big)^\ell,
\end{equation}
where $B_w^\beta =  \beta B_w + (1-\beta )$ and $\Delta \mathbf{W}^{(\ell)} = \mathbf{W}^{(\ell)} - \tilde{\mathbf{W}}^{(\ell)}$.
\end{lemma}

\begin{lemma} [Upper bound of $\Delta h_{\max}^{(\ell)}$ for \textbf{GCNII}] \label{lemma:gcnii_delta_h_max_bound_2}
Let suppose Assumption~\ref{assumption:norm_bound} hold. 
Then, the maximum change between the node embeddings  on two different set of weight parameters for any node at the $\ell$th layer is bounded by
\begin{equation}
    \begin{aligned}
    \Delta h_{\max}^{(\ell)}\leq \beta B_x \big( (1-\alpha) \ell + \alpha \sqrt{d} \big) \big(\max\{1, \alpha\sqrt{d} B_w^\beta\} \big)^{\ell-1} \Big( \|\Delta \mathbf{W}^{(1)}\|_2 + \ldots + \|\Delta \mathbf{W}^{(\ell)}\|_2 \Big),
    \end{aligned}
\end{equation}
where $B_w^\beta =  \beta B_w + (1-\beta )$ and $\Delta \mathbf{W}^{(\ell)} = \mathbf{W}^{(\ell)} - \tilde{\mathbf{W}}^{(\ell)}$.
\end{lemma}

Besides, in Lemma~\ref{lemma:gcnii_delta_z_max_bound_5}, we derive the upper bound on the maximum change of node embeddings before the activation function, i.e., $\Delta z_{\max}^{(\ell)} = \max_i \| [\mathbf{Z}^{(\ell)} - \tilde{\mathbf{Z}}^{(\ell)} ]_{i,:} \|_2$, which will be used to derive the gradient related upper bounds.
\begin{lemma} [Upper bound of $\Delta z_{\max}^{(\ell)}$ for \textbf{GCNII}] \label{lemma:gcnii_delta_z_max_bound_5}
Let suppose Assumption~\ref{assumption:norm_bound} hold. 
Then, the maximum change between the node embeddings before the activation function  on two different set of weight parameters for any node at the $\ell$th layer is bounded by
\begin{equation}
    \Delta z_{\max}^{(\ell)}  \leq 
    \beta B_x \big( (1-\alpha) \ell + \alpha \sqrt{d} \big) (\max\{ 1, \alpha\sqrt{d} B_w^\beta\})^{\ell-1} \Big( \|\Delta \mathbf{W}^{(1)}\|_2 + \ldots + \|\Delta \mathbf{W}^{(\ell)}\|_2 \Big),
\end{equation}
where $B_w^\beta =  \beta B_w + (1-\beta )$ and $\Delta \mathbf{W}^{(\ell)} = \mathbf{W}^{(\ell)} - \tilde{\mathbf{W}}^{(\ell)}$.
\end{lemma}

Then, in Lemma~\ref{lemma:gcnii_d_max_bound_3}, we decompose the change of the model parameters into two terms which depend on
\begin{itemize}
    \item The maximum change of gradient passing from $(\ell+1)$th layer to the $\ell$th layer $\Delta d_{\max}^{(\ell)} = \max_i \Big\| \Big[\frac{\partial \mathbf{H}^{(\ell)}}{\partial \mathbf{H}^{(\ell-1)}} - \frac{\partial \tilde{\mathbf{H}}^{(\ell)}}{\partial \tilde{\mathbf{H}}^{(\ell-1)}} \Big]_{i,:} \Big\|_2$, 
    \item The maximum gradient passing from $(\ell+1)$th layer to the $\ell$th layer $d_{\max}^{(\ell)} = \max_i \Big\| \Big[\frac{\partial \mathbf{H}^{(\ell)} }{\partial \mathbf{H}^{(\ell-1)}} \Big]_{i,:} \Big\|_2$.
\end{itemize}

\begin{lemma} [Upper bound of $\mathbf{d}^{(\ell)}_{\max}, \Delta \mathbf{d}^{(\ell)}_{\max}$ for \textbf{GCNII}] \label{lemma:gcnii_d_max_bound_3}
Let suppose Assumption~\ref{assumption:norm_bound} hold and let denote $C_1 = \max\{1, \alpha \sqrt{d} B_w^\beta\}$, $C_2 = \sqrt{d} + B_{\ell,d}^{\alpha,\beta} B_x $, $B_w^\beta = \beta B_w + (1-\beta)$, and $B_{\ell,d}^{\alpha,\beta} = \max\big\{ \beta \big( (1-\alpha)L + \alpha\sqrt{d} \big), (1-\alpha) L B_w^\beta + 1 \big\}$. 
Then, the maximum gradient passing from layer $\ell+1$ to layer $\ell$  for any node is bounded by
\begin{equation}
    d_{\max}^{(\ell)} \leq \frac{2}{\gamma} C_1^{L-\ell+1} \|\Delta \bm{\theta} \|_2,
\end{equation}
and the maximum change between the gradient passing from layer $\ell+1$ to layer $\ell$  on two different set of weight parameters for any node is bounded by
\begin{equation}
    \begin{aligned}
    \Delta d_{\max}^{(\ell)} &\leq \frac{2}{\gamma} C_1^{L-\ell} \Big( (\alpha \beta L + \beta B_w + 1) C_1^L C_2 + 1\Big) \| \Delta \bm{\theta} \|_2,
    \end{aligned}
\end{equation}
where $\|\Delta \bm{\theta}\|_2 = \|\bm{v} - \tilde{\bm{v}}\|_2 + \sum_{\ell-1}^L \| \mathbf{W}^{(\ell)} - \tilde{\mathbf{W}}^{(\ell)}\|_2$ denotes the change of two set of parameters.
\end{lemma}

Finally, in Lemma~\ref{lemma:gcnii_G_bound_4}, we decompose the change of the model parameters into two terms which depend on
\begin{itemize}
    \item The change of gradient with respect to the $\ell$th layer weight matrix $\|\Delta \mathbf{G}^{(\ell)}\|_2$, 
    \item The gradient with respect to the $\ell$th layer weight matrix $\|\mathbf{G}^{(\ell)} \|_2$,
\end{itemize}
where $\mathbf{G}^{(L+1)}$ denotes the gradient with respect to the weight $\bm{v}$ of the binary classifier and $\mathbf{G}^{(\ell)}$ denotes the gradient with respect to the weight matrices $\mathbf{W}^{(\ell)}$ of the graph convolutional layer.
Notice that $\|\Delta \mathbf{G}^{(\ell)}\|_2$ reflect the smoothness of GCN model and $\|\mathbf{G}^{(\ell)} \|_2$ correspond the upper bound of gradient.

\begin{lemma} [Upper bound of $\mathbf{G}^{(\ell)}, \Delta \mathbf{G}^{(\ell)}$ for \textbf{GCNII}] \label{lemma:gcnii_G_bound_4}
Let suppose Assumption~\ref{assumption:norm_bound} hold, and let $C_1 = \max\{1, \alpha \sqrt{d} B_w^\beta\}$, $C_2 = \sqrt{d} + B_{\ell,d}^{\alpha,\beta} B_x $, $B_w^\beta = \beta B_w + (1-\beta)$, and $B_{\ell,d}^{\alpha,\beta} = \max\big\{ \beta \big( (1-\alpha)L + \alpha\sqrt{d} \big), (1-\alpha) L B_w^\beta + 1 \big\}$.
Then, the gradient and the maximum change between gradients on two different set of weight parameters are bounded by
\begin{equation}
    \begin{aligned}
    \sum_{\ell=1}^{L+1} \| \mathbf{G}^{(\ell)} \|_2 
    &\leq  \frac{2}{\gamma} (L+1) \beta C_1^L C_2, \\ 
    \sum_{\ell=1}^{L+1} \|\Delta \mathbf{G}^{(\ell)} \|_2 
    &\leq  \alpha \beta  \frac{2}{\gamma} (L+1) \sqrt{d} C_1^L C_2 \Big( (\alpha \beta L + \beta B_w + 2) C_1^L C_2 + 2\beta \Big) \| \Delta \bm{\theta} \|_2,
    \end{aligned}
\end{equation}
where $\|\Delta \bm{\theta}\|_2 = \|\bm{v} - \tilde{\bm{v}}\|_2 + \sum_{\ell-1}^L \| \mathbf{W}^{(\ell)} - \tilde{\mathbf{W}}^{(\ell)}\|_2$ denotes the change of two set of parameters.
\end{lemma}

Equipped with above intermediate results, we now proceed to prove  Lemma~\ref{lemma:uniform_stable_gcnii}.
\begin{proof}
Recall that our goal is to explore the impact of different GCN structures on the uniform stability constant $\epsilon_\text{GCNII}$, which is a function of $\rho_f$, $G_f$, and $L_f$. 
Let denote $C_1 = \max\{1, \alpha \sqrt{d} B_w^\beta\}$, $C_2 = \sqrt{d} + B_{\ell,d}^{\alpha,\beta} B_x $, $B_w^\beta = \beta B_w + (1-\beta)$, and $B_{\ell,d}^{\alpha,\beta} = \max\big\{ \beta \big( (1-\alpha)L + \alpha\sqrt{d} \big), (1-\alpha) L B_w^\beta + 1 \big\}$.
By Lemma~\ref{lemma:gcnii_delta_h_max_bound_2}, we know that the function $f$ is $\rho_f$-Lipschitz continuous, with
\begin{equation}
    \rho_f = \beta C_1^L C_2.
\end{equation}
By Lemma~\ref{lemma:gcnii_G_bound_4}, we know that the function $f$ is $L_f$-smoothness, and the gradient of each weight matrices is bounded by $G_f$, with
\begin{equation}
    \begin{aligned}
    G_f &=  \frac{2}{\gamma} (L+1) \beta C_1^L C_2, \\
    L_f &= \alpha \beta  \frac{2}{\gamma} (L+1) \sqrt{d} C_1^L C_2 \Big( (\alpha \beta L + \beta B_w + 2) C_1^L C_2 + 2\beta \Big).
    \end{aligned}
\end{equation}
where $B_{\ell,d}^{\alpha,\beta} = \max\Big\{ \beta \big( (1-\alpha)L + \alpha\sqrt{d} \big), (1-\alpha) L B_w^\beta + 1 \Big\}$.

By plugging $\epsilon_\text{GCNII}$ into Theorem~\ref{thm:uniform_stability_base}, we obtain the generalization bound of GCNII.

\end{proof}


\subsection{Proof of Lemma~\ref{lemma:gcnii_h_max_bound_1}}

By the definition of $h_{\max}^{(\ell)}$, we have
\begin{equation}
    \begin{aligned}
    h_{\max}^{(\ell)} 
    &= \max_i \| [\sigma\Big( \big(\alpha \mathbf{L} \mathbf{H}^{(\ell-1)} + (1-\alpha) \mathbf{X}\big) \big(\beta \mathbf{W}^{(\ell)} + (1-\beta) \mathbf{I}\big)\Big)]_{i,:} \|_2 \\
    &\underset{(a)}{\leq} \max_i \| [\big(\alpha \mathbf{L} \mathbf{H}^{(\ell-1)} + (1-\alpha) \mathbf{X}\big) \big(\beta \mathbf{W}^{(\ell)} + (1-\beta) \mathbf{I} \big)]_{i,:} \|_2 \\
    &\leq \max_i \| [\big(\alpha \mathbf{L} \mathbf{H}^{(\ell-1)} + (1-\alpha) \mathbf{X}\big) ]_{i,:} \|_2 \| \beta \mathbf{W}^{(\ell)} + (1-\beta) \mathbf{I} \|_2 \\
    &\leq  (\beta B_w + (1-\beta)) \cdot \max_i \| [\big(\alpha \mathbf{L} \mathbf{H}^{(\ell-1)} + (1-\alpha) \mathbf{X}\big) ]_{i,:} \|_2 \\
    &\leq  (\beta B_w + (1-\beta)) \cdot \Big( \alpha \max_i \| [ \mathbf{L} \mathbf{H}^{(\ell-1)} ]_{i,:} \|_2 + (1-\alpha) \max_i \| [ \mathbf{X} ]_{i,:} \|_2 \Big) \\
    &\leq (\beta B_w + (1-\beta)) \cdot \Big( \alpha \sqrt{d} h_{\max}^{(\ell-1)} + (1-\alpha) B_x \Big) \\
    &= \alpha \sqrt{d} (\beta B_w + (1-\beta)) h_{\max}^{(\ell-1)} + (1-\alpha) B_x (\beta B_w + (1-\beta)),
    \end{aligned}
\end{equation}
where inequality $(a)$ is due to $\|\sigma(x)\|_2 \leq \|x\|_2$.

Let $B_w^\beta =  \beta B_w + (1-\beta )$.
By induction, we have
\begin{equation}
    \begin{aligned}
    h_{\max}^{(\ell)} 
    &\leq \Big( (1-\alpha) B_x B_w^\beta\Big) \sum_{\ell^\prime=1}^\ell \Big( \alpha \sqrt{d} B_w^\beta  \Big)^{\ell^\prime-1} + \Big( \alpha \sqrt{d} B_w^\beta \Big)^{\ell} B_x. \\
    \end{aligned}
\end{equation}
If $\alpha \sqrt{d} B_w^\beta \leq 1$ we have,
\begin{equation}
    h_{\max}^{(\ell)} \leq (1-\alpha)B_x B_w^\beta \ell + B_x = \Big((1-\alpha)\ell B_w^\beta +1\Big)B_x,
\end{equation}
and if $\alpha \sqrt{d} B_w^\beta > 1$ we have
\begin{equation}
    \begin{aligned}
    h_{\max}^{(\ell)} 
    &\leq \Big( (1-\alpha) B_x B_w^\beta \Big) \ell \Big( \alpha \sqrt{d} B_w^\beta  \Big)^{\ell-1} + \Big( \alpha \sqrt{d} B_w^\beta \Big)^{\ell} B_x \\
    &\leq \Big((1-\alpha)\ell B_w^\beta +1\Big)B_x  \cdot \Big( \alpha \sqrt{d} B_w^\beta \Big)^\ell.
    \end{aligned}
\end{equation}

By combining results above, we have
\begin{equation}
    h_{\max}^{(\ell)} \leq \Big((1-\alpha)\ell B_w^\beta +1\Big)B_x  \cdot \Big( \max\{1, \alpha \sqrt{d} B_w^\beta \} \Big)^\ell.
\end{equation}

\subsection{Proof of Lemma~\ref{lemma:gcnii_delta_h_max_bound_2}}

By the definition of $\Delta h_{\max}^{(\ell)}$, we have
\begin{equation}
    \begin{aligned}
    \Delta h_{\max}^{(\ell)}
    &= \max_i \| \mathbf{h}_i^{(\ell)} - \tilde{\mathbf{h}}_i^{(\ell)}\|_2 \\
    &= \max_i \Big\| \Big[\sigma\Big( \big(\alpha \mathbf{L} \mathbf{H}^{(\ell-1)} + (1-\alpha) \mathbf{X}\big) \big(\beta \mathbf{W}^{(\ell)} + (1-\beta) \mathbf{I}\big)\Big) \\
    &\quad - \sigma\Big( \big(\alpha \mathbf{L} \tilde{\mathbf{H}}^{(\ell-1)} + (1-\alpha) \mathbf{X}\big) \big(\beta \tilde{\mathbf{W}}^{(\ell)} + (1-\beta) \mathbf{I}\big)\Big)\Big]_{i,:} \Big\|_2 \\ 
    &\leq \max_i \Big\| \Big[ \big(\alpha \mathbf{L} \mathbf{H}^{(\ell-1)} + (1-\alpha) \mathbf{X}\big) \big(\beta \mathbf{W}^{(\ell)} + (1-\beta) \mathbf{I}\big) \\
    &\quad - \big(\alpha \mathbf{L} \tilde{\mathbf{H}}^{(\ell-1)} + (1-\alpha) \mathbf{X}\big) \big(\beta \tilde{\mathbf{W}}^{(\ell)} + (1-\beta) \mathbf{I}\big) \Big]_{i,:} \Big\|_2 \\
    &\underset{(a)}{\leq} \max_i \Big\| \Big[ \big(\alpha \mathbf{L} \mathbf{H}^{(\ell-1)} + (1-\alpha) \mathbf{X}\big) \big(\beta (\mathbf{W}^{(\ell)} - \tilde{\mathbf{W}}^{(\ell)}) \big) \Big]_{i,:} \Big\|_2 \\
    &\quad + \max_i \Big\| \Big[ \big(\alpha \mathbf{L} (\mathbf{H}^{(\ell-1)} - \tilde{\mathbf{H}}^{(\ell-1)}) \big) \big(\beta \tilde{\mathbf{W}}^{(\ell)} + (1-\beta) \mathbf{I}\big) \Big]_{i,:} \Big\|_2 \\
    &\leq \max_i \| [\alpha \mathbf{L} \mathbf{H}^{(\ell-1)} + (1-\alpha) \mathbf{X}]_{i,:} \|_2 \cdot \beta \|\Delta \mathbf{W}^{(\ell)}\|_2 + \alpha (\beta B_w + (1-\beta)) \max_i \| [\mathbf{L} (\mathbf{H}^{(\ell-1)} - \tilde{\mathbf{H}}^{(\ell-1)} )]_{i,:} \|_2 \\
    &\leq \Big( \alpha \sqrt{d} h_{\max}^{(\ell-1)} + (1-\alpha) B_x \Big) \cdot \beta \|\Delta \mathbf{W}^{(\ell)}\|_2 + \alpha \big(\beta B_w + (1-\beta)\big) \sqrt{d} \Delta h_{\max}^{(\ell-1)}, \\
    \end{aligned}
\end{equation}
where $(a)$ is due to the fact that $\|\sigma(x)\|_2 \leq \|x\|_2$.

Let $B_w^\beta =  \beta B_w + (1-\beta )$. Then, by induction we have
\begin{equation}
    \begin{aligned}
    \Delta h_{\max}^{(\ell)} 
    &\leq \alpha \sqrt{d} B_w^\beta  \cdot \Delta h_{\max}^{(\ell-1)} + \beta \Big( \alpha \sqrt{d} h_{\max}^{(\ell-1)} + (1-\alpha) B_x \Big) \|\Delta \mathbf{W}^{(\ell)}\|_2 \\
    &\leq (\alpha \sqrt{d} B_w^\beta)^2 \cdot \Delta h_{\max}^{(\ell-2)} + \beta \Big( \alpha \sqrt{d} h_{\max}^{(\ell-1)} + (1-\alpha) B_x \Big) \|\Delta \mathbf{W}^{(\ell)}\|_2 \\
    &\quad + \beta (\alpha \sqrt{d} B_w^\beta) \Big( \alpha \sqrt{d} h_{\max}^{(\ell-2)} + (1-\alpha) B_x \Big) \|\Delta \mathbf{W}^{(\ell-1)}\|_2 \\
    &\ldots \\
    &\leq \beta \Big( \alpha \sqrt{d} h_{\max}^{(\ell-1)} + (1-\alpha) B_x \Big) \|\Delta \mathbf{W}^{(\ell)}\|_2 \\
    &\quad +  \beta (\alpha \sqrt{d} B_w^\beta) \Big( \alpha \sqrt{d} h_{\max}^{(\ell-2)} + (1-\alpha) B_x \Big) \|\Delta \mathbf{W}^{(\ell-1)}\|_2 + \ldots \\
    &\quad + \beta (\alpha \sqrt{d} B_w^\beta)^{\ell-1} \Big( \alpha \sqrt{d} h_{\max}^{(0)} + (1-\alpha) B_x \Big) \|\Delta \mathbf{W}^{(1)}\|_2.
    \end{aligned}
\end{equation}

If $\alpha \sqrt{d} B_w^\beta \leq 1$, using Lemma~\ref{lemma:gcnii_delta_h_max_bound_2} we have
\begin{equation}
    \begin{aligned}
    \Delta h_{\max}^{(\ell)}  
    &\leq \beta \Big( \alpha \sqrt{d} h_{\max}^{(\ell-1)} + (1-\alpha) B_x \Big) \|\Delta \mathbf{W}^{(\ell)}\|_2 + \beta \Big( \alpha \sqrt{d} h_{\max}^{(\ell-2)} + (1-\alpha) B_x \Big) \|\Delta \mathbf{W}^{(\ell-1)}\|_2 + \ldots \\
    &\quad + \beta \Big( \alpha \sqrt{d} h_{\max}^{(0)} + (1-\alpha) B_x \Big) \|\Delta \mathbf{W}^{(1)}\|_2 \\
    &\leq \beta B_x \Big( (1-\alpha)\ell + \alpha \sqrt{d} \Big) \Big( \|\Delta \mathbf{W}^{(1)}\|_2 + \ldots + \|\Delta \mathbf{W}^{(\ell)}\|_2 \Big),
    \end{aligned}
\end{equation}
and if $\alpha \sqrt{d} B_w^\beta > 1$, using Lemma~\ref{lemma:gcnii_delta_h_max_bound_2} we have
\begin{equation}
    \begin{aligned}
    \Delta h_{\max}^{(\ell)}  
    &\leq \beta \Big( \alpha \sqrt{d} \big((1-\alpha)(\ell-1) B_w^\beta +1\big)B_x  \cdot \big( \alpha \sqrt{d} B_w^\beta \big)^{\ell-1} + (1-\alpha) B_x \Big) \|\Delta \mathbf{W}^{(\ell)}\|_2 \\
    &\quad +  \beta (\alpha \sqrt{d} B_w^\beta) \Big( \alpha \sqrt{d} \big((1-\alpha)(\ell-2) B_w^\beta +1\big)B_x  \cdot \big( \alpha \sqrt{d} B_w^\beta \big)^{\ell-2} + (1-\alpha) B_x \Big) \|\Delta \mathbf{W}^{(\ell-1)}\|_2 + \ldots \\
    &\quad + \beta (\alpha \sqrt{d} B_w^\beta)^{\ell-1} \Big( \alpha \sqrt{d} B_x + (1-\alpha) B_x \Big) \|\Delta \mathbf{W}^{(1)}\|_2 \\
    &\leq \beta B_x \big( (1-\alpha) \ell + \alpha \sqrt{d} \big) (\alpha\sqrt{d} B_w^\beta)^{\ell-1} \Big( \|\Delta \mathbf{W}^{(1)}\|_2 + \ldots + \|\Delta \mathbf{W}^{(\ell)}\|_2 \Big). \\
    \end{aligned}
\end{equation}

By combining the result above, we have
\begin{equation}
    \Delta h_{\max}^{(\ell)}  \leq 
    \beta B_x \big( (1-\alpha) \ell + \alpha \sqrt{d} \big) (\max\{ 1, \alpha\sqrt{d} B_w^\beta\})^{\ell-1} \Big( \|\Delta \mathbf{W}^{(1)}\|_2 + \ldots + \|\Delta \mathbf{W}^{(\ell)}\|_2 \Big).
\end{equation}

\subsection{Proof of Lemma~\ref{lemma:gcnii_delta_z_max_bound_5}}
The proof is similar to the proof of Lemma~\ref{lemma:gcnii_delta_h_max_bound_2}.

\subsection{Proof of Lemma~\ref{lemma:gcnii_d_max_bound_3}}

Let $B_w^\beta =  \beta B_w + (1-\beta )$.
By the definition of $d_{\max}^{(\ell)}$, we have

\begin{equation}
    \begin{aligned}
    d_{\max}^{(\ell)} 
    &= \max_i \|[\mathbf{D}^{(\ell)}]_{i,:}\|_2 \\
    &= \alpha \max_i \| [\mathbf{L}^\top \sigma^\prime(\mathbf{Z}^{(\ell)}) \odot \mathbf{D}^{(\ell+1)} \big( \beta \mathbf{W}^{(\ell)}+(1-\beta) \mathbf{I} \big) ]_{i,:} \|_2  \\
    &\leq \alpha \max_i \| [\mathbf{L}^\top \sigma^\prime(\mathbf{Z}^{(\ell)}) \odot \mathbf{D}^{(\ell+1)}]_{i,:} \|_2 \cdot B_w^\beta \\
    &\underset{(a)}{\leq} \alpha \sqrt{d} B_w^\beta \cdot d_{\max}^{(\ell+1)} \\
    &\leq \frac{2}{\gamma} (\max\{1, \alpha \sqrt{d} B_w^\beta \})^{L-\ell+1}, 
    \end{aligned}
\end{equation}
where inequality $(a)$ is due to the gradient of ReLU activation is element-wise either $0$ and $1$.

By the definition of $\Delta d_{\max}^{(\ell)}$, we have
\begin{equation}
    \begin{aligned}
    \Delta d_{\max}^{(\ell)} &= \max_i \|[\mathbf{D}^{(\ell)} - \tilde{\mathbf{D}}^{(\ell)}]_{i,:}\|_2 \\
    &\leq \max_i \alpha \| [\mathbf{L}^\top \sigma^\prime(\mathbf{Z}^{(\ell)}) \odot \mathbf{D}^{(\ell+1)} \big( \beta \mathbf{W}^{(\ell)}+(1-\beta) \mathbf{I} \big) - \mathbf{L}^\top \sigma^\prime(\tilde{\mathbf{Z}}^{(\ell)}) \odot \tilde{\mathbf{D}}^{(\ell+1)} \big( \beta \tilde{\mathbf{W}}^{(\ell)} - (1-\beta) \mathbf{I} \big) ]_{i,:} \|_2 \\
    &\underset{(a)}{\leq} \max_i \alpha \| [\mathbf{L}^\top \mathbf{D}^{(\ell+1)} \big( \beta \mathbf{W}^{(\ell)}+(1-\beta) \mathbf{I} \big) - \mathbf{L}^\top \tilde{\mathbf{D}}^{(\ell+1)} \big( \beta \mathbf{W}^{(\ell)}+(1-\beta) \mathbf{I} \big) ]_{i,:} \|_2 \\
    &\quad + \max_i \alpha \| [\mathbf{L}^\top \tilde{\mathbf{D}}^{(\ell+1)} \big( \beta \mathbf{W}^{(\ell)}+(1-\beta) \mathbf{I} \big) ]_{i,:} - \mathbf{L}^\top \tilde{\mathbf{D}}^{(\ell+1)} \big( \beta \tilde{\mathbf{W}}^{(\ell)} - (1-\beta) \mathbf{I} \big)]_{i,:} \|_2 \\
    &\quad + \alpha \max_i \| [ \mathbf{L}^\top \tilde{\mathbf{D}}^{(\ell+1)} \big( \beta \tilde{\mathbf{W}}^{(\ell)} - (1-\beta) \mathbf{I} \big) ]_{i,:}\|_2 \max_i \| [\mathbf{Z}^{(\ell)} - \tilde{\mathbf{Z}}^{(\ell)}]_{i,:}\|_2 \\ 
    &\underset{(b)}{\leq} \max_i \alpha \| [\mathbf{L}^\top (\mathbf{D}^{(\ell+1)} - \tilde{\mathbf{D}}^{(\ell+1)}) \big( \beta \mathbf{W}^{(\ell)} + (1-\beta) \mathbf{I} \big) ]_{i,:} \|_2 \\
    &\quad + \max_i \alpha \| [\mathbf{L}^\top \tilde{\mathbf{D}}^{(\ell+1)} \big( \beta (\mathbf{W}^{(\ell)} - \tilde{\mathbf{W}}^{(\ell)}) \big) ]_{i,:} \|_2  \\
    &\quad + \alpha \max_i \| [ \mathbf{L}^\top \tilde{\mathbf{D}}^{(\ell+1)} \big( \beta \tilde{\mathbf{W}}^{(\ell)} - (1-\beta) \mathbf{I} \big) ]_{i,:}\|_2 \max_i \| [\mathbf{Z}^{(\ell)} - \tilde{\mathbf{Z}}^{(\ell)}]_{i,:}\|_2 \\
    &\leq \alpha \sqrt{d} B_w^\beta \cdot \Delta d_{\max}^{(\ell+1)} + \alpha \beta \sqrt{d} d_{\max}^{(\ell+1)} \| \Delta \mathbf{W}^{(\ell)}\|_2  + \alpha\sqrt{d} B_w^\beta d_{\max}^{(\ell+1)} \Delta z_{\max}^{(\ell)} \\
    &= \alpha \sqrt{d} B_w^\beta \cdot \Delta d_{\max}^{(\ell+1)} + \underbrace{\alpha  d_{\max}^{(\ell+1)} \Big( \beta \sqrt{d} \| \Delta \mathbf{W}^{(\ell)}\|_2  +  \sqrt{d} B_w^\beta  \Delta z_{\max}^{(\ell)} \Big) }_{(A)},
    \end{aligned}
\end{equation}
where inequality $(a)$ and $(b)$ is due to the gradient of ReLU activation is element-wise either $0$ and $1$.

Let $C_1 = \max\{1, \alpha \sqrt{d} B_w^\beta \}$ and $B_{\ell,d}^{\alpha,\beta} = \max\big\{ \beta \big( (1-\alpha)L + \alpha\sqrt{d} \big), (1-\alpha) L B_w^\beta + 1 \big\}$.
Recall that $d_{\max}^{(\ell+1)} \leq \frac{2}{\gamma} (\max\{1, \alpha \sqrt{d} B_w^\beta \})^{L-\ell} $ and the upper bound of $\Delta z_{\max}^{(\ell)}$ in Lemma~\ref{lemma:gcnii_delta_z_max_bound_5}, we have 
\begin{equation}
    d_{\max}^{(\ell+1)} \leq \frac{2}{\gamma} C_1^{L-\ell},~
    \Delta z_{\max}^{(\ell)} \leq \beta B_x B_{\ell,d}^{\alpha,\beta} C_1^{\ell-1} \big( \|\Delta \mathbf{W}^{(1)}\|_2 + \ldots + \|\Delta \mathbf{W}^{(\ell)}\|_2 \big).
\end{equation}
Therefore, we have upper bound $(A)$ by 
\begin{equation}
    \begin{aligned}
    &\alpha \sqrt{d} d_{\max}^{(\ell+1)} \Big( \beta \| \Delta \mathbf{W}^{(\ell)}\|_2  +  B_w^\beta  \Delta z_{\max}^{(\ell)} \Big) \\
    &\leq \alpha \beta  \cdot \frac{2}{\gamma} C_1^{L-\ell} \cdot \Big( \sqrt{d} \| \Delta \mathbf{W}^{(\ell)}\|_2  +  B_x B_{\ell,d}^{\alpha,\beta} C_1^{\ell} \big( \|\Delta \mathbf{W}^{(1)}\|_2 + \ldots + \|\Delta \mathbf{W}^{(\ell)}\|_2 \big) \Big) \\
    &\leq \alpha \beta \frac{2}{\gamma} C_1^{L} \big( \sqrt{d} + B_{\ell,d}^{\alpha,\beta} B_x \big) \|\Delta \bm{\theta} \|_2 \\
    &\leq  \alpha \beta \frac{2}{\gamma} C_1^L C_2 \|\Delta \bm{\theta} \|_2,
    \end{aligned}
\end{equation}
where $C_2 = \sqrt{d} + B_{\ell,d}^{\alpha,\beta} B_x$ and $B_{\ell,d}^{\alpha,\beta} = \max\Big\{ \beta \big( (1-\alpha)L + \alpha\sqrt{d} \big), (1-\alpha) L B_w^\beta + 1 \Big\}$.
By plugging it back and by induction, we have
\begin{equation}
    \begin{aligned}
    \Delta d_{\max}^{(\ell)} 
    &\leq C_1 \Delta d_{\max}^{(\ell+1)} + \alpha \beta \frac{2}{\gamma} C_1^{L} C_2 \|\Delta \bm{\theta} \|_2 \\
    &\leq C_1^2 \Delta d_{\max}^{(\ell+2)} + (1 + C_1) \cdot \alpha \beta \frac{2}{\gamma} C_1^{L} C_2 \|\Delta \bm{\theta} \|_2 \\
    &\leq C_1^{L-\ell+1} \Delta d_{\max}^{(L+1)} + (1+C_1 + \ldots + C_1^{L-\ell+1}) \cdot \alpha \beta \frac{2}{\gamma} C_1^{L} C_2 \|\Delta \bm{\theta} \|_2 \\
    &\leq C_1^{L-\ell} \Big( \Delta d_{\max}^{(L+1)} +  \alpha \beta \frac{2}{\gamma} L C_1^L C_2 \|\Delta \bm{\theta} \|_2 \Big).  
    \end{aligned}
\end{equation}
Let first explicit upper bound $\Delta d_{\max}^{(L+1)}$. By definition, $\Delta d_{\max}^{(L+1)}$ as 
\begin{equation}
    \begin{aligned}
    \Delta d_{\max}^{(L+1)} 
    &= \max_i \|\mathbf{d}_i^{(L+1)} - \tilde{\mathbf{d}}_i^{(L+1)}\|_2 \\
    &= \max_i \left\|\frac{\partial \text{Loss}(f(\mathbf{h}_i^{(L)}), y_i)}{\partial \mathbf{h}_i^{(L)}} - \frac{\partial  \text{Loss}(\tilde{f}(\tilde{\mathbf{h}}_i^{(L)}), y_i)}{\partial \tilde{\mathbf{h}}_i^{(L)}}\right\|_2 \\
    &\leq  \frac{2}{\gamma} \max_i \left\| \frac{\partial f(\mathbf{h}_i^{(L)})}{\partial \mathbf{h}_i^{(L)}} - \frac{\partial f(\tilde{\mathbf{h}}_i^{(L)})}{\partial \tilde{\mathbf{h}}_i^{(L)}} \right \|_2 \\
    &\leq \frac{2}{\gamma} \Big( \Delta h_{\max}^{(L)} + (h_{\max}^{(L)} + 1) \Delta \bm{v} \Big) \\
    &= \frac{2}{\gamma} \Big[ B_x B_{\ell,d}^{\alpha,\beta} C_1^{L-1} \big( \|\Delta \mathbf{W}^{(1)}\|_2 + \ldots + \|\Delta \mathbf{W}^{(L)}\|_2 \big) + \big( B_x B_{\ell,d}^{\alpha,\beta} C_1^L + 1 \big) \|\Delta \bm{v} \|_2 \Big] \\
    &\leq \frac{2}{\gamma} \Big( B_{\ell}^{\alpha} B_x C_1^L + 1 \Big) \|\Delta \bm{\theta}\|_2 \\
    &= \frac{2}{\gamma} \Big( B_{\ell}^{\alpha} B_x C_1^L + 1 \Big) \|\Delta \bm{\theta}\|_2,
    \end{aligned}
\end{equation}
where $C_1 = \max\{1, \alpha \sqrt{d} B_w^\beta\}$, $B_w^\beta = \beta B_w + (1-\beta)$ and 
$B_{\ell,d}^{\alpha,\beta} = \max\Big\{ \beta \big( (1-\alpha)L + \alpha\sqrt{d} \big), (1-\alpha) L B_w^\beta + 1 \Big\}$.
By plugging it back, we have
\begin{equation}
    \begin{aligned}
    \Delta d_{\max}^{(\ell)} &\leq 
    C_1^{L-\ell} \Big( \frac{2}{\gamma} \big( B_{\ell,d}^{\alpha,\beta} B_x C_1^L + 1 \big) \|\Delta \bm{\theta}\|_2  +  \alpha \beta \frac{2}{\gamma} L C_1^L (\sqrt{d} + B_{\ell,d}^{\alpha,\beta} B_x) \|\Delta \bm{\theta} \|_2 \Big)  \\
    &\leq C_1^{L-\ell} \frac{2}{\gamma} \Big( B_{\ell,d}^{\alpha,\beta} B_x C_1^L + 1 + \alpha \beta L C_1^L (\sqrt{d} + B_{\ell,d}^{\alpha,\beta} B_x)  \Big) \|\Delta \bm{\theta} \|_2 \\
    &\leq C_1^{L-\ell} \frac{2}{\gamma} \Big( (\alpha \beta L + 1 )(\sqrt{d} + B_{\ell,d}^{\alpha,\beta} B_x) C_1^L + 1 \Big) \| \Delta \bm{\theta} \|_2 \\
    &= C_1^{L-\ell} \frac{2}{\gamma} \Big( (\alpha \beta L + 1 ) C_1^L C_2 + 1 \Big) \| \Delta \bm{\theta} \|_2.
    \end{aligned}
\end{equation}

\subsection{Proof of Lemma~\ref{lemma:gcnii_G_bound_4}}
By the definition of $\mathbf{G}^{(\ell)}$, we have
\begin{equation}
    \begin{aligned}
    \| \mathbf{G}^{(\ell)} \|_2 
    &= \frac{1}{m} \left\|\beta \cdot [\alpha \mathbf{L} \mathbf{H}^{(\ell-1)} + (1-\alpha) \mathbf{X}]^\top \mathbf{D}^{(\ell)} \odot \sigma^\prime(\bm{Z}^{(\ell)}) \right\|_2 \\
    &\leq \frac{1}{m} \beta \left\|[\alpha \mathbf{L} \mathbf{H}^{(\ell-1)} + (1-\alpha) \mathbf{X}]^\top \mathbf{D}^{(\ell)} \right\|_2 \\
    &\leq  \beta \Big( \alpha \sqrt{d} h_{\max}^{(\ell-1)} + (1-\alpha) B_x \Big) d_{\max}^{(\ell)}. \\
    \end{aligned}
\end{equation}
Plugging in the results from Lemma~\ref{lemma:gcnii_h_max_bound_1} and Lemma~\ref{lemma:gcnii_d_max_bound_3}, we have
\begin{equation}\label{eq:gradient_norm_GCNII}
    \begin{aligned}
    \| \mathbf{G}^{(\ell)} \|_2 &\leq 
     \frac{2}{\gamma} \beta B_x \Big( \alpha \sqrt{d} \big((1-\alpha)(\ell-1) B_w^\beta +1\big)  \cdot C_1^{\ell-1} + (1-\alpha) \Big) C_1^{L-\ell+1} \\
    &\leq  \frac{2}{\gamma} \beta B_x \big( (1-\alpha)\ell + \alpha \sqrt{d} \big) C_1^{L} \leq  \frac{2}{\gamma} \beta C_1^L C_2,
    \end{aligned}
\end{equation}
where $C_2=\sqrt{d} + B_{\ell,d}^{\alpha,\beta} B_x$.
Similarly, by the definition of $\mathbf{G}^{(L+1)}$, we have
\begin{equation}
    \| \mathbf{G}^{(L+1)} \|_2 \leq \frac{2}{\gamma} h_{\max}^{(L)} \leq \frac{2}{\gamma} C_1^L C_2.
\end{equation}
Therefore, we have
\begin{equation}
    \sum_{\ell=1}^{L+1} \| \mathbf{G}^{(\ell)} \|_2 \leq  \frac{2}{\gamma} (L+1) \beta C_1^L C_2.
\end{equation}
Furthermore, we can upper bound the difference between gradient for $\ell\in[L]$ as
\begin{equation}
    \begin{aligned}
    \|\mathbf{G}^{(\ell)} - \tilde{\mathbf{G}}^{(\ell)}\|_2
    &= \frac{1}{m} \left\| \beta [\alpha \mathbf{L} \mathbf{H}^{(\ell-1)} + (1-\alpha) \mathbf{X}]^\top \mathbf{D}^{(\ell+1)} \odot \sigma^\prime(\mathbf{Z}^{(\ell)}) - \beta [\alpha \mathbf{L} \tilde{\mathbf{H}}^{(\ell-1)} + (1-\alpha) \mathbf{X}]^\top \tilde{\mathbf{D}}^{(\ell+1)} \odot \sigma^\prime(\tilde{\mathbf{Z}}^{(\ell)}) \right\|_2  \\
    &\underset{(a)}{\leq} \frac{1}{m} \beta \| [\alpha \mathbf{L} \mathbf{H}^{(\ell-1)} + (1-\alpha)\mathbf{X}]^\top \mathbf{D}^{(\ell+1)}  -  [\alpha \mathbf{L} \tilde{\mathbf{H}}^{(\ell-1)} + (1-\alpha) \mathbf{X}]^\top \mathbf{D}^{(\ell+1)}  \|_2\ \\
    &\quad + \frac{1}{m} \beta \| [\alpha \mathbf{L} \tilde{\mathbf{H}}^{(\ell-1)} + (1-\alpha) \mathbf{X}]^\top \mathbf{D}^{(\ell+1)} -  [\alpha \mathbf{L} \tilde{\mathbf{H}}^{(\ell-1)} + (1-\alpha) \mathbf{X}]^\top \tilde{\mathbf{D}}^{(\ell+1)}  \|_2\\
    &\quad + \frac{1}{m} \beta \| [\alpha \mathbf{L} \tilde{\mathbf{H}}^{(\ell-1)} + (1-\alpha) \mathbf{X}]^\top \tilde{\mathbf{D}}^{(\ell+1)}  \odot \big( \sigma^\prime(\mathbf{Z}^{(\ell)})  - \sigma^\prime(\tilde{\mathbf{Z}}^{(\ell)})  \big) \|_2 \\
    &\underset{(b)}{\leq}  \beta \Big( \alpha \max_i \|[ [\mathbf{L} (\mathbf{H}^{(\ell-1)} - \tilde{\mathbf{H}}^{(\ell-1)})]^\top \mathbf{D}^{(\ell+1)} ]_{i,:}\|_2 \\
    &\quad + \max_i \|[ [\alpha \mathbf{L} \tilde{\mathbf{H}}^{(\ell-1)} + (1-\alpha)\mathbf{X}]^\top (\mathbf{D}^{(\ell+1)} - \tilde{\mathbf{D}}^{(\ell+1)}) ]_{i,:}\|_2 + \\
    &\quad + \max_i \|  [[\alpha \mathbf{L} \tilde{\mathbf{H}}^{(\ell-1)} + (1-\alpha) \mathbf{X}]^\top \tilde{\mathbf{D}}^{(\ell+1)} ]_{i,:}\|_2 \max_i \| [\mathbf{Z}^{(\ell)} - \tilde{\mathbf{Z}}^{(\ell)} ]_{i,:}\|_2 \Big) \\
    &\leq  \beta \Big( \underbrace{\alpha \sqrt{d} \Delta h_{\max}^{(\ell-1)} d_{\max}^{(\ell+1)} }_{(A)}+ \underbrace{(\alpha \sqrt{d} h_{\max}^{(\ell-1)} + (1-\alpha) B_x) \Delta d_{\max}^{(\ell+1)} }_{(B)}\\
    &\quad + \underbrace{\big(\alpha \sqrt{d} h_{\max}^{(\ell-1)} + (1-\alpha)B_x \big) d_{\max}^{(\ell+1)} \Delta z_{\max}^{(\ell)}}_{(C)} \Big),
    \end{aligned}
\end{equation}
where inequality $(a)$ and $(b)$ is due to the gradient of ReLU activation is element-wise either $0$ or $1$.

We can bound $(A)$ as
\begin{equation}
    \alpha \sqrt{d} \Delta h_{\max}^{(\ell-1)} d_{\max}^{(\ell+1)} \leq \alpha \beta \sqrt{d} \frac{2}{\gamma} C_1^L C_2 \| \Delta \bm{\theta} \|_2.
\end{equation}
To upper bound $(B)$ and $(C)$, let first consider the upper bound of the following term
\begin{equation}
\begin{aligned}
    \alpha \sqrt{d} h_{\max}^{(\ell-1)} + (1-\alpha) B_x 
    &\leq \alpha \sqrt{d} \Big( (1-\alpha)(\ell-1) B_w^\beta + 1 \Big) B_x C_1^{\ell-1} + (1-\alpha) B_x \\
    &\leq \alpha \sqrt{d} \Big( (1-\alpha)\ell B_w^\beta + 1 \Big) B_x C_1^{\ell-1} \\
    &\leq \alpha \sqrt{d}  C_1^{\ell-1}  C_2.
\end{aligned}
\end{equation}
Therefore, we can bound $(B)$ as
\begin{equation}
    \begin{aligned}
    \Big(\alpha \sqrt{d} h_{\max}^{(\ell-1)} + (1-\alpha) B_x \Big) \Delta d_{\max}^{(\ell+1)} 
    &\leq \alpha \sqrt{d} C_1^{\ell-1}  C_2 \cdot C_1^{L-\ell} \frac{2}{\gamma} \Big( (\alpha \beta L + 1 )C_1^LC_2 + 1 \Big) \| \Delta \bm{\theta} \|_2 \\
    &\leq \alpha \sqrt{d} \frac{2}{\gamma} C_1^L C_2 \Big( (\alpha \beta L + 1 )C_1^LC_2 + 1 \Big) \| \Delta \bm{\theta} \|_2,
    \end{aligned}
\end{equation}
and we can bound $(C)$ by
\begin{equation}
\begin{aligned}
    \big(\alpha \sqrt{d} h_{\max}^{(\ell-1)} + (1-\alpha)B_x \big) d_{\max}^{(\ell+1)} \Delta z_{\max}^{(\ell)} 
    &\leq \alpha \sqrt{d} C_1^{\ell-1}  C_2 \cdot \frac{2}{\gamma} C_1^{L-\ell} \cdot  \beta C_1^{\ell-1} C_2 \| \Delta \bm{\theta} \|_2 \\
    &\leq \alpha \beta \sqrt{d} \frac{2}{\gamma} (C_1^L C_2)^2  \| \Delta \bm{\theta} \|_2.
\end{aligned}
\end{equation}
By combining result above, we have
\begin{equation}
    \begin{aligned}
    \|\mathbf{G}^{(\ell)} - \tilde{\mathbf{G}}^{(\ell)}\|_2 
    &\leq \alpha \beta   \frac{2}{\gamma} \sqrt{d} C_1^L C_2 \Big( 2 \beta + (\alpha \beta L + 1 )C_1^LC_2 \Big). 
    \end{aligned}
\end{equation}
Similarly, we can upper bound the difference between gradient for the weight of the binary classifier as
\begin{equation}
    \|\mathbf{G}^{(L+1)} - \tilde{\mathbf{G}}^{(L+1)}\|_2 \leq \frac{2}{\gamma} \Delta h_{\max}^{(L)} \leq \beta \frac{2}{\gamma} C_1^L C_2 \|\Delta \bm{\theta}\|_2.
\end{equation}
which is upper bound by the right hand side of the previous equation.
Therefore, we have
\begin{equation}
    \sum_{\ell=1}^{L+1} \|\mathbf{G}^{(\ell)} - \tilde{\mathbf{G}}^{(\ell)}\|_2 \leq \alpha \beta   \frac{2}{\gamma} (L+1) \sqrt{d} C_1^L C_2 \Big( 2 \beta + (\alpha \beta L + 1 )C_1^LC_2 \Big). 
\end{equation}

\section{Generalization bound for DGCN}\label{supp:proof_dgcn}

In the following, we provide proof of the generalization bound of DGCN in Theorem~\ref{thm:dgcn_result}.
Recall that the update rule of DGCN is defined as
\begin{equation}
    \mathbf{Z} = \sum_{\ell=1}^L \alpha_\ell \mathbf{H}^{(\ell)},~ \mathbf{H}^{(\ell)} = \mathbf{L}^\ell \mathbf{X} (\beta_\ell \mathbf{W}^{(\ell)} + (1-\beta_\ell) \mathbf{I} ).
\end{equation}

The training of DGCN is an empirical risk minimization with respect to a set of parameters $\bm{\theta} = \{ \mathbf{W}^{(1)}, \ldots, \mathbf{W}^{(L)}, \mathbf{v} \}$, i.e.,
\begin{equation}
    \mathcal{L}(\bm{\theta}) = \frac{1}{m} \sum_{i=1}^m \Phi_\gamma (-p(f(\mathbf{z}_i), y_i)),~
    f(\mathbf{z}_i) = \tilde{\sigma}(\mathbf{v}^\top \mathbf{z}_i),
\end{equation}
where $\mathbf{h}_i^{(L)}$ is the node representation of the $i$th node at the final layer,  $f(\mathbf{h}_i^{(L)})$ is the predicted label for the $i$th node,
$\tilde{\sigma}(x) = \frac{1}{\exp(-x) + 1}$ is the sigmoid function, and loss function $\Phi_\gamma (-p(z, y))$ is $\frac{2}{\gamma}$-Lipschitz continuous with respect to its first input $z$. For simplification, let $\text{Loss}(z,y)$ denote $\Phi_\gamma (-p(z, y))$ in the proof.

For analysis purpose, we suppose $\alpha_\ell$ and $\beta_\ell$ are hyper-parameters that are pre-selected before training and fixed during the training phase. However, in practice, these two parameters are tuned during the training phase. 

To establish the generalization of DGCN as stated in Theorem~\ref{thm:uniform_stability_base}, we utilize the result on transductive uniform stability from~\cite{el2006stable} (Theorem~\ref{thm:uniform_stability_base_supp} in Appendix~\ref{supp:proof_gcn}).
Then, in Lemma~\ref{lemma:uniform_stable_dgcn}, we derive the uniform stability constant for DGCN, i.e., $\epsilon_\text{DGCN}$. 

\begin{lemma}\label{lemma:uniform_stable_dgcn} 
The uniform stability constant for DGCN is computed as $\epsilon_\texttt{DGCN} = \frac{2 \eta  \rho_f G_f}{m} \sum_{t=1}^T (1+\eta L_F)^{t-1} $ where
\begin{equation}
    \begin{aligned}
    \rho_f &= (\sqrt{d})^L B_x,~ G_f = \frac{2}{\gamma} (L+1) (\sqrt{d})^L B_x, \\
    L_f &= \frac{2}{\gamma} (L+1)  (\sqrt{d})^L B_x \Big( (\sqrt{d})^L B_x \max\{1, B_w\} + 1)  \Big).
    \end{aligned}
\end{equation}
\end{lemma}

By plugging the result in Lemma~\ref{lemma:uniform_stable_dgcn} back to Theorem~\ref{thm:uniform_stability_base_supp}, we establish the generalization bound for DGCN.

The key idea of the proof is to decompose the change of the network output into two terms (in Lemma~\ref{lemma:universal_boound_on_model_output}) which depend on 
\begin{itemize}
    \item (Lemma~\ref{lemma:dgcn_h_max_bound_1}) The maximum change of node representation $\Delta z_{\max}^{(\ell)} = \max_i \| [\mathbf{Z} - \tilde{\mathbf{Z}}]_{i,:}\|_2$, 
    \item (Lemma~\ref{lemma:dgcn_delta_h_max_bound_2}) The maximum node representation $z_{\max}^{(\ell)} = \max_i \|[ \mathbf{Z}]_{i,:}\|_2$. 
\end{itemize}

\begin{lemma} [Upper bound of $z_{\max}$ for \textbf{DGCN}] \label{lemma:dgcn_h_max_bound_1}
Let suppose Assumption~\ref{assumption:norm_bound} hold. Then, the maximum node embeddings for any node at the $\ell$th layer is bounded by
\begin{equation}
    z_{\max} \leq (\sqrt{d})^L B_x \max\{1, B_w\}.
\end{equation}
\end{lemma}

\begin{lemma} [Upper bound of $\Delta z_{\max}$ for \textbf{DGCN}] \label{lemma:dgcn_delta_h_max_bound_2}
Let suppose Assumption~\ref{assumption:norm_bound} hold. 
Then, the maximum change between the node embeddings  on two different set of weight parameters for any node at the $\ell$th layer is bounded by
\begin{equation}
    \begin{aligned}
    \Delta z_{\max} &\leq (\sqrt{d})^L B_x (\|\Delta \mathbf{W}^{(1)}\|_2 + \ldots + \|\Delta \mathbf{W}^{(L)}\|_2),
    \end{aligned}
\end{equation}
where $\Delta \mathbf{W}^{(\ell)} = \mathbf{W}^{(\ell)} - \tilde{\mathbf{W}}^{(\ell)}$.
\end{lemma}

Then, in Lemma~\ref{lemma:dgcn_G_bound_4}, we decompose the change of the model parameters into two terms which depend on
\begin{itemize}
    \item The change of gradient with respect to the $\ell$th layer weight matrix $\|\Delta \mathbf{G}^{(\ell)}\|_2$, 
    \item The gradient with respect to the $\ell$th layer weight matrix $\|\mathbf{G}^{(\ell)} \|_2$, 
\end{itemize}
where $\|\Delta \mathbf{G}^{(\ell)}\|_2$ reflect the smoothness of APPNP model and $\|\mathbf{G}^{(\ell)} \|_2$ correspond the upper bound of gradient.

\begin{lemma} [Upper bound of $\| \mathbf{G}^{(\ell)} \|_2, \|\Delta \mathbf{G}^{(\ell)} \|_2$  for \textbf{DGCN}] \label{lemma:dgcn_G_bound_4}
Let suppose Assumption~\ref{assumption:norm_bound} hold.
Then, the gradient and the maximum change between gradients on two different set of weight parameters are bounded by
\begin{equation}
    \begin{aligned}
    \sum_{\ell=1}^{L+1} \| \mathbf{G}^{(\ell)} \|_2 
    &\leq \frac{2}{\gamma} (L+1) (\sqrt{d})^L B_x, \\
    \sum_{\ell=1}^{L+1} \|\Delta \mathbf{G}^{(\ell)} \|_2 
    &\leq \frac{2}{\gamma} (L+1) (\sqrt{d})^L B_x \Big( (\sqrt{d})^L B_x \max\{1, B_w\} + 1)  \Big) \| \Delta \bm{\theta} \|_2, \\
    \end{aligned}
\end{equation}
where $\|\Delta \bm{\theta}\|_2 = \|\mathbf{v} - \tilde{\mathbf{v}}\|_2 + \sum_{\ell-1}^L \| \mathbf{W}^{(\ell)} - \tilde{\mathbf{W}}^{(\ell)}\|_2$ denotes the change of two set of parameters.
\end{lemma}

Equipped with above intermediate results, we now proceed to prove  Lemma~\ref{lemma:uniform_stable_dgcn}.

\begin{proof}
Recall that our goal is to explore the impact of different GCN structures on the uniform stability constant $\epsilon_\text{DGCN}$, which is a function of $\rho_f$, $G_f$, and $L_f$. 
By Lemma~\ref{lemma:dgcn_delta_h_max_bound_2}, we know that the function $f$ is $\rho_f$-Lipschitz continuous, with $\rho_f = (\sqrt{d})^L B_x$.
By Lemma~\ref{lemma:dgcn_G_bound_4}, we know the function $f$ is $L_f$-smoothness, and the gradient of each weight matrices is bounded by $G_f$, with 
\begin{equation}
    G_f \leq \frac{2}{\gamma} (L+1) (\sqrt{d})^L B_x,~
    L_f \leq \frac{2}{\gamma} (L+1) (\sqrt{d})^L B_x \Big( (\sqrt{d})^L B_x \max\{1, B_w\} + 1)  \Big).
\end{equation}
By plugging $\epsilon_\text{DGCN}$ into Theorem~\ref{thm:uniform_stability_base}, we obtain the generalization bound of DGCN.

\end{proof}

\subsection{Proof of Lemma~\ref{lemma:dgcn_h_max_bound_1}}

By the definition of $z_{\max}$, we have
\begin{equation}
    \begin{aligned}
    z_{\max} &= \max_i \left\| \left[ \sum_{\ell=1}^L \alpha_\ell \mathbf{L}^\ell \mathbf{X} \big( \beta_\ell \mathbf{W}^{(\ell)} + (1-\beta_\ell) \mathbf{I} \big) \right]_{i,:} \right\|_2 \\
    &\leq \sum_{\ell=1}^L \alpha_\ell \max_i \| [\mathbf{L}^\ell \mathbf{X} \big( \beta_\ell \mathbf{W}^{(\ell)} + (1-\beta_\ell) \mathbf{I} \big) ]_{i,:}\|_2 \\
    &\underset{(a)}{\leq} \sum_{\ell=1}^L \alpha_\ell (\sqrt{d})^\ell B_x \big( \beta_\ell B_w + (1-\beta_\ell)\big) \\
    &\underset{(b)}{\leq}  (\sqrt{d})^L B_x \max\{1, B_w\},
    \end{aligned}
\end{equation}
where inequality $(a)$ follows from Lemma~\ref{lemma:laplacian_1_norm_bound} and inequality $(b)$ is due to the fact that $\sum_{\ell=1}^L \alpha_\ell = 1$ and $\alpha_\ell \in [0,1]$.

\subsection{Proof of Lemma~\ref{lemma:dgcn_delta_h_max_bound_2}}

By the definition of $\Delta z_{\max}$, we have

\begin{equation}
    \begin{aligned}
    \Delta z_{\max} &= \max_i \| [\mathbf{Z} - \tilde{\mathbf{Z}}]_{i,:} \|_2 \\
    &= \max_i \Big\| \Big[\sum_{\ell=1}^L \alpha_\ell \mathbf{L}^\ell \mathbf{X} \big( \beta_\ell \mathbf{W}^{(\ell)} + (1-\beta_\ell) \mathbf{I} \big) - \sum_{\ell=1}^L \alpha_\ell \mathbf{L}^\ell \mathbf{X} \big( \beta_\ell \tilde{\mathbf{W}}^{(\ell)} + (1-\beta_\ell) \mathbf{I} \big) \Big]_{i,:} \Big\|_2 \\
    &\leq \sum_{\ell=1}^L \alpha_\ell \max_i \Big\| \Big[ \mathbf{L}^\ell \mathbf{X} \big( \beta_\ell \mathbf{W}^{(\ell)} + (1-\beta_\ell) \mathbf{I}  \big) - \mathbf{L}^\ell \mathbf{X} \big( \beta_\ell \tilde{\mathbf{W}}^{(\ell)} + (1-\beta_\ell) \mathbf{I}  \big) \Big]_{i,:} \Big\|_2 \\
    &= \sum_{\ell=1}^L \alpha_\ell \beta_\ell \max_i \Big\| \Big[ \mathbf{L}^\ell \mathbf{X} (\mathbf{W}^{(\ell)} - \tilde{\mathbf{W}}^{(\ell)} )  \Big]_{i,:} \Big\|_2 \\
    &\underset{(a)}{\leq} \sum_{\ell=1}^L \alpha_\ell \beta_\ell (\sqrt{d})^\ell B_x \|\Delta \mathbf{W}^{(\ell)}\|_2 \\
    &\leq (\sqrt{d})^L B_x \Big( \|\Delta \mathbf{W}^{(1)}\|_2 + \ldots + \|\Delta \mathbf{W}^{(L)}\|_2\Big), 
    \end{aligned}
\end{equation}
where inequality $(a)$ follows from Lemma~\ref{lemma:laplacian_1_norm_bound} and inequality $(b)$ is due to the fact that $\sum_{\ell=1}^L \alpha_\ell = 1$ and $\alpha_\ell \in [0,1]$.

\clearpage
\subsection{Proof of Lemma~\ref{lemma:dgcn_G_bound_4}}
Recall that $\mathbf{G}^{(\ell)}$ is defined as the gradient of loss with respect to $\mathbf{W}^{(\ell)}$, which can be formulated as
\begin{equation}
    \begin{aligned}
    \| \mathbf{G}^{(\ell)} \|_2 &= \Big\| \frac{1}{m}\sum_{i=1}^m \frac{\partial \text{Loss}( f(\mathbf{z}_i), y_i)}{\partial \mathbf{z}_i} \frac{\partial \mathbf{z}_i}{\partial \mathbf{h}_i^{(\ell)}} \frac{\partial \mathbf{h}_i^{(\ell)}}{\partial \mathbf{W}^{(\ell)}} \Big\|_2 \\
    &\leq \max_i \Big\| \frac{\partial \text{Loss}(f(\mathbf{z}_i), y_i)}{\partial \mathbf{z}_i} \frac{\partial \mathbf{z}_i}{\partial \mathbf{h}_i^{(\ell)}} \frac{\partial \mathbf{h}_i^{(\ell)}}{\partial \mathbf{W}^{(\ell)}} \Big\|_2 \\
    &\leq  \alpha_\ell \beta_\ell (\sqrt{d})^\ell B_x \max_i \Big\| \frac{\partial \text{Loss}(f(\mathbf{z}_i), y_i)}{\partial \mathbf{z}_i} \Big\|_2 \\
    &\underset{(a)}{\leq}  \frac{2}{\gamma} \alpha_\ell \beta_\ell (\sqrt{d})^\ell B_x \\
    &\leq \frac{2}{\gamma}(\sqrt{d})^L B_x,
    \end{aligned}
\end{equation}
where $(a)$ is due to the fact that loss function is $\frac{2}{\gamma}$-Lipschitz conitnous.

Besides, recall that $\mathbf{G}^{(L+1)}$ denotes the gradient with respect to the weight of binary classifier. Therefore, we have
\begin{equation}
    \|\mathbf{G}^{(L+1)}\|_2 \leq \frac{2}{\gamma} h_{\max}^{(L)} \underset{(a)}{\leq} \frac{2}{\gamma}(\sqrt{d})^L B_x,
\end{equation}
where inequality $(a)$ follows from Lemma~\ref{lemma:dgcn_delta_h_max_bound_2}.

Therefore, combining the results above, we have
\begin{equation}
    \sum_{\ell=1}^{L+1} \| \mathbf{G}^{(\ell)} \|_2 \leq \frac{2}{\gamma} (L+1) (\sqrt{d})^L B_x.
\end{equation}

Furthermore, by the definition of $\|\mathbf{G}^{(\ell)} - \tilde{\mathbf{G}}^{(\ell)}\|_2 $, we have
\begin{equation}
    \begin{aligned}
    \|\mathbf{G}^{(\ell)} - \tilde{\mathbf{G}}^{(\ell)}\|_2 &\leq \frac{2}{\gamma} \max_i \| \frac{\partial f(\mathbf{z}_i)}{\partial \mathbf{z}_i } \frac{\partial \mathbf{z}_i }{\partial \mathbf{W}}-  \frac{\partial f(\tilde{\mathbf{z}}_i ) }{\partial \tilde{\mathbf{z}}_i } \frac{\partial \tilde{\mathbf{z}}_i}{\partial \tilde{\mathbf{W}}}\|_2 \\
    &=  \frac{2}{\gamma} \alpha_\ell \beta_\ell (\sqrt{d})^\ell B_x \left\| \frac{\partial f(\mathbf{z}_i)}{\partial \mathbf{z}_i } -  \frac{\partial f(\tilde{\mathbf{z}}_i ) }{\partial \tilde{\mathbf{z}}_i } \right\|_2 \\
    &\leq  \frac{2}{\gamma}  \alpha_\ell \beta_\ell B_x (\sqrt{d})^\ell \Big(\Delta z_{\max} + (z_{\max} + 1) \| \Delta \mathbf{v} \|_2 \Big) \\
    &\leq  \frac{2}{\gamma}  (\sqrt{d})^L B_x \Big( (\sqrt{d})^L B_x \max\{1, B_w\} + 1)  \Big) \| \Delta \bm{\theta} \|_2. \\
    \end{aligned}
\end{equation}
Similarly, we can upper bound the difference between gradient for the weight of the binary classifier as
\begin{equation}
    \begin{aligned}
    \|\mathbf{G}^{(L+1)} - \tilde{\mathbf{G}}^{(L+1)}\|_2 
    &\leq \frac{2}{\gamma} \Delta h_{\max}^{(L)} \\
    &\leq \frac{2}{\gamma} (\sqrt{d})^L B_x B_w,
    \end{aligned}
\end{equation}
which is bounded by the right hand side of the last equation. 
Therefore, we have
\begin{equation}
    \sum_{\ell=1}^{L+1} \|\mathbf{G}^{(\ell)} - \tilde{\mathbf{G}}^{(\ell)}\|_2 \leq \frac{2}{\gamma}  (L+1) (\sqrt{d})^L B_x \Big( (\sqrt{d})^L B_x \max\{1, B_w\} + 1)  \Big) \| \Delta \bm{\theta} \|_2.
\end{equation}
\section{Omitted proofs of useful lemmas}\label{supp:omitted_proof_uniform_stability}
\subsection{Proof of Lemma~\ref{lemma:useful_lemma_for_main_theorems}}

Let $f_k(\bm{\theta})$ denote applying function $f$ with parameter $\bm{\theta}$ on the $k$th data point and $\bm{\theta}_t$ as the weight parameter at the $t$th iteration.
Because function $f_k(\bm{\theta})$ is $\rho_f$-Lipschitz with respect to parameter $\bm{\theta}$, we can bound the difference between model output by the difference between parameters $\bm{\theta}_T^{ij}, \bm{\theta}_T$, i.e.,
\begin{equation}\label{eq:gcn_f_lip}
    \epsilon = \max_k | f_k(\bm{\theta}_T^{ij}) - f_k(\bm{\theta}_T) | \leq \rho_f \| \boldsymbol{\theta}_T^{ij} - \boldsymbol{\theta}_T \|_2. 
\end{equation}

Let $\nabla f(\boldsymbol{\theta}) = \frac{1}{m} \sum_{k=1}^m \nabla f_k(\boldsymbol{\theta})$ and $\nabla f(\boldsymbol{\theta}^{ij}) = \frac{1}{m} \sum_{k=1}^m \nabla f_k(\boldsymbol{\theta}^{ij})$ denote the full-batch gradient computed on the original and perturbed dataset respectively.

 Recall that both models have the same initialization $\bm{\theta}_0 = \bm{\theta}_0^{ij}$.
At each iteration the full-batch gradient are computed on $(m-1)$ data points that identical in two dataset, and only $1$ data point that are different in two dataset. Therefore, we can bound 
the change of model parameters $\bm{\theta}_t^{ij}, \bm{\theta}_t$ after one gradient update step as
\begin{equation}
    \begin{aligned}
    \| \boldsymbol{\theta}_{t+1}^{ij} - \boldsymbol{\theta}_{t+1} \|_2 
    &\leq  \Big( 1 + (1-\frac{1}{m}) \eta L_f \Big) \| \boldsymbol{\theta}_{t}^{ij} - \boldsymbol{\theta}_{t} \|_2 + \frac{2 \eta G_f}{m} \\
    &\leq (1+ \eta L_f) \| \boldsymbol{\theta}_{t}^{ij} - \boldsymbol{\theta}_{t} \|_2 + \frac{2 \eta G_f}{m}.
    \end{aligned}
\end{equation}
Then after $T$ iterations, we can bound the different between two parameters as
\begin{equation}
    \| \boldsymbol{\theta}_T^{ij} - \boldsymbol{\theta}_T \|_2 \leq \frac{2\eta G_f}{m} \sum_{t=1}^T (1+\eta L_F)^{t-1}.
\end{equation}
By plugging the result back to Eq.~\ref{eq:gcn_f_lip}, we have
\begin{equation}\label{eq:stability_bound_to_rho_L_G}
    \epsilon = \frac{2 \eta  \rho_f G_f}{m} \sum_{t=1}^T (1+\eta L_F)^{t-1}. 
\end{equation}

\subsection{Upper bound on the $\ell_1$-norm of Laplacian matrix}
\begin{lemma} [Lemma A.3 in~\cite{liao2020pac}] \label{lemma:laplacian_1_norm_bound}
Let $\mathbf{A}$ denote the adjacency matrix of a self-connected undirected graph $\mathcal{G}(\mathcal{V},\mathcal{E})$, and $\mathbf{D}$ denote its corresponding degree matrix, and $d$ denote the maximum number of node degree.
We define the graph Laplacian matrix as $\mathbf{L} = \mathbf{D}^{-1/2} \mathbf{A} \mathbf{D}^{-1/2}$. Then we have
\begin{equation}
    \max_i \| [\mathbf{L}]_{i,:}\|_2\leq \sqrt{d}.
\end{equation}
\end{lemma}
\begin{proof}
By the definition of Laplacian matrix, we know that the $i$th row and $j$th column of $\mathbf{L}$ is defined as $L_{i,j} = \frac{1}{\sqrt{\deg(i)} \sqrt{\deg(j)}}$. Therefore, we have
\begin{equation}
    \begin{aligned}
    \max_i \sum_{j=1}^N L_{i,j} 
    &= \max_i \sum_{j\in \mathcal{N}(i)} \frac{1}{\sqrt{\deg(i)} \sqrt{\deg(j)}} \\ 
    &\underset{(a)}{\leq} \max_i \sum_{j\in \mathcal{N}(i)} \frac{1}{\sqrt{\deg(i)}} \\
    &= \max_i \sqrt{\deg(i)} \leq \sqrt{d},
    \end{aligned}
\end{equation}
where $(a)$ is due to $\frac{1}{\sqrt{\deg(i)}} \leq 1$ for any node $i$.

\end{proof}

\end{document}